%% file: main.tex
\documentclass[a4paper,11pt]{article}

\input{math_commands.tex}

\input{macros.tex}

\title{
\toptitlebar
{{\center\baselineskip 18pt
                      {\Large\bf Enhancing Optimizer Stability:\\ Momentum Adaptation of The NGN Step-size}}
} 
\bottomtitlebar}
\date{}
\author[1]{Rustem Islamov}
\author[2]{Niccolò Ajroldi}
\author[2,3,4]{Antonio Orvieto}
\author[1]{Aurelien Lucchi}
\affil[1]{University of Basel, Switzerland}
\affil[2]{Max Planck Institute for Intelligent Systems, Germany}
\affil[3]{ELLIS Institute Tübingen, Germany}
\affil[4]{Tübingen AI Center, Germany}

\begin{document}

\maketitle

\begin{abstract}
  Modern optimization algorithms that incorporate momentum and adaptive step-size offer improved performance in numerous challenging deep learning tasks. However, their effectiveness is often highly sensitive to the choice of hyperparameters, especially the step-size.  Tuning these parameters is often difficult, resource-intensive, and time-consuming. Therefore, recent efforts have been directed toward enhancing the stability of optimizers across a wide range of hyperparameter choices~\citep{schaipp2024momo}. In this paper, we introduce an algorithm that matches the performance of state-of-the-art optimizers while improving stability to the choice of the step-size hyperparameter through a novel adaptation of the \algname{NGN} step-size method \citep{orvieto2024adaptive}. Specifically, we propose a momentum-based version (\algname{NGN-M}) that attains the standard convergence rate of $\mathcal{O}(1/\sqrt{K})$ under less restrictive assumptions, without the need for interpolation condition or assumptions of bounded stochastic gradients or iterates, in contrast to previous approaches. Additionally, we empirically demonstrate that the combination of the \algname{NGN} step-size with momentum results in enhanced robustness to the choice of the step-size hyperparameter while delivering performance that is comparable to or surpasses other state-of-the-art optimizers.
\end{abstract}

\section{Introduction}

Adaptive methods such as \algname{Adam}~\citep{Kingma2015adam} and \algname{RMSprop} \citep{hinton2012neural} are widely used in machine learning due to their established advantages over (momentum) \algname{SGD}, particularly in tasks such as training Transformers~\citep{brown2020language, touvron2021training, touvron2023llama}. These methods adaptively scale the step-size across different dimensions (parameters) based on their respective statistics, effectively acting as a diagonal preconditioning. 

Although these methods perform well in practice, existing theoretical analyses typically require stringent assumptions on the noise structure of the stochastic gradients, such as sub-Gaussian noise \citep{li2024convergence} or affine noise models \citep{wang2024closing,zhang2024convergence}: Relaxing these assumptions remains an open challenge. Another well-known issue of \algname{Adam} is its performance sensitivity to the step-size hyperparameter~\citep{wilson2017marginal,choi2019empirical}, particularly when training Transformers, where loss spikes are commonly observed \citep{molybog2023theory, wortsman2023small}. This often necessitates careful adjustments of the hyperparameters throughout the training process \citep{zhang2022opt, chowdhery2023palm}, which can be costly in terms of computational resources \citep{sharir2020cost}. Consequently, there has been growing interest in developing optimization methods that are more robust to hyperparameter selection~\citep{schaipp2024momo}. In addition to adapting the step-size, \algname{Adam} and other state-of-the-art optimizers also rely on momentum~\citep{polyak1964momentum}, a broadly used technique that has been shown to enhance performance both theoretically~\citep{cutkosky2020momentum,fatkhullin2024momentum,islamov2024motef,islamov2025double} and practically \citep{choi2019empirical, fu2023momentum, jelassi2022towards}. Besides speeding up convergence, momentum is known as a technique to reduce the variance of stochastic algorithms  \citep{ma2018quasi, cutkosky2019momentum}, improving stability as well as generalization in some settings~\citep{jelassi2022towards}.

In this work, we address the aforementioned drawbacks of \algname{Adam} by developing a new algorithm based on the recently proposed \algname{NGN} step-size \citep{orvieto2024adaptive}, an improved variant of the Stochastic Polyak Step-size~\citep{loizou2021stochastic} that has demonstrated strong resilience to step-size hyperparameter tuning. In particular, \algname{NGN} was shown never to diverge for any choice of the step-size hyperparameter in the convex setting, and to exhibit strong curvature adaptation properties strengthened by theoretical guarantees. However, the step-size of \citet{orvieto2024adaptive} simply adapts the learning rate through a scalar multiplier, leaving to future work the incorporation of momentum and coordinate-wise variants -- needed in complex problems such as optimizing transformers, as motivated above.  Here, we develop a momentum and step-size adaptive version of \algname{NGN} designed to enhance robustness in terms of hyperparameter selection. We also present a theoretical analysis alongside a practical evaluation of this approach, showcasing its improvements over current state-of-the-art methods.


In summary, our contributions are as follows:

\begin{enumerate}
    \item We introduce a new  algorithm named \algname{NGN-M} that combines the \algname{NGN} step-size with momentum. We theoretically show that \algname{NGN-M} achieves a convergence rate $\cO(\nicefrac{1}{\sqrt{K}})$ in the convex regime without the typical requirements of interpolation or bounded gradient assumptions found in earlier works;
    
    \item We focus on the problem of adapting the step-size rule towards a coordinate-wise diagonal preconditioning. By integrating this diagonal step-size strategy with momentum, we develop a new variant of \algname{NGN}, called \algname{NGN-MD};

    \item  The theoretical results are supported by extensive empirical validation in various deep learning settings where we demonstrate that \algname{NGN-M} and \algname{NGN-MD} not only preserve the robustness property of the \algname{NGN} step-size, but improve it further in many cases. The step-size hyperparameter resilience comes together with better or comparable performance  to state-of-the-art algorithms.

\end{enumerate}

\section{Related Works}

\paragraph{Polyak Step-size.} When training a deep network with standard optimizers, tuning the learning rate is crucial but time-consuming and resource-intensive~\citep{GoodBengCour16}. This issue is at the root of recent research focusing on transferring hyperparameters across architectures at different scales, therefore avoiding expensive tuning pipelines~\citep{yang2022tensor, yang2023spectral, bordelon2023depthwise}. Yet, in the convex setting, choosing the learning rate can already be difficult -- an issue that was studied already in \citet{polyak1987introduction} and gave rise to the first adaptive method: the Polyak Stepsize~(\algname{PS}). Recently, there has been a renewed interest in adapting \algname{PS} to modern settings~\citep{loizou2021stochastic, orvieto2022dynamics, jiang2024adaptive}, delivering a theoretically principled way to scale the gradient magnitude during training adaptively. \algname{PS}-inspired methods have gained increasing interest for their simplicity and adaptability, as they utilize local curvature and smoothness information to accelerate algorithms and facilitate faster convergence.
\citet{orvieto2024adaptive} recently introduced a variant of the Stochastic Polyak step-size, called \algname{NGN}, which further enhances the robustness of the step-size hyperparameter and solidifies the link to Gauss-Newton preconditioning. The theoretical analysis in~\citet{orvieto2024adaptive} demonstrated that \algname{NGN} does not diverge regardless of the choice of the step-size hyperparameter, and converges fast when the step-size is appropriately tuned. In contrast, the current theory of the \algname{SPS} step-size with fixed step-size hyperparameters \citep{loizou2021stochastic} proves convergence to the exact solution only if the interpolation condition holds\footnote{In our notation, this means that $\sigma_{\rm int}^2 = 0.$}.

\paragraph{Polyak Step-size and Heavy-ball Momentum.} 
Heavy-ball momentum methods, stemming from the work of~\citet{polyak1964momentum}, have gained significant attention over the years due to their benefits, including acceleration on convex quadratics \citep{jain2018accelerating, lee2022trajectory, bollapragada2022fast}, convex-like \citep{wang2022provable}, and non-convex problems \citep{cutkosky2020momentum}, as well as their variance reduction abilities \citep{ma2018quasi, cutkosky2019momentum}. This has led to growing interest in the combination of Polyak step-size and heavy-ball momentum, which is an active area of research \citep{barre2020complexity, samer2022adaptive, barre2020complexity, wang2023generalized, oikonomou2024stochastic, gower2025analysis}. Recently, \citet{schaipp2024momo} demonstrated that a geometrically principled combination of \algname{SPS} and momentum leads to lower sensitivity to the step-size hyperparameter, although they did not provide strong theoretical convergence guarantees.

\paragraph{Diagonal Polyak Step-size.} Coordinate-wise adaptive step-sizes are essential in training Transformer architectures due to the varying parameter-wise scaling and conditioning of the problem \citep{noci2022signal, zhang2024transformers}. Algorithms employing diagonal step-sizes, such as \algname{Adam} and \algname{SignSGD} \citep{bernstein2018signsgd}, typically outperform non-diagonal methods in language modeling tasks by addressing issues such as class imbalance (where certain words appear more frequently than others) \citep{kunstner2023noise, kunstner2024heavy}  and heavy-tailed noise~\citep{zhang2019gradient,zhang2020adaptive, compagnoni2025adaptive}. It is, therefore, paramount in current setups to deliver adaptive step-size improvements targeted to the coordinate-wise (diagonal) regime. However, most Polyak-step-size-based algorithms only focus on a single step-size for all parameters \citep{loizou2021stochastic,wang2023generalized, gower2021sgd, oikonomou2024stochastic, orvieto2024adaptive}. Only a few works propose a diagonal-wise modification of Polyak-step-size by either using \algname{Adam} preconditioner \citep{schaipp2024momo} as a weight matrix or incorporating second-order information from the objective function \citep{li2022sp2, richtarik2024local}.

\Cref{tab:comparison} provides a theoretical comparison of various Polyak step-size-based algorithms that incorporate momentum and/or diagonal step-size, highlighting the differences between the theoretical results presented in this work and those from prior works.

\begin{table*}[t]
    \centering
    \caption{Summary of existing methods exploiting Polyak-type adaptive step-sizes and their convergence guarantees. \textbf{Mom.}=Supports momentum; \textbf{Diag.}=Supports diagonal step-sizes. $\sigma_{\rm int}^2$ is defined in \Cref{sec:theory}. $x^*$ defines an optimal solution to \eqref{eq:problem}. $\cO$ notation hides absolute and problem-dependent constant factors and logarithmic terms in the rate.}
    \label{tab:comparison}
    \resizebox{\textwidth}{!}{
        \begin{tabular}{ccccc}
            \toprule
            \textbf{Method} & 
            \makecellnew{ \textbf{Rate} ${}^{(a)}$} &
            \makecellnew{ \textbf{Mom.} }&
            \makecellnew{ \textbf{Diag.} } &
            \textbf{Comments} 
            \\ \toprule

            \makecellnew{\algname{SPS}${}_{\max}$ 
            \citep{loizou2021stochastic}} &
            $\cO(\nicefrac{1}{K}+\sigma^2_{\rm int})$ &
            \xmark &
            \xmark &
            \makecellnew{Conv. to non-vanishing \\ neighbourhood} \\
            
            \midrule

            \makecellnew{\algname{ALR-SMAG} 
            \citep{wang2023generalized}} &
            $\cO((1-\rho)^K+\sigma^2_{\rm int})$ &
            \cmark &
            \xmark &
            \makecellnew{Strong convexity \\ Conv. to non-vanishing \\ neighbourhood} \\
            
            \midrule

            \makecellnew{\algname{Momo}   \citep{schaipp2024momo}} &
            $\cO(\nicefrac{1}{\sqrt{K}})$&
            \cmark &
            \xmark &
            \makecellnew{Bounded stoch. gradients \\ 
            Interpolation} \\
            \midrule

            \makecellnew{\algname{Momo-Adam}   \citep{schaipp2024momo}} &
            \xmark &
            \cmark &
            \cmark &
            \makecellnew{\algname{Momo} framework \\ for \algname{Adam}} \\
            \midrule

            \makecellnew{\algname{MomSPS}${}_{\max}$   \citep{oikonomou2024stochastic}} &
            $\cO(\nicefrac{1}{K}+\sigma^2_{\rm int})$ &
            \cmark &
            \xmark &
            \makecellnew{Conv. to non-vanishing \\ neighbourhood} \\

            \midrule





            \makecellnew{\algname{NGN}   \citep{orvieto2024adaptive}} &
            $\cO(\nicefrac{1}{\sqrt{K}})$&
            \xmark &
            \xmark &
            $-$ \\

            \midrule 



            \makecellnew{\algname{NGN-M} (Alg. \ref{alg:ngn_m})  \\ {\bf [This work]}} &
            $\cO(\nicefrac{1}{\sqrt{K}})$&
            \cmark &
            \xmark &
            $-$ \\

            \midrule

            \makecellnew{\algname{NGN-MDv1} (Alg. \ref{alg:ngn_md}) \\ {\bf [This work]}} &
            \xmark &
            \cmark &
            \cmark &
            \makecellnew{Combination of \\ \algname{NGN-M} and \algname{RMSprop}} \\

            \midrule

            \makecellnew{\algname{NGN-MDv2} (Alg. \ref{alg:ngn_md}) \\ {\bf [This work]}} &
            \xmark &
            \cmark &
            \cmark &
            \makecellnew{Combination of \\ \algname{NGN-M} and \algname{NGN-D}} \\

            \midrule

            \makecellnew{\algname{NGN-D} (Alg. \ref{alg:ngn_d}) \\ {\bf [This work]}} &
            $\cO(\nicefrac{1}{\sqrt{K}})$&
            \xmark &
            \cmark &
            $-$ \\

            \bottomrule 
        
        \end{tabular}
        }
\end{table*}

\section{Algorithm design of \algname{NGN-M} and \algname{NGN-D}}

In \citet{orvieto2024adaptive}, the \algname{NGN} step‐size is derived by applying a Gauss–Newton update on a regularized first‐order expansion of $r(x) \eqdef \sqrt{f(x)}$. At the current point $x^k$, they linearized $r(x^k+p) \approx r(x^k) + \nabla r(x^k)^\top p.$ Thus the next iterate is given as $x^{k+1} = x^k + p^k$ where 
\begin{align}\label{eq:nojnwfenoj}\textstyle
p^k \eqdef \argmin\limits_{p} \left[(r(x^k) + \nabla r(x^k)^\top p)^2 + \frac{1}{2c}\|p\|^2\right].
\end{align}

It turns out that the problem above has a closed-form solution 
\begin{equation*}\textstyle
    p^k = -\gamma_k\nabla f(x^k) \text{ where } \gamma_k \eqdef \frac{c}{1+\frac{c}{2f(x^k)}\|\nabla f(x^k)\|^2},   
\end{equation*} 
with $\gamma_k$ representing the \algname{NGN} step-size. In \citet{orvieto2024adaptive}, convergence guarantees were established for both convex and general non-convex settings. Importantly, the convex analysis shows that \algname{NGN} exhibits a non-divergence property, regardless of the step-size hyperparameter $c$ (see Theorem $4.5$ in \citep{orvieto2024adaptive}). Due to this property, the \algname{NGN} step-size is a strong candidate to achieve better robustness w.r.t. the choice of the step-size.

\subsection{How to Add Momentum and What to Expect?}

There are several approaches to combining the adaptive Polyak-type step-size with heavy-ball momentum. Broadly, existing algorithms can be divided into two categories: the first category involves computing the Polyak step-size in the usual manner and incorporating it into the standard heavy-ball update \citep{oikonomou2024stochastic}. In contrast, algorithms from the second category first determine an update direction using exponential weighted averaging of the stochastic gradient and momentum variable, and then compute the Polyak-type step-size based on the computed direction \citep{wang2023generalized, schaipp2024momo}. Following this principled approach, we test two possible versions for combining the \algname{NGN} step-size and momentum: 

\begin{align*}
    \textstyle \text{Ver.} 1: \begin{cases}
        \gamma_k = \frac{c}{1+\frac{c}{2f_{S_k}(x^k)}\|\nabla f_{S_k}(x^k)\|^2} \\
        {\color{darkgreen} m^{k}} = \beta {\color{darkgreen} m^{k-1}} + (1-\beta) \gamma_k\nabla f_{S_k}(x^k) \\
        x^{k+1} = x^k - {\color{darkgreen} m^{k}}
    \end{cases} \quad \text{Ver.} 2: \begin{cases}
        {\color{darkred} m^{k}} = \beta {\color{darkred} m^{k-1}} + (1-\beta) \nabla f_{S_k}(x^k)\\
        \gamma_k = \frac{c}{1+\frac{c}{2f_{S_k}(x^k)}\|{\color{darkred} m^{k}}\|^2} \\
        x^{k+1} = x^k - \gamma_k {\color{darkred} m^{k}}
    \end{cases}.
\end{align*}

Before we proceed, we should answer the question: {\it ``What do we expect from the combination of \algname{NGN} step-size and momentum?''}. First, we aim to preserve, and ideally enhance, \algname{NGN}'s robustness to the step-size hyperparameter. Additionally, we seek improved performance, achieving accelerated convergence akin to the advantage of \algname{SGD} with momentum (\algname{SGDM}) over standard \algname{SGD} in convex settings. With these goals in mind, we now show that version $1$ meets all of these criteria, while version $2$ is less suitable.
To gain some intuition regarding the performance of these two variants, we start by conducting a simple experiment on a quadratic function $f(x) = \frac{1}{2}\|Ax-b\|^2$ where $A$ is a data matrix from the normalized Diabetes dataset \citep{smith1988diabetes} and $b$ is a vector of labels. Based on the results from \Cref{fig:synthetic_function} (left), we observe that variant $1$ achieves accelerated convergence as \algname{SGDM} for middle-range step-size hyperparameters ($c \in \{10^1, 10^2\}$) and does not diverge for large step-size parameter ($c \in \{10^3\}$). Conversely, version $2$ has a worse convergence rate than version $1$ for middle-range step-size parameters and diverges for large ones. Therefore, we theoretically analyze and practically test version $1$, which we call \algname{NGN-M}.

\subsection{Evidence of Robustness of \algname{NGN-M}}\label{sec:toy_function}

To illustrate the advantages of the design choice \algname{NGN-M}, we first consider the Rosenbrock function $f(x,y) = (x-1)^2 + 100(y-x^2)^2,$ whose minimizer is at $(1,1).$ Starting from $(-1.2, 1)$, we run both \algname{NGN-M} and \algname{SGDM} over a wide range of constant step-size hyperparameters $\{10^{-3}, \dots, 10^2\}.$ As shown in \Cref{fig:synthetic_function}, we observe that $(i)$ for small step-size hyperparameter both methods successfully converge to $(1,1);$ $(ii)$ \algname{SGDM} already diverges for the step-size hyperparameter $10^{-2};$ By contrast, \algname{NGN-M} remains stable even up to $c=10^2$, thanks to its adaptive step-size that automatically adjusts with the local curvature. \Cref{fig:rosenbrock_convergence} further traces the optimization trajectories: \algname{NGN-M} converges reliably for every tested value of $c$, whereas \algname{SGDM} fails outside its narrow stability window. Finally, in \Cref{sec:synth_function_appendix} we repeat these experiments on a synthetic multimodal function and find that \algname{NGN-M} consistently finds the global minimum, while \algname{SGDM} typically becomes trapped in a nearby suboptimal local minimum.

\begin{figure}[t]
    \begin{minipage}{0.48\textwidth}
    \centering
    \vspace{-7mm}
    \begin{tabular}{cc}
    \hspace{-8mm}\includegraphics[width=0.55\linewidth]{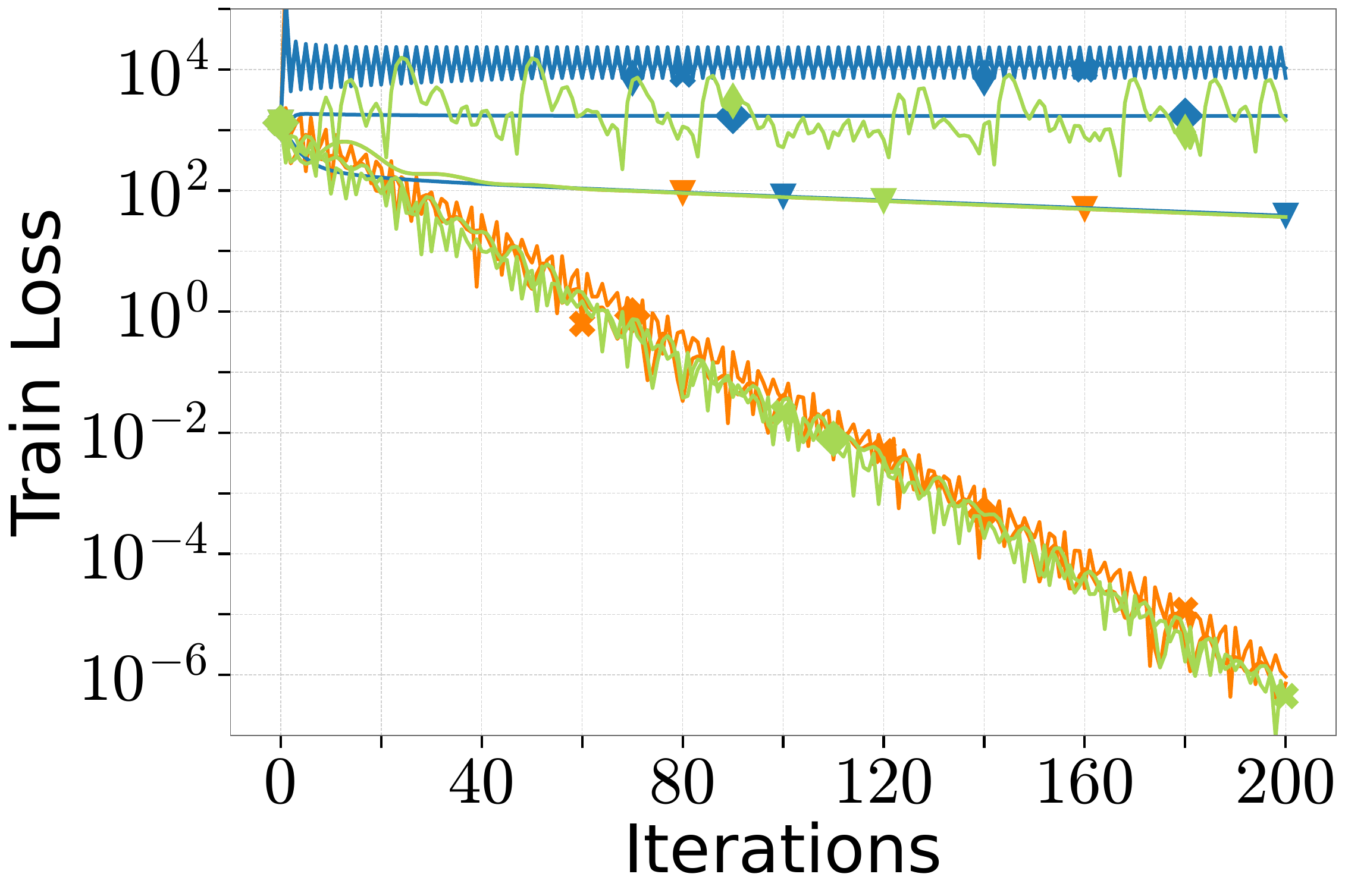} &
    \hspace{-8mm}\includegraphics[width=0.55\linewidth]{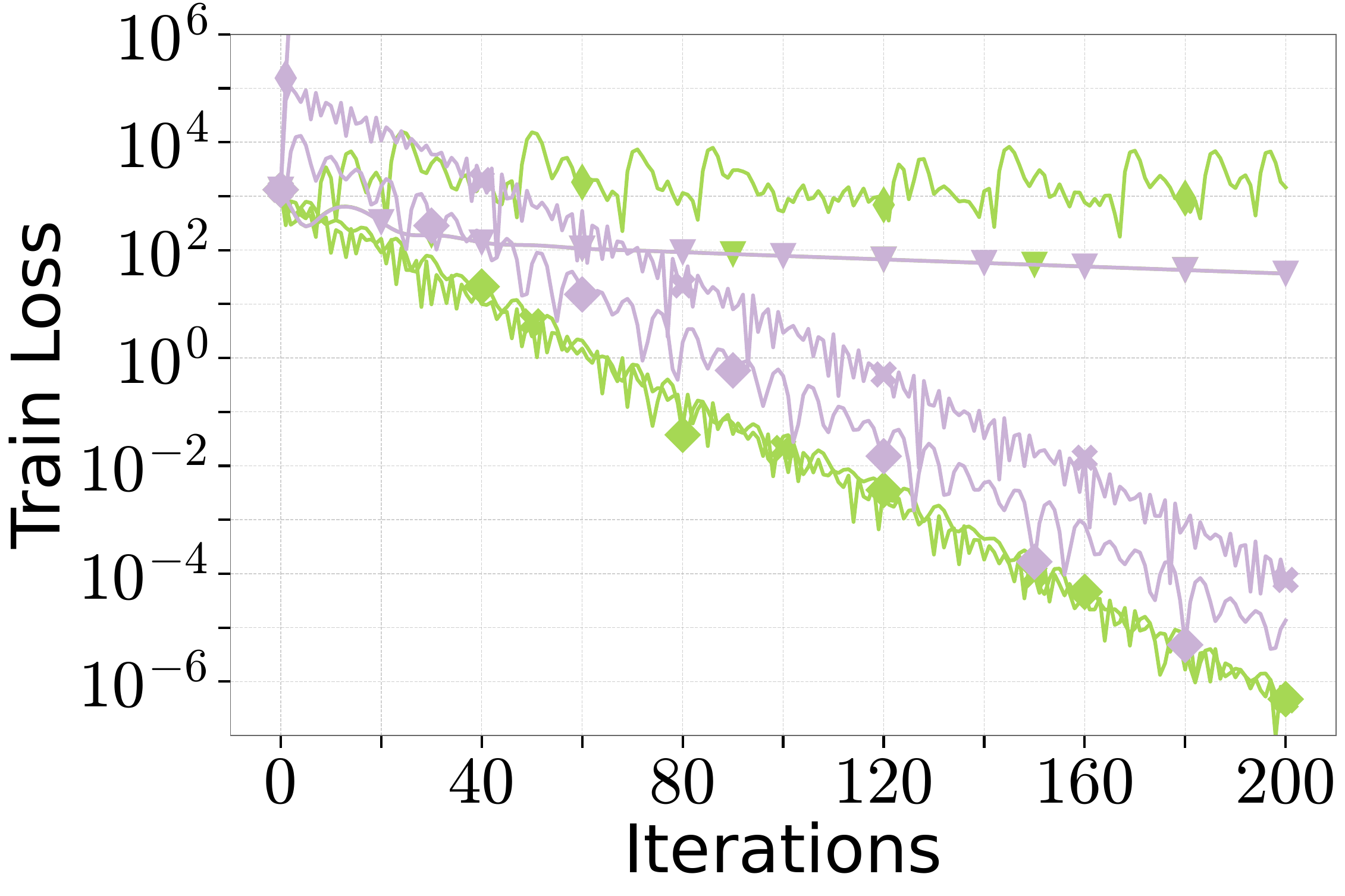} \\

    \hspace{-3mm}\includegraphics[width=0.6\linewidth]{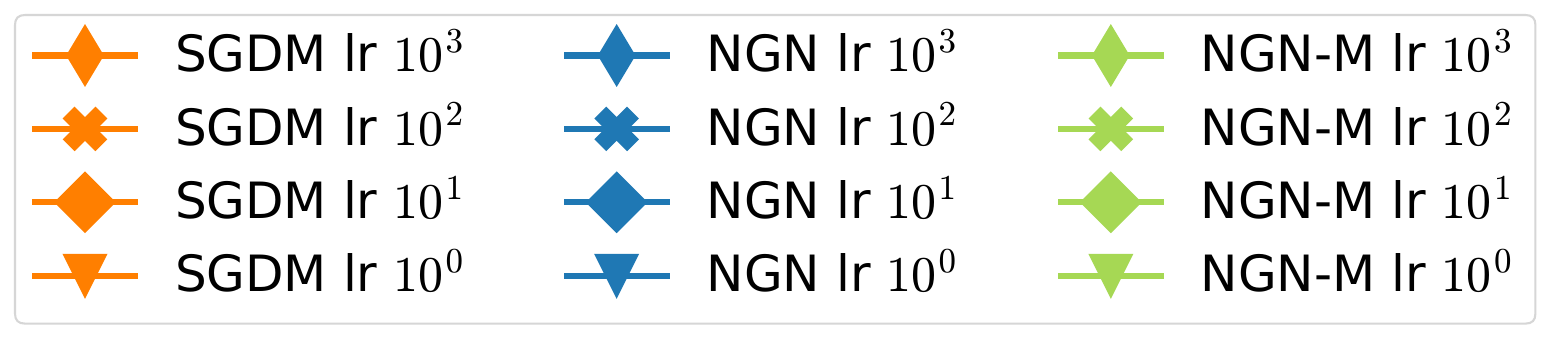} &
   \hspace{-5mm} \includegraphics[width=0.4\linewidth]{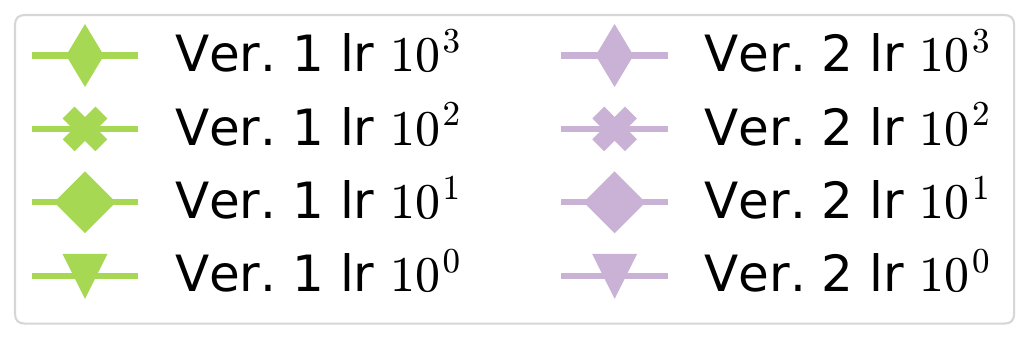}   
    \end{tabular}
    \end{minipage}
    \hspace{0.04\textwidth}
    \begin{minipage}{0.48\textwidth}
    \centering
    \vspace{-18mm}
    \begin{tabular}{cc}
    \hspace{-3mm}\includegraphics[width=0.48\linewidth]{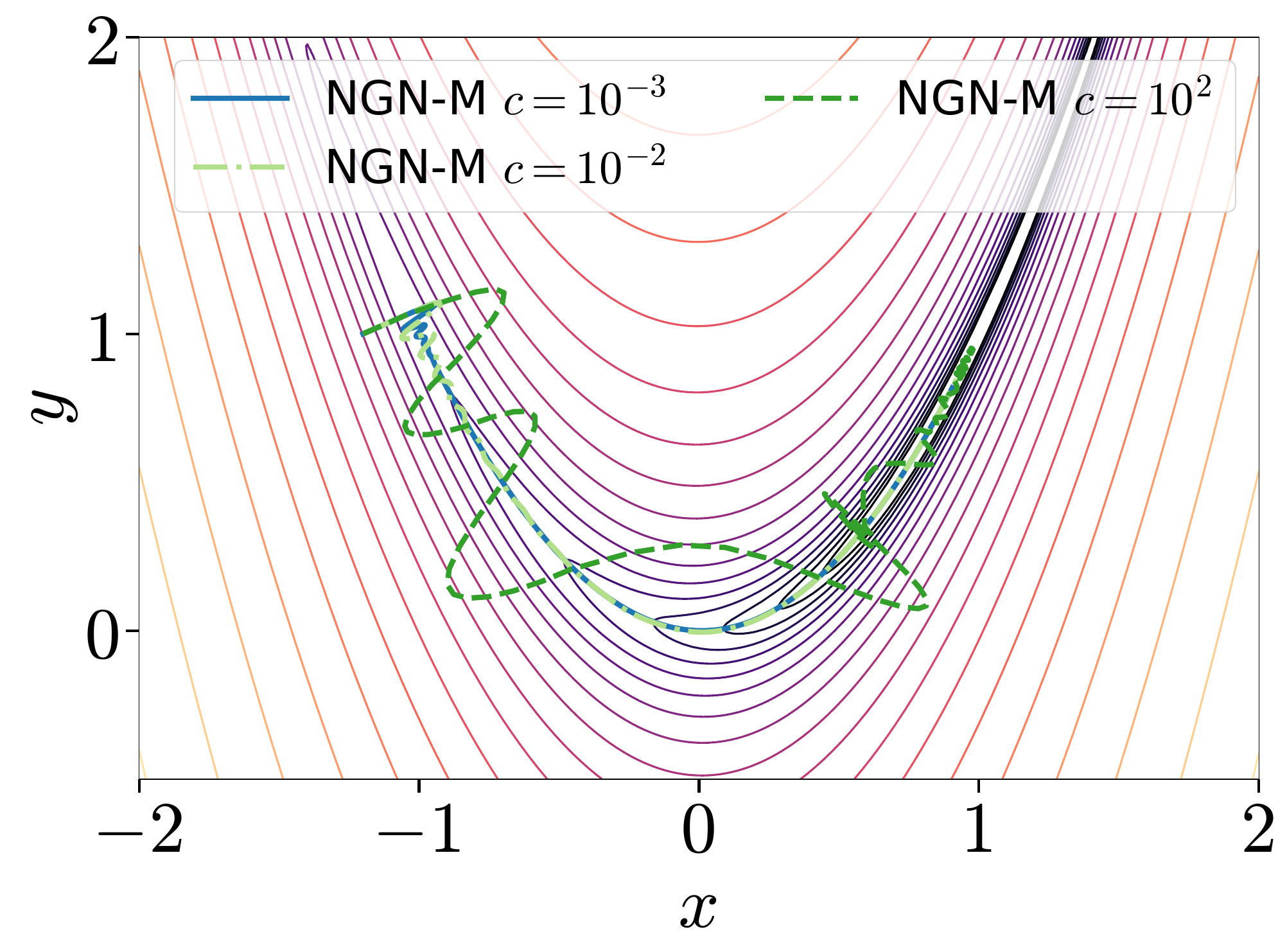} &
    \hspace{-3mm}\includegraphics[width=0.48\linewidth]{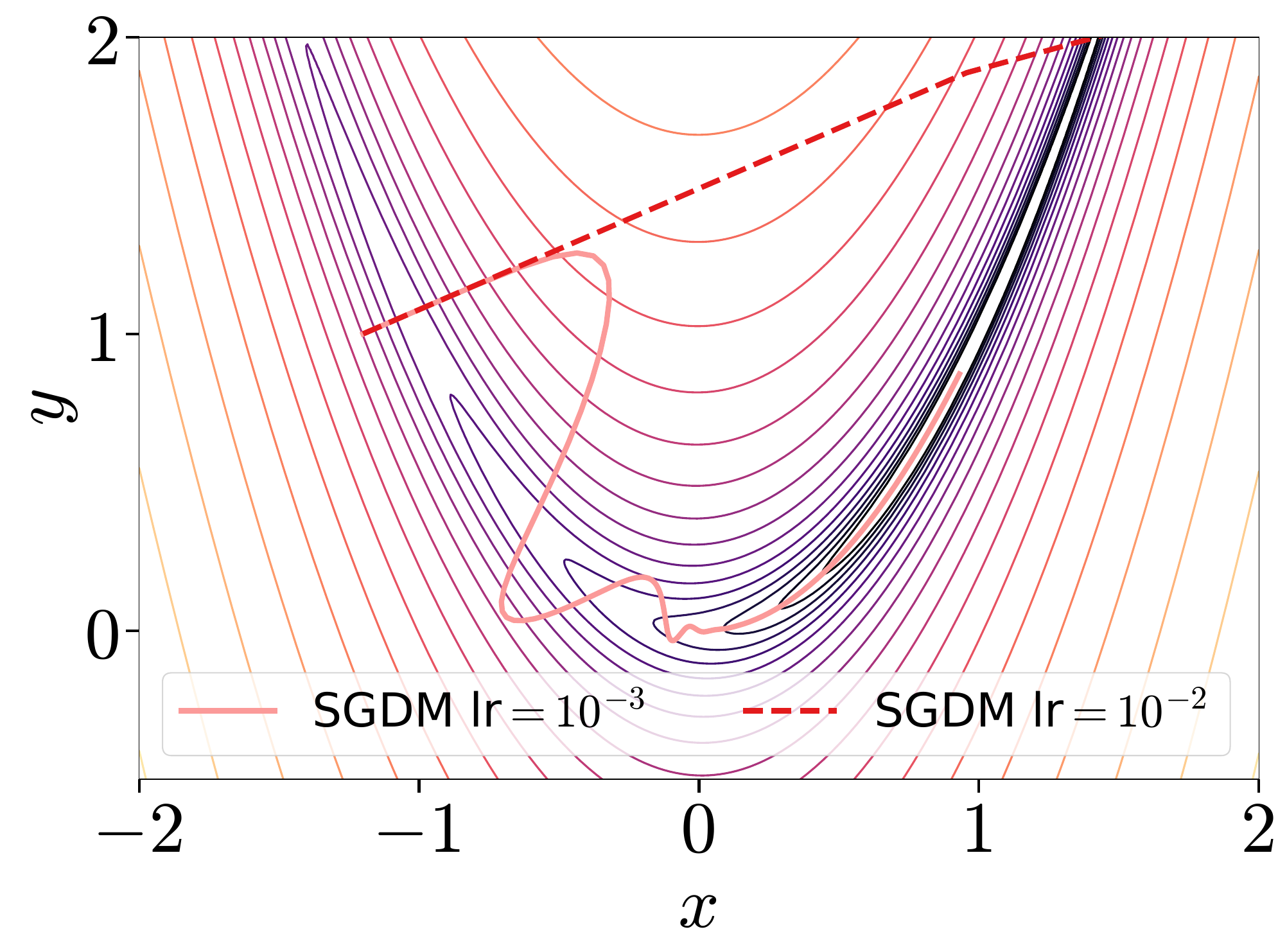}
    \end{tabular}
    \end{minipage}
    \caption{{\bf Left:} Comparison of \algname{SGDM}, \algname{NGN}, \algname{NGN-M} for linear regression on normalized Diabetes dataset varying a step-size hyperparameter. {\bf Second left:} Comparison of two options on how momentum can be used in combination with \algname{NGN} step-size. {\bf Third and fourth:} Comparison of \algname{SGDM} and \algname{NGN-M} on the Rosenbrock function.}
    \label{fig:synthetic_function}
\end{figure}

\subsection{Diagonal Step-size for \algname{NGN}}\label{sec:diagonal_ngn}

We propose two alternatives to make \algname{NGN} step-size coordinate-wise adaptive. In the first approach, we modify an approach of (\ref{eq:nojnwfenoj}): The next iterate $x^{k+1}$ is obtained by minimizing an approximation of the regularized first-order Taylor expansion of $r(x) \eqdef \sqrt{f(x)}$ around $x^k$, namely, $x^{k+1} = x^k + p^k$ where for a preconditioning matrix $\mathbf{\Sigma}_k$
\begin{equation}\label{eq:njinwejnflqlwknqlwknqw}\textstyle
p^k = \argmin\limits_{p} \left[(r(x^k) + \nabla r(x^k)^\top p)^2 + \frac{1}{2c}\|p\|^2_{\boldsymbol{\Sigma}_k}\right].
\end{equation}
The intuition is that $\boldsymbol{\Sigma}_k\in\R^{d\times d}$ can penalize each parameter with its own weight while in vanilla \algname{NGN} the penalization is the same for all parameters, and $f$ is an objective function we aim to minimize. Performing simple derivations (see \Cref{sec:derive_diagonal_ngn}), we obtain the following update rule
\begin{equation}\label{eq:njonjsnao}
    \textstyle
    x^{k+1} = x^k - \frac{c}{1+\frac{c}{2f(x^k)}\|\nabla f(x^k)\|^2_{\boldsymbol{\Sigma}_k^{-1}}}\boldsymbol{\Sigma}_k^{-1} \nabla f(x^k).
\end{equation}
Note that by choosing $\boldsymbol{\Sigma}_k$ to be an identity matrix, the step-size $\gamma_k$ in (\ref{eq:njonjsnao}) reduces to the vanilla \algname{NGN} step-size. 

Alternatively, we can adopt a simpler, coordinate-wise rule: For each parameter $j$, we replace the full gradient norm in the \algname{NGN} step-size with its own partial derivative $\nabla_j f_{S_k}(x^k)$. Both of the described per‐coordinate variants can be further adjusted by an \algname{RMSprop}-style preconditioner $\mD_k = \text{diag}((\mD_k)_{(1)}, \dots, (\mD_k)_{(d)})$ and lead to the following update rule (see Alg.~\ref{alg:ngn_md} for a full description)
\begin{gather*}
    \textstyle \text{\algname{NGN-MDv1}}: \begin{cases}
        \gamma_k = \frac{c}{1+\frac{c}{2f(x^k)}\|\nabla f_{S_k}(x^k)\|^2_{\mD_k^{-1}}} \\
        \mathbf{\Sigma}_k^{-1} = \gamma_k \mD_k^{-1}
    \end{cases} \quad \text{\algname{NGN-MDv2}}: \begin{cases}
        \gamma_k^{(j)} = \frac{c / (\mD_k)_{(j)}}{1+\frac{c / (\mD_k)_j}{2f(x^k)}(\nabla_j f_{S_k}(x^k))^2} \\
        \mathbf{\Sigma}_k^{-1} = \text{diag}(\gamma_k^{(1)}, \dots, \gamma_k^{(d)})
    \end{cases}\\
    x^{k+1} = x^k - (1-\beta_1)\mathbf{\Sigma}_k^{-1}\nabla f_{S_k}(x^k) + \beta_1(x^k - x^{k-1})
\end{gather*}
We highlight that both versions have the same number of hyperparameters as \algname{Adam}. From an empirical evaluation of two versions of \algname{NGN-MD} in \Cref{fig:stability_all_type_test_acc}, we observe that the first choice improves the performance of \algname{NGN-M} while maintaining robustness to step-size hyperparameter. A more detailed discussion on the two versions of \algname{NGN-MD} algorithms is deferred to \Cref{sec:ngn_md_design}.

In the special case $\beta_1=0$ and $\mathbf{\Sigma}_k=\mI$, \algname{NGN-MDv2} reduces to \algname{NGN-D} (Alg.~\ref{alg:ngn_d}). To the best of our knowledge, \algname{NGN-D} is the first algorithm that uses a per-parameter Polyak-type step-size while achieving the standard $\cO(1/\sqrt{K})$ rate under smoothness and bounded noise variance assumptions; see detailed discussion in \Cref{sec:convergence_ngn_d}.

\begin{algorithm}[!t]
    \caption{\algname{NGN-M}}
    \label{alg:ngn_m}
    \begin{algorithmic}[1]
        \State \textbf{Input:} $x^{-1} = x^0\in\R^d,$ step-size hyperparameter $c > 0,$ momentum parameter $\beta \in [0,1)$
        \For{$k = 0,1,\dots,K-1$}
        \State Sample a batch $S_k\subseteq[n]$ 
        \State 
        $\gamma_k = \frac{c}{1+\frac{c}{2f_{S_k}(x^k)}\|\nabla f_{S_k}(x^k)\|^2}$
        \State 
        $x^{k+1} = x^k - (1-\beta)\gamma_k\nabla f_{S_k}(x^k) + \beta(x^k-x^{k-1})$
        \EndFor
    \end{algorithmic}	
\end{algorithm}

\section{Theoretical Analysis of \algname{NGN-M}}\label{sec:theory}

\subsection{Problem Formulation and Notation}

We consider the classic Empirical Risk Minimization (ERM) problem that typically appears when training machine learning models, namely, 
\begin{align}\label{eq:problem}\textstyle
    \min\limits_{x\in\R^d}\left[f(x) \eqdef \frac{1}{n}\sum_{i=1}^n f_i(x)\right],
\end{align}
where $x$ are the parameters of a model we aim to train, $n$ is the number of data points in the dataset, $d$ is the number of parameters, $x^*$ denotes the solution to \eqref{eq:problem}, and $f_i$ represents the loss associated with the $i$-th data point/batch. We assume that each $f_i$ is differentiable and non-negative\footnote{Common losses, e.g. cross-entropy, satisfy this condition.} and that the global optimal value is bounded, i.e. $f^* = \argmin_xf(x) \in\R$. Moreover, we assume that we have access to mini-batch stochastic losses $f_S$ during training such that $f_S^* \eqdef \argmin_xf_S(x) < \infty$ for any $S\subseteq [n]$ picked uniformly at random.


We analyze the convergence of \algname{NGN-M} under assumptions that are often used in the analysis of the Polyak step-size \citep{loizou2021stochastic, orvieto2022dynamics, orvieto2024adaptive, oikonomou2024stochastic, schaipp2024momo}.

\begin{assumption}\label{asmp:convexity_smoothness}

    Each $f_i$ is convex and $L$-smooth, i.e., for all $x,y\in\R^d$ and $i\in[n]$ we have $\<\nabla f_i(x), y-x> \ge f_i(x) - f_i(y)$ and $\|\nabla f_i(x) - \nabla f_i(y)\|\le L\|x-y\|.$
\end{assumption}

\begin{assumption}\label{asmp:bounded_iterpolation} The interpolation  $\sigma^2_{\rm int}\eqdef \mathbb{E}_S[f^* - f_S^*]$ and positive $\sigma^2_{\rm pos} \eqdef \mathbb{E}_S[f^*_{S}]$ errors are bounded. We say that the interpolation holds if $\sigma^2_{\rm int} = 0,$ where $S$ is a sampled mini-batch.
\end{assumption}

\begin{algorithm}[t]
    \caption{\algname{NGN-MD}}
    \label{alg:ngn_md}
    \begin{algorithmic}[1]
        \State \textbf{Input:} $x^0\in\R^d,$ step-size hyperparameter $c > 0,$ momentum parameters $\beta_1, \beta_2 \in [0,1),$ stabilization parameter $\varepsilon>0$, second-order momentum $v^0 = 0$
        \For{$k = 0,1,\dots,K-1$}
        \State Sample a batch 
        $S_k\subseteq[n]$
        \State 
        $v^k = \beta_2 v^{k-1} + (1-\beta_2)(\nabla f_{S_k}(x^k) \odot \nabla f_{S_k}(x^k))$
        \State 
        $\mD_k = {\rm diag}(\varepsilon\mI + \sqrt{v^k/(1-\beta_2^k)})$
        \State 
        For \algname{NGN-MDv1:}
        $\gamma_k = \frac{c}{1+\frac{c}{2f_{S_k}(x^k)}\|\nabla f_{S_k}(x^k)\|^2_{\mD_k^{-1}}}$ \hfill 
        \State 
        For \algname{NGN-MDv1:}
        $\boldsymbol{\Sigma}^{-1}_k = \gamma_k\mD_k^{-1}$ 
        \State 
        For \algname{NGN-MDv2:} $\boldsymbol{\Sigma}_k^{-1} = {\rm diag}(\gamma_k^{(1)}, \dots, \gamma_k^{(d)})$ where
        $\gamma_k^{(j)} = 
        \frac{c/(\mD_k)_{(j)}}{1+\frac{c}{2f_{S_k}(x^k)\cdot (\mD_k)_{(j)}}(\nabla_j f_{S_k}(x^k))^2}
        $
        \State 
        $x^{k+1} = x^k - (1-\beta_1)\boldsymbol{\Sigma}_k^{-1}\nabla f_{S_k}(x^k) + \beta_1(x^k-x^{k-1})$
        \EndFor
    \end{algorithmic}
\end{algorithm}

\subsection{Convergence Guarantees}

\begin{restatable}[]{theorem}{theoremngnmconvex}\label{th:theorem_ngn_m_convex}
Let Assumptions~\ref{asmp:convexity_smoothness}, \ref{asmp:bounded_iterpolation} hold. Let the step-size hyperparameter $c > 0$ and the momentum parameter $\beta = \frac{\lambda}{1+\lambda}$ be constants where $\lambda \le \min\{cL, 0.5(1+cL)^{-1}(1+2cL)^{-1}\}$. Then the iterates of \algname{NGN-M} (Alg. \ref{alg:ngn_m}) satisfy
\begin{gather}
\textstyle
\E{f(\overline{x}^{K-1}) - f(x^*)}
        \le \frac{\|x^0-x^*\|^2(1+2cL)^2}{c K}
        + 8cL(1+2cL)^2\sigma^2_{\rm int}
        + 2cL\max\left\{2c L - 1, 0\right\}\sigma^2_{\rm pos},\notag
\end{gather}
where $\overline{x}^{K-1}$ is chosen uniformly at random from $\{x^0, \dots, x^{K-1}\}.$ Moreover, if we set $c = \cO(\nicefrac{1}{\sqrt{K}})$ then we obtain $\E{f(\overline{x}^{K-1}) - f(x^*)} \le \cO(\nicefrac{1}{\sqrt{K}}).$
\end{restatable}

The convergence of \algname{NGN-M} is provided in the convex setting, which is motivated by recent works that observe convex-like structures in the loss landscape of neural networks \citep{islamov2024loss, hoang2024empirical} and agreement between convex theory and practice \citep{schaipp2025surprising}. Importantly, we show that $(i)$ when the constant $c$ is sufficiently small, \algname{NGN-M} attains the same convergence rate as \algname{SGDM} \citep{garrigos2023handbook}. Moreover, for any choice of $c$, we demonstrate that the \algname{NGN-M} iterates provably converge to a neighborhood of the optimum and thereafter remain within it; $(ii)$ Unlike prior works, our analysis does not rely on strong assumptions such as bounded gradients, interpolation, or a bounded domain; $(iii)$ For small values of $c$, \algname{NGN-M} converges to the exact solution while algorithms such as \algname{MomSPS} and \algname{ALR-SMAG} were shown to converge up to a non-vanishing neighborhood of the solution only\footnote{In fact, this is an inherited property of \algname{SPS} analysis from \citep{loizou2021stochastic}.}. Notably, the non-vanishing neighborhood disappears when the problem satisfies interpolation: We refer to \Cref{tab:comparison} for more details and exact rates; $(iv)$ The momentum parameter $\beta$ is theoretically recommended to be set sufficiently small. A default value of $\beta=0.9$ is commonly used and works well in our experiments. This discrepancy between theoretical guidance and practical implementation has also been observed in prior works on momentum \citep{ghadimi2015global, liu2020improved, wang2023generalized, wang2022provable, oikonomou2024stochastic}.  Interestingly, for simple functions we can establish convergence even when $\beta$ is large (see \Cref{sec:stability_simple_problems}), indicating that the small-$\beta$ requirement may be an artifact of the existing proving techniques rather than an inherent algorithmic limitation of \algname{NGN-M}. We leave a comprehensive study of arbitrary $\beta$ values across general convex objectives for future work; $(v)$ While Theorem \ref{th:theorem_ngn_m_convex} requires knowing the total iteration count $K$ to ensure convergence, this might be impractical: We therefore also prove convergence using a diminishing step-size of order $1/\sqrt{k}$ in \Cref{sec:decaying_stepsize}, which removes the need to preset $K$; $(vi)$ Finally, we corroborate our analysis as we run \algname{NGN-M} with the theory-derived values of $c$ to a quadratic problem that satisfies all our assumptions: We observe \algname{NGN-M}'s rapid convergence with theoretical step-size hyperparameters in practice---see \Cref{sec:theoretical_stepsize} and \Cref{fig:quadratic_convergence} therein.


\begin{figure}
    \centering
    \begin{tabular}{ccc}
        \includegraphics[width=0.3\linewidth]{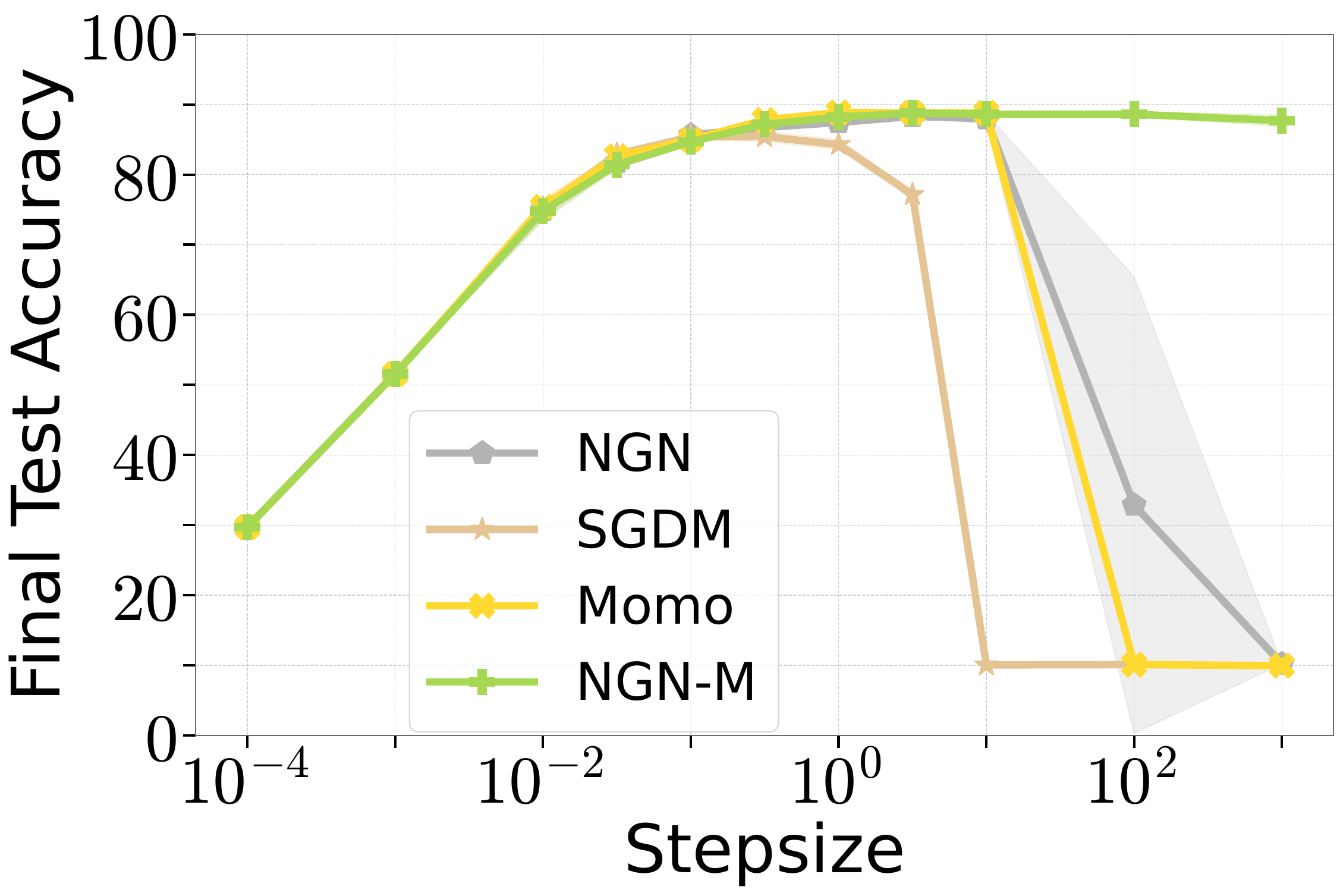}  & 
        \includegraphics[width=0.3\linewidth]{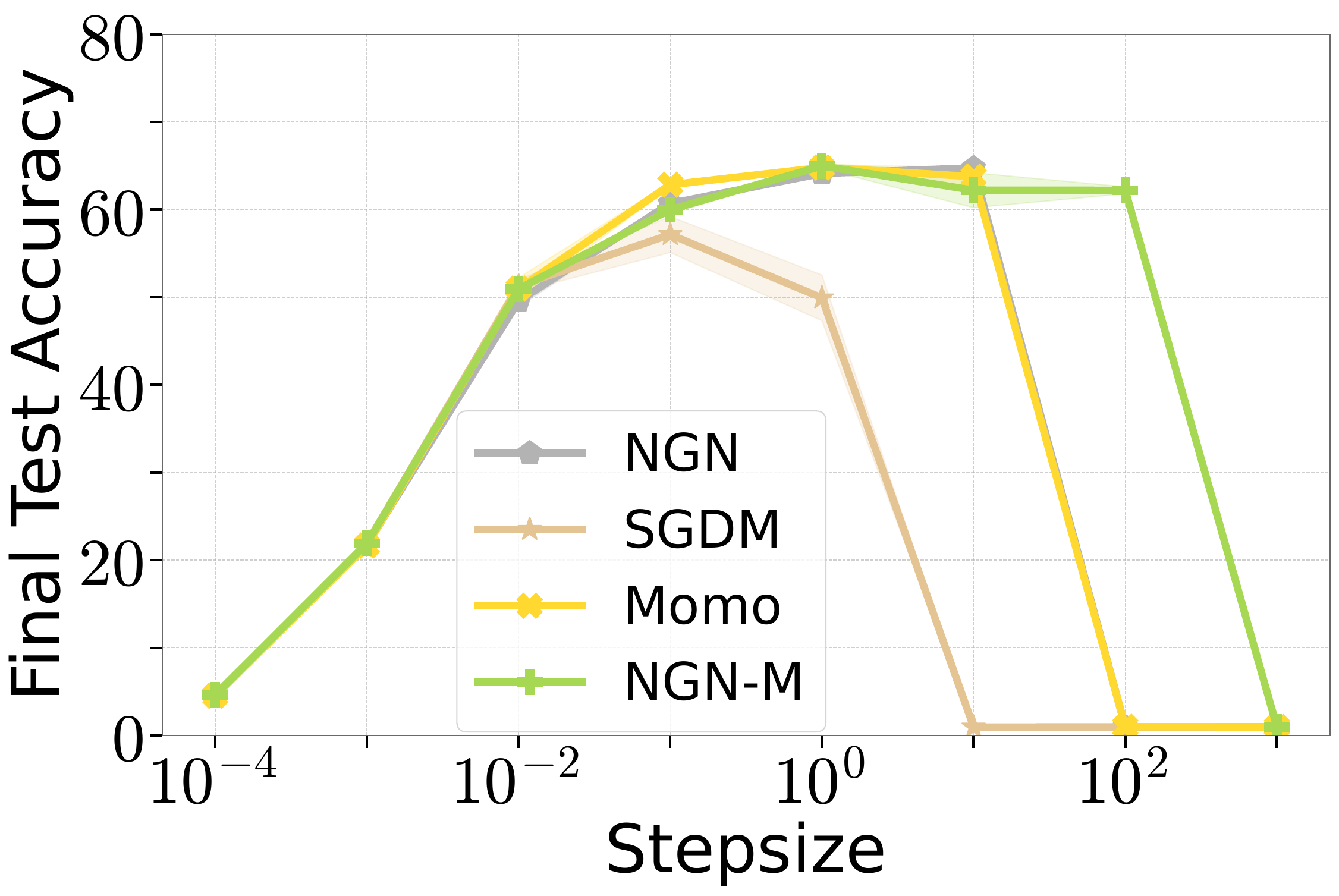} & 
        \includegraphics[width=0.3\linewidth]{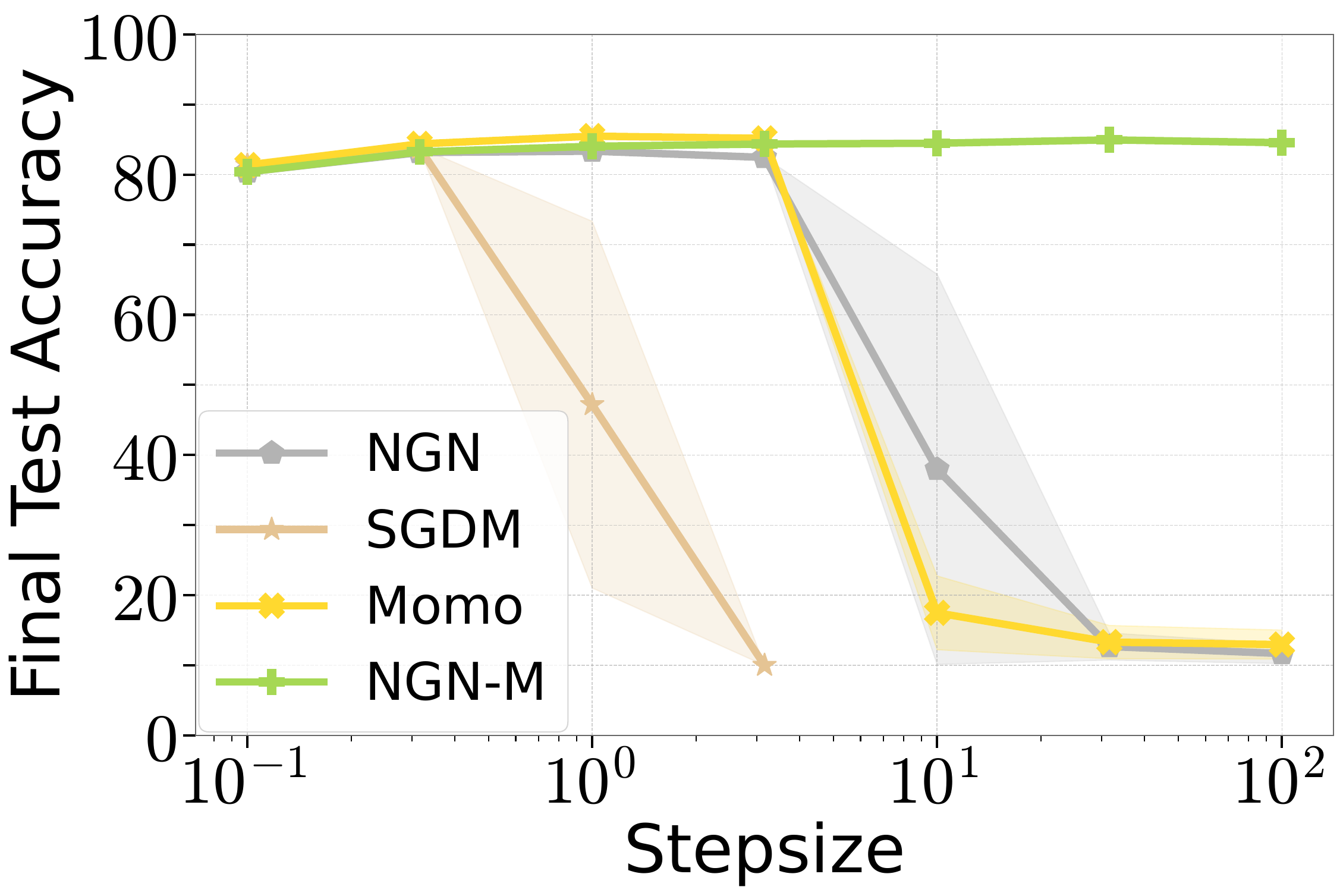} \\
        \includegraphics[width=0.3\linewidth]{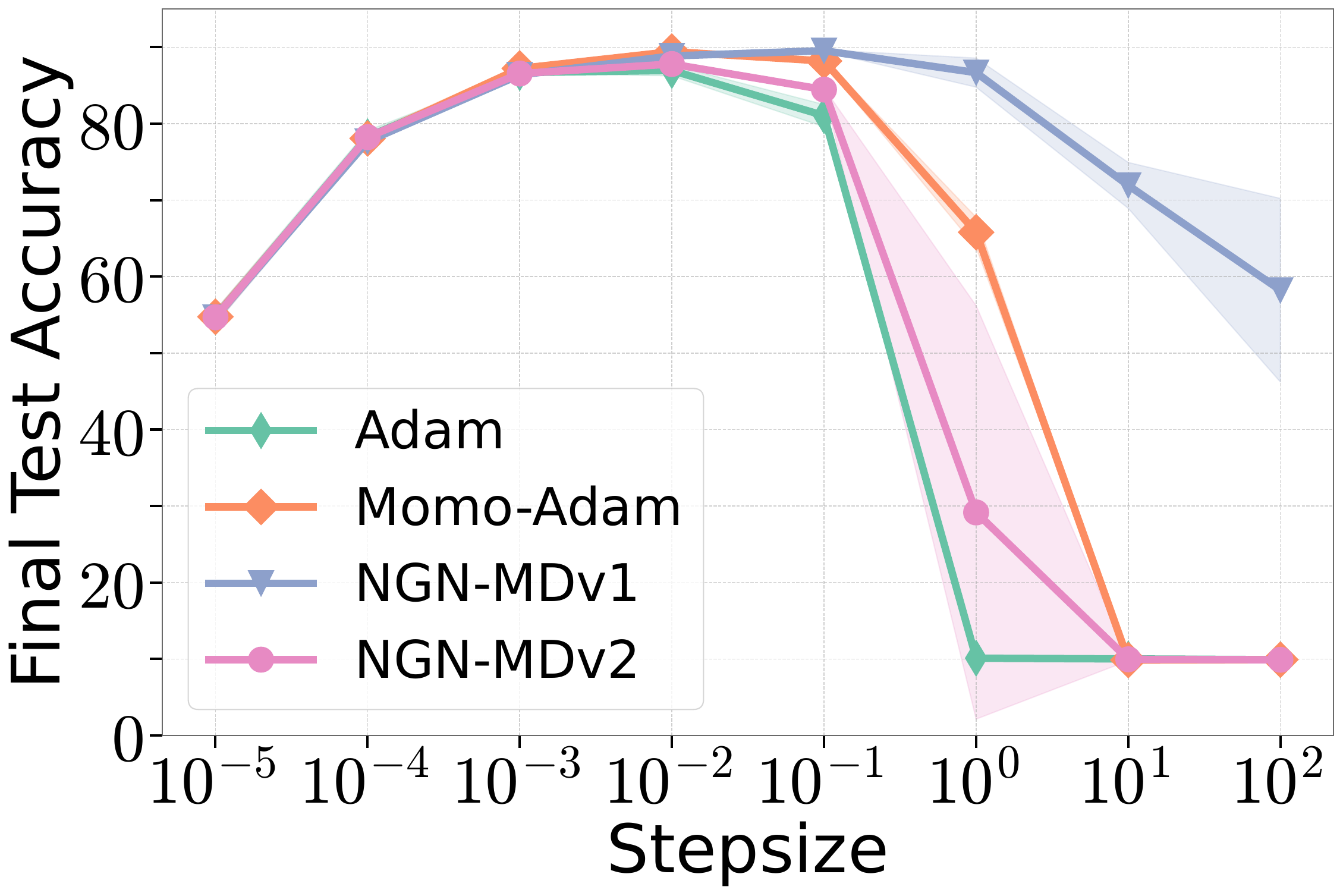} &
        \includegraphics[width=0.3\linewidth]{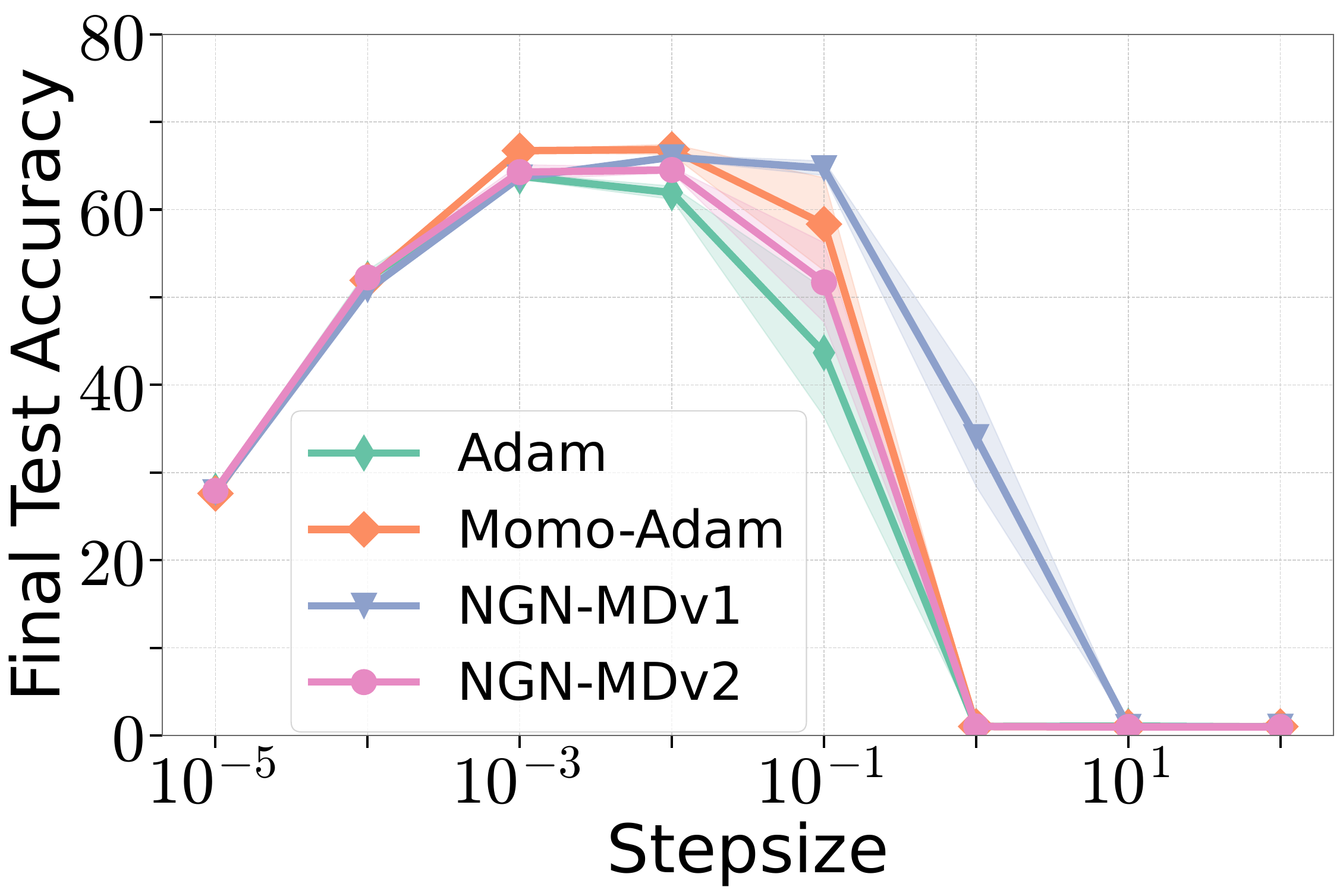} & 
        \includegraphics[width=0.3\linewidth]{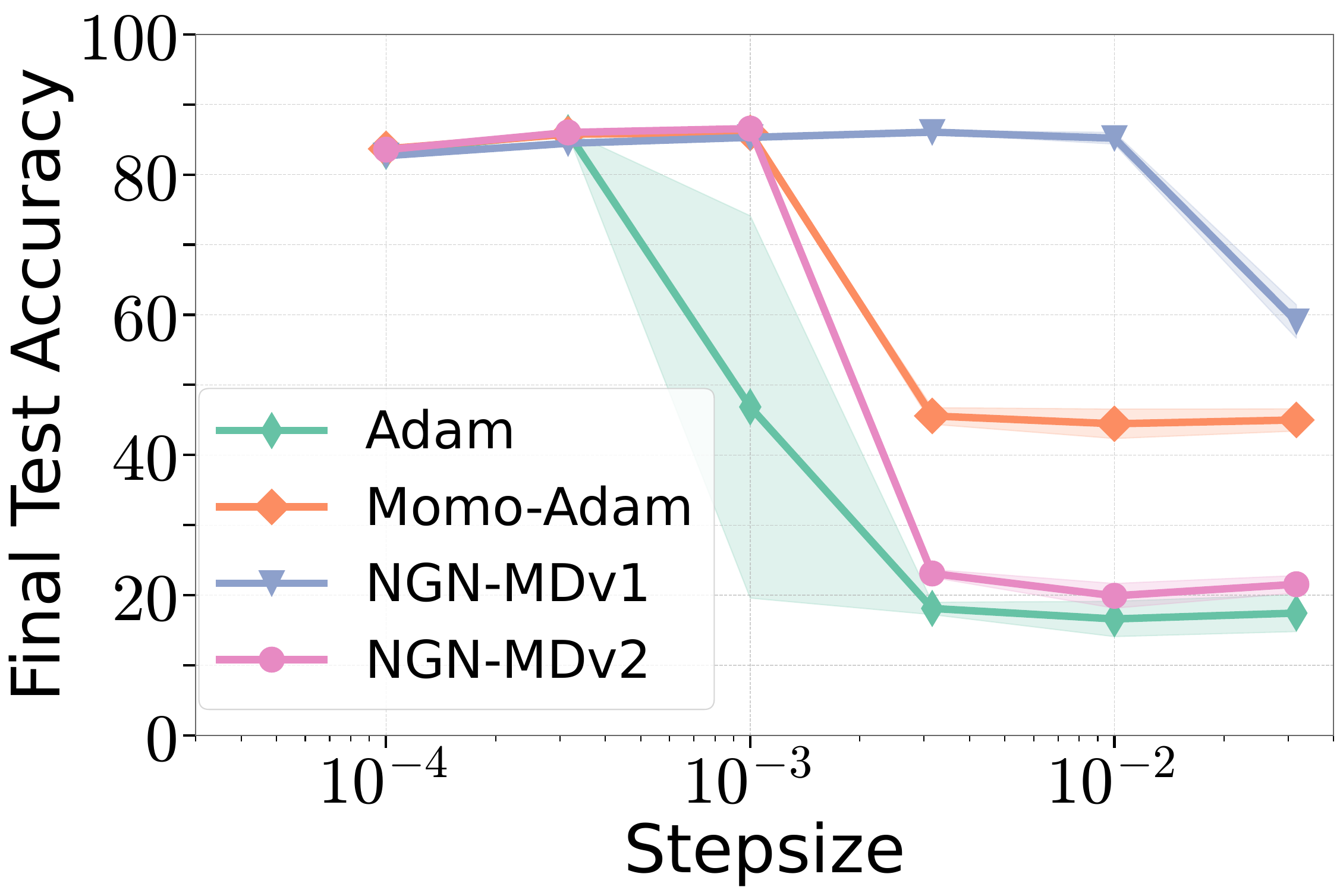} \\
        {\small Resnet20 for CIFAR 10} &
        {\small Resnet110 for CIFAR 100} & 
        {\small ViT for CIFAR 10}
    \end{tabular}
    \caption{Stability performance of algorithms varying step-size hyperparameter ($c$ for \algname{NGN-M}, \algname{NGN-MDv1} and \algname{NGN-MDv2}, $\alpha_0$ for \algname{Momo} and \algname{Momo-Adam}, and step-size for \algname{SGDM} and \algname{Adam}). For \algname{NGN-M} and \algname{NGN-MDv1}, we observe that the range of the step-size hyperparameters that provide competitive performance is wider than that for other algorithms. We refer to \Cref{fig:stability_momentum_train_loss,fig:stability_adam_type_train_loss,fig:stability_adam_type_train_loss_additional_workloads}, \ref{fig:also_adabound_adabelief_lion_resnet20} and \ref{fig:rebuttals_nlp_tasks} for train loss stability and for the results on additional workloads.}
    \label{fig:stability_all_type_test_acc}
\end{figure}

\paragraph{Key Ingredients of the Proof.} We discuss the key steps of the proof to highlight the main challenges in the analysis. 

First, we make use of the Iterative Moving Average (IMA) formulation of momentum \citep{sebbouh2021almost}. Specifically, we define a sequence of virtual iterates $\{z^k\}$ whose update rule is of the form
\begin{equation*}\textstyle
    z^{k+1} = x^k -\gamma_k \nabla f_{S_k}(x^k),\quad   x^{k+1} = \frac{\lambda}{1+\lambda}x^k + \frac{1}{1+\lambda}z^{k+1}, \quad \text{where } z^0 \eqdef x^0 \text{ and } \beta = \frac{\lambda}{1+\lambda}.
\end{equation*}

Next, one of the key technical strategies we follow is splitting the step-size $\gamma_k$ into two parts: a fixed term $\rho = \frac{c}{(1+cL)(1+2cL)} = \cO(c)$ and a changing term $\wtilde{\gamma}_k \le \frac{3c^2L}{1+2cL} = \cO(c^2)$. This decomposition of the step-size $\gamma_k$ enables us to regulate the balance between the descent term, which drives improvement in the objective, and the error term, which reflects possible inaccuracies. More precisely, the descent term is weighted by $c$ while the error term proportional to $\sigma^2_{\rm int}$ is weighted by $c^2$, which suggests that $c$ has to be chosen to tradeoff the two terms to lead to the exact convergence similarly to the standard analysis of \algname{SGD} \citep{garrigos2023handbook}. In contrast, \algname{MomSPS} and \algname{Momo} algorithms achieve the exact convergence only under the interpolation regime.

\section{Experiments}\label{sec:experiments_main}

We now turn to the empirical evaluation of the proposed algorithms against several benchmarks. The detailed experiment setup, including the choice of hyperparameters as well as additional experimental results and details, can be found in \Cref{sec:exp_appendix}. The best performance of algorithms is reported in Tables \ref{tab:empirical_comparison_momentum_appendix} (momentum-based algorithms),\ \ref{tab:empirical_comparison_adam_type_lion_adabound_adabelief} (algorithms with momentum and component-wise step-size), and \ref{tab:empirical_comparison_cd_type} (algorithms with component-wise step-size). For clarity and quick reference, all links to the paper’s empirical results are summarized in \Cref{tab:summary_experiments}.

\begin{figure}[t]
    \centering
    \begin{tabular}{ccc}
       \includegraphics[width=0.3\linewidth]{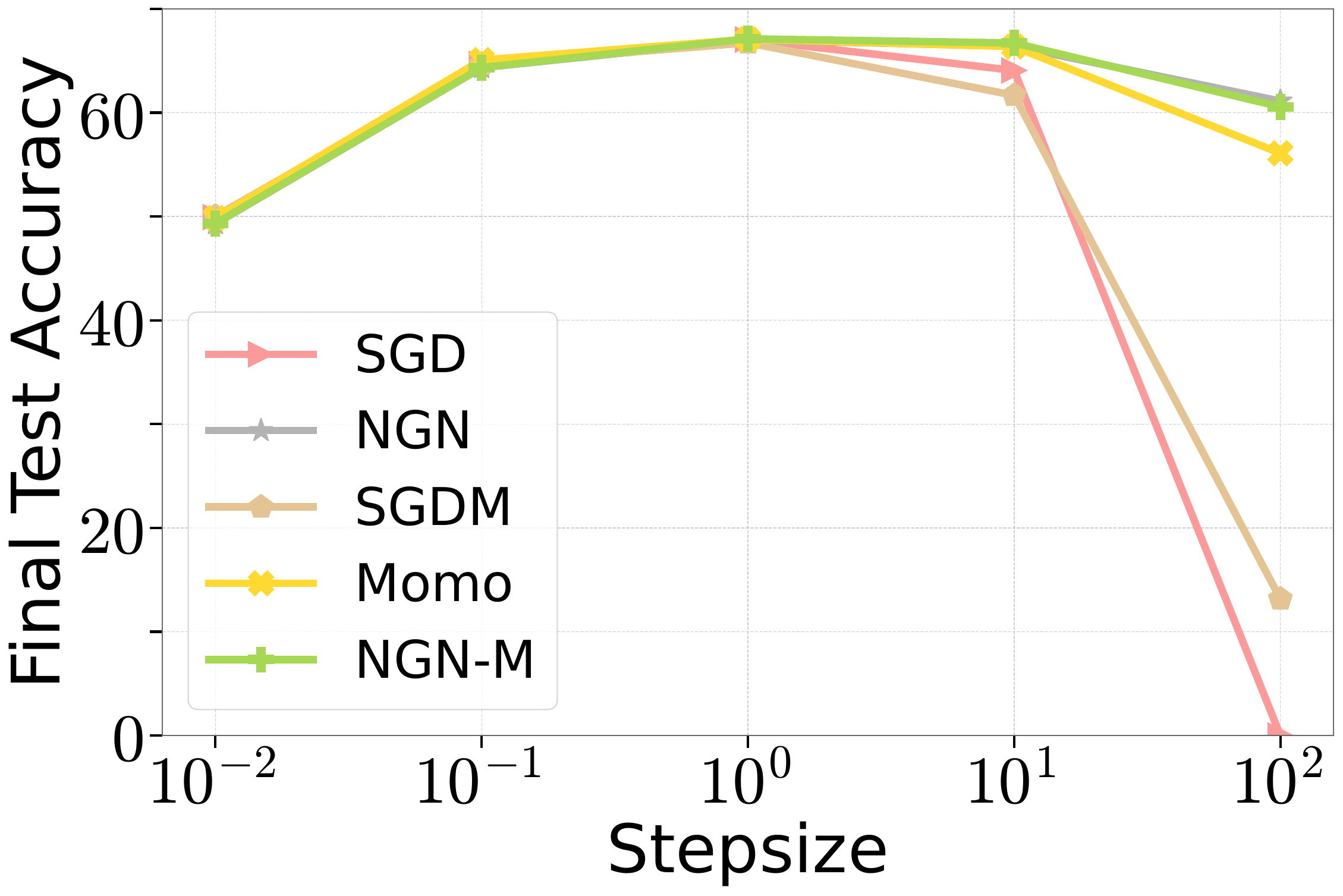} &
       \includegraphics[width=0.3\linewidth]{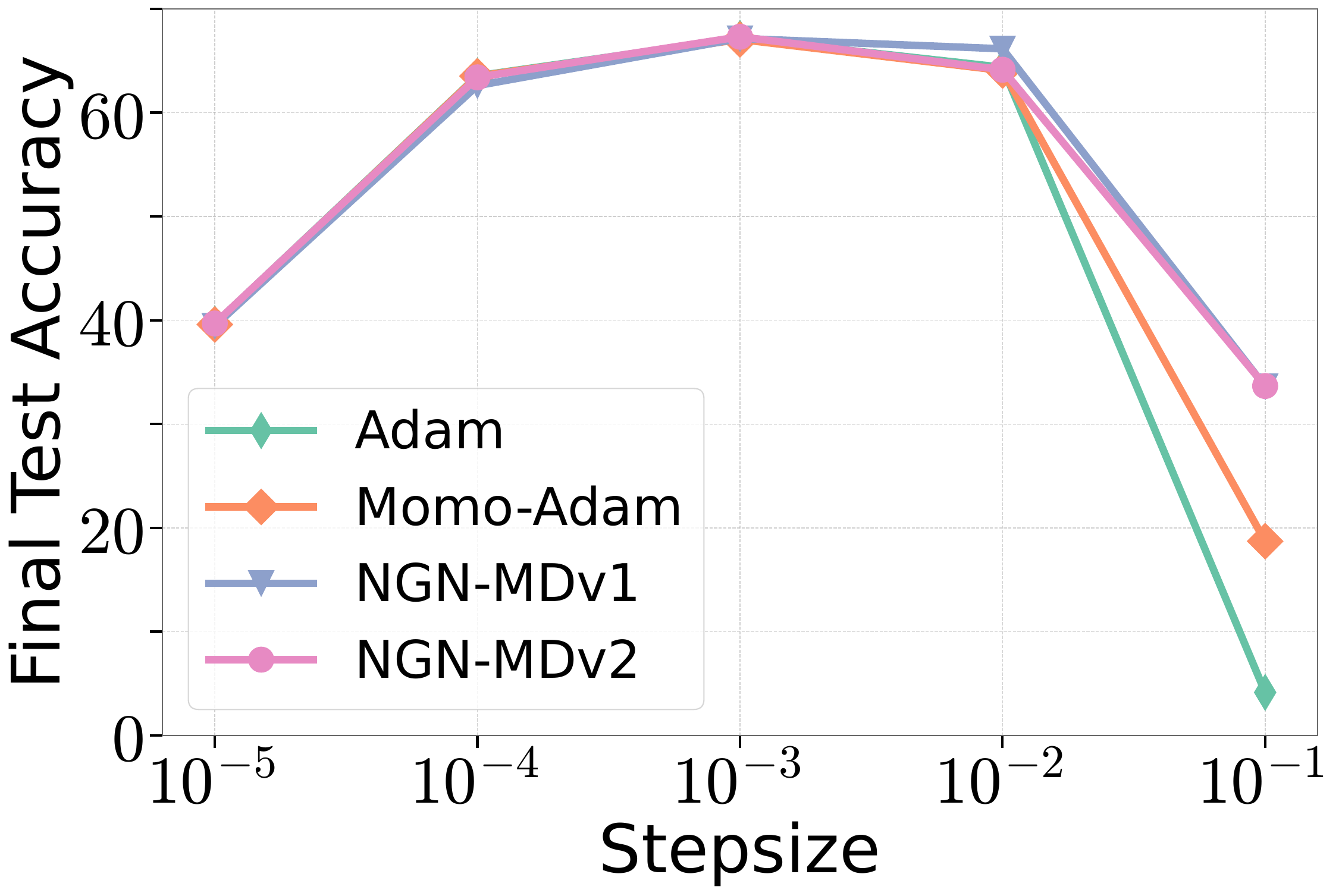} &
       \includegraphics[width=0.3\linewidth]{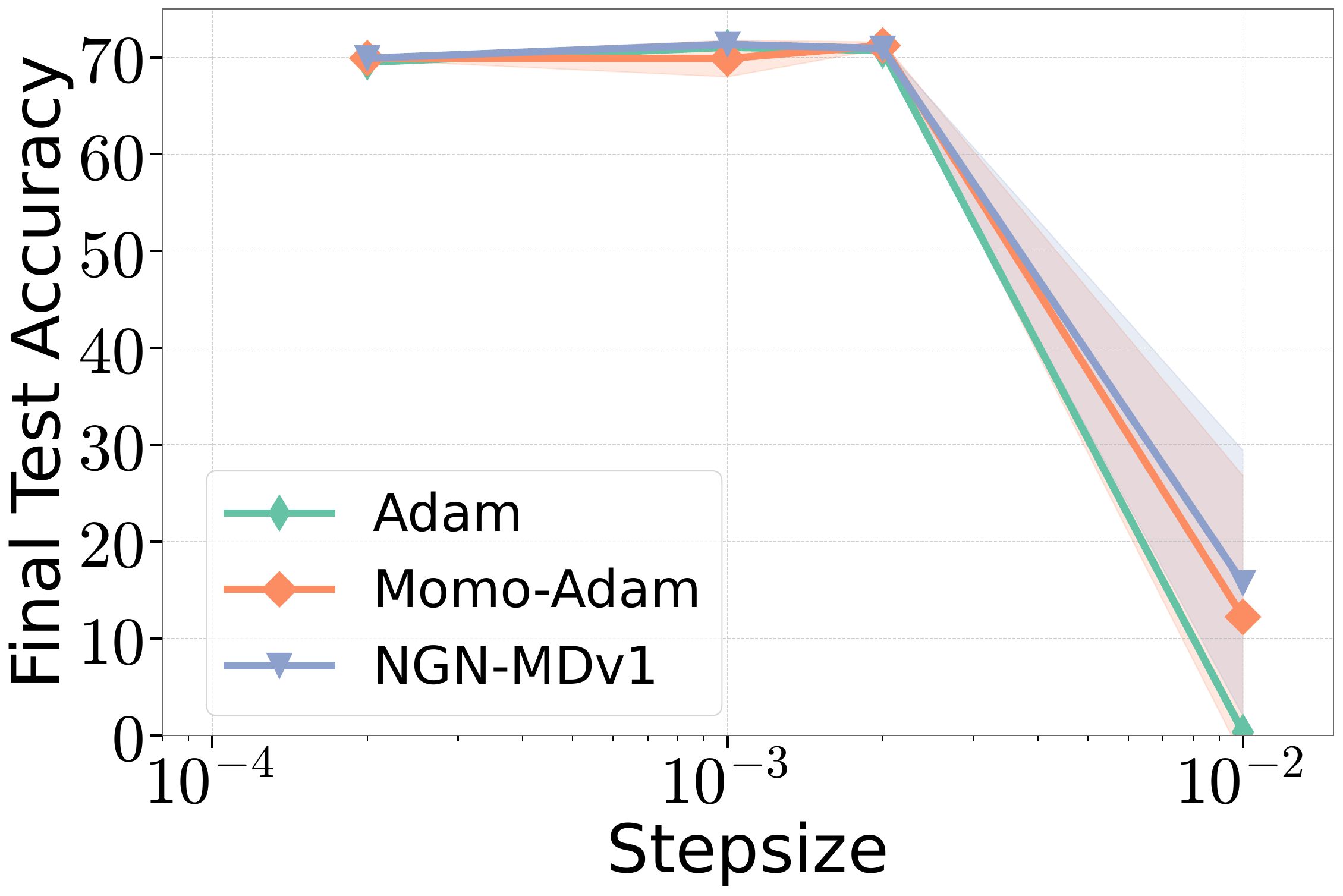} \\
       {\small Resnet18 for ImageNet1k}  &
       {\small Resnet18 for ImageNet1k} &
       {\small ViT-tiny for ImageNet1k} 
       
    \end{tabular}
    \caption{Stability performance on ImageNet1k varying the step-size hyperparameter. 
    \algname{NGN-M} and \algname{NGN-MDv1} achieve higher accuracy for a wider range of the step-size hyperparameters. We refer to \Cref{fig:stability_imagenet_train_loss} for results on train loss stability and additional results on ImageNet32.}
    \label{fig:stability_all_type_imagenet_test_acc}
\end{figure}

\paragraph{Comparison on Standard Benchmarks.}

First, we test the performance of \algname{NGN-M} against other methods that use momentum, such as \algname{SGDM}, \algname{Momo}, \algname{MomSPS}, \algname{ALR-SMAG}, and \algname{NGN}. The tests include the training of Resnet20 \citep{he2016deep} and ViT \citep{dosovitskiy2021image}  on the CIFAR10 dataset \citep{krizhevsky2014cifar10}, and Resnet110 on CIFAR100. Second, we test the performance of \algname{NGN-MD} against \algname{Adam} and \algname{Momo-Adam} that -- contrary to \algname{NGN-M} -- both use component-wise preconditioning. All experiments in this section do not use learning rate schedulers or weight decay.

From Tables~\ref{tab:empirical_comparison_momentum_appendix} and \ref{tab:empirical_comparison_adam_type_lion_adabound_adabelief} we observe that the best performance of \algname{NGN-M} and \algname{NGN-MDv1} matches the results of other algorithms: \algname{NGN-M} and \algname{NGN-MDv1} exhibit competitive performance across all settings we tested. Importantly, \algname{NGN-M} and \algname{NGN-MDv1} demonstrate significantly greater robustness to the choice of the step-size hyperparameter. Indeed, \Cref{fig:stability_all_type_test_acc} shows that the range of step-size hyperparameter that allows \algname{NGN-M} and \algname{NGN-MDv1} to perform optimally is much wider: We can, for instance, use step-sizes that are $1$-$2$ orders of magnitude larger than the optimal one without a significant drop in the performance. This is particularly evident when training ResNet20 and ViT models. Besides, we clearly observe that momentum consistently improves the stability of \algname{NGN} across all settings. We refer to \Cref{sec:exp_appendix} for additional ablation studies against other optimizers and results when training NLP models.

\begin{figure}[t]
    \centering
    \begin{tabular}{cccc}
        \hspace{-2mm}\includegraphics[width=0.24\linewidth]{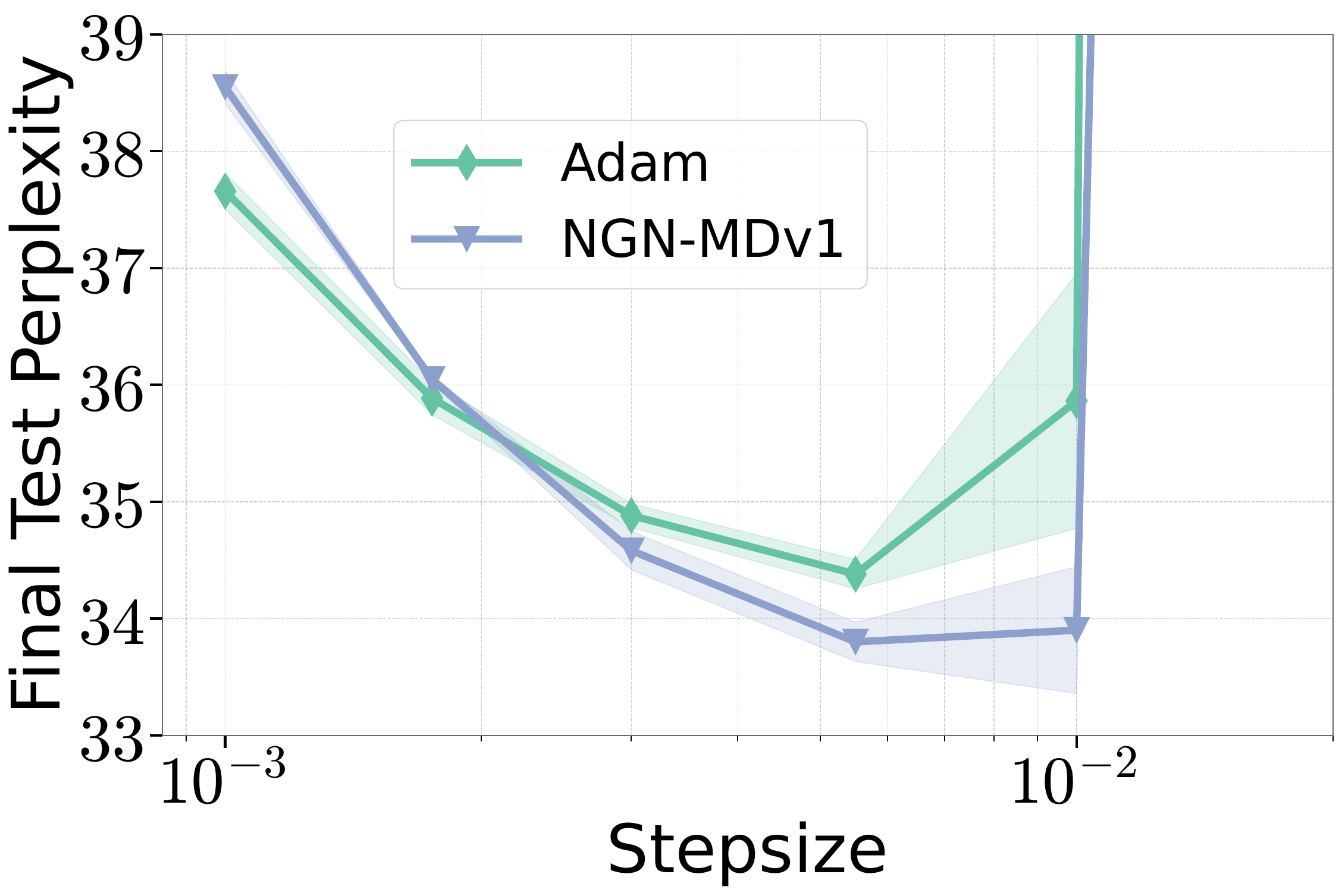} &
       \hspace{-3mm}\includegraphics[width=0.24\linewidth]{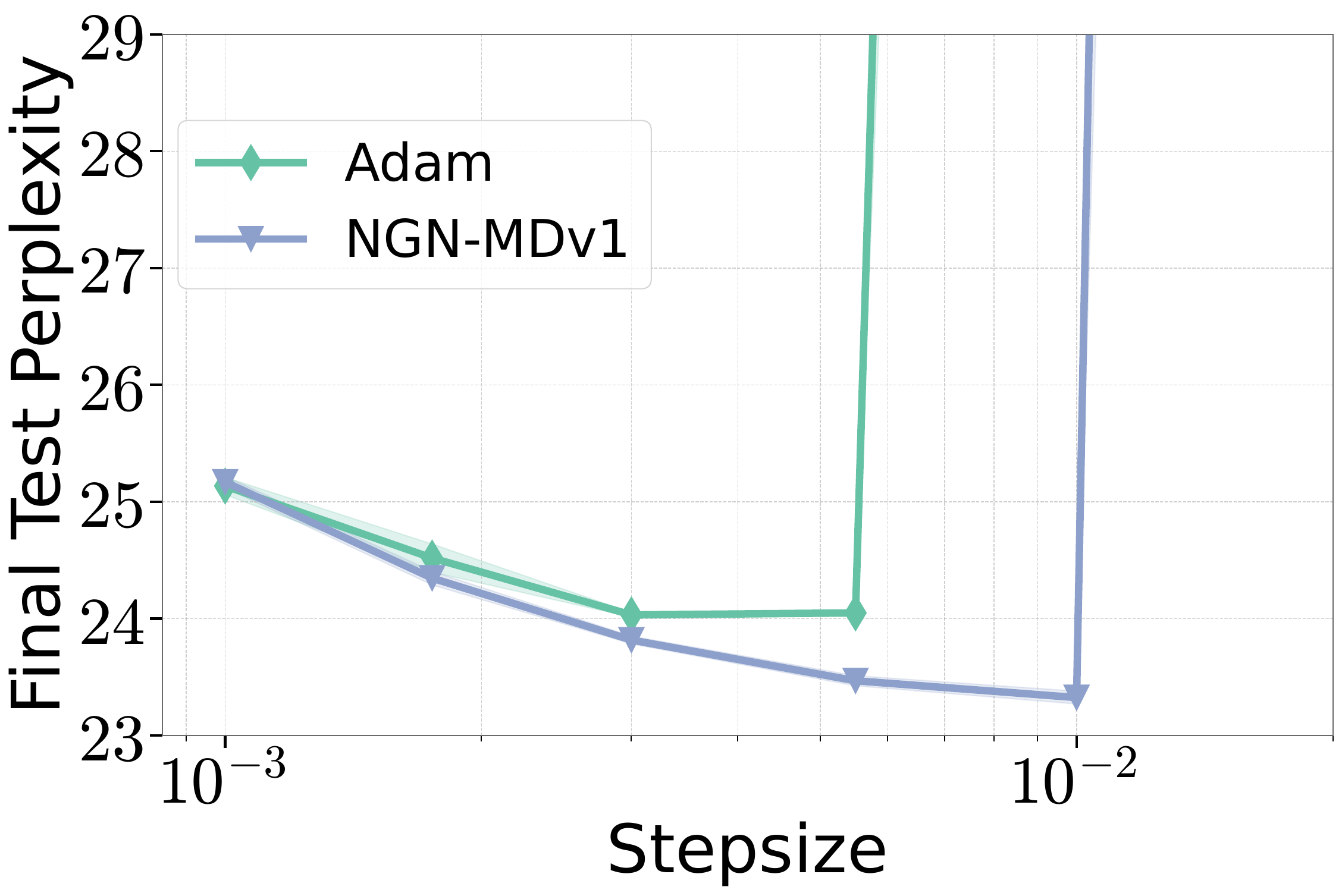}  &  
       \hspace{-3mm}\includegraphics[width=0.24\linewidth]{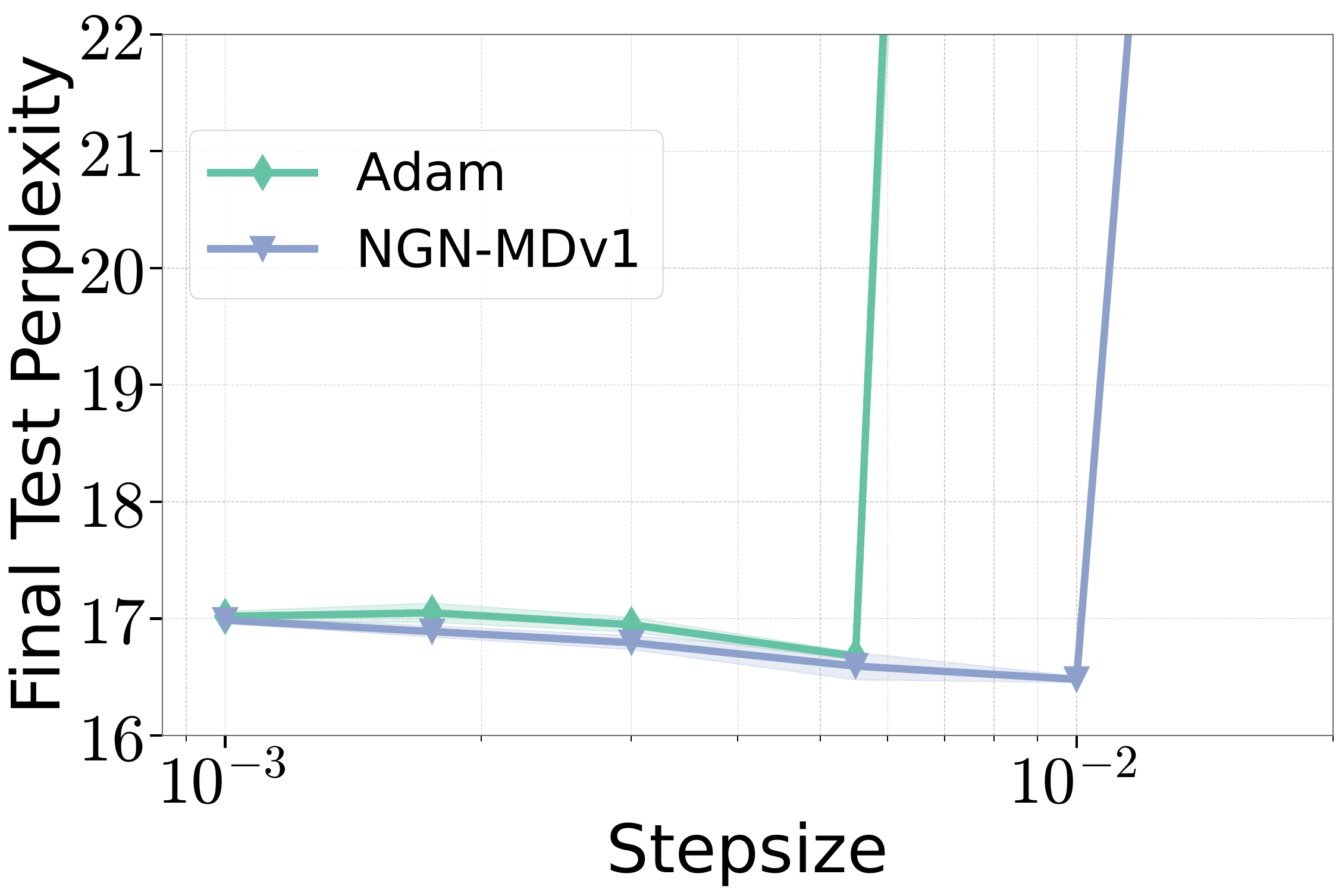} &
       \hspace{-3mm}\includegraphics[width=0.24\linewidth]{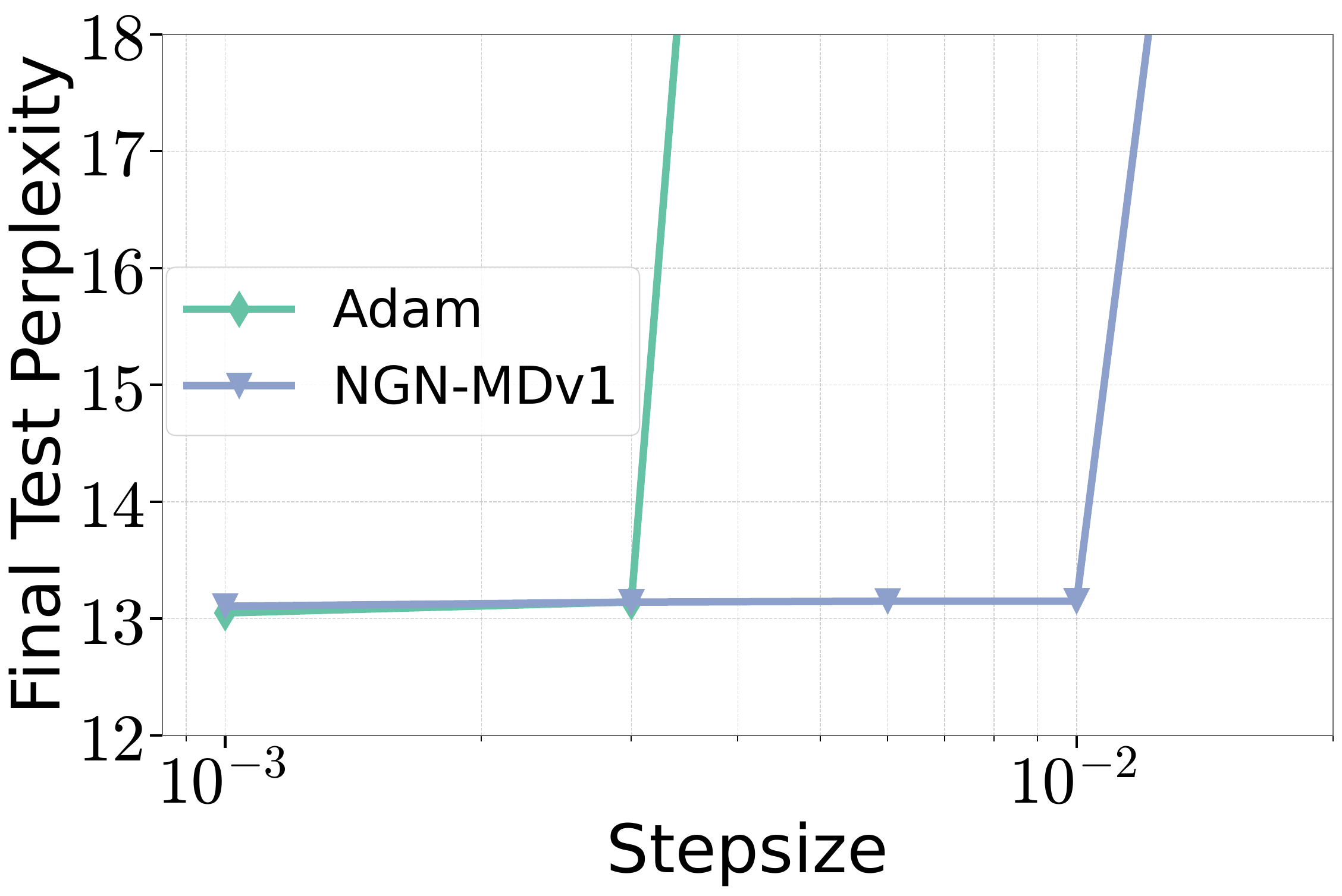} 
       \\
       {\small $70$M Transformer++} &
       {\small $160$M Transformer++} &
       {\small $410$M Transformer++} &
       {\small $1$B Transformer++}
    \end{tabular}
    
    \caption{Language Modeling on SlimPajama. Stability comparison with respect to the step-size hyperparameter across different model sizes and optimizers.  At all model capacities, \algname{NGN-MDv1} achieves similar or lower perplexity, showing better stability and improved performance at larger learning rates. We refer to  \Cref{fig:update_magnitude_lr0003,fig:update_magnitude_lr001,fig:update_magnitude_lr003,fig:llm_training_dynamics} for the results that report update magnitude when training $160$M model and training dynamics across all model sizes.}
    \label{fig:llm_results}
\end{figure}

\begin{figure*}[t]
    \centering
    \begin{tabular}{cccc}
       \multicolumn{2}{c}{\includegraphics[width=0.45\linewidth]{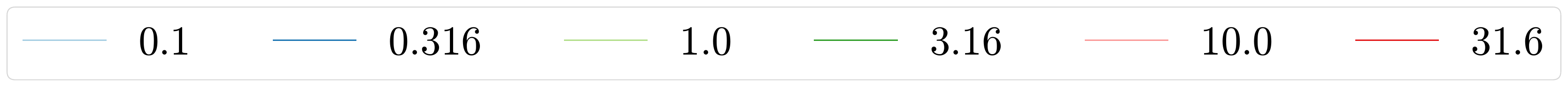}} &
       \multicolumn{2}{c}{\includegraphics[width=0.45\linewidth]{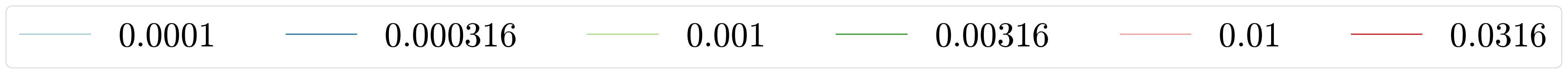}}\\
       \hspace{-3mm}\includegraphics[width=0.26\linewidth]{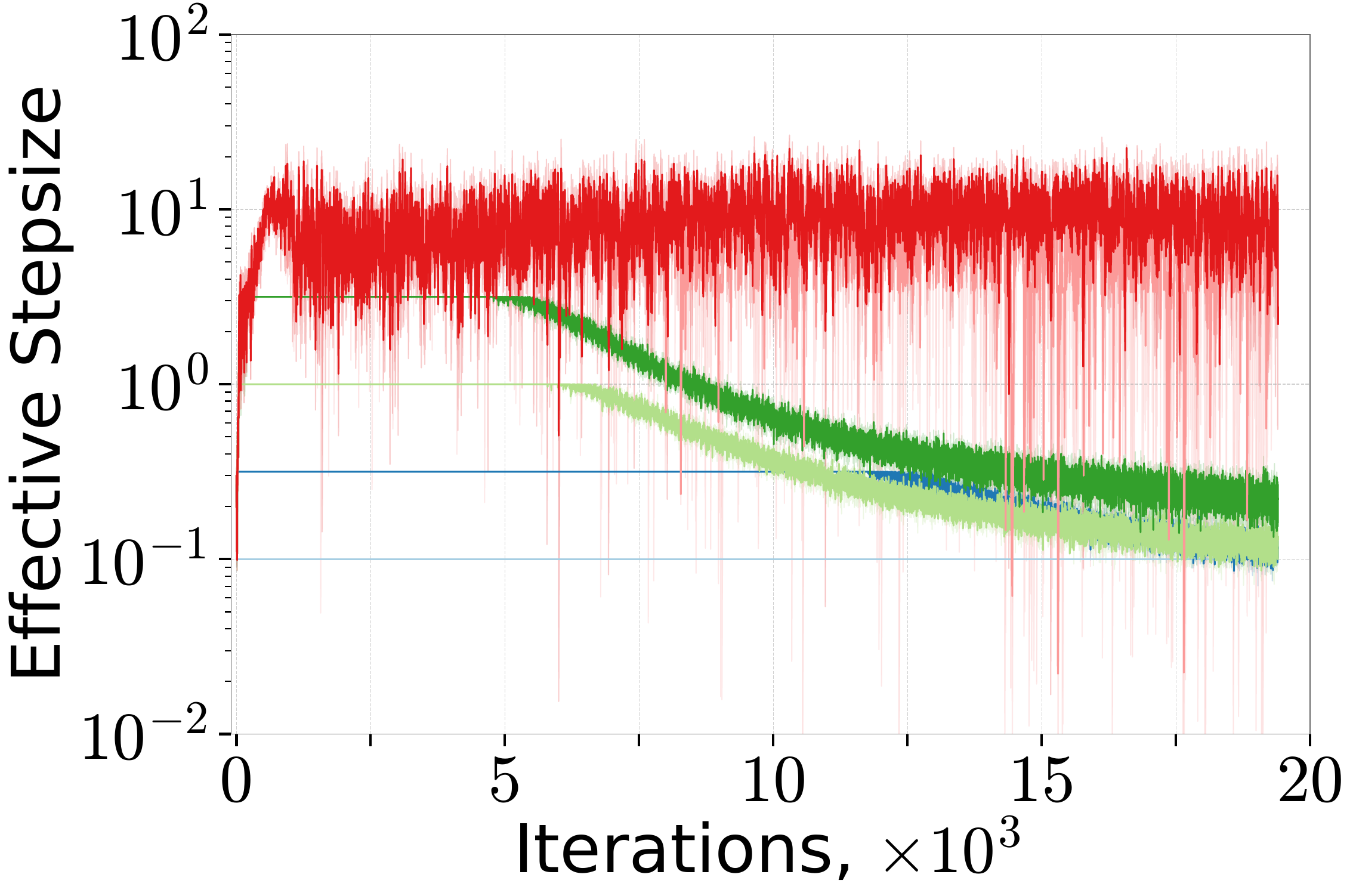}  &  
       \hspace{-5mm}\includegraphics[width=0.26\linewidth]{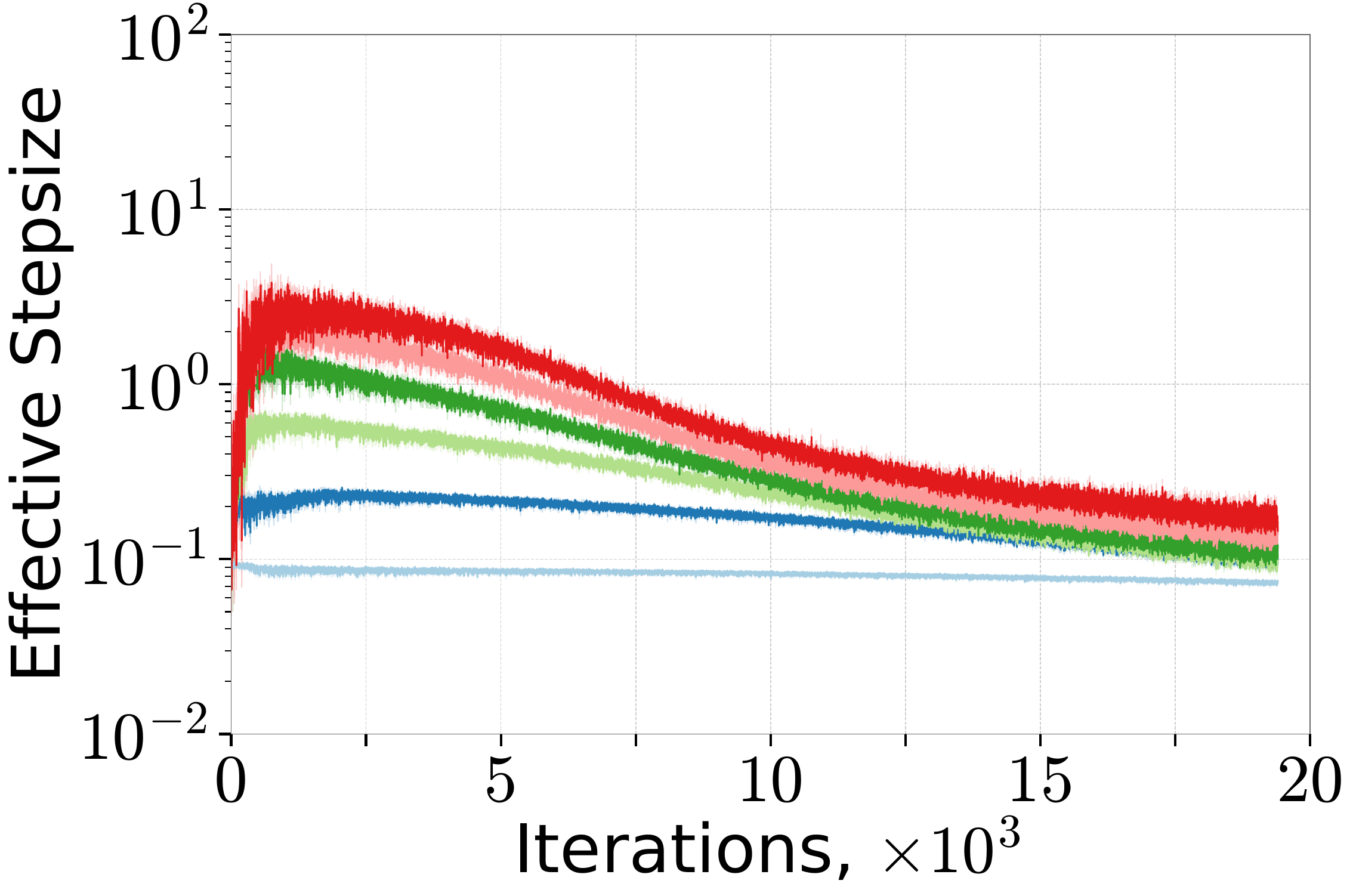} &
       \hspace{-5mm}\includegraphics[width=0.26\linewidth]{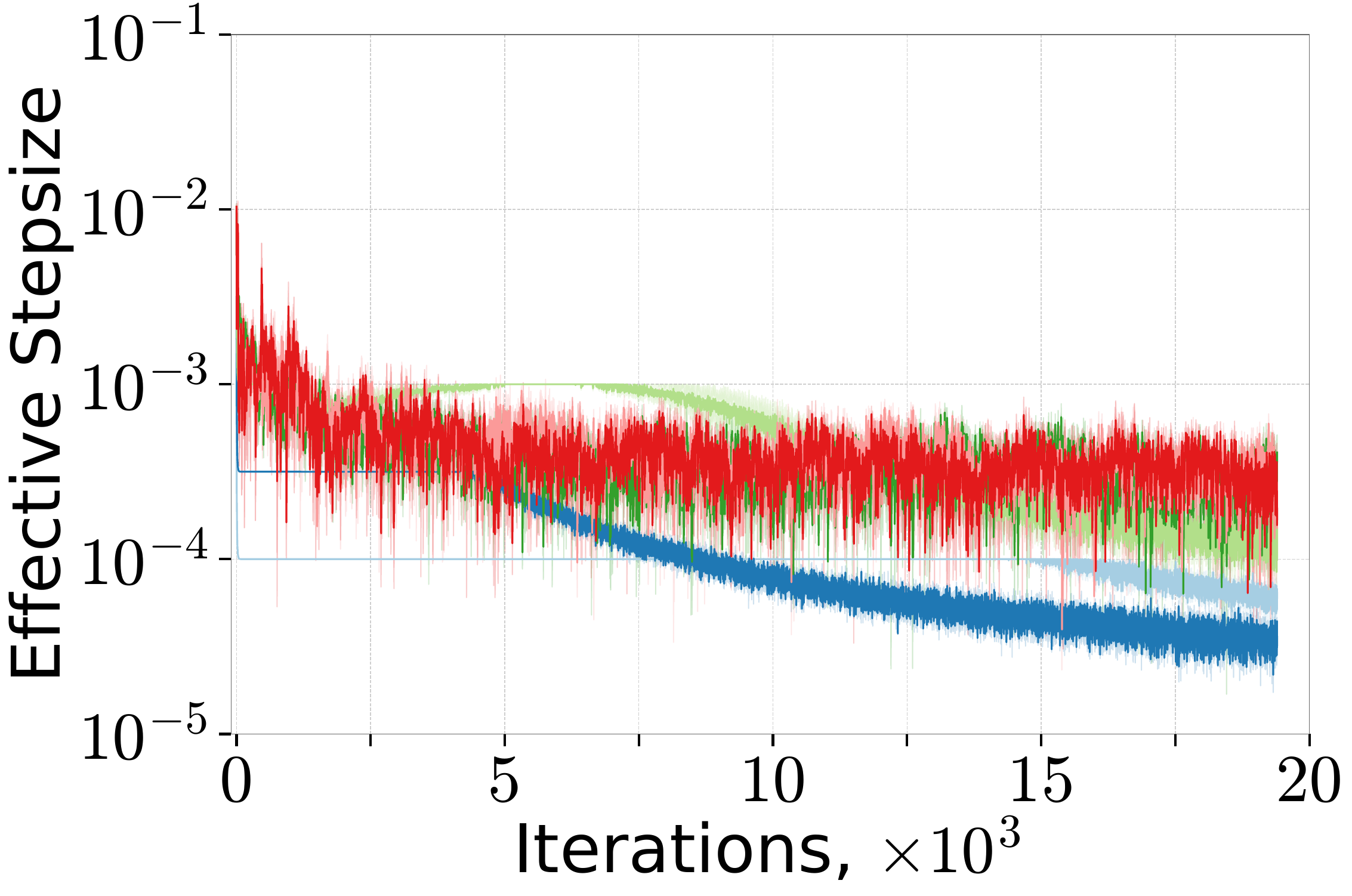}  &
       \hspace{-5mm}\includegraphics[width=0.26\linewidth]{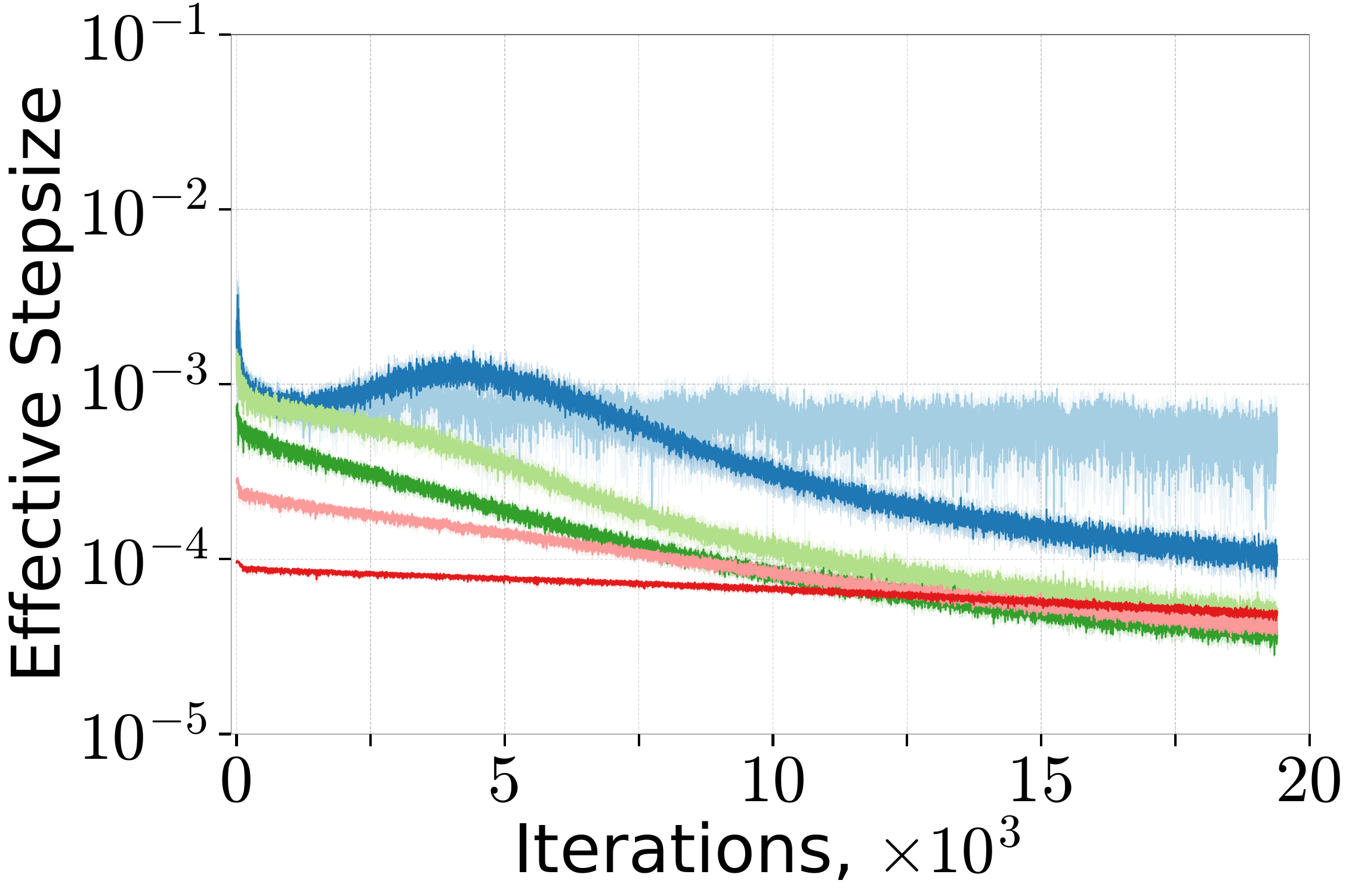} \\
       {\small \algname{Momo}} &
       {\small \algname{NGN-M}} & 
       {\small \algname{Momo-Adam}} &
       {\small \algname{NGN-MDv1}}
    \end{tabular}
    
    \caption{The step-size of \algname{Momo}, \algname{NGN-M} ({\bf two left}), \algname{Momo-Adam} and \algname{NGN-MDv1} ({\bf two right}) during the training of ViT on CIFAR10. We demonstrate the step-sizes $\tau_k$ for \algname{Momo} and \algname{Momo-Adam} and $\gamma_k$ for \algname{NGN-M} and \algname{NGN-MDv1} varying step-size parameters $\alpha_0$  and $c$ correspondingly. We refer to \Cref{fig:stepsize_momentum_appendix,fig:stepsize_adam_type_appendix} for the results in training Resnet20.}
    \label{fig:stepsize}
\end{figure*}

\paragraph{Vision Experiments on ImageNet.}

Having observed promising results on workloads of small and medium size, we switch to larger tasks and datasets. 
We first train a ResNet18 on ImageNet1k \citep{deng2009imagenet}. 
This represents the first task in which we pair our proposed algorithms with a learning rate schedule. As illustrated in Figure \ref{fig:stability_all_type_imagenet_test_acc} and \Cref{tab:empirical_comparison_momentum_appendix}, \algname{NGN-M} achieves the highest validation accuracy, while exhibiting higher robustness across larger step-sizes, improving over both \algname{NGN} and \algname{Momo}. Among adaptive methods, \algname{NGN-MDv1} compares favorably against \algname{Adam} and \algname{MomoAdam}, while once again achieving higher performance on a wider range of learning rates (\Cref{tab:empirical_comparison_adam_type_lion_adabound_adabelief}). \Cref{sec:imagenet_appendix} reports additional ablations on ImageNet32 and train loss stability results.

Finally, we test the effectiveness of the proposed algorithms on vision transformers \citep{dosovitskiy2021image}. These models are trained for a longer horizon compared to convolutional architectures, are notoriously sensitive to initial learning rate, and require adaptive step-sizes. We follow the protocol of \citet{schaipp2024momo}, which includes cosine annealing, but without any weight decay regularization. As highlighted in \Cref{fig:stability_all_type_imagenet_test_acc} and \Cref{tab:empirical_comparison_adam_type_lion_adabound_adabelief}, \algname{NGN-MDv1} achieves the highest validation accuracy across adaptive methods. Moreover, at a larger learning rate, \algname{Adam} diverges, whereas both \algname{MomoAdam} \algname{NGN-MDv1} maintain more stable training dynamics.

\paragraph{Language Modeling.}

Pre-training Large Language Models represents a challenging optimization task. 
To achieve competitive performance, optimizers with adaptive step-size are needed, and preventing instabilities in low-precision training often requires careful hyperparameter tuning. 

To evaluate the capability of \algname{NGN-MDv1} in this setting, we train decoder-only transformers \citep{radford2019gpt} with $70$M, $160$M, $410$M, and $1$B parameters around Chinchilla optimum \citep{hoffmann2022chinchilla} on SlimPajama-627B \citep{cerebras2023slimpajama}. 
For each model, we retune the learning rate, using $3$ seeds for the first three models and $1$ seed for the $1$B. Appendix \ref{sec:exp_appendix} provides additional details about the training and tokenization pipeline. 

\Cref{fig:llm_results} and Table \ref{tab:empirical_comparison_adam_type_lion_adabound_adabelief} report the final validation perplexity when training language models varying a model size. We note that \algname{NGN-MDv1} matches the performance of \algname{Adam} across all model sizes.  However, \algname{NGN-MDv1} achieves competitive performance even for a step-size hyperparameter $c=10^{-2}$ while \algname{Adam}'s performance drops significantly.  
This phenomenon is consistent across all scales we tested, suggesting that the optimal learning rate of \algname{NGN-MDv1} is shifted towards larger values, but also that the algorithm is less sensitive to such a hyperparameter. We additionally discuss how to introduce weight decay in \algname{NGN-MDv1} and report additional ablations on its role in this training task in \Cref{sec:weight_decay}.





\paragraph{Effective Step-size of \algname{NGN-M} and \algname{NGN-MDv1}.}

The first observation from the results in \Cref{fig:stepsize} is that the effective step-size of \algname{NGN-M} and \algname{NGN-MDv1} is always adaptive: if the step-size hyperparameter $c$ is large enough the effective step-size sharply increases in the beginning up to a peak, and then it gradually decreases till the end of the training. From this perspective, \algname{NGN-M} and \algname{NGN-MDv1} step-sizes are close to annealing step-size schedulers widely used in practice. In contrast, the effective step-size of \algname{Momo} and \algname{Momo-Adam} is not adaptive for sufficiently large step-size hyperparameter $\alpha_0$ during the initial part or all of the training. In other words, these algorithms reduce to \algname{SGDM} and \algname{Adam}, which is one of the reasons for the reduced resilience property of \algname{Momo} and \algname{Momo-Adam} in comparison with \algname{NGN-M} and \algname{NGN-MDv1}. The effective step-sizes in training Resnet20 are provided \Cref{fig:stepsize_momentum_appendix,fig:stepsize_adam_type_appendix} while comparison against \algname{Adam}'s effective step-size is reported in \Cref{fig:rebuttals_resnet20_stepsize,fig:rebuttals_vit_stepsize}. Moreover, we report the update magnitudes when training a $160$M language model in \Cref{fig:update_magnitude_lr0003,fig:update_magnitude_lr001,fig:update_magnitude_lr003}. All aforementioned results demonstrate that the \algname{NGN} step-size is more conservative: it decreases the effective step-size when necessary to stabilize the training, even for large values of the step-size hyperparameter $c.$ This feature is a key factor behind robustness of \algname{NGN-M} and \algname{NGN-MDv1} in practice.

\section{Conclusion and Future Work}

This work introduced several novel adaptations of the \algname{NGN} step-size method, incorporating support for momentum and/or diagonal step-size. We provided a theoretical analysis of the convergence rates for these algorithms and conducted an extensive empirical evaluation of their performance. The experimental results show that combining momentum with the \algname{NGN} step-size yields high robustness to step-size hyperparameter choices and performs competitively with state-of-the-art algorithms across various settings.

Given the significant complexity of the task, we defer the theoretical explanation of the step-size resilience properties of \algname{NGN-M} for large values of $\beta$ and analysis in the non-convex setting to future work. Furthermore, while the two proposed methods for incorporating weight decay into \algname{NGN-MDv1} outperform \algname{AdamW} in training language models, they still exhibit some sensitivity to the step-size hyperparameter. This may, in part, be due to the limited understanding of the expected effects of the weight decay technique, a topic that requires further investigation. We acknowledge that computing \algname{NGN} step-size at a large scale may cause runtime overhead, and discuss this limitation in \Cref{sec:computation_cost} by providing train and optimization times. We also recognize that integrating \algname{NGN-MDv1} with advanced parallelism schemes---such as Tensor Parallelism \citep{shoeybi2019megatron} or ZeRO-2 \citep{rajbhandari2020zero}---introduces additional compute and communication overhead, and will require further adaptation of the algorithm. Nevertheless, our results provide valuable guidance for developing inherently more stable optimizers. As a next step, it would be fascinating to investigate whether the resilience of emerging methods like \algname{Muon} \citep{jordan2024muon} can be further improved by incorporating the \algname{NGN} step-size.

\section*{Acknowledgement}

Rustem Islamov and Aurelien Lucchi acknowledge the financial support of the Swiss National Foundation, SNF grant No 207392. Antonio Orvieto acknowledges the financial support of the Hector Foundation.

\bibliography{references.bib}
\bibliographystyle{plainnat}

\newpage
\appendix

\counterwithin{figure}{section}

\vbox{
  {\hrule height 2pt \vskip 0.15in \vskip -\parskip}
  \centering
  {\LARGE\bf Appendix\par}
  {\vskip 0.2in \vskip -\parskip \hrule height 0.5pt \vskip 0.09in}
}

\newcommand\invisiblepart[1]{%
  \refstepcounter{part}%
  \addcontentsline{toc}{part}{\protect\numberline{\thepart}#1}%
}

\invisiblepart{Appendix}
\setcounter{tocdepth}{2}
\localtableofcontents

\section{Equivalent Formulations of \algname{NGN-M}}

We remind that the iterates of \algname{NGN-M} are the following
\begin{align*}
x^{k+1} &= 
x^k 
- (1-\beta)\gamma_k\nabla f_{S_k}(x^k) 
+ \beta(x^k-x^{k-1})\\
&= x^k 
- (1-\beta)\frac{c}{1+\frac{c}{2f_{S_k}(x^k)}\|\nabla f_{S_k}(x^k)\|^2}\nabla f_{S_k}(x^k) 
+ \beta(x^k-x^{k-1}).
\end{align*}
We can rewrite the update rule using Iterative-Moving Average (IMA) approach presented in Proposition 1.6, \citet{sebbouh2021almost}.

\begin{lemma}[Proposition C.8 \citep{oikonomou2024stochastic}, Lemma 7.3 in  \citep{garrigos2023handbook}]\label{lem:equivalence}
    The iterates $\{x^k\}$ generated by \algname{NGN-M} are equivalent to the sequence $\{x^k\}$ generated by IMA update 
    \begin{align}\label{eq:ima_update}
    z^{k+1} = z^k - \gamma_k \nabla f_{S_k}(x^k), 
    \quad x^{k+1} = \frac{\lambda}{1+\lambda}x^k + \frac{1}{1+\lambda}z^{k+1},
    \end{align}
    where 
    \begin{equation}\label{eq:lambda_and_beta}
        \beta = \frac{\lambda}{1+\lambda}, \quad \quad z^{k+1} = x^{k+1} + \lambda(x^{k+1}-x^k), \quad \text{ and } \quad x^{-1}=z^0 = x^0.
    \end{equation}
\end{lemma}
\begin{proof}
    Let the sequences $\{x^k\}$ and $\{z^k\}$ be defined according to \Cref{eq:ima_update}. Let $\beta$ be defined as $\frac{\lambda}{1+\lambda}$. Then we have
    \begin{align*}
        x^{k+1} &= \frac{\lambda}{1+\lambda} x^k 
        + \frac{1}{1+\lambda}z^{k+1}\\
        &= \frac{\lambda}{1+\lambda} x^k 
        +  \frac{1}{1+\lambda}(z^k - \gamma_k\nabla f_{S_k}(x^k))\\
        &= \frac{\lambda}{1+\lambda} x^k 
        +  \frac{1}{1+\lambda}((1+\lambda)x^k - \lambda x^{k-1} - \gamma_k\nabla f_{S_k}(x^k))\\
        &= x^k 
        - \frac{1}{1+\lambda}\gamma_k\nabla f_{S_k}(x^k) 
        + \frac{\lambda}{1+\lambda}(x^k-x^{k-1}).
    \end{align*}
    It remains to use \eqref{eq:lambda_and_beta} as  we have $\beta = \frac{\lambda}{1+\lambda}$ and $1-\beta = 1- \frac{\lambda}{1+\lambda} = \frac{1}{1+\lambda}.$ 

\end{proof}

\section{Technical Lemmas and Definitions}

\begin{definition}\label{def:coordinate_smoothness}
    We say that the function $\phi$ admits $\mL$-smooth with parameters $\mL\eqdef (L_1,\dots, L_d), L_j \ge 0~\forall j\in[d],$ if the following inequality holds for all $x, h \in \R^d$
    \begin{equation}\label{eq:coordinate_smoothness}
    \textstyle
        \phi\left(x+h\right) \le \phi(x) + \<\nabla \phi(x),h>+ \frac{1}{2}h^\top \mL h.
    \end{equation}
\end{definition}
\begin{remark}
    If we set for all $j\in[d]$ $L_j \eqdef L$  then \Cref{def:coordinate_smoothness} reduces to standard $L$-smoothness.
\end{remark}
This assumption is typically used in the context of coordinate adaptive algorithms such as \algname{SignSGD} \citep{bernstein2018signsgd, safaryan2021stochastic}.

\begin{definition}\label{def:pl_condition}
    The function $\phi\colon \R^d \to \R$ satisfies {\it P{\L}-condition} with constant $\mu >0$ if for all $x,y\in\R^d$ we have 
    \begin{equation}\label{eq:pl_condition}
        \textstyle \|\nabla f(x)\|^2  \ge 2\mu(f(x) - f^*).
    \end{equation}
\end{definition}

\begin{assumption}\label{asmp:bounded_variance} We assume that the coordinate-wise variance of the stochastic estimator is bounded, i.e. for all $x\in\R^d$ and $j\in[d]$ we have
\begin{eqnarray}\label{eq:bounded_variance}
    \textstyle
    \mathbb{E}_S\left[|(\nabla_j f_S(x) - \nabla_j f(x)|^2\right] \le \sigma_j^2.
\end{eqnarray}
\end{assumption}

\begin{lemma}[Lemma 4.9 from \citep{orvieto2024adaptive}]\label{lem:lemmaD22}
Let each $f_i$ be $L$-smooth for all $i$, then the step-size of \algname{NGN} satisfies
\begin{eqnarray}
    \gamma_k \in \left[\frac{c}{1+c L}, c\right].
\end{eqnarray}
    
\end{lemma}

\begin{lemma}[Lemma 4.2 from \citep{orvieto2024adaptive}]\label{lem:lemmaD23}
Let each $f_i$ be $L$-smooth for all $i$, then the iterates of \algname{NGN} satisfy
\begin{eqnarray}
    \gamma_k^2\|\nabla f_{S_k}(x^k)\|^2 \le \frac{4c L}{1+2c L}\gamma_k(f_{S_k}(x^k)-f_{S_k}^*)
    + \frac{2c^2L}{1+c L}\max\left\{\frac{2c L - 1}{2c L+1}, 0\right\}f_{S_k}^*.
\end{eqnarray}
\end{lemma}

\begin{lemma}[Gradient Upper Bound]\label{lem:smooth_inequality}
    Let $\phi\colon \R^d \to \R$ satisfy \Cref{def:coordinate_smoothness}. Then, for all $x\in \R^d$ and all $j\in[d]$ we have
    \begin{equation}\label{eq:smooth_inequality}
        2L_j(f(x)-f^*) \ge (\nabla_j f(x))^2.
    \end{equation}
\end{lemma}
\begin{proof}
    From \Cref{def:coordinate_smoothness} we have
    \begin{equation*}
        f^* = \min\limits_{x\in\R^d} f(x) 
        \le \min\limits_{h_j\in\R}f(x + h_je_j) 
        \le f(x) 
        + \min\limits_{h_j\in \R}\left[\nabla_jf(x)h_j 
        + \frac{L_j}{2}h_j^2\right].
    \end{equation*}
    Now we can explicitly compute the minimum in the right-hand side. The optimal value is achieved at
    \begin{equation*}
        h_j^* \eqdef -\frac{1}{L_j}\nabla_jf(x),
    \end{equation*}
    therefore,
    \begin{eqnarray*}
        f^* &\le& f(x) 
        + \nabla_j f(x)h_j^* 
        + \frac{L_j}{2}(h_j^*)^2\\
        &=& f(x) 
        - \frac{1}{L_j}(\nabla_jf(x))^2 
        + \frac{1}{2L_j}(\nabla_jf(x))^2\\
        &=& f(x) 
        - \frac{1}{2L_j}(\nabla_jf(x))^2,
    \end{eqnarray*}
    which equivalent to the statement of the lemma.
\end{proof}

\section{Convergence of \algname{NGN-D}}\label{sec:convergence_ngn_d}

First, we provide \algname{NGN-D} pseudocode and the main convergence results.

\begin{algorithm}[H]
    \caption{\algname{NGN-D}}
    \label{alg:ngn_d}
    \begin{algorithmic}[1]
        \State \textbf{Input:} $x^0\in\R^d,$ step-size parameter $c > 0$
        \For{$k = 0,1,\dots,K-1$}
        \State Sample a batch $S_k\subseteq[n]$ and compute $f_{S_k}$ and $\nabla f_{S_k}(x^k)$
        \State Compute $\gamma_k^{(j)} = \frac{c}{1+\frac{c}{2f_{S_k}(x^k)}(\nabla_j f_{S_k}(x^k))^2}$
        \State Update 
        $$x^{k+1}_{(j)} = x^k_{(j)} - \gamma_k^{(j)}\nabla_j f_{S_k}(x^k) $$
        \EndFor
    \end{algorithmic}	
\end{algorithm}

\begin{restatable}[]{theorem}{theoremngndnonconvex}\label{th:theorem_ngn_d_nonconvex}
    Let each $f_i$ satisfies \Cref{def:coordinate_smoothness}. Assume that \Cref{asmp:bounded_variance} holds. Then the iterates of \algname{NGN-D} (Alg. \ref{alg:ngn_d}) with step-size parameters $\{c_j\}_{j=1}^d$ such that $c_j \le \nicefrac{1}{2L_j}$ satisfy 
    \begin{eqnarray}\label{eq:theorem_ngn_d_nonconvex}
    \min\limits_{0\le k <K}\E{\|\nabla f(x^k)\|^2} \le \frac{12(f(x^0) - f^*)}{c_{\min} K} + \frac{1}{c_{\min}}\sum_{j=1}^d 18L_jc_j^2\sigma_j^2,
\end{eqnarray}
where $c_{\min} \eqdef \min_{j\in[d]}c_j.$ Moreover, if $c_j = \cO(\varepsilon^2)$ for all $j\in[d]$ then after $K = \cO(\varepsilon^{-4})$ we obtain  $\min\limits_{0\le k < K}\E{\|\nabla f(x^k)\|^2} \le \cO(\varepsilon^2)$.
\end{restatable}

\algname{NGN-D} converges with classic rate $\cO(\nicefrac{1}{\sqrt{K}})$ similar to \algname{Adagrad} \citep{ward2020adagrad}. We highlight that, to the best of our knowledge, \algname{NGN-D} is the first algorithm that uses diagonal Polyak-type stepsize and converges under standard smoothness and bounded variance assumptions without requirements of bounded gradients and interpolation.

\begin{restatable}[]{theorem}{theoremngndpl}\label{th:theorem_ngn_d_pl}
    Let $f$ satisfies P{\L}-condition and each $f_i$ satisfies \Cref{def:coordinate_smoothness}. Assume that \Cref{asmp:bounded_variance} holds. Then the iterates of \algname{NGN-D} (Alg. \ref{alg:ngn_d}) with step-size parameters $\{c_j\}_{j=1}^d$ such that $c_j \le \min\{\nicefrac{1}{2L_j}, \nicefrac{6}{\mu}\}$ satisfy 
    \begin{eqnarray}\label{eq:theorem_ngn_pl}
    \E{f(x^K) - f^*} \le (1-\nicefrac{\mu c_{\min}}{6})^K(f(x^0)-f^*) + \frac{9}{\mu c_{\min}}\sum_{j=1}^d L_jc_j^2\sigma_j^2,
\end{eqnarray}
where $c_{\min} \eqdef \min_{j\in[d]}c_j.$ Moreover, if $c_j = \cO(\varepsilon)$ for all $j\in[d]$ then after $K = \max\{\cO(\varepsilon^{-1}),\cO(1)\}\log\varepsilon^{-1}$ iterations we obtain $\E{f(x^K)- f^*} \le \cO(\varepsilon).$
\end{restatable}

To the best of our knowledge, this is the first result of the convergence of the Polyak-like step-size algorithm under the P{\L}-condition. The convergence guarantees are similar to that of \algname{SGD} \citep{garrigos2023handbook}.

Now we are ready to derive the step-size bounds. 

\begin{lemma}[Step-size Bounds]\label{lem:stepsize_bounds}
Let  $f_{S_k}(x) \colon \R^d \to \R$ be a stochastic loss of batch $S_k$ at iteration $k.$ Let $f_{S_k}(x)$ satisfy Definition~\eqref{def:coordinate_smoothness}. Consider $\gamma^k_j$ as in \algname{NGN-D} (\Cref{alg:ngn_d}), then we have
\begin{equation}\label{eq:stepsize_bounds}
    \gamma^k_j \in \left[\frac{c_j}{1+c_jL_j}, c_j\right].
\end{equation}
\end{lemma}
\begin{proof}
    From Lemma~\ref{lem:smooth_inequality} we have $2L_j(f_{S_k}(x^k) - f_{S_k}^*) \ge (\nabla_jf_{S_k}(x^k))^2.$ Since we assume that each $f_{S_k}^* \ge 0,$ then $2L_jf_{S_k}(x^k)  \ge (\nabla_jf_{S_k}(x^k))^2,$ or equivalently, 
    \[
    0 \le  \frac{(\nabla_jf_{S_k}(x))^2}{2f_{S_k}(x)} \le L_j.
    \]
    Therefore, for all $j\in [d]$ we have 
    \[
    \gamma_j^k = \frac{c_j}{1+\frac{c_j}{2f_{S_k}(x^k)}(\nabla_j f_{S_k}(x^k))^2} \le \frac{c_j}{1} = c_j,
    \]
    and
    \[
    \gamma_j^k = \frac{c_j}{1+\frac{c_j}{2f_{S_k}(x^k)}(\nabla_j f_{S_k}(x^k))^2} \ge \frac{c_j}{1+ c_jL_j},
    \]
    that concludes the proof.
\end{proof}

\begin{lemma}[Fundamental Equality]\label{lem:fundamental_equality}
Consider $\gamma^k_j$ as in \algname{NGN-D} (\Cref{alg:ngn_d}). Then the following equality holds
\begin{equation}\label{eq:fundamental_equality}
    \gamma^k_j(\nabla_jf_{S_k}(x^k))^2 = 2\left(\frac{c_j-\gamma^k_j}{c_j}\right)f_{S_k}(x^k).
\end{equation}
\end{lemma}
\begin{proof}
    From \algname{NGN-D} (\Cref{alg:ngn_d}) we have
    \begin{equation*}
        \left(1+\frac{c_j}{2f_{S_k}(x^k)}(\nabla_j f_{S_k}(x^k))^2\right)\gamma^k_j = c_j,
    \end{equation*}
    which one can rewrite as 
    \begin{equation*}
        \frac{c_j}{2f_{S_k}(x^k)}(\nabla_j f_{S_k}(x^k))^2\gamma^k_j = c_j - \gamma^k_j.
    \end{equation*}
    It is left to divide both sides by $\frac{2f_{S_k}(x^k)}{c_j}.$
\end{proof}

\subsection{Convergence in General Non-convex Setting}

\theoremngndnonconvex*
\begin{proof}
    First, we write separable \Cref{def:coordinate_smoothness}

\begin{eqnarray}\label{eq:proof_step1}
    f(x^{k+1}) - f(x^k) &=& f\left(x^k - \sum_{j=1}^d \gamma_j^k\nabla_jf_{S_k}(x^k)e_j\right) - f(x^k)\notag \\
    &\le& -\sum_{j=1}^d \nabla_j f(x^k)\cdot \gamma^k_j\nabla_j f_{S_k}(x^k) + \frac{1}{2}\sum_{j=1}^dL_j(\gamma^k_j\nabla_j f_{S_k}(x^k))^2\notag \\
    &\le& -\sum_{j=1}^d \nabla_j f(x^k)\cdot \gamma^k_j\nabla_j f_{S_k}(x^k) + \frac{1}{2}\sum_{j=1}^dL_j\sigma_j^2(\nabla_j f_{S_k}(x^k))^2.
\end{eqnarray}
Note that both $\gamma^k_j$ and $\nabla_jf_{S_k}(x^k)$ depend on the realization $S_k$, thus we can not directly apply conditional expectation with respect to $x^k,$ as in this case we would have to analyze the product $\gamma^k_j\nabla_jf_{S_k}(x^k)$. Given bounds of the step-size $\gamma^k_j$ from Lemma~\ref{lem:stepsize_bounds}, we can write the step-size as follows
\[
\gamma^k_j = \frac{c_j}{1+c_jL_j} + \nu^k_j \frac{c_j^2L_j}{1+c_jL_j},
\]
where $\nu^k_j \in [0,1]$ is a random variable. Varying the value of $\nu^k_j$ from $0$ to $1$ we cover the whole range of $\gamma^k_j.$ Thus, we continue as follows
\begin{eqnarray*}
    &&-\gamma^k_j\nabla_jf(x^k)\nabla_jf_{S_k}(x^k)\\
    &=& - \frac{c_j}{1+c_jL_j}\nabla_jf(x^k)\nabla_jf_{S_k}(x^k) 
    - \frac{c_j^2L_j}{1+c_jL_j}\nu^k_j\nabla_jf(x^k)\nabla_jf_{S_k}(x^k)\\
    &\le& - \frac{c_j}{1+c_jL_j}\nabla_jf(x^k)\nabla_jf_{S_k}(x^k) 
    + \frac{c_j^2L_j}{1+c_jL_j}|\nu^k_j|\cdot |\nabla_jf(x^k)\nabla_jf_{S_k}(x^k)|\\
    &\le& - \frac{c_j}{1+c_jL_j}\nabla_jf(x^k)\nabla_jf_{S_k}(x^k) 
    + \frac{c_j^2L_j}{1+c_jL_j}\cdot |\nabla_jf(x^k)\nabla_jf_{S_k}(x^k)|.
\end{eqnarray*}

Now we use the inequality $|ab| \le \frac{1}{2}a^2 + \frac{1}{2}b^2 + \frac{1}{2}|a-b|^2,$ and derive
\begin{align*}
    2\Ek{|\nabla_jf(x^k)\nabla_jf_{S_k}(x^k)|}
    &\le |\nabla_jf(x^k)|^2 
    + \Ek{|\nabla_jf_{S_k}(x^k)|^2}
    + \Ek{|\nabla_jf(x^k) - \nabla_jf_{S_k}(x^k)|^2}\\
    &\le 2|\nabla_jf(x^k)|^2 
    + 2\Ek{|\nabla_jf(x^k) - \nabla_jf_{S_k}(x^k)|^2}\\
    &\le 2|\nabla_jf(x^k)|^2 
    + 2\sigma_j^2.
\end{align*}
Therefore, we get 
\begin{align}\label{eq:proof_step2}
    -\Ek{\gamma^k_j\nabla_jf(x^k)\nabla_jf_{S_k}(x^k)}
    &\le - \frac{c_j}{1+c_jL_j}|\nabla_jf(x^k)|^2 
    + \frac{c_j^2L_j}{1+c_jL_j}\left(|\nabla_jf(x^k)|^2 + \sigma_j^2\right)\notag \\
    &= -c_j\left(\frac{1-c_jL_j}{1+c_jL_j}\right)|\nabla_jf(x^k)|^2 
    + \frac{c_j^2L_j}{1+c_jL_j}\sigma_j^2.
\end{align}
We plug in \eqref{eq:proof_step2} into \eqref{eq:proof_step1} and get
\begin{eqnarray*}
    \Ek{f(x^{k+1})} - f(x^k) &\le& 
    -\sum_{j=1}^d \left(\Ek{\gamma_j^k\nabla_jf(x^k)\nabla_jf_{S_k}(x^k)}
    + \frac{L_jc_j^2}{2}\Ek{|\nabla_jf_{S_k}(x^k)|^2}\right)\\
    &\le& \sum_{j=1}^d\left(\left[-c_j\left(\frac{1-c_jL_j}{1+c_jL_j}\right) 
    + \frac{L_jc_j^2}{2}\right]|\nabla_j f(x^k)|^2\right.\\
    &&\quad + \left.\left[\frac{c_j^2L_j}{1+c_jL_j} + \frac{L_jc_j^2}{2}\right]\sigma^2_j\right).
\end{eqnarray*}

If $c_j \le \frac{1}{2L_j},$ we get
\begin{eqnarray*}
    \Ek{f(x^{k+1})} - f(x^k)
    &\le& \sum_{j=1}^d\left(-\frac{c_j}{12}|\nabla_j f(x^k)|^2
    + \frac{3L_jc_j^2}{2}\sigma^2_j\right).
\end{eqnarray*}
\end{proof}

We continue as follows
\begin{eqnarray}\label{eq:nowojswqdqf}
    \Ek{f(x^{k+1})} - f(x^k)
    &\le& -\frac{c_{\min}}{12}\|\nabla f(x^k)\|^2
    +  \sum_{j=1}^d\frac{3L_jc_j^2}{2}\sigma^2_j.
\end{eqnarray}
Taking full expectation and unrolling the recursion above for all iterations $\{0,\dots,K-1\}$. Thus, we obtain
\begin{eqnarray*}
    \min\limits_{0\le k < K}\E{\|\nabla f(x^k)\|^2}\le \frac{1}{K}\sum_{k=0}^{K-1}\E{\|\nabla f(x^k)\|^2} \le 
    \frac{12}{c_{\min} K}(f(x^0) - f^*) 
    + \frac{18}{c_{\min}}\sum_{j=1}^d L_jc_j^2\sigma_j^2.
\end{eqnarray*}
If we choose each $c_j = \frac{c_{0,j}}{\sqrt{K}}$ such that $c_{0,j} \le \frac{1}{2L_j}$ we ensure that $c_j \le \frac{1}{2L_j}$ as well. Plugging this step-size into the bound we get
\begin{align*}
    \min\limits_{0\le k < K}\E{\|\nabla f(x^k)\|^2}
    &\le \frac{12}{\frac{c_{0,\min}}{\sqrt{K}} K}(f(x^0) - f^*) 
    + \frac{18}{\frac{c_{0,\min}}{\sqrt{K}}}\sum_{j=1}^d L_j\sigma_j^2\frac{c_{0,j}^2}{K}\\
    &\le \frac{12}{c_{0,\min}\sqrt{K} }(f(x^0) - f^*) 
    + \frac{18}{c_{0,\min}\sqrt{K}}\sum_{j=1}^d L_j\sigma_j^2c_{0,j}^2,
\end{align*}
where $c_{0,\min} \eqdef \min\limits_{j\in[d]} c_{0,j}.$ If we choose $K = \cO(\varepsilon^{-4})$ we get that 
\[
\min\limits_{0\le k < K}\E{\|\nabla f(x^k)\|^2} = \cO(1/\sqrt{K}) = \cO(\varepsilon^2).
\]

\subsection{Convergence under P{\L}-condition}

\theoremngndpl*
\begin{proof}
    We obtain \eqref{eq:nowojswqdqf} and use \Cref{def:pl_condition} 
    \begin{align*}\label{eq:nowojswqdqf}
    \Ek{f(x^{k+1})} - f(x^k)
    &\le -\frac{c_{\min}}{12}\|\nabla f(x^k)\|^2
    +  \sum_{j=1}^d\frac{3L_jc_j^2}{2}\sigma^2_j\\
    &\le -\frac{\mu c_{\min}}{6}(f(x^k) - f^*)
    +  \sum_{j=1}^d\frac{3L_jc_j^2}{2}\sigma^2_j
\end{align*}
Subtracting $f^*$ from both sides of the inequality above and taking full expectation  we obtain
\begin{align*}
    \E{f(x^{k+1}) - f^*} &\le (1-\nicefrac{\mu c_{\min}}{6})\E{f(x^k)-f^*} 
    + \sum_{j=1}^d\frac{3L_jc_j^2}{2}\sigma^2_j.
\end{align*}
Unrolling the recursion above for $\{0,\dots,K-1\}$ iterations we derive
\begin{align*}
    \E{f(x^{K}) - f^*} &\le (1-\nicefrac{\mu c_{\min}}{6})^K(f(x^0)-f^*)
    + \frac{1}{c_{\min}}\sum_{j=1}^d\underbrace{\frac{9L_j\sigma_j^2}{\mu}}_{A_j}c^2_j.
\end{align*}

Now we follow the proof of Lemma A.3 in \citet{garrigos2023handbook}. Let us choose $c_j = \min\{\nicefrac{1}{2L_j}, \nicefrac{\varepsilon}{2dA_j}\}$. Together with the choice of $K \ge \max\limits_{j\in[d]} \max\left\{\frac{1}{\varepsilon}\frac{12A_j}{\mu}, \frac{12L_j}{\mu}\right\}\log\frac{2(f(x^0)-f^*)}{\varepsilon}$ we get 
\[
    (1-\nicefrac{\mu c_{\min}}{6})^K(f(x^0) -f^*) \le \frac{\varepsilon}{2}.
\]
Now we have two cases:
\begin{enumerate}
    \item $c_{\min}$ does not depend on $\varepsilon,$ then we have
    \[
    \frac{1}{c_{\min}}A_jc_j^2 \le \cO(\varepsilon^2).
    \]
    \item $c_{\min}$ does depend on $\varepsilon,$ i.e. $c_{\min}=\cO(\varepsilon),$ then we have 
    \[
    \frac{1}{c_{\min}}A_jc_j^2 \le \cO(\varepsilon).
    \]
\end{enumerate}
Therefore, combining all together we get
\[
\E{f(x^K) - f^*} \le \cO(\varepsilon)
\]
after $K \ge \max\limits_{j\in[d]} \max\left\{\frac{1}{\varepsilon}\frac{12A_j}{\mu}, \frac{12L_j}{\mu}\right\}\log\frac{2(f(x^0)-f^*)}{\varepsilon}$ iterations.

\end{proof}

\section{Convergence of \algname{NGN-M}}\label{sec:proof_ngnm}

\theoremngnmconvex*

\begin{remark}
    In fact, if $\lambda \le \frac{1}{(1+cL)(1+2cL)},$ then it implies that $\lambda \le \frac{1}{cL}$ because $\frac{1}{x} > \frac{1}{(1+x)(1+2x)}$ for any $x > 0.$
\end{remark}
\begin{proof}
    To prove the convergence of \algname{NGN-M} we consider IMA formulation \Cref{eq:ima_update}:
    \begin{align*}
        x^{-1} = z^0 = x^0, \quad z^{k+1} = z^k - \gamma_k\nabla f_{S_k}(x^k), \quad x^{k+1} = \frac{\lambda}{1+\lambda}x^k + \frac{1}{1+\lambda}z^{k+1},
    \end{align*}
    where $\beta = \frac{\lambda}{1+\lambda}, z^{k+1} = x^{k+1} + \lambda (x^{k+1} - x^k).$
    
    At iteration $k=0$ we have
    \[
    z^1 = z^0 - \gamma_0 \nabla f_{S_0}(x^0)  = x^0 - \gamma_0\nabla f_{S_0}(x^0).
    \]
    Therefore, we get
    \begin{eqnarray}\label{eq:nbijweniqfqd}
        \|z^1 - x^*\|^2 &=& \|z^0 - x^*\|^2 
        - 2\gamma_0\<\nabla f_{S_0}(x^0), z^0-x^*> 
        + \gamma_0^2\|\nabla f_{S_0}(x^0)\|^2\notag \\
        &\overset{{ \text{Lem. }} \ref{lem:lemmaD23}}{\le}& \|z^0 - x^*\|^2 
        - 2\gamma_0\<\nabla f_{S_0}(x^0), x^0-x^*> 
        + \frac{4cL}{1+2cL}\gamma_0(f_{S_0}(x^0) - f_{S_0}^*)\notag \\
        && \quad + \frac{2c^2L}{1+cL}\max\left\{\frac{2cL-1}{2cL+1},0\right\}f_{S_0}^*.
    \end{eqnarray}
    Let $\gamma_0 = \rho + \wtilde{\gamma}_0$ where $\rho = \frac{c}{(1+cL)(1+2cL)}$. Then we have 
    \begin{align*}
        \wtilde{\gamma}_0 &= \gamma_0 - \rho\\
        &\overset{{ \text{Lem. }} \ref{lem:lemmaD22}}{\le} c - \frac{c}{(1+cL)(1+2cL)}\\
        &= c\frac{1+3cL+2c^2L^2 - 1}{(1+cL)(1+2cL)}\\
        &= c^2L \frac{3+3cL}{(1+cL)(1+2cL)}\\
        &= \frac{3c^2L}{1+2cL}.
    \end{align*}
    Using the above we continue from (\ref{eq:nbijweniqfqd})
    \begin{align}
        \|z^1 - x^*\|^2 &\overset{\text{conv.}}{\le} \|z^0 - x^*\|^2 
        - 2\gamma_0(f_{S_0}(x^0) - f_{S_0}(x^*))
        + \frac{4cL}{1+2cL}\gamma_0(f_{S_0}(x^0) - f_{S_0}^*)
        \notag \\
        & \quad 
        + \frac{2c^2L}{1+cL}\max\left\{\frac{2cL-1}{2cL+1},0\right\}f_{S_0}^*\notag\\
        &\le \|z^0 - x^*\|^2 
        - 2\rho(f_{S_0}(x^0) - f_{S_0}(x^*)) - 2\wtilde{\gamma}_0(f_{S_0}(x^0) - f_{S_0}^*) 
        + 2\wtilde{\gamma}_0(f_{S_0}(x^*) - f_{S_0}^*)
        \notag \\
        & \quad 
        + \frac{4cL}{1+2cL}\gamma_0(f_{S_0}(x^0) - f_{S_0}^*)
        + \frac{2c^2L}{1+cL}\max\left\{\frac{2cL-1}{2cL+1},0\right\}f_{S_0}^*\notag\\
        &= \|z^0 - x^*\|^2 
        - 2\rho(f_{S_0}(x^0) - f_{S_0}(x^*)) - 2\left(\gamma_0 - \rho- \frac{2cL}{1+2cL}\gamma_0\right)(f_{S_0}(x^0) - f_{S_0}^*) 
        \notag \\
        & \quad 
        + 2\wtilde{\gamma}_0(f_{S_0}(x^*) - f_{S_0}^*)
        + \frac{2c^2L}{1+cL}\max\left\{\frac{2cL-1}{2cL+1},0\right\}f_{S_0}^*.
    \end{align}
    Here we have 
    \begin{align*}
        \gamma_0 - \rho  - \frac{2cL}{1+2cL}\gamma_0 &= \frac{1}{1+2cL}\gamma_0 - \rho\\
        &= \frac{1}{1+2cL}\gamma_0 - \frac{c}{(1+cL)(1+2cL)}\\
        &\overset{{\rm Lem.} \ref{lem:lemmaD22}}{\ge} \frac{1}{1+2cL}\frac{c}{1+cL} - \frac{c}{(1+cL)(1+2cL)}\\
        &= 0,
    \end{align*}
    $\wtilde{\gamma}_0 \le \frac{3c^2L}{1+2cL}$, and $f_{S_0}(x^0) - f_{S_0}^* \ge 0.$ Hence, we get
    \begin{align*}
        \|z^1 - x^*\|^2 &\le \|z^0 - x^*\|^2 
        - 2\rho(f_{S_0}(x^0) - f_{S_0}(x^*)) 
        + \frac{6c^2L}{1+2cL}(f_{S_0}(x^*) - f_{S_0}^*)
        \notag \\
        &\quad + \frac{2c^2L}{1+cL}\max\left\{\frac{2cL-1}{2cL+1},0\right\}f_{S_0}^*.
    \end{align*}
    Rearranging terms and taking expectation we get
    \begin{align}\label{eq:sfqionqf}
    2\rho\E{f(x^0) - f(x^*)} &\le \E{\|z^1-x^*\|^2} - \|z^0-x^*\|^2 
    + \frac{6c^2L}{1+2cL}\sigma^2_{\rm int}
    \notag \\
    &\quad + \frac{2c^2L}{1+cL}\max\left\{\frac{2cL-1}{2cL+1},0\right\}\sigma^2_{\rm pos}.
    \end{align}
    
    Next, for $k > 0$ we can use the relation $z^{k} = x^k + \lambda (x^k-x^{k-1}).$ We expand $\|z^{k+1} - x^*\|^2$
    \begin{eqnarray*}
        \|z^{k+1} - x^*\|^2 &=& \|z^k - x^*\|^2 
        - 2\gamma_k\<\nabla f_{S_k}(x^k), z^k-x^*>
        + \gamma_k^2\|\nabla f_{S_k}(x^k)\|^2\\
        &\overset{{ \text{Lem. }} \ref{lem:equivalence}}{=}&
        \|z^k-x^*\|^2 
        - 2\gamma_k\<\nabla f_{S_k}(x^k),x^k-x^*>
        - 2\gamma_k\lambda\<\nabla f_{S_k}(x^k),x^k-x^{k-1}>\\
        && \;+\; \gamma_k^2\|\nabla f_{S_k}(x^k)\|^2\\
        &\overset{{ \text{conv.}} }{\le}& \|z^k-x^*\|^2 
        - 2\gamma_k(f_{S_k}(x^k) - f_{S_k}(x^*))
        - 2\gamma_k\lambda(f_{S_k}(x^k) -f_{S_k}(x^{k-1}))\\
        && \;+\; \gamma_k^2\|\nabla f_{S_k}(x^k)\|^2\\
        &\overset{{ \text{Lem. }} \ref{lem:lemmaD23}}{\le}& \|z^k-x^*\|^2 
        - 2\gamma_k(f_{S_k}(x^k) - f_{S_k}(x^*))
        - 2\gamma_k\lambda(f_{S_k}(x^k) -f_{S_k}(x^{k-1}))\\
        && \;+\; \frac{4c L}{1+2c L}\gamma_k(f_{S_k}(x^k)-f_{S_k}^*)
        + \frac{2c^2L}{1+c L}\max\left\{\frac{2c L - 1}{2c L+1}, 0\right\}f_{S_k}^*.
    \end{eqnarray*}
    Let $\gamma_k = \rho+\wtilde{\gamma}_k$, where $\rho,\wtilde{\gamma}_k \ge 0,$ and $\rho$ is a constant step-size independent of $S_k$ which will be defined later. Therefore, we have
     \begin{eqnarray}
        \|z^{k+1} - x^*\|^2 
        &\le& \|z^k-x^*\|^2 
        - 2\rho(f_{S_k}(x^k) - f_{S_k}(x^*))
        - 2\wtilde{\gamma}_k(f_{S_k}(x^k) - f_{S_k}(x^*))\notag\\
        && \;-\; 2\gamma_k\lambda_k(f_{S_k}(x^k) - f_{S_k}^*) 
        + 2\gamma_k\lambda(f_{S_k}(x^{k-1}) - f_{S_k}^*)\notag\\
        && \;+\; 
        \frac{4c L}{1+2c L}\gamma_k(f_{S_k}(x^k)-f_{S_k}^*)
        + \frac{2c^2L}{1+c L}\max\left\{\frac{2c L - 1}{2c L+1}, 0\right\}f_{S_k}^*\notag\\
        &=&  \|z^k-x^*\|^2 
        - 2\rho(f_{S_k}(x^k) - f_{S_k}(x^*))
        - 2\wtilde{\gamma}_k(f_{S_k}(x^k) - f_{S_k}^*)
        + 2\wtilde{\gamma}_k(f_{S_k}(x^*) - f_{S_k}^*)\notag\\
        && \;-\; 2\gamma_k\lambda(f_{S_k}(x^k) - f_{S_k}^*) 
        + 2\gamma_k\lambda(f_{S_k}(x^{k-1}) - f_{S_k}^*)\notag\\
        && \;+\; \frac{4c L}{1+2c L}\gamma_k(f_{S_k}(x^k)-f_{S_k}^*)
        + \frac{2c^2L}{1+c L}\max\left\{\frac{2c L - 1}{2c L+1}, 0\right\}f_{S_k}^*\notag\\
        &=& \|z^k-x^*\|^2 
        - 2\rho(f_{S_k}(x^k) - f_{S_k}(x^*))
        - 2\left(\wtilde{\gamma}_k + \gamma_k\lambda - \frac{2cL}{1+2cL}\gamma_k\right)(f_{S_k}(x^k) - f_{S_k}^*)\notag\\
        && \;+\; 2\wtilde{\gamma}_k(f_{S_k}(x^*) - f_{S_k}^*)  
        + 2\gamma_k\lambda(f_{S_k}(x^{k-1}) - f_{S_k}^*)\notag\\
        && \;+\; 
        \frac{2c^2L}{1+c L}\max\left\{\frac{2c L - 1}{2c L+1}, 0\right\}f_{S_k}^*.\label{eq:sdfjnooqf}
    \end{eqnarray}
    We need to find $\rho$ such that 
    \[
    \wtilde{\gamma}_k + \gamma_k\lambda - \frac{2cL}{1+2cL}\gamma_k \ge 0
    \]
    Since $\wtilde{\gamma}_k = \gamma_k - \rho$, then we have 
    \begin{align*}
        &\; \gamma_k - \rho + \gamma_k \lambda - \frac{2cL}{1+2cL}\gamma_k \ge 0\\
        \Leftrightarrow& \; \gamma_k\left(1+\lambda - \frac{2cL}{1+2cL}\right) \ge \rho.
    \end{align*}
    The inequality above is satisfied if it is satisfied for the lower bound on $\gamma_k$ (which is $\nicefrac{c}{1+cL}$), i.e.
    \begin{align*}
        \frac{c}{1+cL}\left(\frac{1}{1+2cL}+\lambda\right) \ge \rho.
    \end{align*}
    We can take $\rho = \frac{c}{(1+cL)(1+2cL)}$ since $\lambda \ge 0$.
    \begin{align*}
        \wtilde{\gamma}_k &= \gamma_k - \rho\\
        &\le c - \frac{c}{(1+cL)(1+2cL)}\\
        &= c\frac{1+3cL+2c^2L^2 - 1}{(1+cL)(1+2cL)}\\
        &\le c^2L\frac{3+3cL}{(1+cL)(1+2cL)}\\
        &= \frac{3c^2L}{1+2cL}.
    \end{align*}
    Using the above, we get from (\ref{eq:sdfjnooqf})
    \begin{align*}
        \|z^{k+1} - x^*\|^2 
        &\le \|z^k-x^*\|^2 
        - 2\rho(f_{S_k}(x^k) - f_{S_k}(x^*))
        +   2c\lambda(f_{S_k}(x^{k-1}) - f_{S_k}(x^*))\notag\\
        & \;+\;  
        2c\lambda(f_{S_k}(x^{*}) - f_{S_k}^*)
        + \frac{6c^2L}{1+2cL}(f_{S_k}(x^*) - f_{S_k}^*)\\
        & \;+\; \frac{2c^2L}{1+c L}\max\left\{\frac{2c L - 1}{2c L+1}, 0\right\}f_{S_k}^*.\notag
    \end{align*}
    Taking expectations we get 
    \begin{eqnarray}
        \E{\|z^{k+1} - x^*\|^2} 
        &\le & \E{\|z^k-x^*\|^2 }
        - 2\rho\E{f(x^k) - f(x^*)}
        +   2c\lambda\E{f(x^{k-1}) - f(x^*)}\notag\\
        && \;+\;  
       \left(2c\lambda+\frac{6c^2L}{1+2cL}\right)\sigma^2_{\rm int}
        +
        \frac{2c^2L}{1+c L}\max\left\{\frac{2c L - 1}{2c L+1}, 0\right\}\sigma^2_{\rm pos}.
    \end{eqnarray}
    Rearranging terms we get 
    \begin{align}\label{eq:lnjbaifbas}
        2\rho\E{f(x^k) - f(x^*)} - 2c\lambda\E{f(x^{k-1}) - f(x^*)} &\le \E{\|z^k-x^*\|^2 } 
        - \E{\|z^{k+1} - x^*\|^2}\notag\\
        & \;+\;  
       \left(2c\lambda+\frac{6c^2L}{1+2cL}\right)\sigma^2_{\rm int}\notag\\
        & \;+\; \frac{2c^2L}{1+c L}\max\left\{\frac{2c L - 1}{2c L+1}, 0\right\}\sigma^2_{\rm pos}.
    \end{align}

    Combining \Cref{eq:sfqionqf} and \Cref{eq:lnjbaifbas} for iterations $\{1,\dots, K-1\}$ we get 
    \begin{align}
        &2\rho\E{f(x^0) - f(x^*)}
        + 2\rho\sum_{k=1}^{K-1}\E{f(x^k) - f(x^*)} 
        - 2c\lambda\sum_{k=1}^{K-1}\E{f(x^{k-1}) - f(x^*)}\notag \\
        &= 2\rho\sum_{k=0}^{K-1}\E{f(x^k) - f(x^*)} 
        - 2c\lambda\sum_{k=0}^{K-2}\E{f(x^{k}) - f(x^*)}\notag \\
        &\le (2\rho -2c\lambda)\sum_{k=0}^{K-1}\E{f(x^k) - f(x^*)}\notag \\
        &\le \|z^0-x^*\|^2
         + \frac{6c^2L}{1+2cL}\sigma^2_{\rm int}
        + \frac{2c^2L}{1+cL}\max\left\{\frac{2cL-1}{2cL+1},0\right\}\sigma^2_{\rm pos}\notag \\
        &\;+\;  \left(2c\lambda+\frac{6c^2L}{1+2cL}\right)(K-1)\sigma^2_{\rm int}
        + (K-1)\cdot \frac{2c^2L}{1+c L}\max\left\{\frac{2c L - 1}{2c L+1}, 0\right\}\sigma^2_{\rm pos}\notag \\
        &\le\|z^0-x^*\|^2 
        + \left(2c\lambda+\frac{6c^2L}{1+2cL}\right)K\sigma^2_{\rm int}
        + K\cdot \frac{2c^2L}{1+c L}\max\left\{\frac{2c L - 1}{2c L+1}, 0\right\}\sigma^2_{\rm pos}.\label{eq:ynjknfjonas}
    \end{align}
    
    We need to ensure that $\rho - c\lambda > 0$ which is satisfied for $\lambda$ such that 
    \begin{align*}
        & \frac{\rho}{2} = \frac{c}{2(1+cL)(1+2cL)} > c\lambda\\
        \Leftrightarrow & 1 > 2\lambda (1+cL)(1+2cL).
    \end{align*}
    Note that we also assume that $\lambda \le cL$.
    Therefore, from (\ref{eq:ynjknfjonas}) we get
    \begin{align}
        \frac{1}{K}\sum_{k=0}^{K-1}\E{f(x^k) - f(x^*)}
        &\le \frac{\|z^0-x^*\|^2}{2(\rho-c\lambda)K}
        + \frac{1}{2(\rho-c\lambda)}\left(2c\lambda+\frac{6c^2L}{1+2cL}\right)\sigma^2_{\rm int} \notag\\
        & \;+\;  
        \frac{1}{2(\rho-c\lambda)}\frac{2c^2L}{1+c L}\max\left\{\frac{2c L - 1}{2c L+1}, 0\right\}\sigma^2_{\rm pos}\notag\\
        &\le \frac{\|z^0-x^*\|^2}{2(\rho-c\lambda)K}
        + \frac{8c^2L}{2(\rho-c\lambda)}\sigma^2_{\rm int} \notag\\
        & \;+\;  
        \frac{1}{2(\rho-c\lambda)}\frac{2c^2L}{1+c L}\max\left\{\frac{2c L - 1}{2c L+1}, 0\right\}\sigma^2_{\rm pos}.
    \end{align}
    Since $\rho - c\lambda \ge \frac{\rho}{2}$ and setting $\overline{x}^k$ be uniformly at random chosen from $\{x^0,\dots,x^{K-1}\}$ we get
    \begin{align}
        \E{f(\overline{x}^k) - f(x^*)}
        &\le \frac{\|z^0-x^*\|^2}{\rho K}
        + \frac{8c^2L}{\rho}\sigma^2_{\rm int}
        +
        \frac{1}{\rho}\frac{2c^2L}{1+c L}\max\left\{\frac{2c L - 1}{2c L+1}, 0\right\}\sigma^2_{\rm pos},
    \end{align}
    where we use the convexity of $f$ and Jensen's inequality. Plugging the value of $\rho = \frac{c}{(1+cL)(1+2cL)}$  inside we get
    \begin{align}
        \E{f(\overline{x}^k) - f(x^*)}
        &\le \frac{\|z^0-x^*\|^2}{cK}(1+cL)(1+2cL)
        + 8cL(1+cL)(1+2cL)\sigma^2_{\rm int}\notag\\
        & \:+\:
        2cL\max\left\{2c L - 1, 0\right\}\sigma^2_{\rm pos}.
    \end{align}
    Choosing $c = \cO(1/\sqrt{K})$ we get
    \begin{align}
        \E{f(\overline{x}^k) - f(x^*)}
        &\le \cO\left(\frac{\|z^0-x^*\|^2}{\sqrt{K}}
        + \frac{\sigma^2_{\rm int}}{\sqrt{K}}
        + \frac{\sigma^2_{\rm pos}}{\sqrt{K}}\max\left\{2c L - 1, 0\right\}\right).
    \end{align}
    Therefore, if $K \ge  \cO(\varepsilon^{-2})$ then $\E{f(\overline{x}^k) - f(x^*)} \le \cO(\varepsilon).$ It remains to notice that $z^0 = x^0$ to derive the statement of the theorem.
\end{proof}

\section{Convergence of \algname{NGN-M} with Decaying Step-size}\label{sec:decaying_stepsize}

\begin{lemma}\label{lem:lemma_summation}
    We have 
    \begin{equation}
        \sum_{k=0}^{K-1}\frac{1}{k+1} \le \log(K+2), \quad \sum_{k=0}^{K-1}\frac{1}{\sqrt{k+1}} \ge \frac{4}{5}\sqrt{K+1}.
    \end{equation}
\end{lemma}
\begin{proof}
    We refer to Lemma A.8 from \citet{garrigos2023handbook}.
\end{proof}

To prove the convergence of \algname{NGN-M} with decaying $c_k$ we consider IMA formulation (see Section A in the paper):
    \begin{align*}
        x^{-1} = z^0 = x^0, \quad z^{k+1} = z^k - \gamma_k\nabla f_{S_k}(x^k), \quad \gamma_k = \frac{c_k}{1+\frac{c_k}{2f_{S_k}(x^k)}\|\nabla f_{S_k}(x^k)\|^2}\\
        x^{k+1} = \frac{\lambda}{1+\lambda}x^k + \frac{1}{1+\lambda}z^{k+1},
    \end{align*}
    where $c_k = \frac{c_0}{\sqrt{k+1}}, \lambda_k = Lc_k, \lambda_0=0.$

\begin{theorem}\label{th:ngnm_decaying}
    Assume that each $f_i$ is convex and $L$-smooth, and that Assumption 3.2 holds. Let the step-size hyperparameter is set $c_k = \frac{c_0}{\sqrt{k}},$ momentum parameter $\lambda_k \le \min\{c_kL, 0.5(1+c_kL)^{-1}(1+2c_kL)^{-1}\}$. Then the iterates of \algname{NGN-M} satisfy
    \begin{align}
        \E{f(\hat{x}^{K-1}) - f(x^*)}
        &\le \frac{5(1+c_0L)(1+2c_0L)\|x^0-x^*\|^2}{4c_0\sqrt{K}}
        + 10Lc_0(1+c_0L)(1+2c_0L)\sigma^2_{\rm int}\frac{\log{(K+2)}}{\sqrt{K}} \notag\\
        & \;+\;  
        5c_0L(1+c_0L)\frac{\log(K+2)}{2\sqrt{K}}\max\left\{2c_0L-1,0\right\}\sigma_{\rm pos}^2,
    \end{align}
    where $\hat{x}^{K-1} = \sum_{k=0}^{K-1}\frac{\rho_k}{\sum_{k=0}^{K-1}\rho_k}x^k, \rho_k = \frac{c_k}{(1+c_kL)(1+2c_kL)}.$
\end{theorem}

\begin{proof}

    At iteration $k=0$ we have
    \[
    z^1 = z^0 - \gamma_0 \nabla f_{S_0}(x^0)  = x^0 - \gamma_0\nabla f_{S_0}(x^0).
    \]
    Therefore, we get
    \begin{eqnarray}\label{eq:nbijweniqfqdddd}
        \|z^1 - x^*\|^2 &=& \|z^0 - x^*\|^2 
        - 2\gamma_0\<\nabla f_{S_0}(x^0), z^0-x^*> 
        + \gamma_0^2\|\nabla f_{S_0}(x^0)\|^2\notag \\
        &\overset{{ \text{Lem. B.6}}}{\le}& \|z^0 - x^*\|^2 
        - 2\gamma_0\<\nabla f_{S_0}(x^0), x^0-x^*> 
        + \frac{4c_0L}{1+2c_0L}\gamma_0(f_{S_0}(x^0) - f_{S_0}^*)\notag \\
        && \quad + \frac{2c^2_0L}{1+c_0L}\max\left\{\frac{2c_0L-1}{2c_0L+1},0\right\}f_{S_0}^*.
    \end{eqnarray}
    Let $\gamma_0 = \rho_0 + \wtilde{\gamma}_0$ where $\rho_0 = \frac{c_0}{(1+c_0L)(1+2c_0L)}$. Then we have 
    \begin{align*}
        \wtilde{\gamma}_0 &= \gamma_0 - \rho_0\\
        &\overset{{ \text{Lem. B.5}}}{\le} c_0 - \frac{c_0}{(1+c_0L)(1+2c_0L)}\\
        &= c_0\frac{1+3c_0L+2c^2_0L^2 - 1}{(1+c_0L)(1+2c_0L)}\\
        &= c^2_0L \frac{3+3c_0L}{(1+c_0L)(1+2c_0L)}\\
        &= \frac{3c^2_0L}{1+2c_0L}.
    \end{align*}
    Using the above we continue from (\ref{eq:nbijweniqfqdddd})
    \begin{align}
        \|z^1 - x^*\|^2 &\overset{\text{conv.}}{\le} \|z^0 - x^*\|^2 
        - 2\gamma_0(f_{S_0}(x^0) - f_{S_0}(x^*))
        + \frac{4c_0L}{1+2c_0L}\gamma_0(f_{S_0}(x^0) - f_{S_0}^*)
        \notag \\
        & \quad 
        + \frac{2c^2_0L}{1+c_0L}\max\left\{\frac{2c_0L-1}{2c_0L+1},0\right\}f_{S_0}^*\notag\\
        &\le \|z^0 - x^*\|^2 
        - 2\rho_0(f_{S_0}(x^0) - f_{S_0}(x^*)) - 2\wtilde{\gamma}_0(f_{S_0}(x^0) - f_{S_0}^*) 
        + 2\wtilde{\gamma}_0(f_{S_0}(x^*) - f_{S_0}^*)
        \notag \\
        & \quad 
        + \frac{4c_0L}{1+2c_0L}\gamma_0(f_{S_0}(x^0) - f_{S_0}^*)
        + \frac{2c^2_0L}{1+c_0L}\max\left\{\frac{2c_0L-1}{2c_0L+1},0\right\}f_{S_0}^*\notag\\
        &= \|z^0 - x^*\|^2 
        - 2\rho_0(f_{S_0}(x^0) - f_{S_0}(x^*)) - 2\left(\gamma_0 - \rho_0- \frac{2c_0L}{1+2c_0L}\gamma_0\right)(f_{S_0}(x^0) - f_{S_0}^*) 
        \notag \\
        & \quad 
        + 2\wtilde{\gamma}_0(f_{S_0}(x^*) - f_{S_0}^*)
        + \frac{2c^2_0L}{1+c_0L}\max\left\{\frac{2c_0L-1}{2c_0L+1},0\right\}f_{S_0}^*.
    \end{align}
    Here we have 
    \begin{align*}
        \gamma_0 - \rho_0  - \frac{2c_0L}{1+2c_0L}\gamma_0 &= \frac{1}{1+2c_0L}\gamma_0 - \rho_0\\
        &= \frac{1}{1+2cL}\gamma_0 - \frac{c_0}{(1+c_0L)(1+2c_0L)}\\
        &\overset{{\rm Lem. B.5} }{\ge} \frac{1}{1+2c_0L}\frac{c_0}{1+c_0L} - \frac{c_0}{(1+c_0L)(1+2cL)}\\
        &= 0,
    \end{align*}
    $\wtilde{\gamma}_0 \le \frac{3c^2_0L}{1+2c_0L}$, and $f_{S_0}(x^0) - f_{S_0}^* \ge 0.$ Hence, we get
    \begin{align*}
        \|z^1 - x^*\|^2 &\le \|z^0 - x^*\|^2 
        - 2\rho_0(f_{S_0}(x^0) - f_{S_0}(x^*)) 
        + \frac{6c^2_0L}{1+2c_0L}(f_{S_0}(x^*) - f_{S_0}^*)
        \notag \\
        &\quad + \frac{2c^2_0L}{1+c_0L}\max\left\{\frac{2c_0L-1}{2c_0L+1},0\right\}f_{S_0}^*.
    \end{align*}
    Rearranging terms and taking the expectation we get
    \begin{align}\label{eq:sfqionqffff}
    2\rho_0\E{f(x^0) - f(x^*)} &\le \E{\|z^1-x^*\|^2} - \|z^0-x^*\|^2 
    + \frac{6c^2_0L}{1+2c_0L}\sigma^2_{\rm int}
    \notag \\
    &\quad + \frac{2c^2_0L}{1+c_0L}\max\left\{\frac{2c_0L-1}{2c_0L+1},0\right\}\sigma^2_{\rm pos}.
    \end{align}
    
    Next, for $k > 0$ we can use the relation $z^{k} = x^k + \lambda_k (x^k-x^{k-1}).$ We expand $\|z^{k+1} - x^*\|^2$
    \begin{eqnarray*}
        \|z^{k+1} - x^*\|^2 &=& \|z^k - x^*\|^2 
        - 2\gamma_k\<\nabla f_{S_k}(x^k), z^k-x^*>
        + \gamma_k^2\|\nabla f_{S_k}(x^k)\|^2\\
        &=&
        \|z^k-x^*\|^2 
        - 2\gamma_k\<\nabla f_{S_k}(x^k),x^k-x^*>
        - 2\gamma_k\lambda_k\<\nabla f_{S_k}(x^k),x^k-x^{k-1}>\\
        && \;+\; \gamma_k^2\|\nabla f_{S_k}(x^k)\|^2\\
        &\overset{{ \text{conv.}} }{\le}& \|z^k-x^*\|^2 
        - 2\gamma_k(f_{S_k}(x^k) - f_{S_k}(x^*))
        - 2\gamma_k\lambda(f_{S_k}(x^k) -f_{S_k}(x^{k-1}))\\
        && \;+\; \gamma_k^2\|\nabla f_{S_k}(x^k)\|^2\\
        &\overset{{ \text{Lem. B.6}}}{\le}& \|z^k-x^*\|^2 
        - 2\gamma_k(f_{S_k}(x^k) - f_{S_k}(x^*))
        - 2\gamma_k\lambda_k(f_{S_k}(x^k) -f_{S_k}(x^{k-1}))\\
        && \;+\; \frac{4c_kL}{1+2c_k L}\gamma_k(f_{S_k}(x^k)-f_{S_k}^*)
        + \frac{2c^2_kL}{1+c_k L}\max\left\{\frac{2c_k L - 1}{2c_k L+1}, 0\right\}f_{S_k}^*.
    \end{eqnarray*}
    Let $\gamma_k = \rho_k+\wtilde{\gamma}_k$, where $\rho,\wtilde{\gamma}_k \ge 0,$ and $\rho$ is a constant step-size independent of $S_k$ which will be defined later. Therefore, we have
     \begin{eqnarray}
        \|z^{k+1} - x^*\|^2 
        &\le& \|z^k-x^*\|^2 
        - 2\rho_k(f_{S_k}(x^k) - f_{S_k}(x^*))
        - 2\wtilde{\gamma}_k(f_{S_k}(x^k) - f_{S_k}(x^*))\notag\\
        && \;-\; 2\gamma_k\lambda_k(f_{S_k}(x^k) - f_{S_k}^*) 
        + 2\gamma_k\lambda(f_{S_k}(x^{k-1}) - f_{S_k}^*)\notag\\
        && \;+\; 
        \frac{4c_k L}{1+2c_k L}\gamma_k(f_{S_k}(x^k)-f_{S_k}^*)
        + \frac{2c^2_kL}{1+c_k L}\max\left\{\frac{2c_k L - 1}{2c_k L+1}, 0\right\}f_{S_k}^*\notag\\
        &=&  \|z^k-x^*\|^2 
        - 2\rho(f_{S_k}(x^k) - f_{S_k}(x^*))
        - 2\wtilde{\gamma}_k(f_{S_k}(x^k) - f_{S_k}^*)
        + 2\wtilde{\gamma}_k(f_{S_k}(x^*) - f_{S_k}^*)\notag\\
        && \;-\; 2\gamma_k\lambda(f_{S_k}(x^k) - f_{S_k}^*) 
        + 2\gamma_k\lambda(f_{S_k}(x^{k-1}) - f_{S_k}^*)\notag\\
        && \;+\; \frac{4c_k L}{1+2c_k L}\gamma_k(f_{S_k}(x^k)-f_{S_k}^*)
        + \frac{2c^2_kL}{1+c_k L}\max\left\{\frac{2c_k L - 1}{2c_k L+1}, 0\right\}f_{S_k}^*\notag\\
        &=& \|z^k-x^*\|^2 
        - 2\rho(f_{S_k}(x^k) - f_{S_k}(x^*))
        - 2\left(\wtilde{\gamma}_k + \gamma_k\lambda - \frac{2c_kL}{1+2c_kL}\gamma_k\right)(f_{S_k}(x^k) - f_{S_k}^*)\notag\\
        && \;+\; 2\wtilde{\gamma}_k(f_{S_k}(x^*) - f_{S_k}^*)  
        + 2\gamma_k\lambda(f_{S_k}(x^{k-1}) - f_{S_k}^*)\notag\\
        && \;+\; 
        \frac{2c^2_kL}{1+c_k L}\max\left\{\frac{2c_k L - 1}{2c_k L+1}, 0\right\}f_{S_k}^*.\label{eq:sdfjnooqfffff}
    \end{eqnarray}
    We need to find $\rho_k$ such that 
    \[
    \wtilde{\gamma}_k + \gamma_k\lambda - \frac{2c_kL}{1+2c_kL}\gamma_k \ge 0
    \]
    Since $\wtilde{\gamma}_k = \gamma_k - \rho_k$, then we have 
    \begin{align*}
        &\; \gamma_k - \rho_k + \gamma_k \lambda_k - \frac{2c_kL}{1+2c_kL}\gamma_k \ge 0\\
        \Leftrightarrow& \; \gamma_k\left(1+\lambda_k - \frac{2c_kL}{1+2c_kL}\right) \ge \rho_k.
    \end{align*}
    The inequality above is satisfied if it is satisfied for the lower bound on $\gamma_k$ (which is $\nicefrac{c}{1+cL}$), i.e.
    \begin{align*}
        \frac{c_k}{1+c_kL}\left(\frac{1}{1+2c_kL}+\lambda\right) \ge \rho.
    \end{align*}
    We can take $\rho_k = \frac{c_k}{(1+c_kL)(1+2c_kL)}$ since $\lambda \ge 0$.
    \begin{align*}
        \wtilde{\gamma}_k &= \gamma_k - \rho_k\\
        &\le c_k - \frac{c_k}{(1+c_kL)(1+2c_kL)}\\
        &= c_k\frac{1+3c_kL+2c_k^2L^2 - 1}{(1+c_kL)(1+2c_kL)}\\
        &\le c^2_kL\frac{3+3c_kL}{(1+c_kL)(1+2c_kL)}\\
        &= \frac{3c^2_kL}{1+2c_kL}.
    \end{align*}
    Using the above, we get from (\ref{eq:sdfjnooqfffff})
    \begin{align*}
        \|z^{k+1} - x^*\|^2 
        &\le \|z^k-x^*\|^2 
        - 2\rho_k(f_{S_k}(x^k) - f_{S_k}(x^*))
        +   2c_k\lambda_k(f_{S_k}(x^{k-1}) - f_{S_k}(x^*))\notag\\
        & \;+\;  
        2c_k\lambda_k(f_{S_k}(x^{*}) - f_{S_k}^*)
        + \frac{6c^2_kL}{1+2c_kL}(f_{S_k}(x^*) - f_{S_k}^*)\\
        & \;+\; \frac{2c^2_kL}{1+c_k L}\max\left\{\frac{2c_k L - 1}{2c_k L+1}, 0\right\}f_{S_k}^*.\notag
    \end{align*}
    Taking expectations, we get 
    \begin{eqnarray}
        \E{\|z^{k+1} - x^*\|^2} 
        &\le & \E{\|z^k-x^*\|^2 }
        - 2\rho_k\E{f(x^k) - f(x^*)}
        +   2c_k\lambda_k\E{f(x^{k-1}) - f(x^*)}\notag\\
        && \;+\;  
       \left(2c_k\lambda_k+\frac{6c^2_kL}{1+2c_kL}\right)\sigma^2_{\rm int}
        +
        \frac{2c^2_kL}{1+c_k L}\max\left\{\frac{2c_k L - 1}{2c_k L+1}, 0\right\}\sigma^2_{\rm pos}.
    \end{eqnarray}
    Rearranging terms, we get 
    \begin{align}\label{eq:lnjbaifbasssss}
        2\rho_k\E{f(x^k) - f(x^*)} - 2c_k\lambda_k\E{f(x^{k-1}) - f(x^*)} &\le \E{\|z^k-x^*\|^2 } 
        - \E{\|z^{k+1} - x^*\|^2}\notag\\
        & \;+\;  
       \left(2c_k\lambda_k+\frac{6c^2_kL}{1+2cL}\right)\sigma^2_{\rm int}\notag\\
        & \;+\; \frac{2c^2_kL}{1+c_k L}\max\left\{\frac{2c_k L - 1}{2c_k L+1}, 0\right\}\sigma^2_{\rm pos}.
    \end{align}

    Combining \eqref{eq:sfqionqffff} and \eqref{eq:lnjbaifbasssss} for iterations $\{1,\dots, K-1\}$ we get 
    \begin{align}
        &2\rho_0\E{f(x^0) - f(x^*)}
        + 2\sum_{k=1}^{K-1}\rho_k\E{f(x^k) - f(x^*)} 
        - 2\sum_{k=1}^{K-1}c_k\lambda_k\E{f(x^{k-1}) - f(x^*)}\notag \\
        &= 2\sum_{k=0}^{K-1}\rho_k\E{f(x^k) - f(x^*)} 
        - 2\sum_{k=0}^{K-2}c_k\lambda_k\E{f(x^{k}) - f(x^*)}\notag \\
        &\le 2\sum_{k=0}^{K-1}(\rho_k -c_k\lambda_k)\E{f(x^k) - f(x^*)}\notag \\
        &\le \|z^0-x^*\|^2
         + \frac{6c^2_0L}{1+2c_0L}\sigma^2_{\rm int}
        + \frac{2c^2_0L}{1+c_0L}\max\left\{\frac{2c_0L-1}{2c_0L+1},0\right\}\sigma^2_{\rm pos}\notag \\
        &\;+\;  \sum_{k=1}^{K-1}\left(2c_k\lambda_k+\frac{6c^2_kL}{1+2c_kL}\right)\sigma^2_{\rm int}
        + \sum_{k=1}^{K-1} \frac{2c^2_kL}{1+c_k L}\max\left\{\frac{2c_k L - 1}{2c_k L+1}, 0\right\}\sigma^2_{\rm pos}.\label{eq:ynjknfjonasssss}
    \end{align}
    Note that choosing $\lambda_k = \min\left\{c_kL, 0.5(1+c_kL)^{-1}(1+2c_kL)^{-1}\right\}$ ensures that $\frac{\rho_k}{2} \ge c_k\lambda_k$. Indeed, we have 
    \begin{align*}
        & \frac{\rho_k}{2} = \frac{c_k}{2(1+c_kL)(1+2c_kL)} > c_k\lambda_k\\
        \Leftrightarrow &\; 1 > 2\lambda_k (1+c_kL)(1+2c_kL).
    \end{align*}
    Therefore, from (\ref{eq:ynjknfjonasssss}) and the facts that $\lambda_0=0$ and $\lambda_k \le c_kL$ we get
    \begin{align}\label{eq:nqibdjqwjkd}
        \sum_{k=0}^{K-1}\rho_k\E{f(x^k) - f(x^*)}
        &\le \|z^0-x^*\|^2
        + \sum_{k=0}^{K-1}\left(2c_k\lambda_k+\frac{6c^2_kL}{1+2c_kL}\right)\sigma^2_{\rm int} \notag\\
        & \;+\;  
        \sum_{k=0}^{K-1}\frac{2c^2_kL}{1+c_k L}\max\left\{\frac{2c L - 1}{2c L+1}, 0\right\}\sigma^2_{\rm pos}\notag\\
        &\le \|z^0-x^*\|^2
        + 8L\sigma^2_{\rm int}\sum_{k=0}^{K-1}c_k^2 \notag\\
        & \;+\;  
        \sum_{k=0}^{K-1}2c^2_kL\max\left\{\frac{2c_k L - 1}{2c_k L+1}, 0\right\}\sigma^2_{\rm pos}.
    \end{align}
    We have by \Cref{lem:lemma_summation}
    \begin{align}\label{eq:fiehfbqkjd}
        &\sum_{k=0}^{K-1}\rho_k = \sum_{k=0}^{K-1}\frac{c_k}{(1+c_kL)(1+2c_kL)} \ge \sum_{k=0}^{K-1}\frac{c_k}{(1+c_0L)(1+2c_0L)} \ge \frac{4c_0\sqrt{K}}{5(1+c_0L)(1+2c_0L)},\notag\\
        &\sum_{k=0}^{K-1}c_k^2 \overset{\text{Lem}~\ref{lem:lemma_summation}}{\le} c_0^2\log(K+2),\\
        &\sum_{k=0}^{K-1}c_k^2\max\left\{\frac{2c_kL-1}{2c_kL+1},0\right\} \le \sum_{k=0}^{K-1}c_k^2\max\left\{\frac{2c_0L-1}{2c_0L+1},0\right\} \le c_0^2\log(K+2)\max\left\{\frac{2c_0L-1}{2c_0L+1},0\right\}.\notag
    \end{align}
    Therefore, using \eqref{eq:fiehfbqkjd}, $z^0=x^0$ in \eqref{eq:nqibdjqwjkd} and dividing both sides in \eqref{eq:nqibdjqwjkd} by $\sum_{k=0}^{K-1}\rho_k$ we derive
    \begin{align}
        \sum_{k=0}^{K-1}\frac{\rho_k}{\sum_{k=0}^{K-1}}\E{f(x^k) - f(x^*)}
        &\le \frac{\|x^0-x^*\|^2}{\sum_{k=0}^{K-1}\rho_k}
        + 8Lc_0^2\sigma^2_{\rm int}\frac{\log{(K+2)}}{\sum_{k=0}^{K-1}\rho_k} \notag\\
        & \;+\;  
        2c_0^2L\frac{\log(K+2)}{\sum_{k=0}^{K-1}\rho_k}\max\left\{\frac{2c_0L-1}{2c_0L+1},0\right\}\sigma_{\rm pos}^2.
    \end{align}
    With an lower bound on $\sum_{k=0}^{K-1}$ and Jensen's inequality we conclude that 
    \begin{align}
        \E{f(\hat{x}^{K-1}) - f(x^*)}
        &\le \frac{5(1+c_0L)(1+2c_0L)\|x^0-x^*\|^2}{4c_0\sqrt{K}}\notag\\
        & \;+\; 10Lc_0(1+c_0L)(1+2c_0L)\sigma^2_{\rm int}\frac{\log{(K+2)}}{\sqrt{K}} \notag\\
        & \;+\;  
        5c_0L(1+c_0L)\frac{\log(K+2)}{2\sqrt{K}}\max\left\{2c_0L-1,0\right\}\sigma_{\rm pos}^2,
    \end{align}
    where $\hat{x}^{K-1} = \sum_{k=0}^{K-1}\frac{\rho_k}{\sum_{k=0}^{K-1}\rho_k}x^k.$
    
\end{proof}

\section{Stability of \algname{NGN-M} on a Simple Problem}\label{sec:stability_simple_problems}

We consider 1D convex functions of the form $f(x) = Lx^2(1+p^2(x))$ that satisfy the following assumption.

\begin{assumption}\label{asmp:bounded_above}
    There exists a constant $C$ such that $C(1+p^2(x)) \ge xp(x)p^\prime(x)).$
\end{assumption}
Note that $1+p^2(x) \ge 1$ and ${\rm deg}(1+p^2(x)) = {\rm deg}(xp(x)p^\prime(x)).$ Therefore, this assumption is mild.

\begin{remark}
    For example, the function $f(x) = x^2(1+x^2)$ (i.e., $p(x) = x$) is convex and satisfies \Cref{asmp:bounded_above} with $C=1.$
\end{remark}

\begin{remark}
    Let $p(x) = \sum_{j=0}^m a_jx^j.$ Then for large values of $x$ in magnitude, $p(x) \sim a_mx^m, p^\prime(x) \sim ma_mx^{m-1}.$ Therefore, the constant $C$ should be expected of order $C \approx m,$ where $m = {\rm deg}(p(x)).$
\end{remark}

The function $f(x)$ is non-negative for any $x\in\R$ and its minimum $f^*=0$ is attained at $x=0$ by design. Let us compute a step of \algname{NGN-M} on this problem
\begin{align}
    x^{k+1} &= x^k - (1-\beta)\frac{c}{1+\frac{c}{2f(x^k)}(f^\prime(x^k))^2}f^\prime(x^k) + \beta(x^k-x^{k-1})\notag\\
    &= x^k - (1-\beta)\frac{2Lc(1+p^2(x^k) + x^kp(x^k)p^\prime(x^k))}{1 + \frac{4L^2c[x^k]^2}{2L[x^k]^2(1+p^2(x^k))}(1+p^2(x^k) + x^kp(x^k)p^\prime(x^k))^2}x^k + \beta(x^k-x^{k-1})\notag\\
    &= x^k - (1-\beta)\underbrace{\frac{2Lc(1+p^2(x^k) + x^kp(x^k)p^\prime(x^k))}{1 + \frac{2Lc}{1+p^2(x^k)}(1+p^2(x^k) + x^kp(x^k)p^\prime(x^k))^2}}_{\eqdef\hat{\gamma}_k}x^k + \beta(x^k-x^{k-1}).
\end{align}

Note that the convexity of $f$ implies that 
\begin{align}\label{eq:from_convexity}
    f(0) &\ge f(x) + f^\prime(x)(0-x)\notag\\
    0 &\ge Lx^2(1+p^2(x)) - 2Lx^2(1+p^2(x) + xp(x)p^\prime(x))\notag\\
    0 &\ge -Lx^2(1+p^2(x)) - 2Lx^3p(x)p^\prime(x)\notag\\
    xp(x)p^\prime(x) &\ge -\frac{1}{2}(1+p^2(x)).
\end{align}
In particular, \eqref{eq:from_convexity} implies that $1+p^2(x) + xp(x)p^\prime(x) \ge \frac{1}{2}(1+p^2(x)) > 0.$ Therefore, we can obtain lower and upper bounds on $\hat{\gamma}_k$.

\begin{lemma}\label{lem:gammahat_bound}
    Let \Cref{asmp:bounded_above} hold with a constant $C > 0$ and $f(x) = x^2(1+p^2(x))$ be convex. Let $c \ge \frac{1}{2L}$. Then we have $\hat{\gamma}_k \in \left[\frac{1}{2(1+C)}, 2\right]$.
\end{lemma}

\begin{proof}
    Indeed, the upper bound on $\hat{\gamma}_k$ follows from the following inequality
    \begin{align}
        \hat{\gamma}_k &= \frac{2Lc(1+p^2(x^k) + x^kp(x^k)p^\prime(x^k))}{1 + \frac{2Lc}{1+p^2(x^k)}(1+p^2(x^k) + x^kp(x^k)p^\prime(x^k))^2}\notag\\
        &\le \frac{2Lc(1+p^2(x^k) + x^kp(x^k)p^\prime(x^k))}{\frac{2Lc}{1+p^2(x^k)}(1+p^2(x^k) + x^kp(x^k)p^\prime(x^k))^2}\notag\\
        &= \frac{1+p^2(x^k)}{1+p^2(x^k) + x^kp(x^k)p^\prime(x^k)} \le 2,
    \end{align}
    due to \eqref{eq:from_convexity}.
    The lower bound can be obtained as follows
    \begin{align}
        \hat{\gamma}_k &= \frac{2Lc(1+p^2(x^k) + x^kp(x^k)p^\prime(x^k))}{1 + \frac{2Lc}{1+p^2(x^k)}(1+p^2(x^k) + x^kp(x^k)p^\prime(x^k))^2}\notag\\
        &= \frac{2Lc(1+p^2(x^k) + x^kp(x^k)p^\prime(x^k))(1+p^2(x^k))}{(1+p^2(x^k)) + 2Lc(1+p^2(x^k) + x^kp(x^k)p^\prime(x^k))^2}\notag\\
        &\ge \frac{2Lc(1+p^2(x^k) + x^kp(x^k)p^\prime(x^k))(1+p^2(x^k))}{2(1+p^2(x^k)+x^kp(x^k)p^\prime(x^k)) + 2Lc(1+p^2(x^k) + x^kp(x^k)p^\prime(x^k))^2}\notag\\
        &= \frac{Lc(1+p^2(x^k))}{1 + Lc(1+p^2(x^k) + x^kp(x^k)p^\prime(x^k))}\notag\\
        &= \frac{Lc(1+p^2(x^k))}{1 + Lc(1+p^2(x^k) + C(1+p^2(x^k)))}\notag\\
        &\ge \frac{Lc(1+p^2(x^k))}{2Lc(1+C)(1+p^2(x^k)))}=\frac{1}{2(1+C)}
    \end{align}
\end{proof}

The update rule of \algname{NGN-M} can be rewritten as 
\begin{align}
    x^{k+1} = x^k - (1-\beta)\hat{\gamma}_kx^k + \beta(x^k-x^{k-1}).
\end{align}
Let us consider the joint dynamics of $w^k \eqdef ([x^k]^\top, [x^{k-1}]^\top)^\top \in \R^{2d}.$ We have that 
\begin{align}
    w^k &= \begin{pmatrix}
        x^k \\ x^{k-1}
    \end{pmatrix} = \begin{pmatrix}
        x^k - (1-\beta)\hat{\gamma}_k x^k + \beta(x^k-x^{k-1}) \\ x^{k-1}
    \end{pmatrix}\notag\\
    &= \begin{pmatrix}
        \mI - (1-\beta)\hat{\gamma}_k \mI + \beta\mI & -\beta \mI \\ \mI & \textbf{0}
    \end{pmatrix}\begin{pmatrix}
        x^k \\ x^{k-1}
    \end{pmatrix} = \mG w^{k-1},
\end{align}
where 
\begin{align}
    \mG \eqdef \begin{pmatrix}
        \mI - (1-\beta)\hat{\gamma}_k \mI + \beta\mI & -\beta \mI \\ \mI & \textbf{0}
    \end{pmatrix}.
\end{align}

Now we are ready to prove the convergence of \algname{NGN-M} on this simple problem for any value $c\ge \frac{1}{2}.$

\begin{theorem}\label{th:polynomial}
    Let $f(x) = x^2(1+p^2(x))$ be convex and \Cref{asmp:bounded_above} holds. Let $\beta \ge \frac{(2(1+C)-1)^2}{(2(1+C)+1)^2}$ and $c\ge \frac{1}{2L}$. Then the iterates of \algname{NGN-M} on $f(x)$ converge to the minimum $f^*=0$.
\end{theorem}
\begin{proof}
    We follow the standard proof of \algname{SGD} with Polyak momentum \citep{polyak1964momentum}. At this stage, we need to estimate the eigenvalues of $\mG.$ To do so, we will proceed with a permutation matrix $\Pi$\footnote{The permutation matrix $\Pi$ is defined as $\Pi_{ij} = \begin{cases} 1 & i \text{ odd }, j = i \\ 1 & i \text{ even }, j = 2n + i \\ 0 & \text{else} \end{cases}$. Note that permutation matrices preserve eigenvalues.} which transforms the matrix $\mG$ to the block-diagonal matrix as
    \begin{equation}
    \mG = \begin{pmatrix} \mG_1 & 0 & \dots & 0 \\
    \cdots & \cdots & \cdots & \cdots \\
    0 & 0 & \dots & \mG_d \end{pmatrix},
    \end{equation}
    where
    \begin{equation}
    \mG_i :=  \begin{pmatrix} 1 + \beta - (1-\beta)\hat{\gamma}_k & - \beta \\
    1 & 0 \end{pmatrix}
    \end{equation}
    Since the matrix $\mG$ is a block-diagonal matrix, we have $\| \mG \| \leq \max_i \| \mG_i \|$. Therefore, the problem is now simplified to bounding the spectral radii of the individual blocks $\mG_i$, for $i = 1, 2, \ldots, d$. The two eigenvalues $u_1$ and $u_2$ of $\mG_i$ are the roots of the quadratic
    \begin{equation}
q(u) := u^2 - (1 + \beta - (1-\beta)\hat{\gamma}_k) u + \beta = 0,
\end{equation}
which take different values depending on the discriminant $\Delta := (1 + \beta - (1-\beta)\hat{\gamma}_k)^2 - 4 \beta$. Let us find the values of $\beta$ when the discriminant is negative. We need to satisfy the inequality 
\begin{align}
    (1 + \beta - (1-\beta)\hat{\gamma}_k)^2 - 4 \beta \le 0 &\Leftrightarrow (1+\beta)^2 + (1-\beta)^2\hat{\gamma}_k^2 -2(1+\beta)(1-\beta)\hat{\gamma}_k - 4\beta \le 0\notag\\
    &\Leftrightarrow (1-\beta)^2 + (1-\beta)^2\hat{\gamma}_k^2 -2(1+\beta)(1-\beta)\hat{\gamma}_k \le 0\notag\\
    &\Leftrightarrow (1-\beta)(1+\hat{\gamma}_k^2) \le 2(1+\beta)\hat{\gamma}_k\notag\\
    &\Leftrightarrow \frac{1+\hat{\gamma}_k^2}{2\hat{\gamma}_k} \le \frac{1+\beta}{1-\beta}.
\end{align}
    Since the function $\frac{1+y^2}{2y}$ for $y\in\left[\frac{1}{2(1+C)}, 2\right]$ attains the maximum $\frac{4(1+C)^2+1}{4(1+C)}$ at $y=\frac{1}{2(1+C)}$, then we satisfy the last inequality, and consequently the discriminant is non-positive, if we choose 
    \begin{align}
        \frac{4(1+C)^2+1}{4(1+C)} \le \frac{1+\beta}{1-\beta}.
    \end{align}
    The above inequality is satisfied for $\beta \in \left[\frac{(2(1+C)-1)^2}{(2(1+C)+1)^2}, 1\right).$ Therefore, we obtain that for such choice of $\beta$ we have $\Delta_i \le 0$ for all $i\in[d].$ Therefore, the zeros of the quadratic $q(u)$ are complex, and are equal in absolute value
    \begin{align}
        |u_1| = |u_2| = \sqrt{\beta} < 1.
    \end{align}
    This gives us that $\|\mG_i\|\le \sqrt{\beta} < 1.$ Therefore, the algorithm converges for any value of $\beta$ in this range.  

    It remains to use Lemma 11 from \citet{foucart2012lecture} which says that for a given matrix $\mA \in \R^{d \times d}$, and $\epsilon > 0$, there exists a matrix norm $\| \cdot \|$ such that
    \begin{equation}
    \| \mA \| \leq \rho(\mA) + \epsilon,
    \end{equation}
    where $\rho(\mA) = \max \{ |\lambda| : \lambda \text{ eigenvalue of } \mA \}$ (spectral radius of $\mA$).

    Asymptotically~\footnote{A non-asymptotic version of the analysis can be derived using Theorem 5 by~\cite{wang2021modular}} (as $k \to \infty$, one can show (see Theorem 12 in~\cite{foucart2012lecture}) that
    \begin{equation}
    \| w^k \|_2 = \cO(\rho(\mG)^k),
    \end{equation}
    where $\rho(\mG) \le \sqrt{\beta} < 1$ in our analysis. Therefore, \algname{NGN-M} with hyperparameters $c \ge \frac{1}{2}$ and $\beta\ge \frac{1}{0}$ converges.
    \end{proof}

    \begin{remark}
        For example, \algname{NGN-M} converges on $f(x) = x^2(1+x^2)$ for any $c\ge \frac{1}{2}$ and $\beta \ge \frac{9}{25}.$
    \end{remark}

    Theorem \ref{th:polynomial} shows that \algname{NGN-M} remains stable even with an arbitrarily large step-size hyperparameter $c$. Thanks to the adaptive nature of \algname{NGN} step-size, the actual update scale is automatically shrunk when necessary, preserving convergence. Importantly, this is possible with a choice of momentum parameter $\beta$ close to $1$, which extends the results of \Cref{sec:theory}. We acknowledge that our current analysis is restricted to the special convex class of 1D functions $f(x) = x^2(1+p^2(x))$ satisfying \Cref{asmp:bounded_above}. Extending such stability guarantees to wider function classes with large momentum $\beta$ remains a significant open challenge.

To support the theoretical result, we test the performance of \algname{NGN-M} and \algname{GDM} (Gradient Descent with Momentum) on the problem $f(x) = x^2(1+x^2)$, which is convex and satisfies \Cref{asmp:bounded_above}; see \Cref{fig:x2_x4}. We run both algorithms, varying the step-size hyperparameter in $\{10^{-4}, \dots, 10^4\}.$ We run algorithms for $10^5$ iterations. We stop training if the loss reaches a threshold $10^{-15}$ or exceeds $10^{10}$ for the first time. We observe that $(i)$ for small step-size hyperparameters, both methods converge but do not reach the threshold $10^{-15};$ $(ii)$ \algname{NGN-M} reaches the threshold even for extremely large values of the step-size hyperparameter while \algname{GDM} diverges. $(iii)$ the fastest convergence of \algname{GDM} is achieved with the step-size hyperparameter $10^{-2}$ after $691$ iterations while the fastest convergence of \algname{NGN-M} is achieved with $c=10^{1}$ after $269$ iterations. In other details, \algname{NGN-M} achieves faster convergence and much more stable to the choice of the step-size hyperparameter. These results align well with our theoretical analysis.

    \begin{figure}[h]
    \centering
    \includegraphics[width=0.5\linewidth]{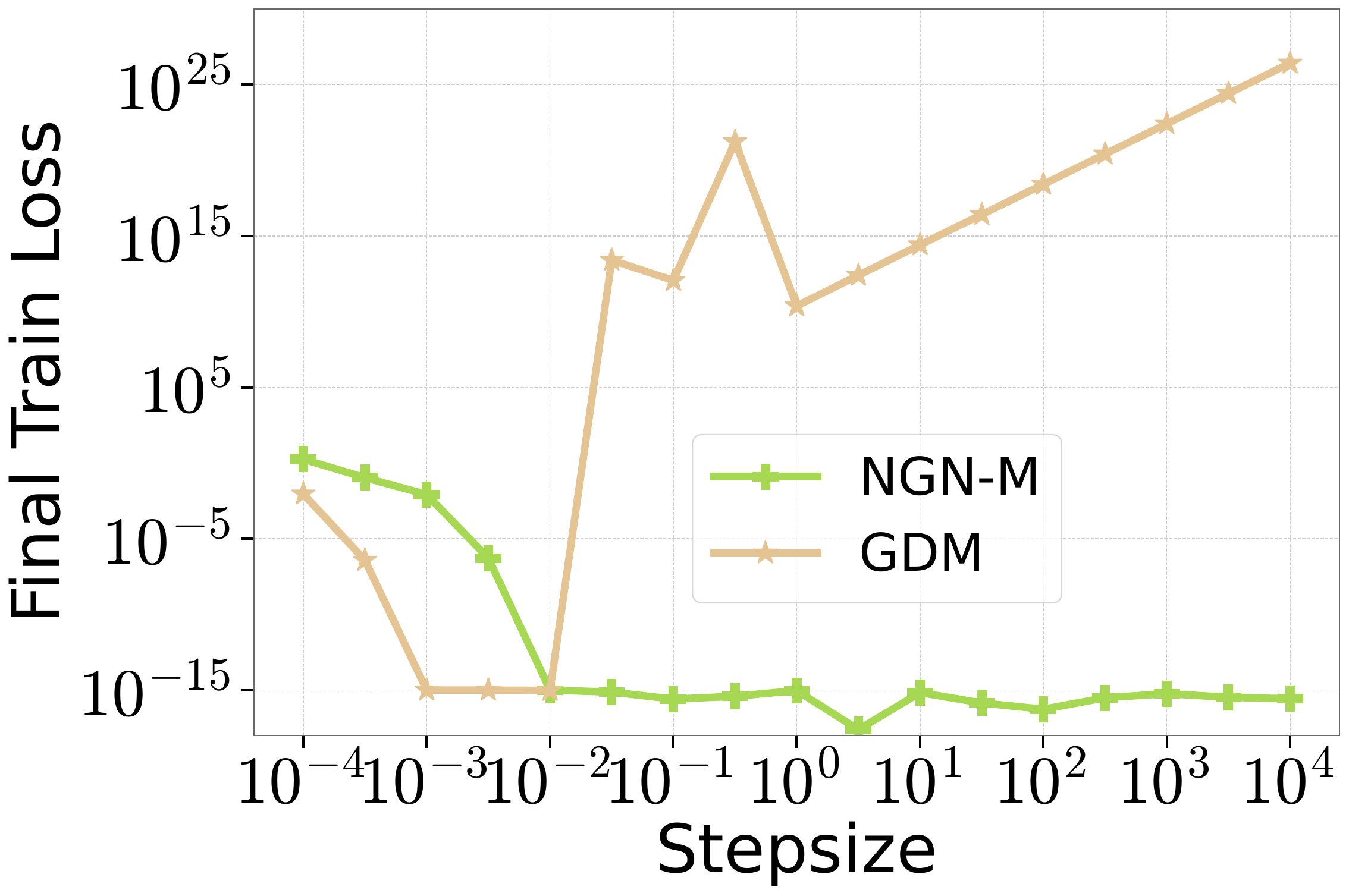} 
    
    \caption{Comparison of \algname{SGDM} and \algname{NGN-M} when minimizing a function $f(x) = x^2 + x^4$.}
    \label{fig:x2_x4}
\end{figure}

\newpage

\section{How to Derive Diagonal \algname{NGN}-based Step-size?}\label{sec:derive_diagonal_ngn}

Here we provide derivations of how combine \algname{NGN}
and diagonal step-size following \Cref{sec:diagonal_ngn} for completeness. 

We consider the following model

\begin{equation}\label{eq:nfljnqfdq}
p^k = \argmin\limits_{p\in\R^d} \left[f_{\boldsymbol{\Sigma}_k,c}(x^k+p) \eqdef (r(x^k) + \nabla r(x^k)^\top p)^2 + \frac{1}{2c}\|p\|^2_{\boldsymbol{\Sigma}_k}\right],
\end{equation}
where $r(x) = \sqrt{f(x)}$. We compute the gradient of RHS of (\ref{eq:nfljnqfdq}) w.r.t. $p$ and equal it to zero:
\begin{align}
    \nabla_p f_{\boldsymbol{\Sigma}_k,c}(x^k+p) &= 2\left(r(x^k)+\nabla r(x^k)^\top p\right)\nabla r(x^k) + \frac{1}{c}\boldsymbol{\Sigma}_k p\notag\\
    &= \left(2\nabla r(x^k)\nabla r(x^k)^\top + \frac{1}{c}\boldsymbol{\Sigma}_k\right)p + 2r(x^k) \nabla r(x^k).\notag
\end{align}
Therefore, we have
\begin{align*}
    p^k = - \left(2\nabla r(x^k)\nabla r(x^k)^\top + \frac{1}{c}\boldsymbol{\Sigma}_k\right)^{-1}2r(x^k)\nabla r(x^k).
\end{align*}
Using Shermann-Morrison formula $(\mA+uv^\top)^{-1} = \mA^{-1} - \frac{\mA^{-1}uv^\top \mA^{-1}}{1+u^\top \mA^{-1}v}$ with $\mA = \nicefrac{1}{c}\boldsymbol{\Sigma}_k$ we derive
\begin{align*}
    p^k &= -\left(c\boldsymbol{\Sigma}_k^{-1} - \frac{2c^2\boldsymbol{\Sigma}_k^{-1}\nabla r(x^k)\nabla r(x^k)^\top \boldsymbol{\Sigma}_k^{-1}}{1 + 2c\nabla r(x^k)^\top \boldsymbol{\Sigma}_k^{-1} \nabla r(x^k)}\right)2r(x^k)\nabla r(x^k)\\
    &= -2cr(x^k)\left(1 - \frac{2c\nabla r(x^k)^\top \boldsymbol{\Sigma}_k^{-1}\nabla r(x^k)}{1+ 2c\nabla r(x^k) \boldsymbol{\Sigma}_k^{-1}\nabla r(x^k)}\right)\boldsymbol{\Sigma}_k^{-1}\nabla r(x^k)\\
    &= -\frac{2cr(x^k)}{1+ 2c\nabla r(x^k) \boldsymbol{\Sigma}_k^{-1}\nabla r(x^k)}\boldsymbol{\Sigma}_k^{-1}\nabla r(x^k).
\end{align*}
Now we plug-in $r(x^k) = \sqrt{f(x^k)}$ and $\nabla r(x^k) = \frac{1}{2\sqrt{f(x^k)}}\nabla f(x^k)$ and obtain
\begin{align*}
    p^k &= -\frac{2c\sqrt{f(x^k)}}{1+2c\frac{1}{4f(x^k)}\nabla f(x^k)^\top \boldsymbol{\Sigma}_k^{-1}\nabla f(x^k)}\frac{1}{2\sqrt{f(x^k)}}\boldsymbol{\Sigma}_k^{-1}\nabla f(x^k)\\
    &= \frac{c}{1+\frac{c}{2f(x^k)}\|\nabla f(x^k)\|^2_{\boldsymbol{\Sigma}^{-1}_k}}\boldsymbol{\Sigma}_k^{-1}\nabla f(x^k).
\end{align*}

\subsection{Design Comparison of \algname{NGN-MDv1} and \algname{NGN-MDv2}}\label{sec:ngn_md_design}

The derivations in \eqref{eq:njonjsnao} are used to provide an intuition of how one can add a diagonal step-size into \algname{NGN} by choosing the regularization matrix $\boldsymbol{\Sigma}_k$. By choosing $\boldsymbol{\Sigma}_k = \mD_k$ we recover the update direction of \algname{NGN-MDv1}. In this case, we have only one global \algname{NGN} step-size in front of $\mD_k$. The design of \algname{NGN-MDv2} follows a more straightforward intuition. In particular, it can be seen as a direct extension of \algname{NGN} to diagonal case by replacing the squared gradient norm $\|\nabla f_{S_k}(x^k)\|^2$ by the squared partial derivative $(\nabla_j f_{S_k}(x^k))^2$ for each parameter $j\in[d]$. 

The main difference in comparison with \algname{Adam} is the order in which the preconditioning and momentum is applied. In both \algname{NGN-MDv1} and \algname{NGN-MDv2} we average the preconditioned updates $\boldsymbol{\Sigma}_k^{-1}\nabla f_{S_k}(x^k),$ i.e. we first apply preconditioning and momentum later. In contrast, in \algname{Adam} the stochastic gradients are averaged to construct new momentum term, and then the momentum is preconditioned. In other words, the momentum is applied first and then it is followed by preconditioning. We believe this change might be one of the reasons behind the step-size hyperparameter resilience as well.

In practice, we found out that the tuned performance of \algname{NGN-MDv1} is slightly better than that of \algname{NGN-MDv2}. Moreover, \algname{NGN-MDv1} demonstrates higher resilience to the choice of the step-size hyperparameter than \algname{NGN-MDv2}.

\subsection{Computation Cost of \algname{NGN-MD}}\label{sec:computation_cost}

Implementing any version of \algname{NGN-MD} in practice might be slightly more computationally expensive. However, we highlight that computing a step of \algname{NGN-MD} does not involve matrix-vector operations since the preconditioner is a diagonal matrix, and the matrix notation is used only for the convenience of presentation. The additional computation cost that we have in \algname{NGN-MDv1} is the computation of $\|\nabla f_{S_k}(x^k)\|^2_{\mathbf{D}_k^{-1}}$. This can na\"ively be done by one additional pass over the gradient and summing the terms $\frac{1}{(\mathbf{D}_k){j}}(\nabla_j f_{S_k}(x^k))^2$ for $j\in[d].$ This operation does not require additional matrix multiplication. However, it can be computed more efficiently while updating $\mathbf{D}_k$. The rest of the \algname{NGN-MDv1} implementation does not add any significantly costly operations in comparison with \algname{Adam}. 


\begin{table*}[t]
    \centering
    \caption{Train time of \algname{Adam} and \algname{NGN-MDv1} when training language models.}
    \label{tab:train_time}
    \resizebox{\linewidth}{!}{
        \begin{tabular}{cccc}

            \toprule
           \textbf{Model} & 
           {\bf Method} &
            \makecellnew{{\bf Time per Iteration (sec)}} &
            \makecellnew{{\bf Time per Optimizer Update (sec)}}
            \\ \toprule

            70M &
            \makecellnew{\algname{AdamW} \\ \algname{NGN-MDv1}} &
            \makecellnew{$1.63_{\pm 0.01}$ \\ $1.65\pm 0.01$} &
            \makecellnew{$0.0048\pm0.0002$ \\ $0.0130\pm0.0002$}
            \\

            \midrule

            160M &
            \makecellnew{\algname{AdamW} \\ \algname{NGN-MDv1}} &
            \makecellnew{$3.33\pm 0.03$ \\ $3.37\pm 0.02$} &
            \makecellnew{$0.0088\pm0.0003$ \\ $0.0239\pm 0.0003$}
            \\

            \midrule

            410M &
            \makecellnew{\algname{AdamW}  \\ \algname{NGN-MDv1}} &
            \makecellnew{$8.41\pm 0.06$ \\ $8.68\pm 0.06$} &
            \makecellnew{$0.0838\pm0.0009$ \\ $0.2154\pm0.0007$}
            \\

            \bottomrule 
        
        \end{tabular}
        }
\end{table*}

We compare in \Cref{tab:train_time} the time per iteration and optimizer update when training language models from \Cref{sec:experiments_main} using \algname{AdamW} and \algname{NGN-MDv1}. We notice that our naive implementation of \algname{NGN-MDv1} is about $2.5$ times slower than PyTorch's \algname{AdamW}. This is expected since our algorithm requires two passes over the gradient. Nevertheless, in this setting training time is dominated by forward and backward computations, keeping \algname{NGN-MDv1} competitive with \algname{AdamW}. Moreover, as noted above, this overhead can be largely eliminated by computing the weighted gradient concurrently with the second-momentum $v^k$ update. We do not aim to provide the most efficient implementation of \algname{NGN-MDv1} as the primary goal of our work is to highlight the stability advantages that \algname{NGN} step-size brings in the training of neural networks.

\subsubsection{Distributed Training}

In a vanilla DDP implementation \citep{li2020pytorch}, computing the weighted gradient norm $\|\nabla f_{S_k}(x^k)\|^2_{\mD_k^{-1}}$ is straightforward since gradients are replicated across devices. We only require an additional all-reduce to synchronize $f_{S_k}(x^k)$ across devices, which is, however, a lightweight communication (just a single float) and, in principle, can even be overlapped with the backward pass.

However, with more sophisticated types of parallelism, like Tensor Parallel \citep{shoeybi2019megatron} or ZeRO-2 \citep{rajbhandari2020zero}, computing the weighted gradient norm introduces additional communication, as gradients are sharded across devices. This could still be implemented efficiently by accumulating squared gradient entries in each device and all-reducing only a single float, but it will, nevertheless, result in a computation and communication overhead for \algname{NGN-MDv1}. We acknowledge that our methods might not be scalable to large distributed training, and adjustments are needed to make \algname{NGN-MDv1} work in this case. Nonetheless, we believe that our findings offer useful insights toward designing more stable optimization algorithms.

\section{How to add weight decay to \algname{NGN-MDv1}?}\label{sec:weight_decay}



Regularization techniques serve a fundamental purpose in minimizing generalization error. Orthogonal to their role for generalization, modern deep learning tasks often benefit from the use of weight decay \citep{xiao2024rethinkingwisdomML}. Despite its widespread application, the role of weight decay is poorly understood. \citet{andriushchenko2023weightdecay} suggested that it might provide implicit regularization by stabilizing the loss in over-parameterized neural networks and helping to balance the bias-variance tradeoff that leads to lower training loss in under-parameterized networks. However, even in the case of \algname{SGD}, there is still uncertainty regarding how the weight decay mechanism should be incorporated, as various implementations may exist \citep{zhang2018three}.

We propose two ways of adding weight decay to \algname{NGN-MDv1}.
The first variant follows the approach of \cite{loshchilov2019decoupled}, adding decoupled weight decay $\lambda$:
\begin{equation}\label{eq:ngn_mdv1_decoupled}
x^{k+1} = x^k - \lambda cx^k - (1-\beta_1)\boldsymbol{\Sigma}_k^{-1}\nabla f_{S_k}(x^k) + \beta_1(x^k-x^{k-1}).
\end{equation}
In this update rule, the weight is added separately from the update direction $\boldsymbol{\Sigma}_k^{-1}\nabla f_{S_k}(x^k).$ We call the resulting algorithm (\ref{eq:ngn_mdv1_decoupled}) \algname{Dec-NGN-MDv1}, that stands for decoupled \algname{NGN-MDv1}.

\subsection{Combining \algname{NGN-MDv1} and Weight Decay Regularization}

We now discuss how to combine \algname{NGN-MDv1} and weight decay, following the idea that weight decay should perform weight regularization.

We consider the following model
\begin{align*}
    f_{\boldsymbol{\Sigma}_k, \lambda}(x^k+p) \eqdef (r(x^k) + \nabla r(x^k)^\top p)^2 + \frac{1}{2c}\|p\|^2_{\boldsymbol{\Sigma}_k} + \frac{\lambda}{2}\|x^k+p\|^2_{\boldsymbol{\Sigma}_k}.
\end{align*}
By taking the gradient of $f_{\boldsymbol{\Sigma}_k,\lambda}$ w.r.t. $p$ we get
\begin{align*}
    0 &= 2(r(x^k) + \nabla r(x^k)^\top p)\nabla r(x^k) + \frac{1}{c}\boldsymbol{\Sigma}_kp + \lambda\boldsymbol{\Sigma}_k(x^k+p)\\
    &= \left(2\nabla r(x^k)\nabla r(x^k)^\top + \frac{1}{c}\boldsymbol{\Sigma}_k + \lambda\boldsymbol{\Sigma}_k\right)p + 2r(x^k)\nabla r(x^k) + \lambda \boldsymbol{\Sigma}_kx^k.
\end{align*}
Therefore, we get
\begin{align*}
    p^k &= - \left(2\nabla r(x^k)\nabla r(x^k)^\top + \frac{1}{c}\boldsymbol{\Sigma}_k + \lambda\boldsymbol{\Sigma}_k\right)^{-1}(2r(x^k)\nabla r(x^k) + \lambda\boldsymbol{\Sigma}_kx^k).
\end{align*}
Using Sherman-Morrison formula $(\mA+uv^\top)^{-1} = \mA^{-1} - \frac{\mA^{-1}uv^\top \mA^{-1}}{1+u^\top \mA^{-1}v}$ with $\mA = (\lambda+\nicefrac{1}{c})\boldsymbol{\Sigma}_k$ and $u=v=\sqrt{2}\nabla r(x^k)$ we get that
\begin{align*}
    &\left(2\nabla r(x^k)\nabla r(x^k)^\top + \frac{1}{c}\boldsymbol{\Sigma}_k + \lambda\boldsymbol{\Sigma}_k\right)^{-1}\\
    &\qquad = \frac{c}{1+\lambda c}\boldsymbol{\Sigma}_k^{-1} - \frac{\frac{2c^2}{(1+\lambda c)^2}\boldsymbol{\Sigma}_k^{-1}\nabla r(x^k)\nabla r(x^k)^\top \boldsymbol{\Sigma}_k^{-1}}{1+\frac{2c}{1+\lambda c}\nabla r(x^k) \boldsymbol{\Sigma}_k^{-1}\nabla r(x^k)}.
\end{align*}
Therefore, we have 
\begin{align*}
    p^k &= - \left(\frac{c}{1+\lambda c}\boldsymbol{\Sigma}_k^{-1} - \frac{\frac{2c^2}{(1+\lambda c)^2}\boldsymbol{\Sigma}_k^{-1}\nabla r(x^k)\nabla r(x^k)^\top \boldsymbol{\Sigma}_k^{-1}}{1+\frac{2c}{1+\lambda c}\nabla r(x^k) \boldsymbol{\Sigma}_k^{-1}\nabla r(x^k)}\right)(2r(x^k)\nabla r(x^k) + \lambda\boldsymbol{\Sigma}_kx^k)\\
    &= -\frac{2cr(x^k)}{1+\lambda c}\left(1 - \frac{\frac{2c}{1+\lambda c} \nabla r(x^k)^\top \boldsymbol{\Sigma}_k^{-1}\nabla r(x^k)}{1+\frac{2c}{1+\lambda c}\nabla r(x^k) \boldsymbol{\Sigma}_k^{-1}\nabla r(x^k)}\right)\boldsymbol{\Sigma}_k\nabla r(x^k)\\
    &\qquad -\frac{\lambda c}{1+\lambda c}x^k + \frac{\frac{2c^2\lambda}{1+\lambda c} \boldsymbol{\Sigma}_k^{-1}\nabla r(x^k)\nabla r(x^k)^\top x^k}{1+\frac{2c}{1+\lambda c}\nabla r(x^k)\boldsymbol{\Sigma}_k^{-1}\nabla r(x^k)}\\
    &= 
    - \frac{2cr(x^k)}{1+\lambda c}\frac{1}{1+\frac{2c}{1+\lambda c}\nabla r(x^k)\boldsymbol{\Sigma}_k^{-1}\nabla r(x^k)}\boldsymbol{\Sigma}_k^{-1}\nabla r(x^k)\\
    &\qquad - \frac{\lambda c}{1+\lambda c}x^k 
    + \frac{\frac{2c^2\lambda}{1+\lambda c}\boldsymbol{\Sigma}_k^{-1}\nabla r(x^k) \nabla r(x^k)^\top x^k}{1+\frac{2c}{1+\lambda c}\nabla r(x^k)\boldsymbol{\Sigma}_k^{-1}\nabla r(x^k)}.
\end{align*}
Using the connection $\nabla r(x^k) = \frac{1}{2\sqrt{f(x^k)}}\nabla f(x^k)$ and $r(x^k) = \sqrt{f(x^k)}$ we get
\begin{align*}
    p^k &= - \frac{2c\sqrt{f(x^k)}}{1+\lambda c}\frac{1}{1+\frac{2c}{4f(x^k)(1+\lambda c)} \nabla f(x^k)^\top \boldsymbol{\Sigma}_k^{-1}\nabla f(x^k)}\boldsymbol{\Sigma}_k^{-1}\frac{1}{2\sqrt{f(x^k)}}\nabla f(x^k)\\
    &\qquad  - \frac{c\lambda}{1+\lambda c}x^k + \frac{\frac{2c^2\lambda}{4f(x^k)(1+\lambda c)} \boldsymbol{\Sigma}_k^{-1}\nabla f(x^k)\nabla f(x^k)^\top x^k}{1+\frac{2c}{4(1+\lambda c)f(x^k)}\nabla f(x^k)^\top \boldsymbol{\Sigma}_k^{-1}\nabla f(x^k)}\\
    &= - \frac{\nicefrac{c}{(1+\lambda c)}}{1+\frac{c}{2f(x^k)(1+\lambda c)}\|\nabla f(x^k)\|^2_{\boldsymbol{\Sigma}_k^{-1}}}\boldsymbol{\Sigma}_k\nabla f(x^k) - \frac{c\lambda}{1+ \lambda c}x^k \\
    &\qquad + \frac{c\lambda}{1+\lambda c}\frac{\frac{c}{2f(x^k)}\nabla f(x^k)^\top x^k}{1+ \frac{c}{2f(x^k)(1+\lambda c)}\|\nabla f(x^k)\|^2_{\boldsymbol{\Sigma}_k^{-1}}}\boldsymbol{\Sigma}_k^{-1}\nabla f(x^k).
\end{align*}
\begin{algorithm}[!t]
    \caption{\algname{NGN-MDv1W}}
    \label{alg:ngn_mdw}
    \begin{algorithmic}[1]
        \State \textbf{Input:} $x^0\in\R^d,$ step-size parameter $c > 0,$ momentum parameters $\beta_1, \beta_2 \in [0,1),$ weight decay parameter $\lambda \ge 0,$ stabilization parameter $\varepsilon>0$
        \For{$k = 0,1,\dots,K-1$}
        \State Sample a batch $S_k\subseteq[n]$ and compute $f_{S_k}$ and $\nabla f_{S_k}(x^k)$
        \State Compute $v^k = \beta_2 v^{k-1} + (1-\beta_2)(\nabla f_{S_k}(x^k) \odot \nabla f_{S_k}(x^k))$
        \State Compute $\mD_k = {\rm diag}(\varepsilon\mI + \sqrt{v^k/(1-\beta_2^k)})$
        \State Compute 
        $$\gamma_k = \frac{\frac{c}{(1+\lambda c)}\left[1-\frac{c\lambda}{2f_{S_k}(x^k)} \nabla f_{S_k}(x^k)^\top x^k\right]_+}{1+\frac{c}{2f_{S_k}(x^k)(1+\lambda c)}\|\nabla f_{S_k}(x^k)\|^2_{\mD_k^{-1}}}$$
        \State Update 
        $x^{k+1} = \frac{1}{1+\lambda c}x^k - (1-\beta_1)\gamma_k\mD_k^{-1}\nabla f_{S_k}(x^k) + \beta_1(x^k-x^{k-1})$
        \EndFor
    \end{algorithmic}
    \begin{tablenotes}
     \item $[\cdot ]_+$ denotes $\max\{0,\cdot\}.$ 
    \end{tablenotes}  
\end{algorithm}
To summarize, the update of \algname{NGN-Dv1W} is the following 
\begin{align}
    x^{k+1} &= x^k + p^k\notag\\
    &= \frac{1}{1+\lambda c}x^k + \frac{c\lambda}{1+\lambda c} \frac{\frac{c}{2f(x^k)}\nabla f(x^k)^\top x^k}{1+\frac{c}{2f(x^k)(1+\lambda c)}\|\nabla f(x^k)\|^2_{\boldsymbol{\Sigma}_k^{-1}}}\boldsymbol{\Sigma}_k^{-1}\nabla f(x^k)\notag\\
    &\qquad - \frac{\nicefrac{c}{(1+\lambda c)}}{1+\frac{c}{2f(x^k)(1+\lambda c)}\|\nabla f(x^k)\|^2_{\boldsymbol{\Sigma}_k^{-1}}}\boldsymbol{\Sigma}_k^{-1}\nabla f(x^k)\notag\\
    &= \frac{1}{1+\lambda c}x^k 
    - \frac{\frac{c}{1+\lambda c}\left(1 - \frac{c\lambda}{2f(x^k)}\nabla f(x^k)^\top x^k\right)}{1+\frac{c}{2f(x^k)(1+\lambda c)}\|\nabla f(x^k)\|^2_{\boldsymbol{\Sigma}_k^{-1}}}\boldsymbol{\Sigma}_k^{-1}\nabla f(x^k).
\end{align}
To prevent the step-size next to $\boldsymbol{\Sigma}_k^{-1}\nabla f(x^k)$ from being negative, the final update has the form
\begin{align}
    x^{k+1} &= \frac{1}{1+\lambda c}x^k 
    - \frac{\frac{c}{1+\lambda c}\left[1 - \frac{c\lambda}{2f(x^k)}\nabla f(x^k)^\top x^k\right]_+}{1+\frac{c}{2f(x^k)(1+\lambda c)}\|\nabla f(x^k)\|^2_{\boldsymbol{\Sigma}_k^{-1}}}\boldsymbol{\Sigma}_k^{-1}\nabla f(x^k),
\end{align}
where $[\cdot]_+ \eqdef \max\{\cdot, 0\}$. Now we can add momentum on top and obtain the following update of \algname{NGN-MDv1W}
\begin{align}
    x^{k+1} &= \frac{1}{1+\lambda c}x^k 
    - \frac{\frac{c}{1+\lambda c}\left[1 - \frac{c\lambda}{2f(x^k)}\nabla f(x^k)^\top x^k\right]_+}{1+\frac{c}{2f(x^k)(1+\lambda c)}\|\nabla f(x^k)\|^2_{\boldsymbol{\Sigma}_k^{-1}}}\boldsymbol{\Sigma}_k^{-1}\nabla f(x^k) + \beta(x^k - x^{k-1}).
\end{align}

This combination of \algname{NGN-MDv1} and weight decay is summarized in \Cref{alg:ngn_mdw}. We highlight that now the weight decay is incorporated inside the adaptive step-size as well as regularizing the coefficient next to $x^k.$

\subsection{Empirical Validation of the Proposed Combinations}

Having two possible ways of adding weight decay to \algname{NGN-MDv1}, we test them  on pretraining a 70M transformer on language modeling. The validation perplexity at the end of training is reported in \Cref{fig:weight_decay_70m_appendix}. We note that when weight decay is turned off, both \algname{NGN-MDv1W} and \algname{Dec-NGN-MDv1} reduce to \algname{NGN-MDv1}.

First, we observe that when weight decay is properly tuned, all algorithms improve over the baseline case with no weight decay, which is consistent with the observation of \citet{xiao2024rethinkingwisdomML} and \citet{andriushchenko2023weightdecay} on \algname{AdamW}. We also note that \algname{Dec-NGN-MDv1} and \algname{NGN-MDv1W} require a smaller weight decay value compared to the other algorithms. 
Finally, the stability and performance of \algname{NGNMDv1} are preserved by both variations, allowing training with larger learning rates, and significantly improving over \algname{AdamW} and \algname{Momo-Adam}.

We do not observe a substantial difference between the two proposed modifications of \algname{NGN-MDv1} for this task. We remark however that these two versions serve substantially different purposes, and pretraining language models might not be the most representative task to evaluate the effect of adding regularization.

\begin{figure*}[t]
    \centering
    \begin{tabular}{cccc}
       \multicolumn{4}{c}{\includegraphics[width=0.7\linewidth]{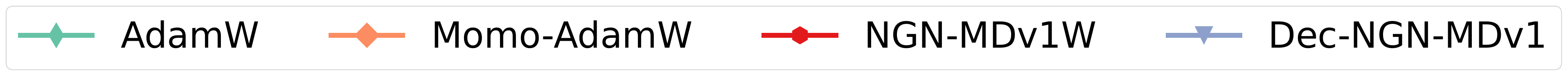} }\\
       \includegraphics[width=0.22\linewidth]{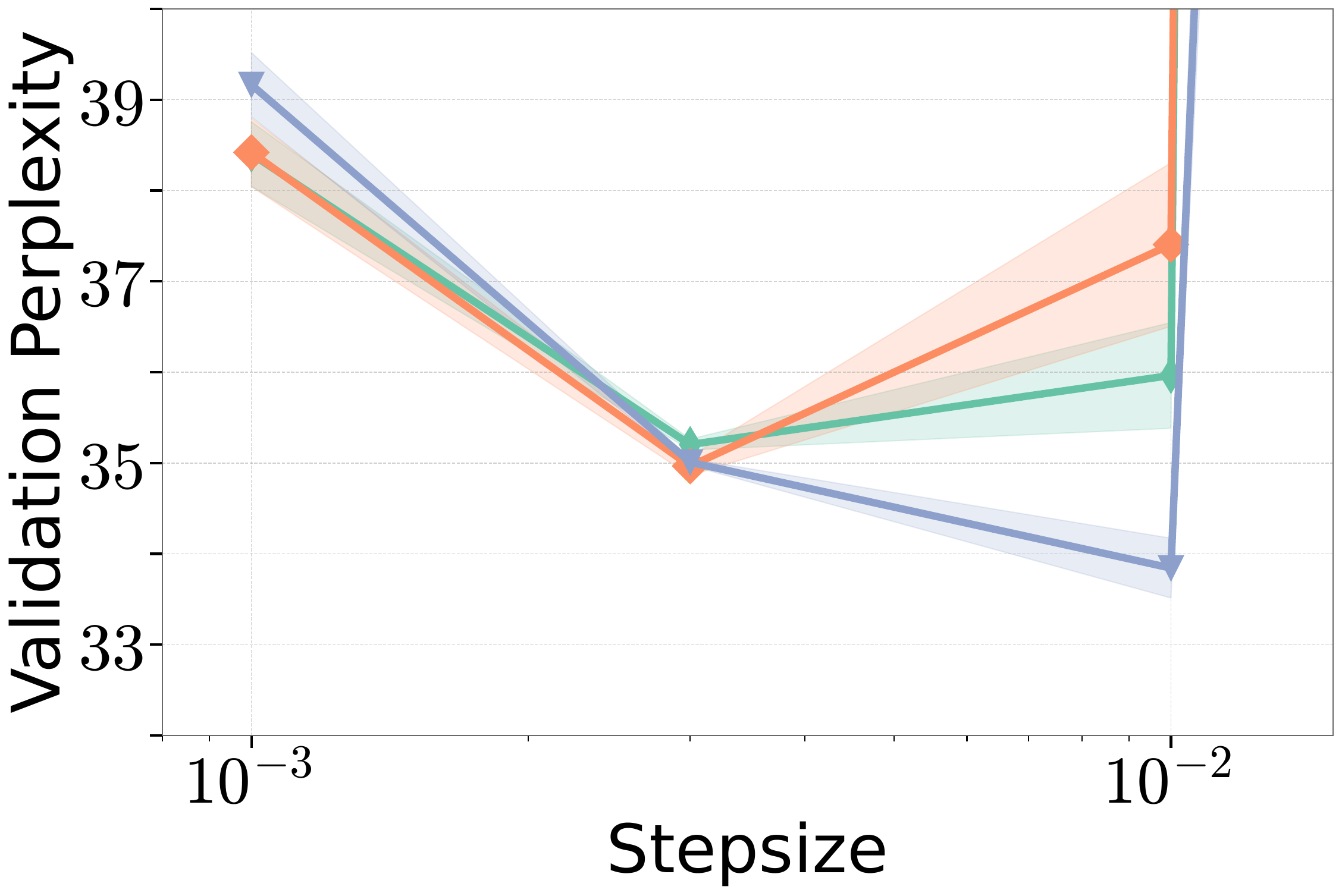}  &  
       \includegraphics[width=0.22\linewidth]{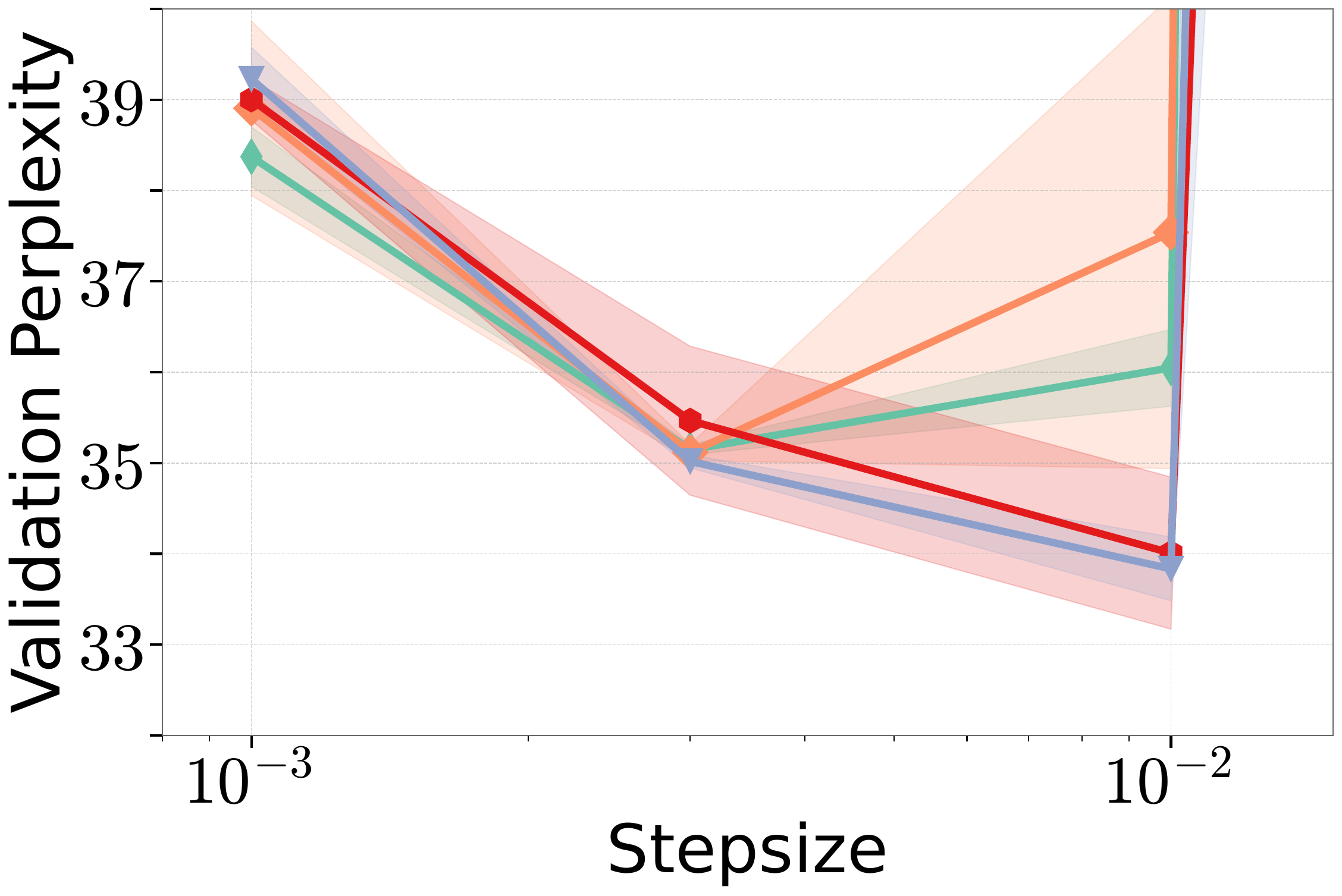} &
       \includegraphics[width=0.22\linewidth]{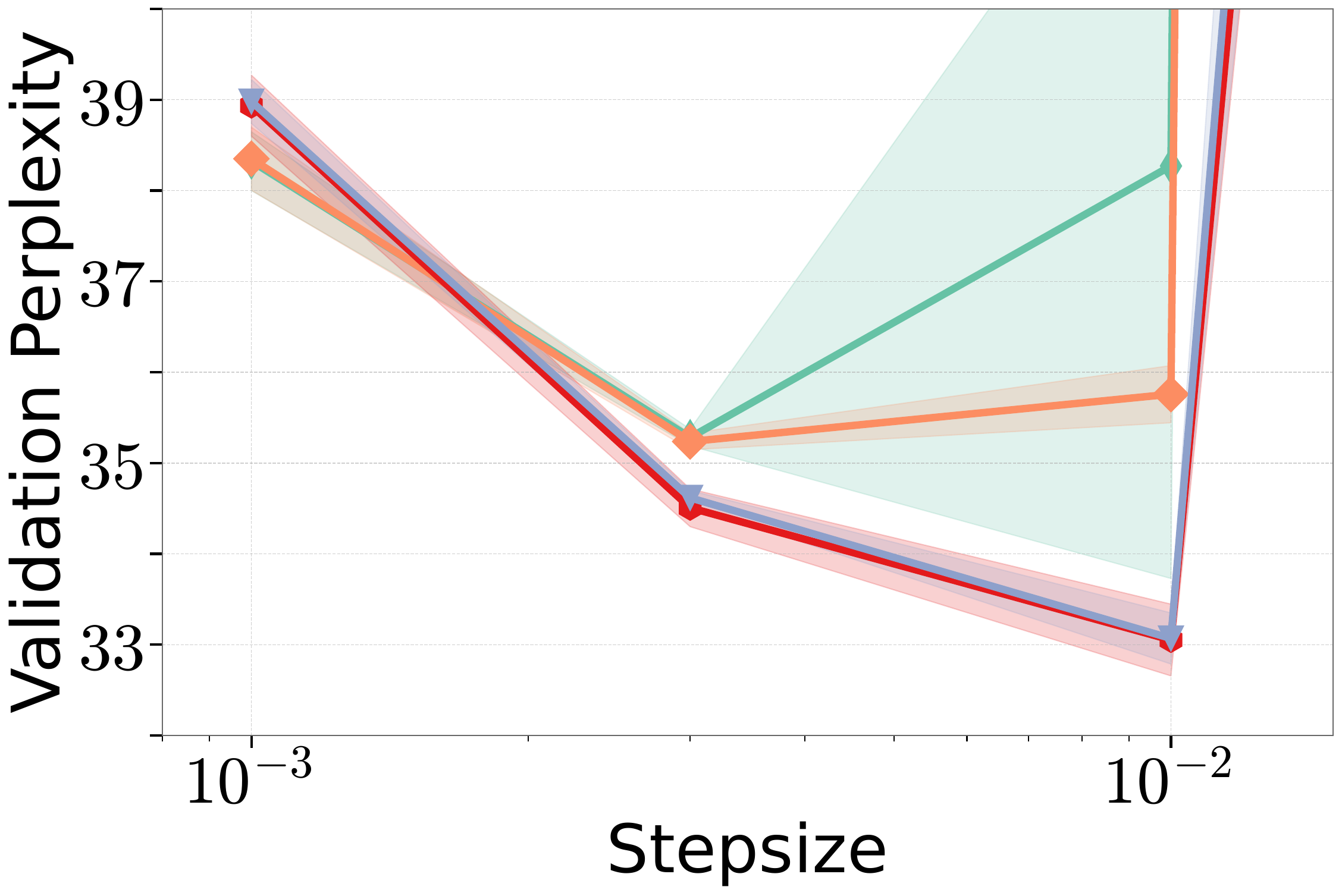} &
       \includegraphics[width=0.22\linewidth]{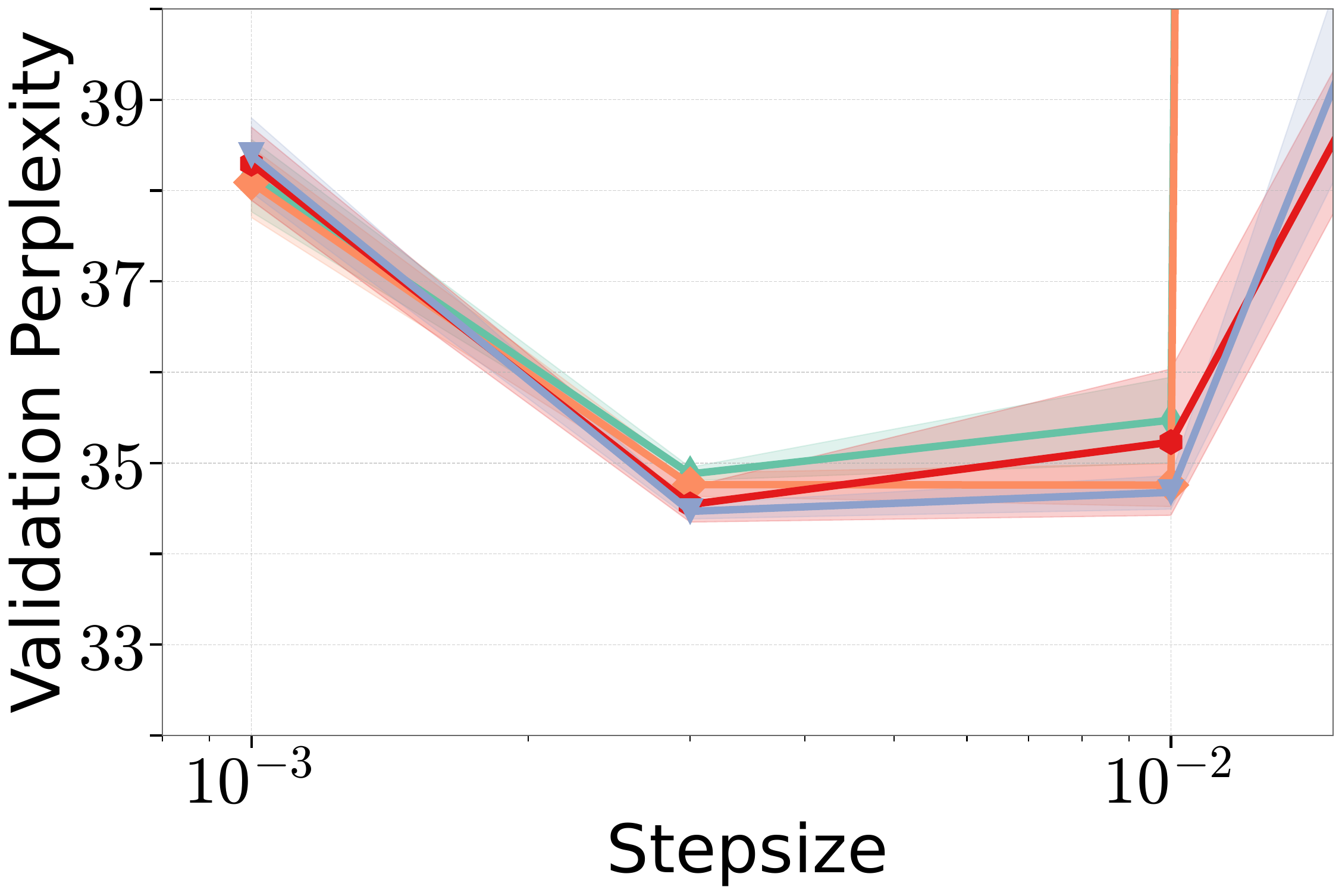} \\
       {\small $\lambda=0$} &
       {\small $\lambda=10^{-3}$} &
       {\small $\lambda=10^{-2}$} & 
       {\small $\lambda=10^{-1}$}\\
    \end{tabular}
    
    \caption{Adding weight decay when pretraining a 70M Transformer$++$. When properly tuned, a value of weight decay $>0$ enhances the performance of all algorithms. \algname{NGN-MDv1} retains his characteristic stability, and achieves smaller perplexity in all scenarios.}
    \label{fig:weight_decay_70m_appendix}
\end{figure*}

\newpage

\section{Additional Experiments on Toy Problems}\label{sec:toy_exp}

\subsection{Additional Experiments on the Problem with Many Minima}\label{sec:synth_function_appendix}

Now, we provide a simple example of minimizing a function 
\begin{equation}\label{eq:synth_function}
    f(x) = (\sin(1+\cos(-\pi+x))-0.2x)^2+(\sin(1+\cos(\pi-x))+0.2x)^4
\end{equation} that has many sharp sub-optimal local and flat global minima. We compare the performance of \algname{NGN-M} and \algname{SGDM} varying the step-size hyperparameter in $\{10^{0}, 10^{1}, 10^2, 10^3\}$ and the starting point in $[-20, 20]$ with a step $\nicefrac{4}{30}$\footnote{This step is chosen small enough so that the initial point can be close to any local minima within $[-20,20].$}. Based on the results in \Cref{fig:synth_function_appendix} (right), we conclude that $(i)$ for small step-sizes, both methods likely get stuck at sub-optimal local minima and reach the global minima only if they are initialized close enough to it; $(ii)$ for large step-sizes, we observe less runs of \algname{SGDM} reaching the global minima; $(iii)$ in contrast, for \algname{NGN-M} with large step-sizes, we observe more runs reaching the global minima. This is possible due to the adaptive nature of the \algname{NGN} step-size that forces \algname{NGN-M} to converge to the flatness of the global minima.

\begin{figure}[h]
    \centering
    \begin{tabular}{cc}
    \includegraphics[width=0.3\linewidth]{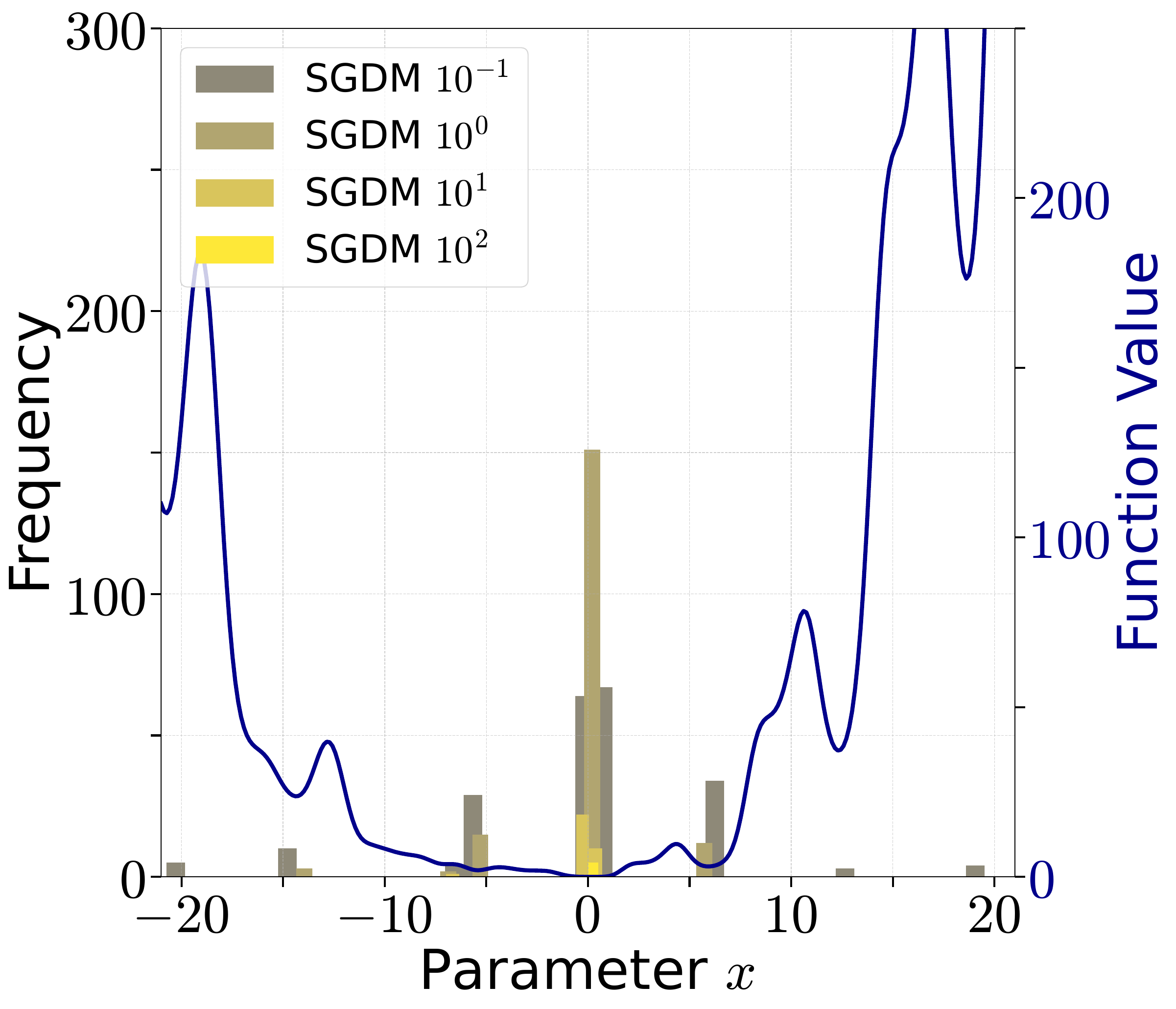} &
    \includegraphics[width=0.3\linewidth]{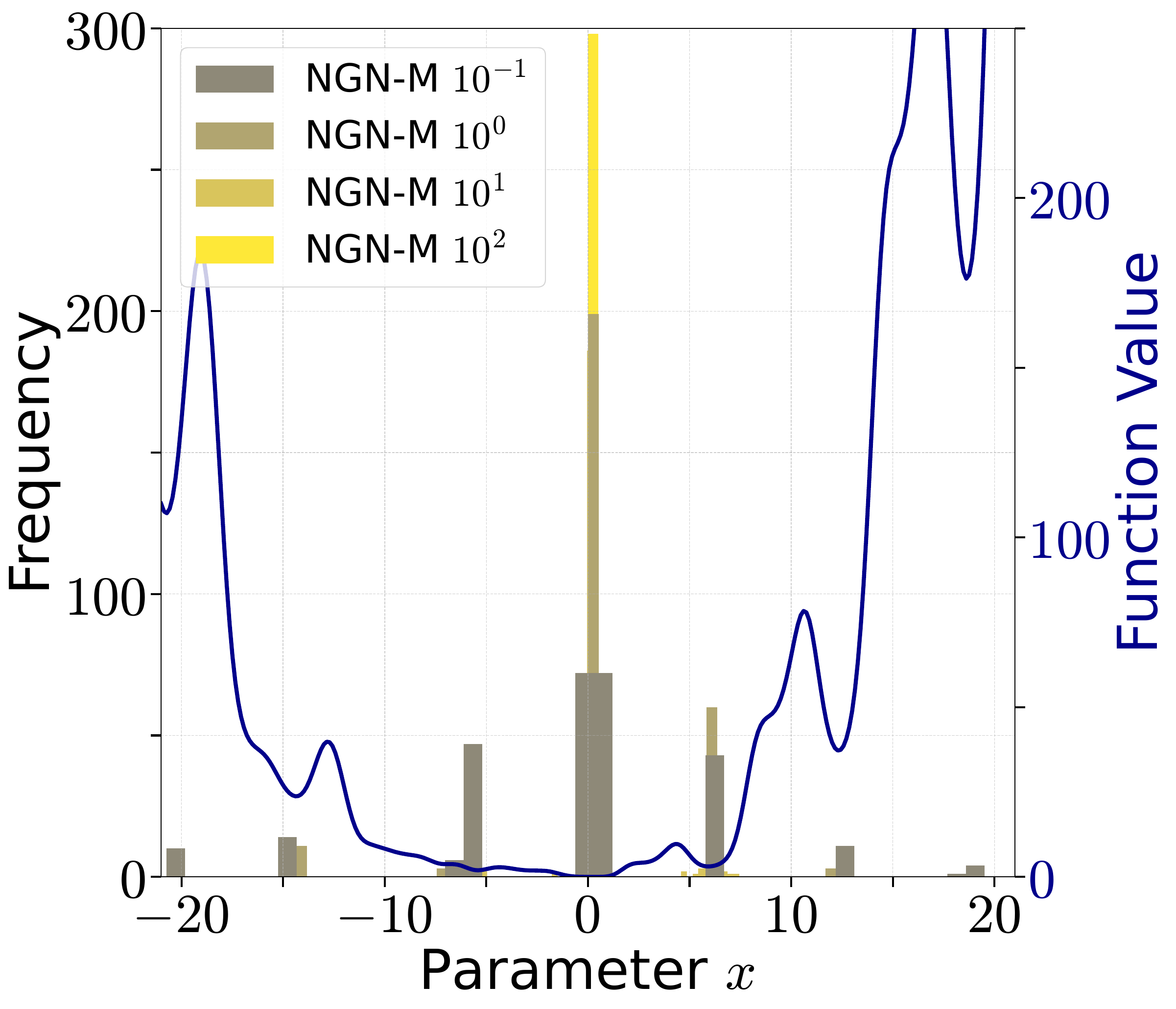} 
    \end{tabular}
    
    \caption{Comparison of \algname{SGDM} and \algname{NGN-M} when minimizing function in \eqref{eq:synth_function}.}
    \label{fig:synth_function_appendix}
\end{figure}

\subsection{Comparison on Rosenbrock Function}

Now we present the results where we compare \algname{NGN-M} and \algname{SGDM} when minimizing the Rosenbrock function. We report the trajectories of optimizers and training dynamics in \Cref{fig:rosenbrock_trajectories} and \Cref{fig:rosenbrock_convergence}. 

We observe that \algname{NGN-M} converges for all values of $c$, indicating its high resilience to the choice of step-size hyperparameter. In contrast, \algname{SGDM} already diverges for the step-size hyperparameter $10^{-2}$. This can be explained by the adaptive nature of \algname{NGN} step-size, which decreases the effective step-size of \algname{NGN-M} for a more stable convergence. This is especially evident from the trajectories of algorithms. Indeed, \algname{NGN-M} effectively moves in the complex valley of the Rosenbrock function, adapting to the local curvature.

\begin{figure}[h]
 \centering

         \includegraphics[width=1\linewidth]{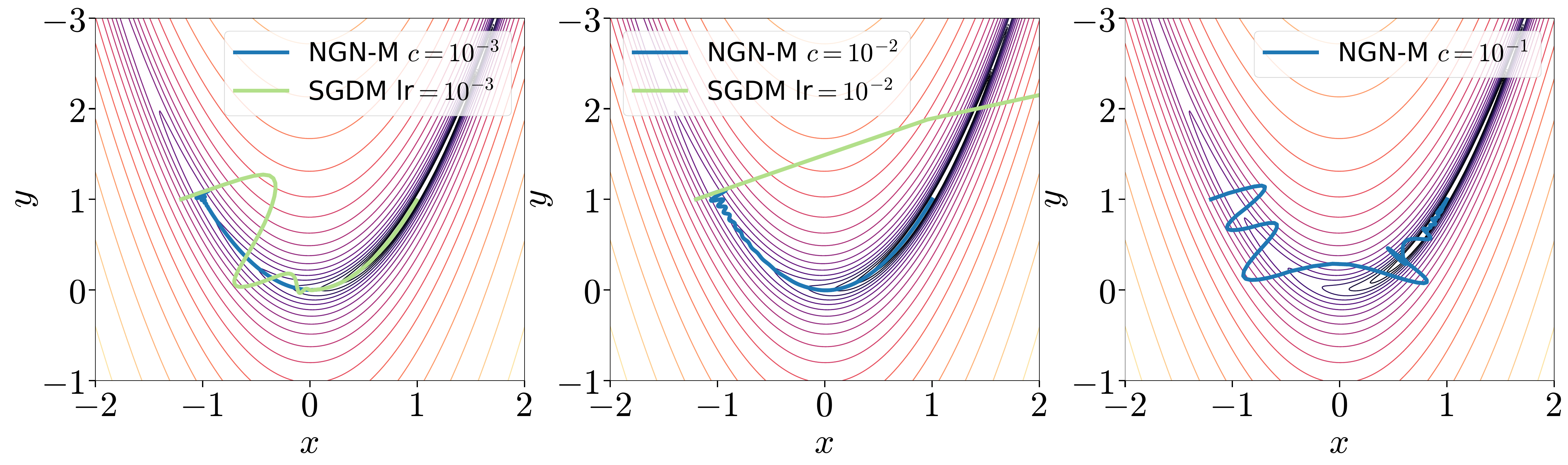}

    \caption{Trajectories of \algname{NGN-M} and \algname{SGDM} when minimizing the Rosenbrock function and varying the step-size hyperparameter.}
    \label{fig:rosenbrock_trajectories}
\end{figure}

\begin{figure*}[h]
 \centering
        \begin{tabular}{c}
            \includegraphics[width=0.8\linewidth]{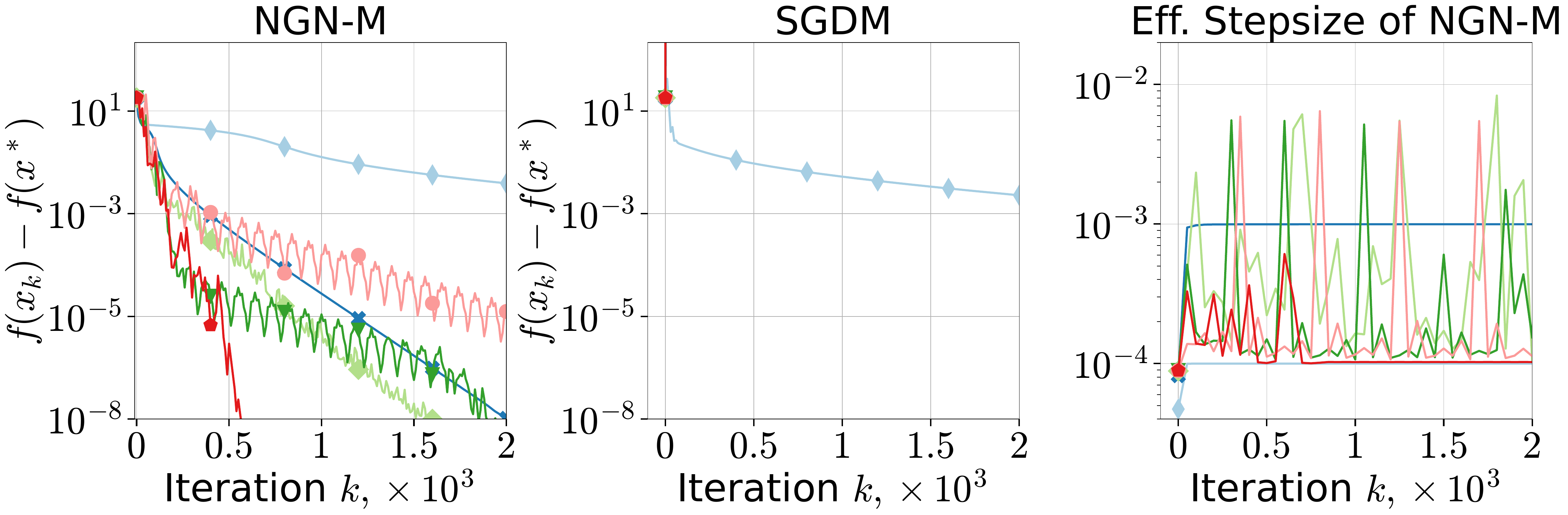} \\
             \includegraphics[width=0.6\linewidth]{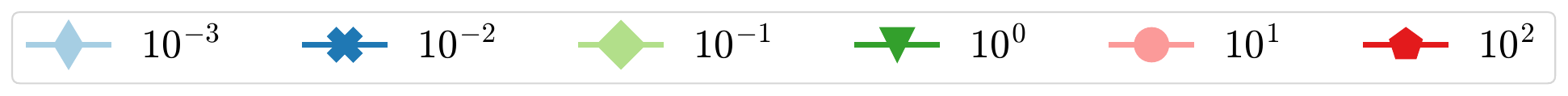}
        \end{tabular}

    \caption{Training dynamics of \algname{NGN-M} and \algname{SGDM} when minimizing the Rosenbrock function and varying the step-size hyperparameter.}
    \label{fig:rosenbrock_convergence}
\end{figure*}

\subsection{Comparison on Quadratic Function with Theoretical Step-size}\label{sec:theoretical_stepsize}

Next, we run \algname{NGN-M} with theoretical choice of step-size hyperparameter $c=1/\sqrt{K}$ and $c_k=1/\sqrt{k}$ (see \Cref{th:theorem_ngn_m_convex} and \Cref{th:ngnm_decaying} for more details) against fixed choices $c\in\{10^{-3},10^{-4}\}$. The comparison is made on quadratic function $f(x) = \frac{1}{2}\|(\mA + r\mI)x- y\|^2$, where $\mA\in\R^{400\times 400}$ and $y\in\R^{400}$ are sampled from standard normal distribution. The constant $r$ controls the condition number of the problem. 

We test the performance of \algname{NGN-M} varying the condition number of the problem and the number of iterations; see \Cref{fig:quadratic_convergence}. We observe that in all the cases, the choice $1/\sqrt{k}$ leads to faster convergence, supporting our theoretical claims. The choice $1/\sqrt{K}$ demonstrates competitive performance as well, but it is slightly pessimistic at the beginning of training. In contrast, the choice $c\in\{10^{-3},10^{-4}\}$, which is a default value in practice, is too small and does not lead to fast convergence.

These experiments demonstrate that when the problem satisfies all assumptions needed in the analysis, the choice of the step-size hyperparameter $c$ given by the convergence theorems is a good starting point in practice and can serve as a baseline when tuning $c.$

\begin{figure*}[h]
 \centering
     \includegraphics[width=0.8\linewidth]{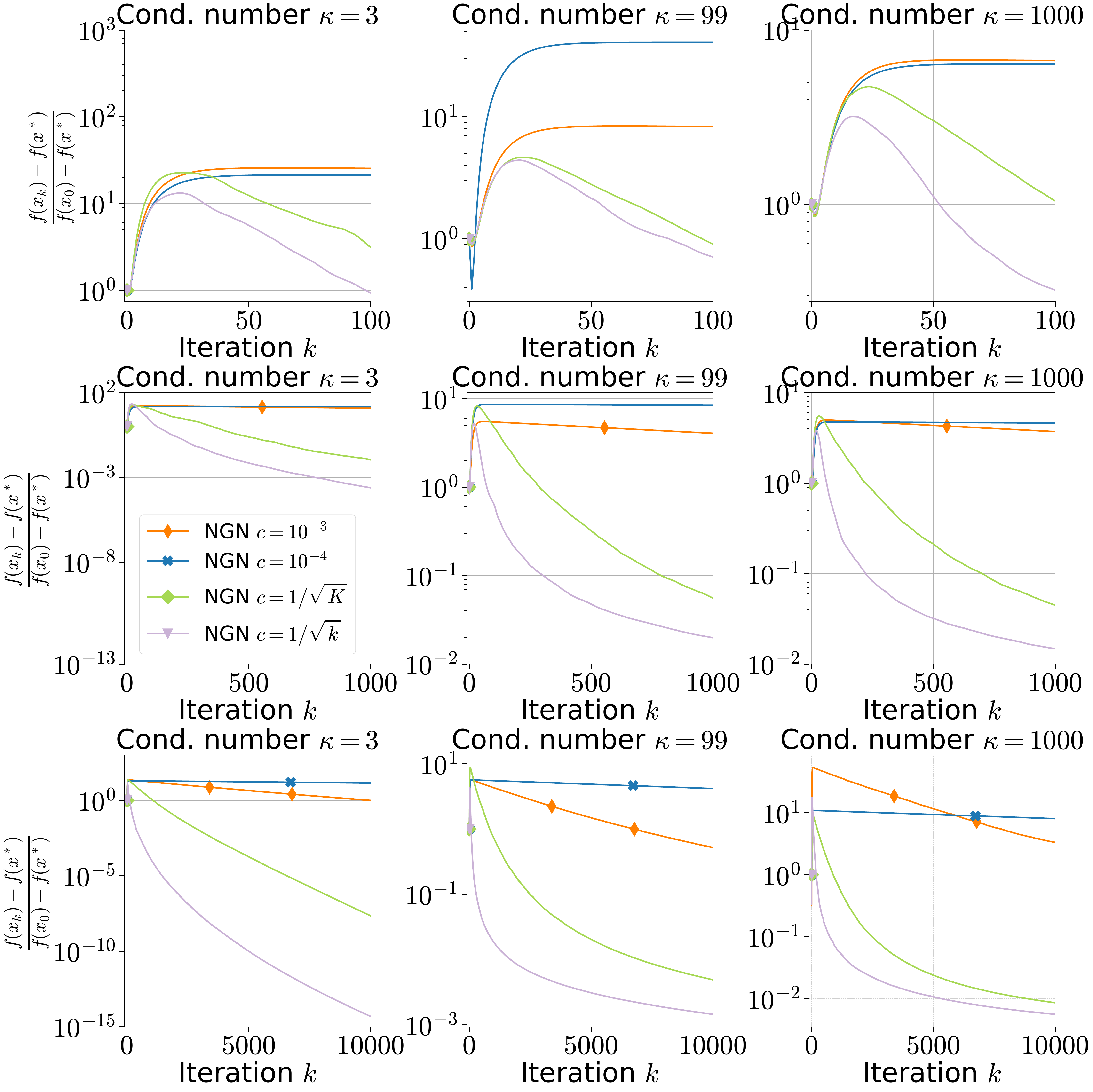}

    \caption{Training dynamics of \algname{NGN-M} with several choices of the step-size hyperparameter varying the condition number of the quadratic problem.}
    \label{fig:quadratic_convergence}
\end{figure*}

\newpage

\section{Additional Experiments and Training Details}\label{sec:exp_appendix}

\subsection{Training Details}

The detailed experiment setup with hyperparameters and training details is presented in \Cref{tab:summary_experiments}. We provide links to the exact model architectures used in our experiments (the links are clickable) as well as links to the tables and figures for each workload. We demonstrate the results averaged across $3$ different random seeds for small and middle-range size experiments. We use standard values of momentum parameters $(\beta_1,\beta_2) = (0.9, 0.999)$ if the opposite is not specified. The step-size hyperparameter is tuned across powers of $10$ (for some workloads we add additional values of the step-size hyperparameter shown in the step-size resilience plots). We use PyTorch \citep{paszke2017automatic} implementation of \algname{Adam}. The implementation of \algname{MomSPS}, \algname{Momo}, \algname{Momo-Adam} are provided in the corresponding papers. Finally, when employing \algname{SGD-M}, we set dampening equal to $0.9$.

For vision transformers experiments, we follow the setup of \citet{schaipp2024momo}, and use Pytorch Image Models codebase \citep{rw2019timm}. We train a \texttt{vit\_tiny\_patch16\_224} for 200 epochs on Imagenet1k, using a cosine learning rate schedule with a linear warmup of 5 epochs. Differently than \citet{schaipp2024momo}, we train in \texttt{bfloat16}, instead of \texttt{float16}, and do not employ weight decay regularization.

For pre-training Transformers on Causal Language Modeling, we build upon the nanoGPT \citep{karpathy2022nanogpt} implementation, augmenting it with Rotational Positional Embedding \citep{su2023roformerRoPE}, RMSNorm \citep{zhang2019rmsnorm}, and SwiGLU \citep{shazeer2020glu}. We call this enhanced version Transformer$++$. Models are trained with a batch size of $256$, context length of $2048$ tokens, vocabolary size of 50280 and make use of GPT-Neox tokenizer \citep{black2022gptneox20b}. We adopt an enhanced training recipe, made popular by large language models such as LLaMa \citep{touvron2023llama}. These modifications include: training in \texttt{bfloat16}; employing a linear learning rate warm-up for 10\% of the training steps, followed by cosine annealing to $10^{-5}$; omitting biases from linear layers; using $(\beta_1,\beta_2) = (0.9, 0.95)$ for all algorithms; clipping gradient norms above $1$; no weight tying between embedding and last linear layer. All models are trained on SlimPajama-627B \citep{cerebras2023slimpajama}, a cleaned and deduplicated version of RedPajama We report validation perplexity on a separate subset of Slim-Pajama consisting of $10$M tokens.
The total compute is estimated following \citet{kaplan2020scaling}, where the estimated number of floating-point operations (FLOPs) is 6 $\times$ Number of Parameters $\times$ Number of Tokens.

Experiments of small and middle size are performed on 1xRTX 4090. We perform ImageNet32 experiments on 2xA100-40GB, and ImageNet1k experiments on 4xA100-SXM4-40GB. For pretraining Transformers on Language Modeling, we employ 8xH100-HBM3-80GB GPUs. With multiple devices in use, we employ Distributed Data Parallel to parallelize the training process. 

\begin{table*}[t]
    \centering
    \caption{Summary of experiment setup with all the details on hyperparameters used in each case.}
    \label{tab:summary_experiments}
    \resizebox{\textwidth}{!}{
        \begin{tabular}{cccccccc}
            \toprule
            {\bf Model} &
            {\bf Dataset} & 
            \makecellnew{{\bf Performance} \\ {\bf Results}} &
            \makecellnew{{\bf Stability} \\ {\bf Results} } &
            \makecellnew{{\bf Effective} \\ {\bf Stepsize} \\ {\bf Results}}& 
            \makecellnew{{\bf Epochs} /\\ {\bf Iterations}}& 
            \makecellnew{{\bf Batch } \\ {\bf Size}}& 
            {\bf Comments}  
            
            \\ \toprule
            \href{https://github.com/akamaster/pytorch_resnet_cifar10/blob/master/resnet.py}{Resnet20} &
            CIFAR10 & 
            Tab.  
            \ref{tab:empirical_comparison_momentum_appendix}, \ref{tab:empirical_comparison_adam_type_lion_adabound_adabelief}, \ref{tab:empirical_comparison_cd_type}
            
             &
            \makecellnew{Fig. \ref{fig:stability_all_type_test_acc}, 
            
            \ref{fig:stability_momentum_train_loss}, 
            \ref{fig:stability_adam_type_train_loss}, 
            \ref{fig:also_adabound_adabelief_lion_resnet20}
            }
            & 
            Fig. \ref{fig:stepsize_momentum_appendix}, \ref{fig:stepsize_adam_type_appendix}, \ref{fig:rebuttals_resnet20_stepsize}  &
            50 & 
            128 & 
            \\
            \midrule

            \href{https://github.com/akamaster/pytorch_resnet_cifar10/blob/master/resnet.py}{Resnet110} &
            CIFAR100 & 
            
            Tab. \ref{tab:empirical_comparison_momentum_appendix},  \ref{tab:empirical_comparison_adam_type_lion_adabound_adabelief} &
            
            \makecellnew{Fig. \ref{fig:stability_all_type_test_acc},

            \ref{fig:stability_momentum_train_loss},
            \ref{fig:stability_adam_type_train_loss},
            \ref{fig:also_adabound_adabelief_lion_resnet20} 
            }&
            &
            100 &
            128 &

            \\ \midrule

            \href{https://github.com/chengyangfu/pytorch-vgg-cifar10/blob/master/vgg.py}{VGG16} & 
            CIFAR10 &
            Tab. \ref{tab:empirical_comparison_momentum_appendix}, \ref{tab:empirical_comparison_adam_type_lion_adabound_adabelief} &
            
            Fig.  \ref{fig:stability_momentum_train_loss}, \ref{fig:stability_adam_type_train_loss} &
            &
            50 &
            128 &

            \\ \midrule
            
            \href{https://github.com/fabian-sp/step-back/blob/main/stepback/models/basic_models.py}{MLP} &
            MNIST &
            Tab. \ref{tab:empirical_comparison_momentum_appendix}, \ref{tab:empirical_comparison_adam_type_lion_adabound_adabelief}  &
            Fig.
            \ref{fig:stability_momentum_train_loss}, 
            \ref{fig:stability_adam_type_train_loss_additional_workloads} &
            &
            10 &
            128 &
            \makecellnew{2 hidden layers \\ of size 100}

            \\ \midrule

            \href{https://github.com/lucidrains/vit-pytorch}{ViT} &
            CIFAR10 &
            Tab. \ref{tab:empirical_comparison_momentum_appendix}, \ref{tab:empirical_comparison_adam_type_lion_adabound_adabelief} &
            \makecellnew{Fig. \ref{fig:stability_all_type_test_acc}, \ref{fig:stability_momentum_train_loss}, 
            \ref{fig:stability_adam_type_train_loss}, 
            \ref{fig:also_adabound_adabelief_lion_resnet20}
            }&
            Fig. \ref{fig:stepsize}, \ref{fig:stepsize_momentum_appendix}, \ref{fig:stepsize_adam_type_appendix}, \ref{fig:rebuttals_vit_stepsize} &
            200 &
            512 &

            \\ \midrule

            \href{https://github.com/fhueb/parameter-agnostic-lzlo/tree/main/model}{LSTM} &
            PTB &
            Tab. \ref{tab:empirical_comparison_adam_type_lion_adabound_adabelief}, \ref{tab:empirical_comparison_cd_type} &
            
            Fig. \ref{fig:stability_adam_type_train_loss_additional_workloads} &
             &
            150 &
            20 &
            \# layers 3
            \\ \midrule

            \href{https://github.com/fhueb/parameter-agnostic-lzlo/tree/main/model}{LSTM} &
            Wikitext-2 &
            Tab. \ref{tab:empirical_comparison_adam_type_lion_adabound_adabelief}, \ref{tab:empirical_comparison_cd_type} &
            Fig. \ref{fig:rebuttals_nlp_tasks} &
            &
            150 &
            20 &
            \# layers 3
            \\ \midrule

            \href{https://github.com/karpathy/ng-video-lecture/blob/52201428ed7b46804849dea0b3ccf0de9df1a5c3/bigram.py}{Transformer} &
            \makecellnew{Rotten \\ Tomatoes} &
            Tab. \ref{tab:empirical_comparison_adam_type_lion_adabound_adabelief}, \ref{tab:empirical_comparison_cd_type} &
            Tab. \ref{fig:rebuttals_nlp_tasks} &
             &
            2000 &
            16 &
            \makecellnew{\# heads 8 \\ \# layers 24}
            \\ \midrule

            \href{https://github.com/karpathy/ng-video-lecture/blob/52201428ed7b46804849dea0b3ccf0de9df1a5c3/bigram.py}{Transformer} &
            \makecellnew{Tiny \\ Shakespeare} &
            Tab. \ref{tab:empirical_comparison_adam_type_lion_adabound_adabelief}, \ref{tab:empirical_comparison_cd_type}&
            Fig. \ref{fig:stability_adam_type_train_loss_additional_workloads}, \ref{fig:rebuttals_nlp_tasks} &
             &
            2000 &
            16 &
            \makecellnew{\# heads 8 \\ \# layers 24}
            \\ \midrule

            \href{https://github.com/kuangliu/pytorch-cifar/blob/master/models/resnet.py}{Resnet18}&
            ImageNet32 &
            Tab. \ref{tab:empirical_comparison_momentum_appendix}, \ref{tab:empirical_comparison_adam_type_lion_adabound_adabelief},  &
            Fig. \ref{fig:stability_imagenet_train_loss} &
             &
            45 &
            128 &
            \makecellnew{constant learning rate \\ schedule; no weight decay }
            \\ \midrule

            \href{https://pytorch.org/vision/main/models/generated/torchvision.models.resnet18.html}{Resnet18} &
            ImageNet1k &
            Tab. \ref{tab:empirical_comparison_momentum_appendix}, \ref{tab:empirical_comparison_adam_type_lion_adabound_adabelief} &
            Fig. \ref{fig:stability_all_type_test_acc}, \ref{fig:stability_imagenet_train_loss} &
             &
            90 &
            256 &
            \makecellnew{learning rate decay every \\$30$ epochs by $0.1$ \\
            no weight decay }
             
            \\ \midrule

             

            \href{https://github.com/huggingface/pytorch-image-models/blob/e3242a52584bbc69f848f762d254e8a23932832c/timm/models/vision_transformer.py\#L2071}{ViT-Tiny} &
            ImageNet1k &
            Tab. \ref{tab:empirical_comparison_adam_type_lion_adabound_adabelief}&
            Fig. \ref{fig:stability_all_type_imagenet_test_acc} &
            &
            200 &
            512 &
             \makecellnew{cosine learning rate \\ schedule with linear warm-up \\ for $5$ epochs \\ no weight decay, \texttt{bfloat16}}
             \\ \midrule

            \href{https://github.com/karpathy/nanoGPT}{70M Transformer$++$} &
            SlimPajama-627B &
            Tab. \ref{tab:empirical_comparison_adam_type_lion_adabound_adabelief}, \ref{tab:train_time} &
            Fig. \ref{fig:llm_results}, \ref{fig:weight_decay_70m_appendix},  \ref{fig:llm_training_dynamics} & 
            &
            2400 &
            256 &
            \makecellnew{dim=512, \# heads 8 \\ \# layers 6, context length $2048$ \\ $(\beta_1,\beta_2)=(0.9, 0.95)$, \texttt{bfloat16} \\
            clipping norm $1$, linear warm-up \\for $10\%$  of iterations}
            
            \\ \midrule

            \href{https://github.com/karpathy/nanoGPT}{70M Transformer$++$} &
            FineWeb &
            Tab. \ref{tab:ngnmdv1_ablation_beta1}, \ref{tab:adam_ablation_beta1}, \ref{tab:ngnmdv1_ablation_beta2}, \ref{tab:adam_ablation_beta2} &
            & 
            &
            4800 &
            128 &
            \makecellnew{dim=512, \# heads 8 \\ \# layers 6, context length $2048$ \\ $(\beta_1,\beta_2)=(0.9, 0.95)$, \texttt{bfloat16} \\
            clipping norm $1$, linear warm-up \\for $10\%$  of iterations}
            
            \\ \midrule

            \href{https://github.com/karpathy/nanoGPT}{160M Transformer$++$} &
            SlimPajama-627B &
            Tab. \ref{tab:empirical_comparison_adam_type_lion_adabound_adabelief}, \ref{tab:train_time} &
            Fig. \ref{fig:llm_results},  \ref{fig:llm_training_dynamics} &
            Fig. \ref{fig:update_magnitude_lr0003}, \ref{fig:update_magnitude_lr001}, \ref{fig:update_magnitude_lr003} &
            4800 &
            256 &
            \makecellnew{dim=768, \# heads 12 \\ \# layers 12, context length $2048$ \\ $(\beta_1,\beta_2)=(0.9, 0.95)$, \texttt{bfloat16} \\ 
            clipping norm $1$, linear warm-up \\ for $10\%$ of iterations}
             
            \\ \midrule

            \href{https://github.com/karpathy/nanoGPT}{410M Transformer$++$} &
            SlimPajama-627B &
            Tab. \ref{tab:empirical_comparison_adam_type_lion_adabound_adabelief}, \ref{tab:train_time} &
            Fig. \ref{fig:llm_results}, \ref{fig:llm_training_dynamics} &
            &
            13500 &
            256 &
            \makecellnew{dim=1024, \# heads 16 \\ \# layers 24, context length $2048$ \\ $(\beta_1,\beta_2)=(0.9, 0.95)$, \texttt{bfloat16} \\ 
            clipping norm $1$, linear warm-up \\ for $10\%$ of iterations}

            \\ \midrule

            \href{https://github.com/karpathy/nanoGPT}{1B Transformer$++$} &
            SlimPajama-627B &
            Tab. \ref{tab:empirical_comparison_adam_type_lion_adabound_adabelief} &
            Fig. \ref{fig:llm_results}, \ref{fig:llm_training_dynamics} &
            &
            13500 &
            256 &
            \makecellnew{dim=2048, \# heads 8 \\ \# layers 16, context length $2048$ \\ $(\beta_1,\beta_2)=(0.9, 0.95)$, \texttt{bfloat16} \\ 
            clipping norm $1$, linear warm-up \\ for $10\%$ of iterations}
             
            \\
            \bottomrule 
        
        \end{tabular}
        }

\end{table*}

\begin{table*}[t]
    \centering
    \caption{The best validation score (with one standard deviation across $3$ runs; accuracy for computer vision tasks; perplexity for NLP tasks) for the best learning rate choice for each method that supports momentum.}
    \label{tab:empirical_comparison_momentum_appendix}
    \resizebox{\textwidth}{!}{
        \begin{tabular}{cccccccc}
            \toprule
            {\bf Model} & {\bf Dataset} & \algname{NGN} & \algname{SGDM} & \algname{NGN-M} & \algname{MomSPS} & \algname{Momo} & \algname{ALR-SMAG}
            \\ \toprule
            Resnet20 &
            CIFAR10 &
            $88.30_{\pm 0.20}$ & 
            $85.42_{\pm 0.70}$ &
            $88.76_{\pm 0.05}$ & 
            $87.20_{\pm 0.38}$ & 
            $88.86_{\pm 0.14}$ & 
            $88.88_{\pm 0.19}$ \\
            \midrule

            Resnet110 &
            CIFAR100 & 
            $64.76_{\pm 0.26}$ &
            $57.16_{\pm 2.06}$ &
            $64.98_{\pm 0.29}$ &
            $63.37_{\pm 0.71}$ &
            $64.81_{\pm 0.33}$ &
            $64.73_{\pm 1.81}$

            \\ \midrule

            VGG16 & 
            CIFAR10 &
            $90.21_{\pm 0.10}$ &
            $89.67_{\pm 0.43}$ &
            $90.42_{\pm 0.06}$ &
            $87.26_{\pm 0.21}$ &
            $90.43_{\pm 0.17}$ &
            $90.49_{\pm 0.35}$ 

            \\ \midrule
            
            MLP &
            MNIST &

            $98.04_{\pm 0.07}$ &
            $97.63_{\pm 0.10}$ &
            $97.97_{\pm 0.08}$ &
            $97.73_{\pm 0.09}$ &
            $97.97_{\pm 0.04}$ &
            $97.64_{\pm 0.06}$ 

            \\ \midrule

            ViT &
            CIFAR10 &
            $83.34_{\pm 0.24}$ &
            $83.74_{\pm 0.11}$ &
            $84.95_{\pm 0.29}$ &
            $83.77_{\pm 0.27}$ &
            $85.47_{\pm 0.27}$ &
            $85.54_{\pm 0.39}$
            \\

            \midrule

            Resnet18 &
            ImageNet32 &
            $48.63$ &
            $48.56$ &
            $48.29$ &
            N/A &
            $48.68$ &
            N/A 
            \\ \midrule

            Resnet18 &
            ImageNet1k &
            $67.00$ &
            $66.73$ &
            $67.12$ &
            N/A &
            $67.09$ &
            N/A
            \\ \midrule

            Transformer &
            \makecellnew{Tiny \\ Shakespeare} &
            $9.27_{\pm 0.19}$ &
            $8.73_{\pm 0.13}$ &
            $7.67_{\pm 0.12}$ &
            N/A &
            $8.80_{\pm 0.19}$ &
            N/A \\ \midrule 

            Transformer &
            \makecellnew{Rotten \\ Tomatoes} &
            $9.01_{\pm 0.22}$ &
            $8.75_{\pm 0.04}$ &
            $7.12_{\pm 0.03}$ &
            N/A &
            $8.65_{\pm 0.03}$ &
            N/A \\ \midrule 

            LSTM &
            Wikitext-2 & 
            $75.33_{\pm 0.15}$ &
            $82.07_{\pm 0.16}$ &
            $75.51_{\pm 0.22}$ &
            N/A &
            $76.09_{\pm 0.40}$ &
            N/A \\
            \bottomrule 
        
        \end{tabular}
        }

\end{table*}

\begin{table*}[t]
    \centering
    \caption{The best validation score (with one standard deviation; accuracy for computer vision tasks; perplexity for NLP tasks) for the best learning rate choice for each method that supports diagonal step-sizes and momentum.}
    \label{tab:empirical_comparison_adam_type_lion_adabound_adabelief}
    \resizebox{\textwidth}{!}{
        \begin{tabular}{ccccccccc}
            \toprule
            {\bf Model} & {\bf Dataset} & \algname{Adam} & \algname{Momo-Adam} & \algname{NGN-MDv1} & \algname{NGN-MDv2} &
            \algname{Lion} &
            \algname{Adabelief} &
            \algname{Adabound}
            \\ \toprule
            Resnet20 &
            CIFAR10 &
            $86.96_{\pm 0.70}$ &
            $89.41_{\pm 0.36}$ &
            $89.53_{\pm 0.11}$ &
            $87.80_{\pm 0.16}$ &
            $88.09_{\pm 0.27}$ &
            $87.47_{\pm 0.48}$ &
            $85.00_{\pm 0.56}$
            \\ \midrule

            Resnet110 &
            CIFAR100 & 
            $64.12_{\pm 0.94}$ &
            $67.10_{\pm 0.53}$ &
            $66.10_{\pm 0.45}$ &
            $64.33_{\pm 0.40}$ &
            $61.85_{\pm 0.77}$ &
            $65.32_{\pm 0.43}$ &
            $61.28_{\pm 0.39}$
            \\ \midrule

            VGG16 & 
            CIFAR10 &
            $90.26_{\pm 0.23}$ &
            $90.95_{\pm 0.28}$ &
            $90.64_{\pm 0.18}$ &
            $90.07_{\pm 0.37}$ &
            N/A &
            N/A &
            N/A

            \\ \midrule
            
            MLP &
            MNIST &
            $97.44_{\pm 0.19}$ &
            $97.96_{\pm 0.10}$ &
            $98.10_{\pm 0.06}$ &
            $97.67_{\pm 0.17}$ &
            N/A &
            N/A &
            N/A

            \\ \midrule

            ViT &
            CIFAR10 &
            $85.96_{\pm 0.23}$ &
            $85.74_{\pm 0.12}$ &
            $85.65_{\pm 0.10}$ &
            $86.56_{\pm 0.11}$ &
            $86.89_{\pm  0.19}$&
            $85.05_{\pm  0.47}$&
            $80.32_{\pm 0.47}$
            
            \\ \midrule

            Transformer &
            \makecellnew{Rotten \\ Tomatoes }&
            $6.80_{\pm 0.07}$ &
            $6.81_{\pm 0.05}$ &
            $6.90_{\pm 0.05}$ &
            $6.83_{\pm 0.05}$ &
            N/A &
            N/A &
            N/A
            \\ \midrule

            Transformer &

            \makecellnew{Tiny \\ Shakespeare} &
            $6.80_{\pm 0.06}$ &
            $6.80_{\pm 0.05}$ &
            $6.89_{\pm 0.06}$ &
            $6.82_{\pm 0.05}$ &
            N/A &
            N/A &
            N/A
            
             \\ \midrule
             
            LSTM &
            PTB &
            $70.95_{\pm 0.08}$ &
            $71.09_{\pm 0.05}$ &
            $70.84_{\pm 0.20}$ &
            $71.37_{\pm 0.17}$&
            N/A &
            N/A &
            N/A
            \\ \midrule

            LSTM &
            Wikitext-2 &
            $81.49_{\pm 1.49}$&
            $82.23_{\pm 0.64}$ &
            $75.24_{\pm 0.21}$ &
            $81.99_{\pm 0.78}$ &
            N/A &
            N/A &
            N/A
            \\ \midrule

            Resnet18 &
            ImageNet32 &
            $48.11$ &
            $48.09$ &
            $48.06$ &
            $47.55$ &
            N/A &
            N/A &
            N/A
            \\ \midrule

            Resnet18 &
            ImageNet1k &
            $67.17$ &
            $67.06$ &
            $67.15$ &
            $67.32$ &
            N/A &
            N/A &
            N/A
            \\ \midrule

            
            ViT-Tiny &
            ImageNet1k &
            $71.05_{\pm 0.16} $&
            $71.22_{\pm 0.36}$ &
            $71.345_{\pm 0.22}$&
            N/A &
            N/A &
            N/A &
            N/A
            \\ \midrule

            \makecellnew{Transformer$++$ \\ 70M} &
            SlimPajama-627B &
            $34.38_{\pm 0.12}$ &
            $34.96_{\pm 0.11}$ &
            $33.84_{\pm 0.33}$ &
            N/A &
            N/A &
            N/A &
            N/A
            \\ \midrule

            \makecellnew{Transformer$++$ \\ 160M} &
            SlimPajama-627B &
            $24.03_{\pm 0.02}$ &
            $24.29_{\pm 0.10}$ &
            $23.32_{\pm 0.06}$ &
            N/A &
            N/A &
            N/A &
            N/A
            \\ 
            \\ \midrule

            \makecellnew{Transformer$++$ \\ 410M} &
            SlimPajama-627B &
            $16.65_{\pm 0.03}$ &
            $17.07_{\pm 0.05}$ &
            $16.48_{\pm 0.03}$ &
            N/A &
            N/A &
            N/A &
            N/A
            \\ 

            \\ \midrule

            \makecellnew{Transformer$++$ \\ 1B} &
            SlimPajama-627B &
            $13.09$ &
            N/A &
            $13.11$ &
            N/A &
            N/A &
            N/A &
            N/A
            \\ 
        
            \bottomrule 
        
        \end{tabular}
        }

\end{table*}


\subsection{Comparison Algorithms that Support Momentum}\label{sec:momentum_appendix}

In the main paper, we provided the test performance only. Now we additionally illustrate the performance of algorithms w.r.t. training loss convergence. \Cref{fig:stability_momentum_train_loss} demonstrates that \algname{NGN-M} is the most robust algorithm for the choice of the step-size hyperparameter from this perspective as well. In \Cref{fig:stability_momentum_train_loss}, we additionally demonstrate the performance of the algorithms on (VGG16 \citep{simonyan2014very}, CIFAR10) and (MLP, MNIST) workloads where \algname{NGN-M} matches the performance of the state-of-the-art algorithms in this setting and archives higher resilience to the step-size hyperparameter choice. The best performance results are reported in \Cref{tab:empirical_comparison_momentum_appendix} and showcase that \algname{NGN-M} always matches the performance of other optimizers or improves it.  

\begin{figure*}[t]
    \centering
    \begin{tabular}{ccc}
        \multicolumn{3}{c}{\includegraphics[width=0.6\linewidth]{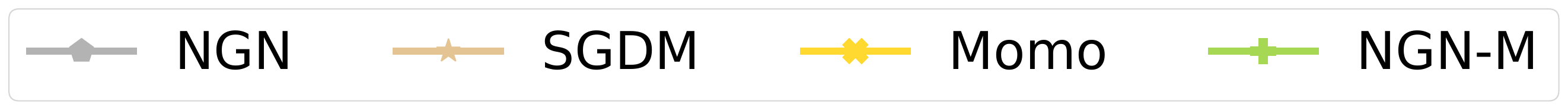}} \\
       \includegraphics[width=0.3\linewidth]{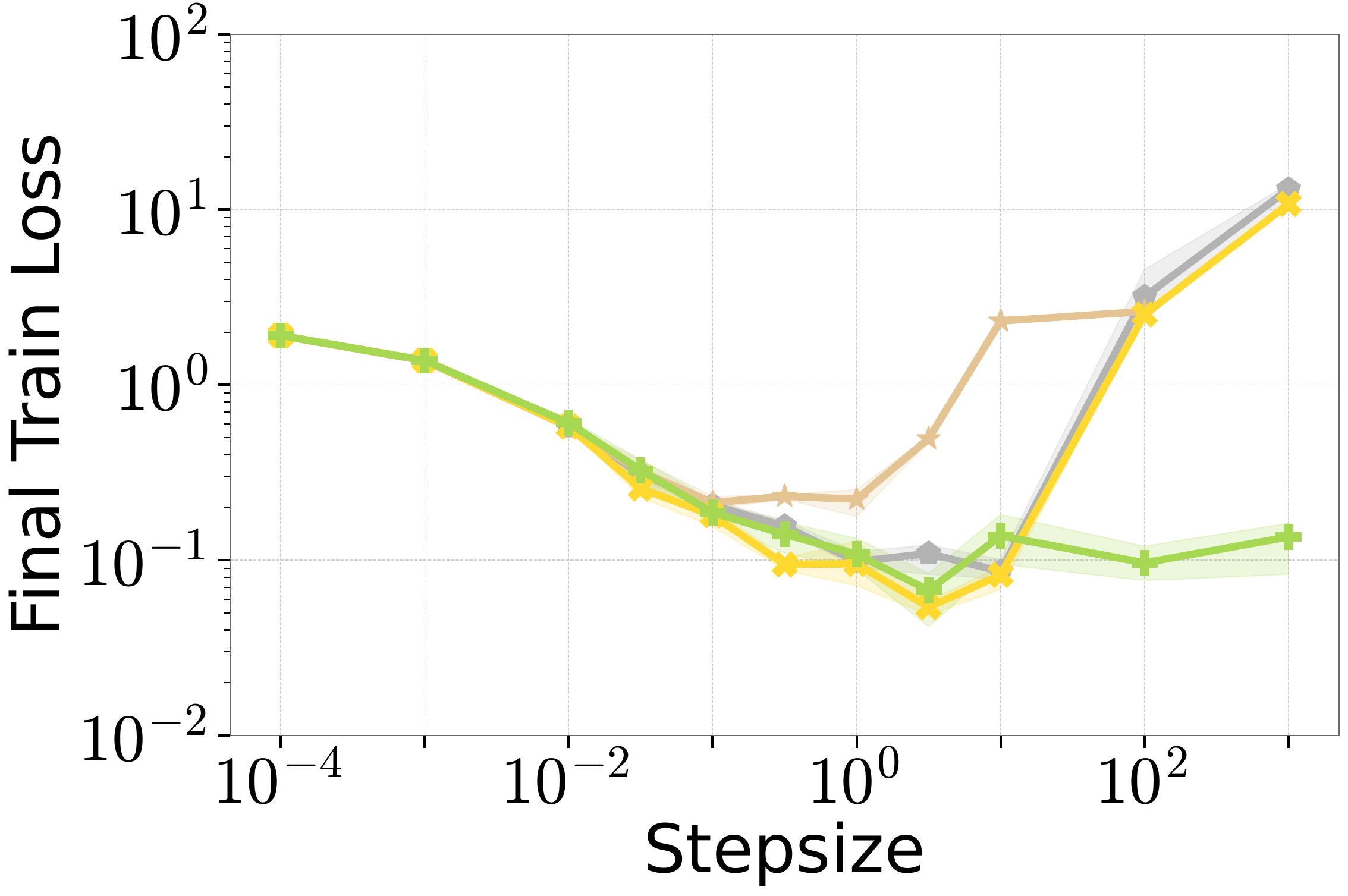} &
       \includegraphics[width=0.3\linewidth]{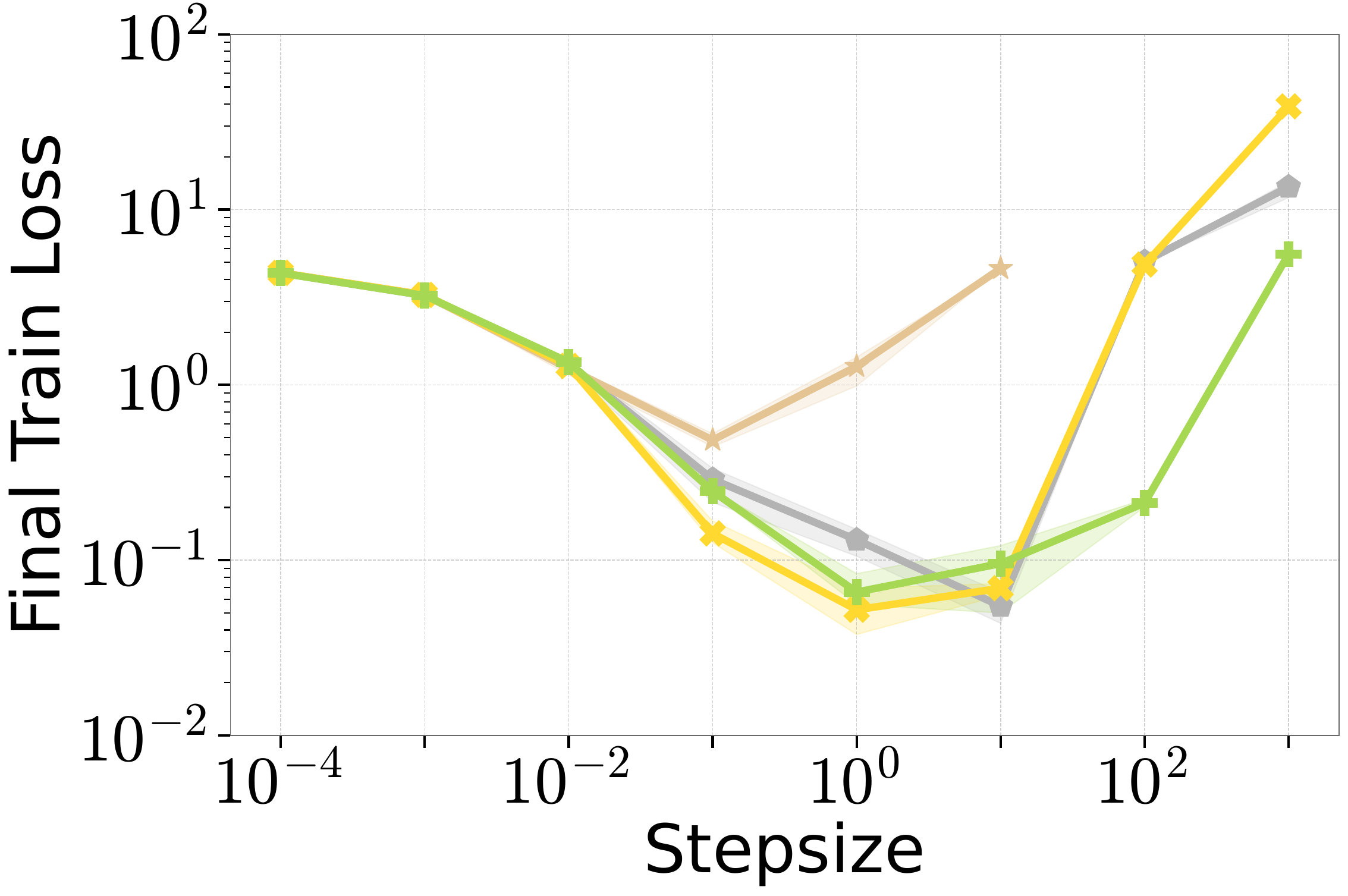} & 
       \includegraphics[width=0.3\linewidth]{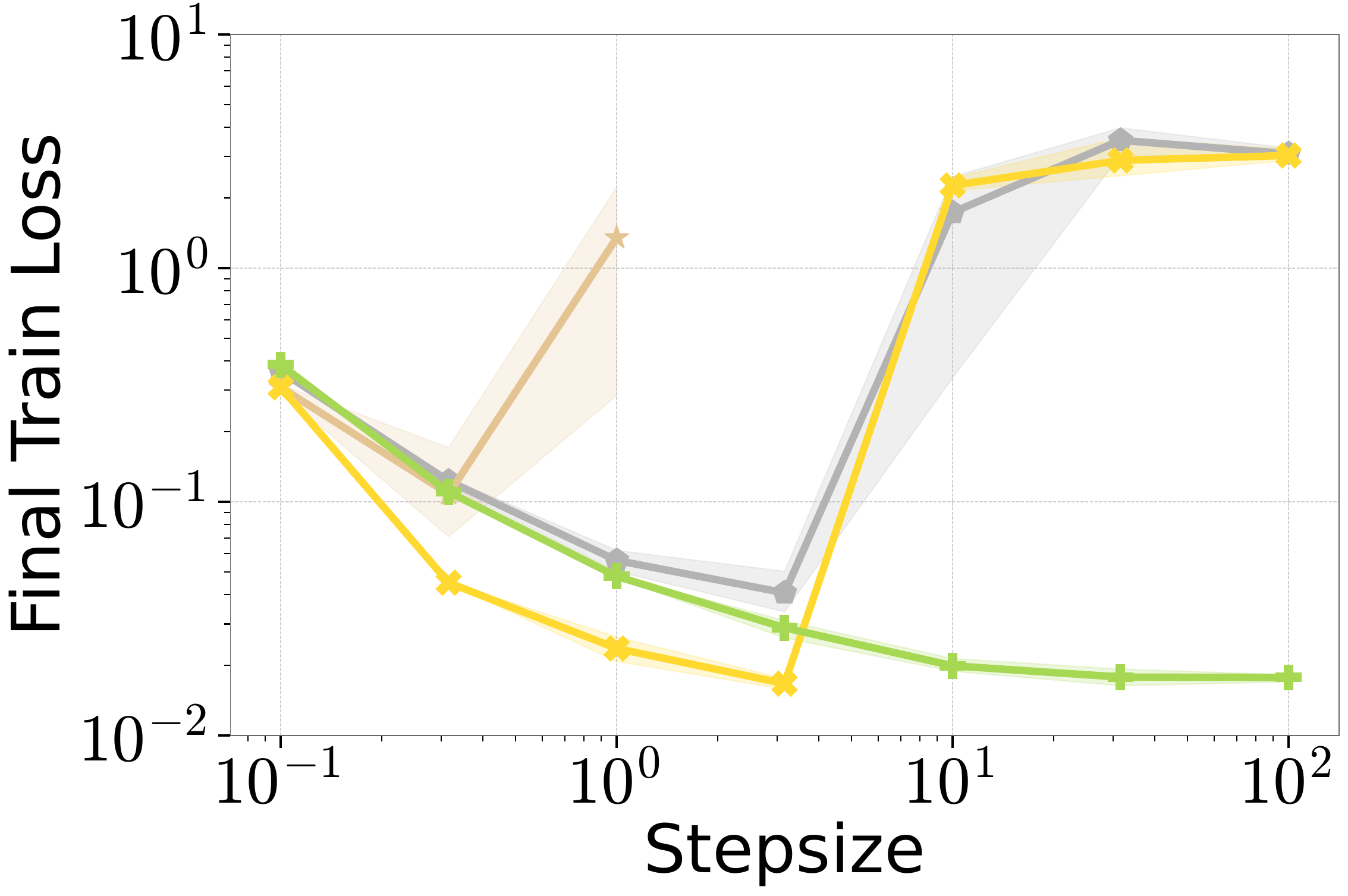} \\
       {\small Resnet20 for CIFAR 10} & 
       {\small Resnet110 for CIFAR 100} &
       {\small ViT for CIFAR 10} \\
       
    \end{tabular}

    \begin{tabular}{cc}
         {\includegraphics[width=0.3\linewidth]{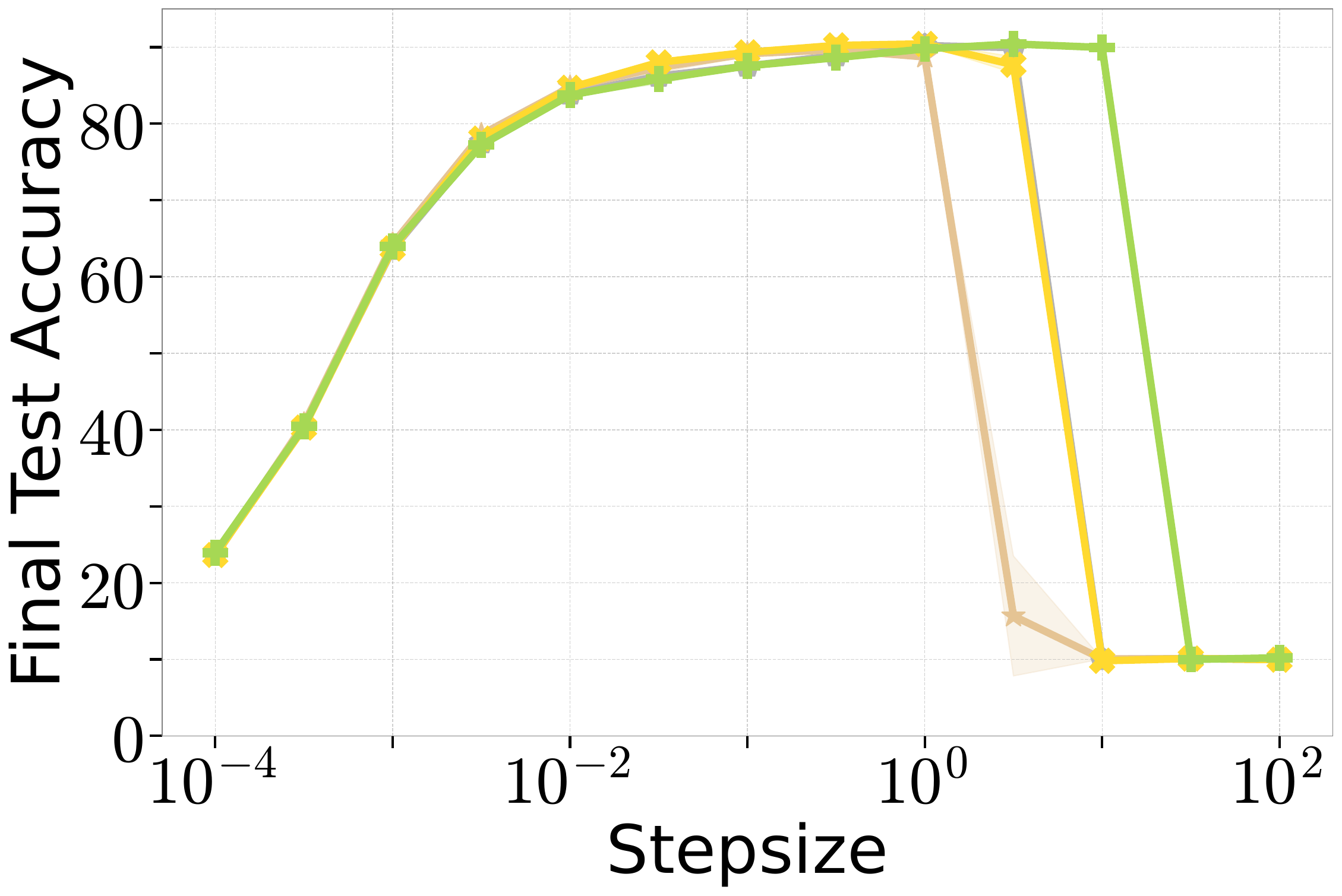}} &
         {\includegraphics[width=0.3\linewidth]{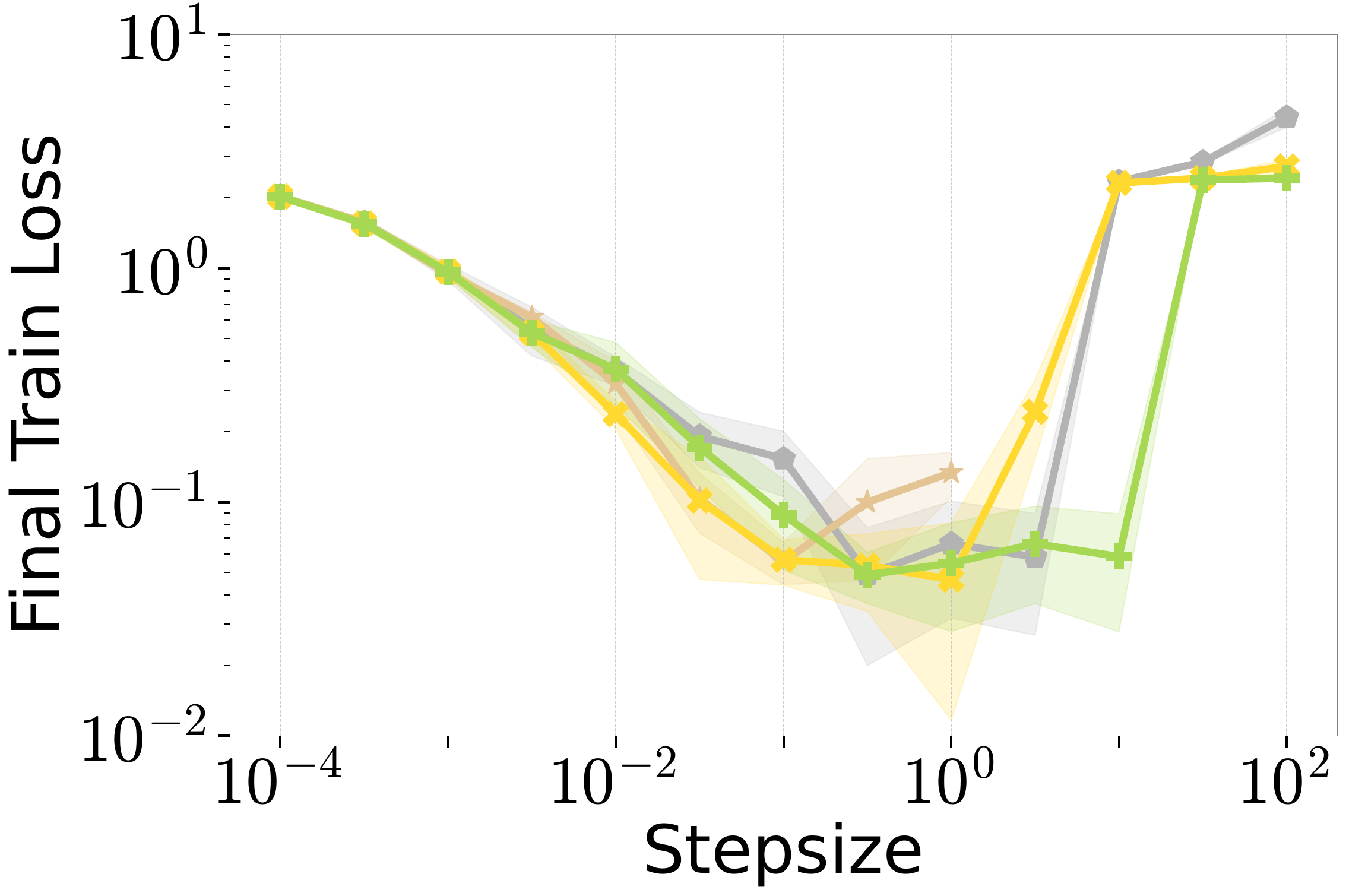}} \\
         {\small VGG16 for CIFAR 10} &
         {\small VGG16 for CIFAR 10} \\
        {\includegraphics[width=0.3\linewidth]{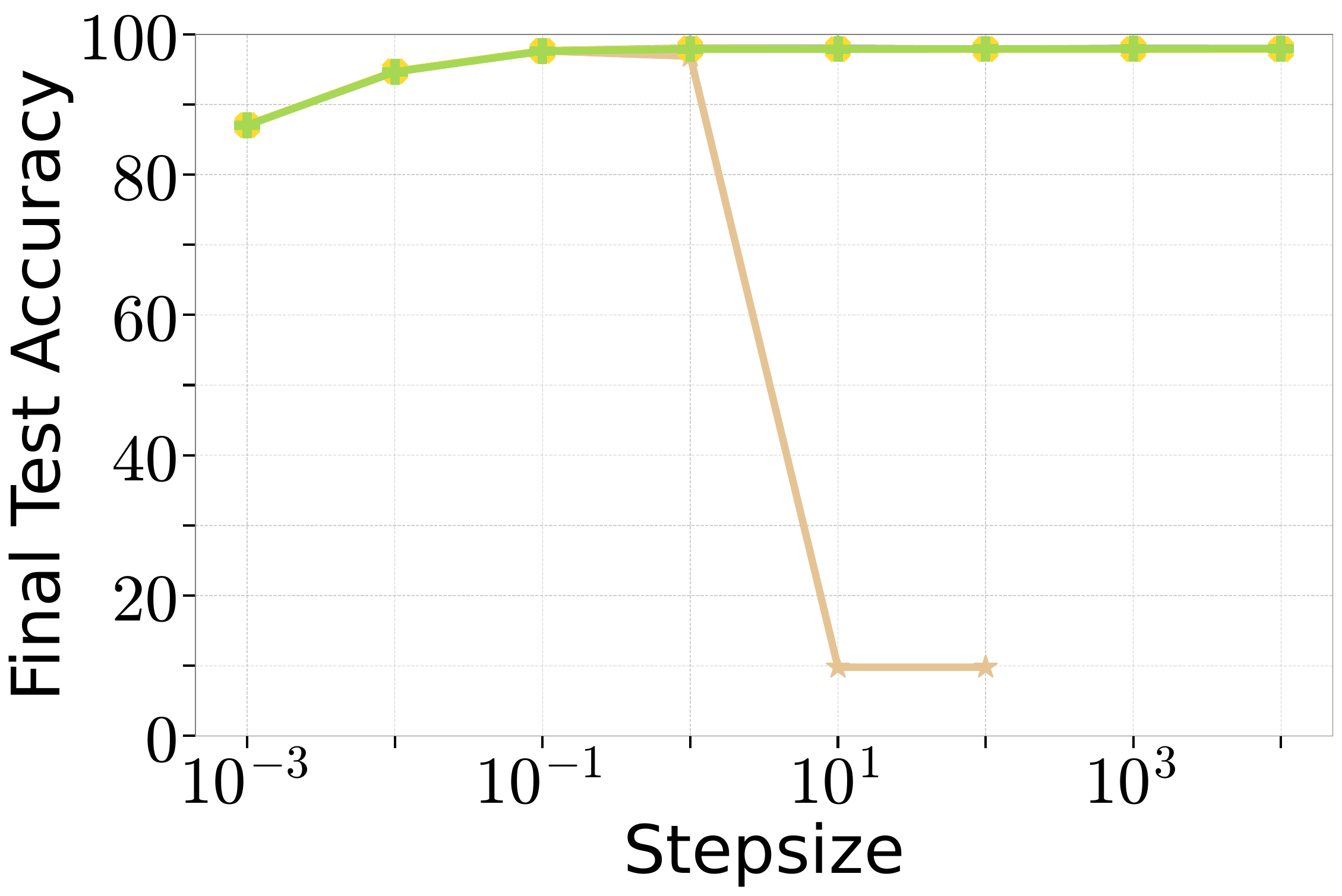}} &
         {\includegraphics[width=0.3\linewidth]{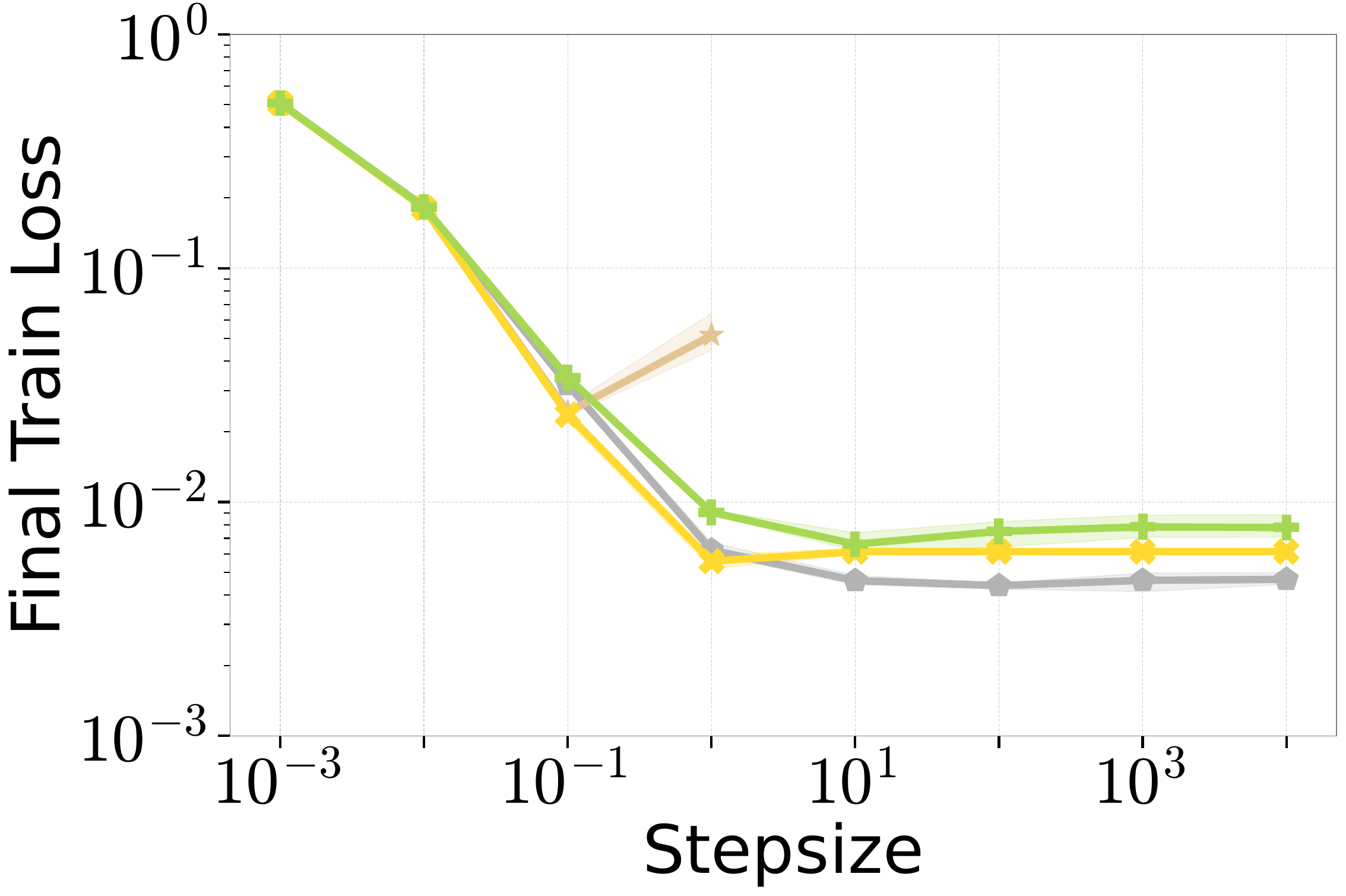}} \\
         {\small MLP for MNIST} &
         {\small MLP for MNIST} 
    \end{tabular}
    
    \caption{Stability performance of algorithms supporting momentum varying step-size hyper-arameter ($c$ for \algname{NGN} and \algname{NGN-M}, $\alpha_0$ for \algname{Momo}, and step-size for \algname{SGDM}). We observe that \algname{NGN-M} achieves the training loss close to the best possible for a wider range of the step-size hyperparameter.}
    \label{fig:stability_momentum_train_loss}
\end{figure*}

\subsection{Comparison of Algorithms that Support Momentum and Diagonal Step-size}\label{sec:adam_type_appendix}

Next, we illustrate the performance of the algorithms that support both momentum and diagonal step-size. According to the results in \Cref{fig:stability_adam_type_train_loss,fig:stability_adam_type_train_loss_additional_workloads}, \algname{NGN-MDv1} achieves the best resilience to the step-size hyperparameter choice among all considered algorithms. Again, \algname{NGN-MDv1} is the most stable algorithm to the choice of step-size hyperparameter w.r.t. training loss convergence. Its best performance is competitive to that of other algorithms but the step-size hyperparameter range that gives such performance is wider.

Moreover, we support our claims about stability on additional workloads such as (VGG16, CIFAR10) (in \Cref{fig:stability_momentum_train_loss}), (MLP, MNIST), (LSTM \citep{hochreiter1997lstm}, PTB \citep{mikolov2010recurrent}), and (Transformer \citep{karpathy2022nanogpt}, Tiny Shakespeare \citep{karpathy2015tinysheakspeare}) workloads. We observe that \algname{NGN-MDv1} attains higher robustness to the choice of the step-size hyperparameter. Finally, the performance results on (LSTM, Wikitext-2 \citep{merity2016pointer}) and (Transformer, Rotten Tomatoes \citep{pang2005rottomatoes}) are reported in \Cref{tab:empirical_comparison_adam_type_lion_adabound_adabelief}. The results demonstrate competitive performance of \algname{NGN-MDv1} against other benchmarks across all considered workloads.

\begin{figure*}[t]
    \centering
    \begin{tabular}{ccc} 
    \multicolumn{3}{c}{\includegraphics[width=0.7\linewidth]{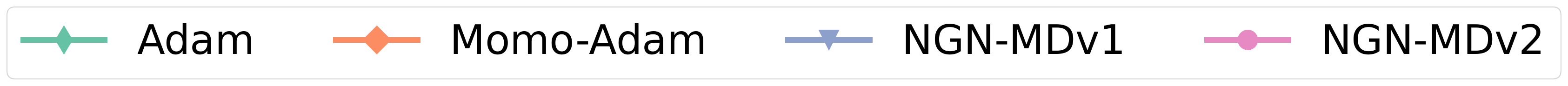}}\\
       \includegraphics[width=0.3\linewidth]{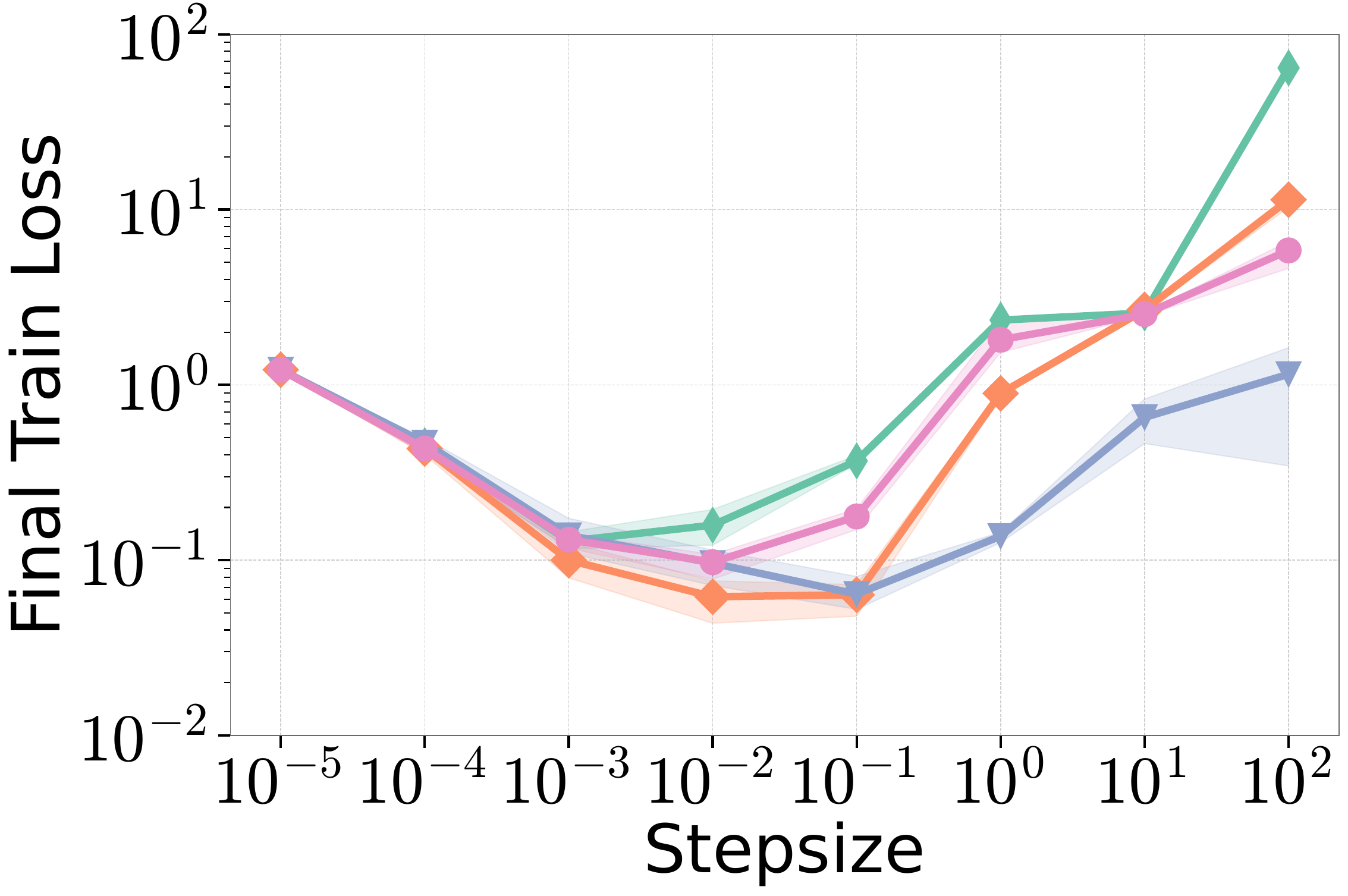} & 
       \includegraphics[width=0.3\linewidth]{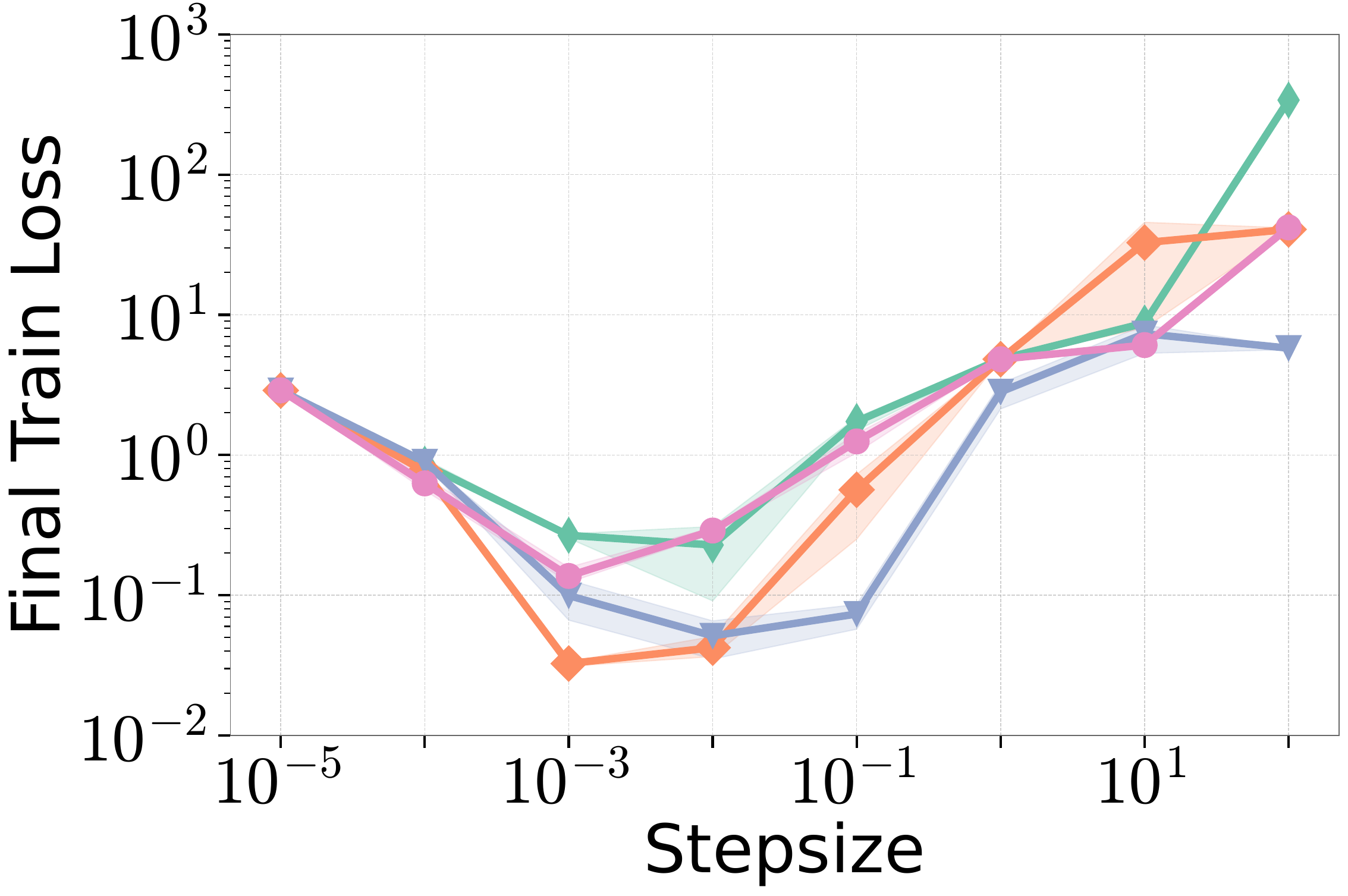} &
       \includegraphics[width=0.3\linewidth]{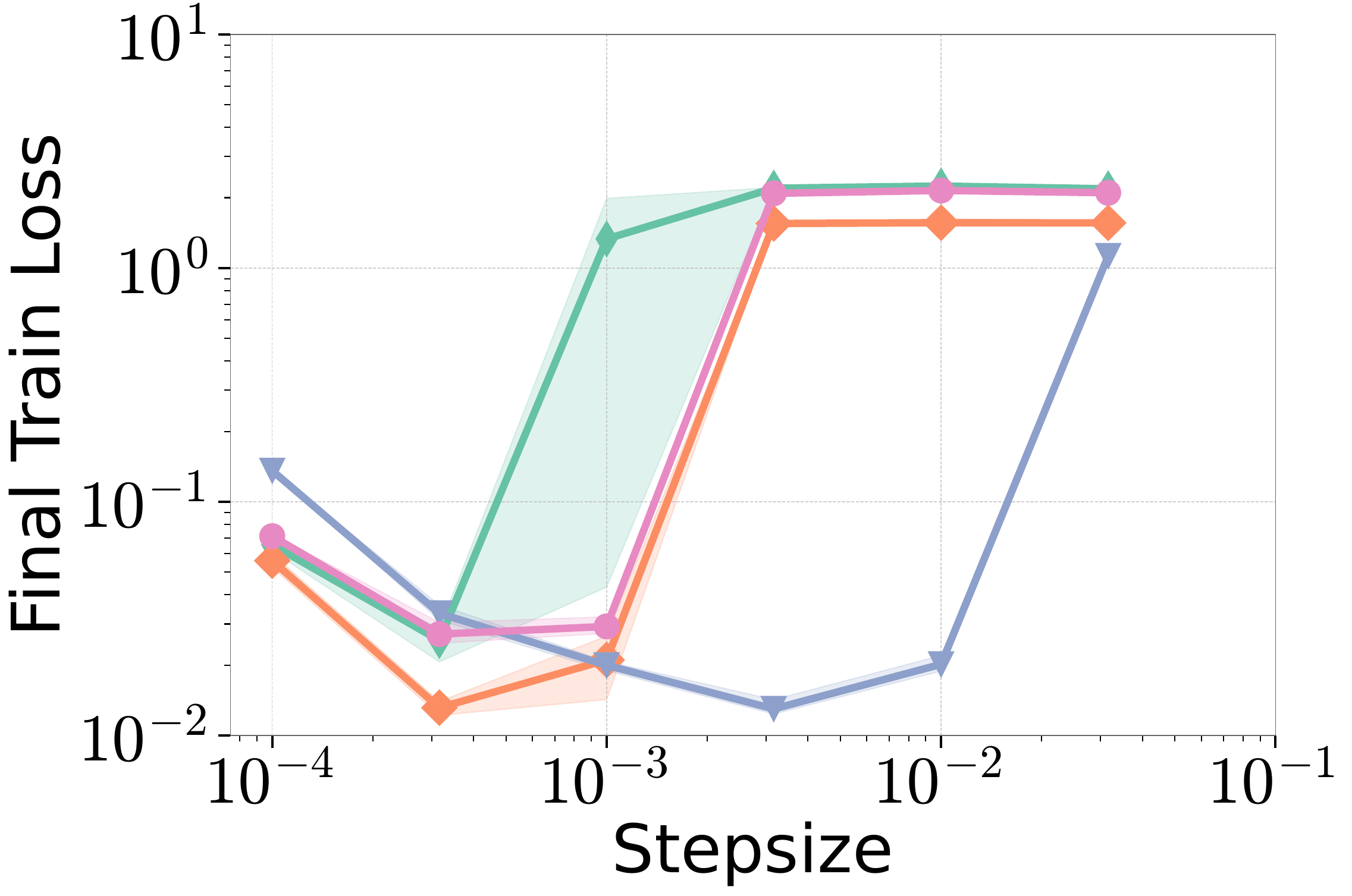}\\
       {\small Resnet20 for CIFAR 10} & 
       {\small Resnet110 for CIFAR 100} &
       {\small ViT for CIFAR 10}
    \end{tabular}
    \begin{tabular}{cc}
         \includegraphics[width=0.3\linewidth]{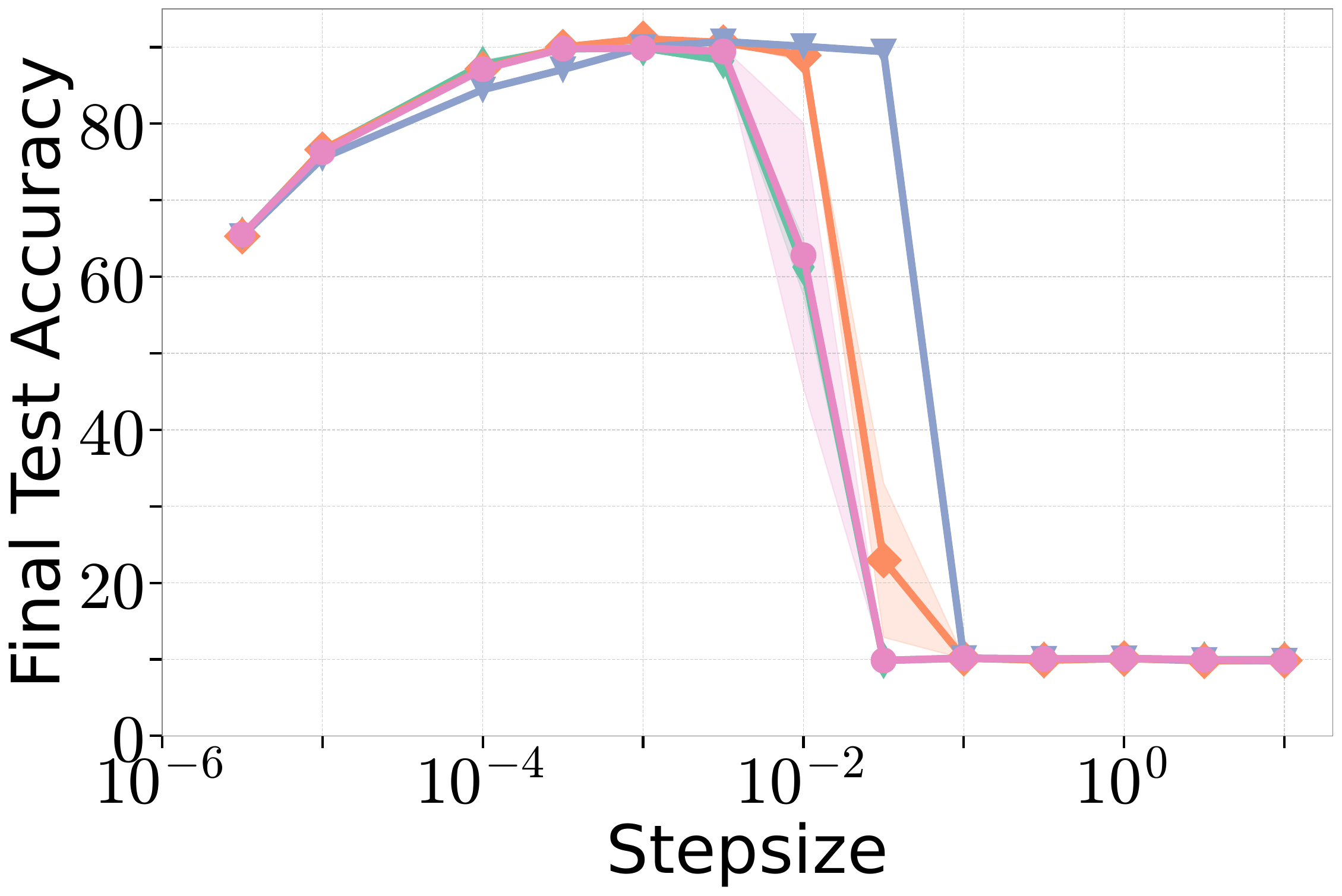} &  
          \includegraphics[width=0.3\linewidth]{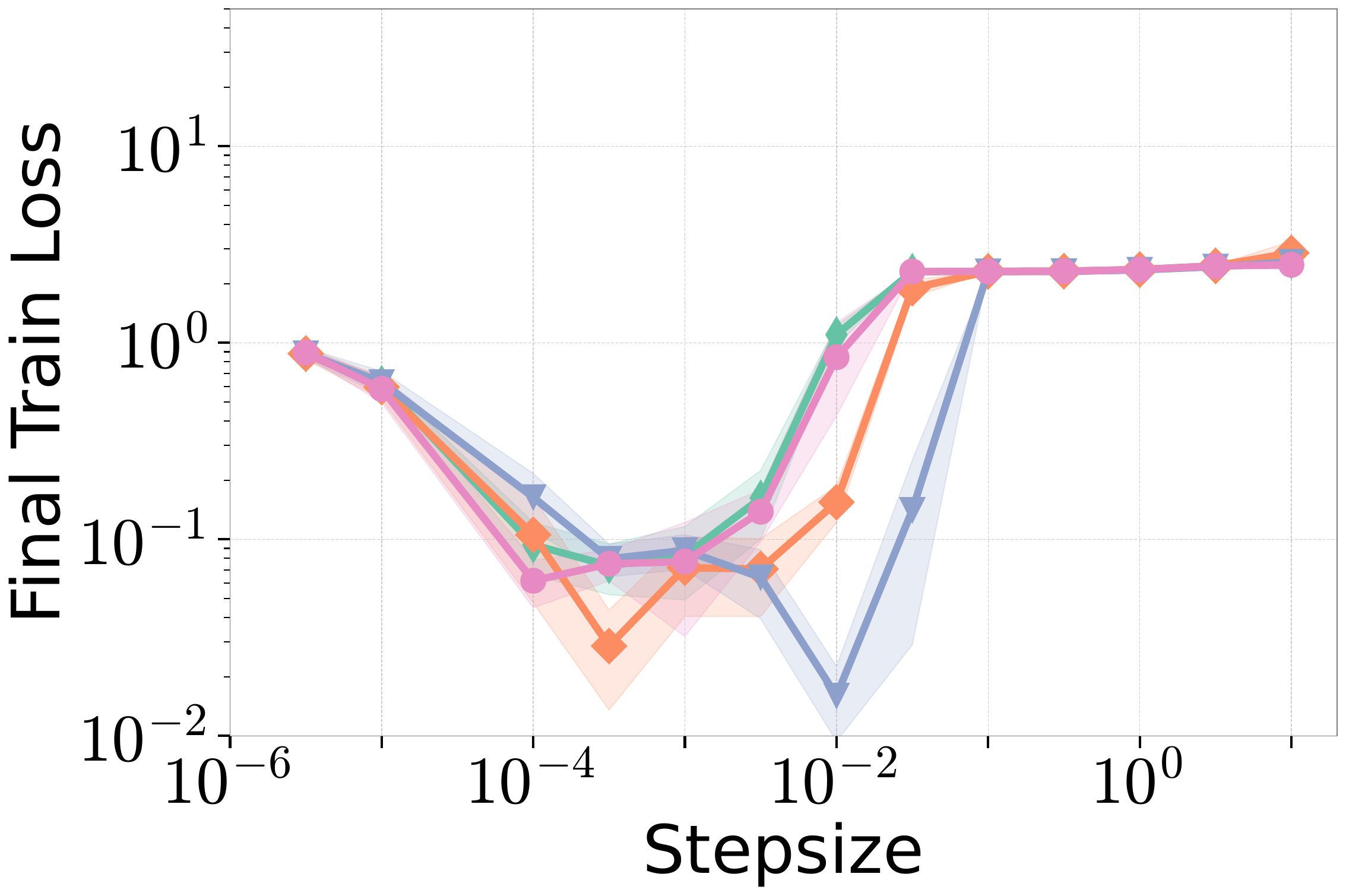} \\
          {\small VGG16 for CIFAR 10} & 
          {\small VGG16 for CIFAR 10}
    \end{tabular}
    
    \caption{Stability performance of algorithms supporting momentum and diagonal step-size varying step-size hyperparameter ($c$ for \algname{NGN-MDv1} and \algname{NGN-MDv2}, $\alpha_0$ for \algname{Momo-Adam}, and step-size for \algname{Adam}). We observe that \algname{NGN-MDv1} achieves the training loss close to the best possible for a wider range of the step-size hyperparameter. }
    \label{fig:stability_adam_type_train_loss}
\end{figure*}

\begin{figure*}[t]
    \centering
    \begin{tabular}{ccc} 
    \multicolumn{3}{c}{\includegraphics[width=0.7\linewidth]{Plots/legend_adam-type_test_acc_stability_comparison_cifar10_vit_512_200.pdf}}\\
       \includegraphics[width=0.3\linewidth]{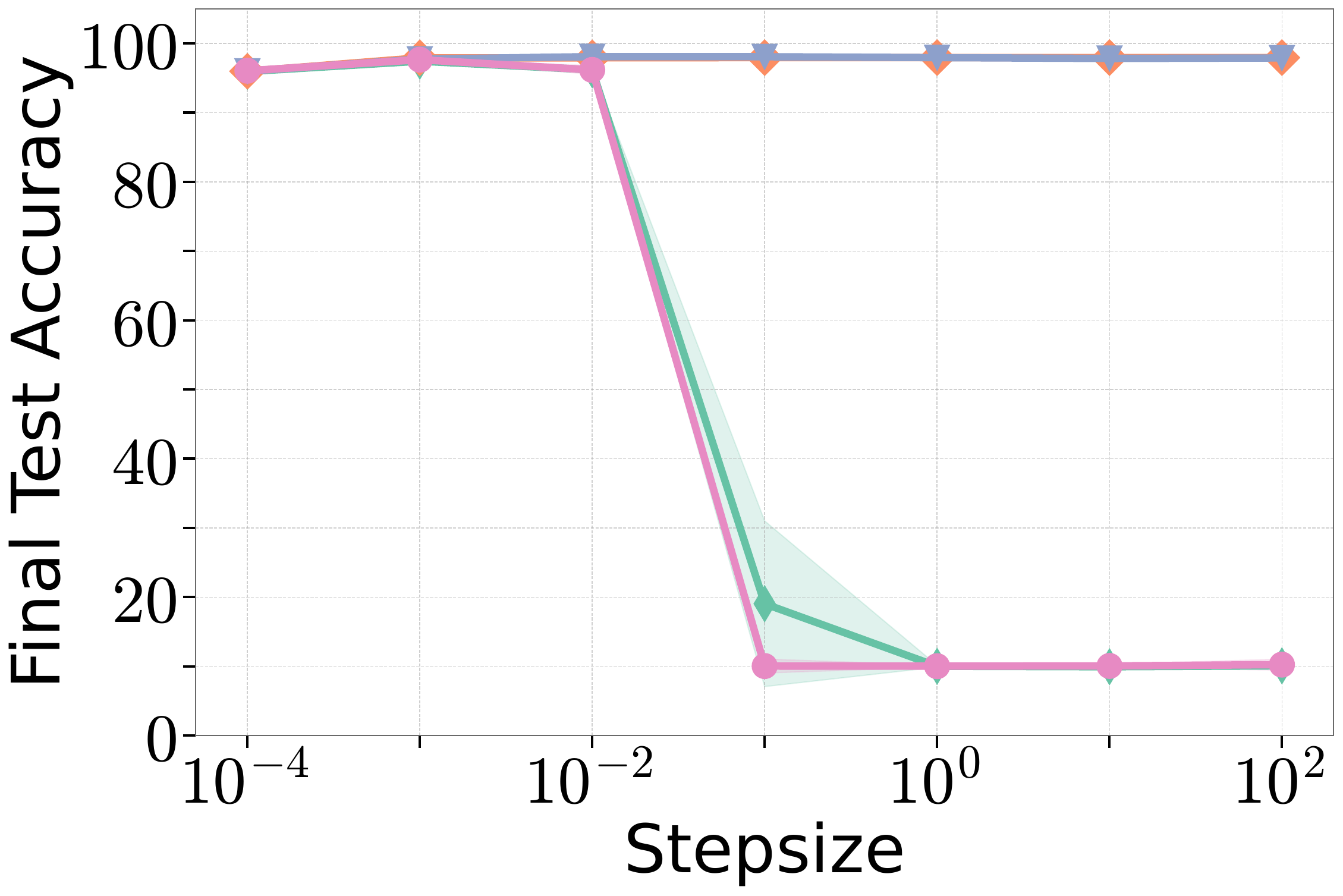} & 
       \includegraphics[width=0.3\linewidth]{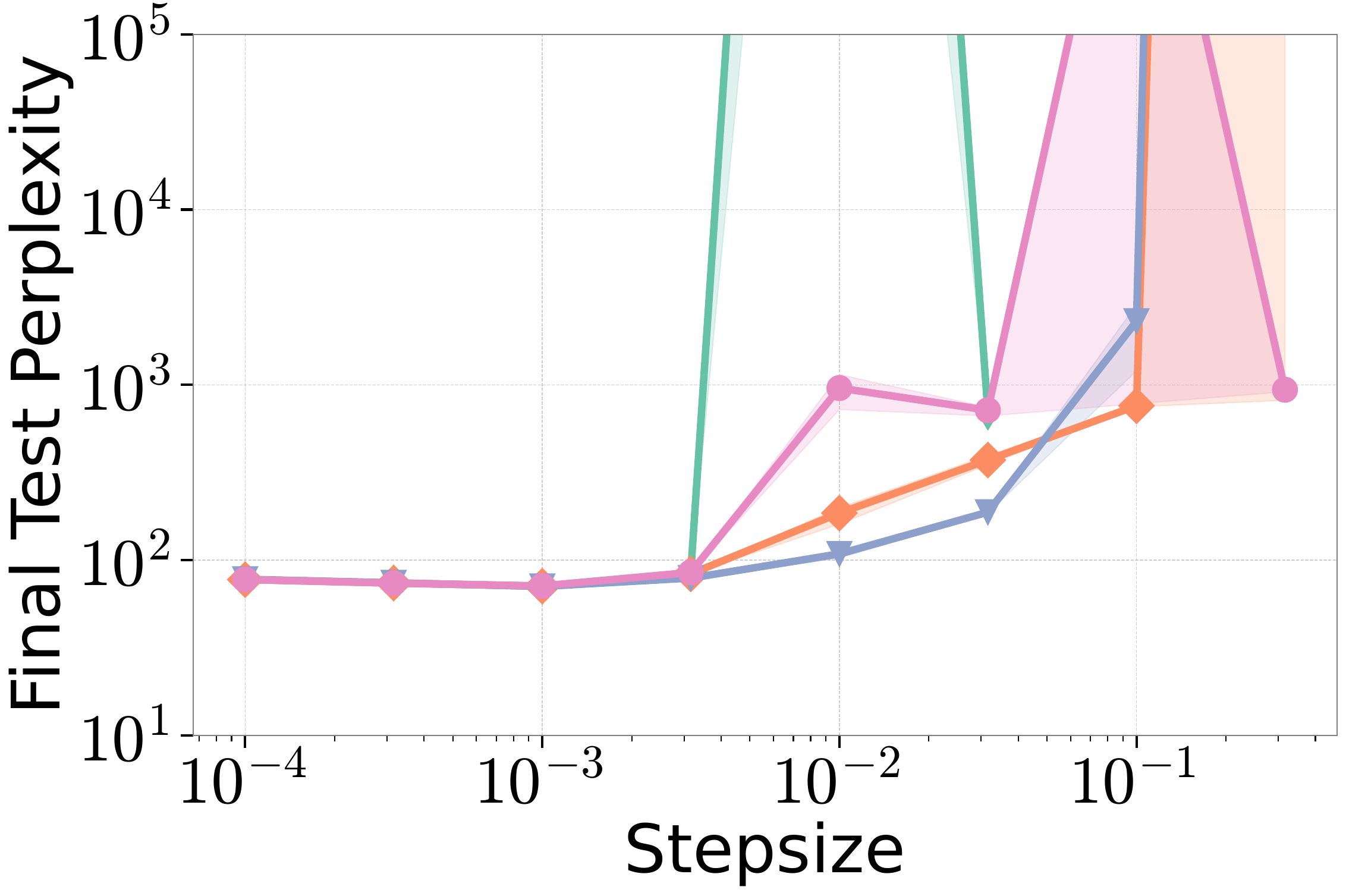} &
       \includegraphics[width=0.3\linewidth]{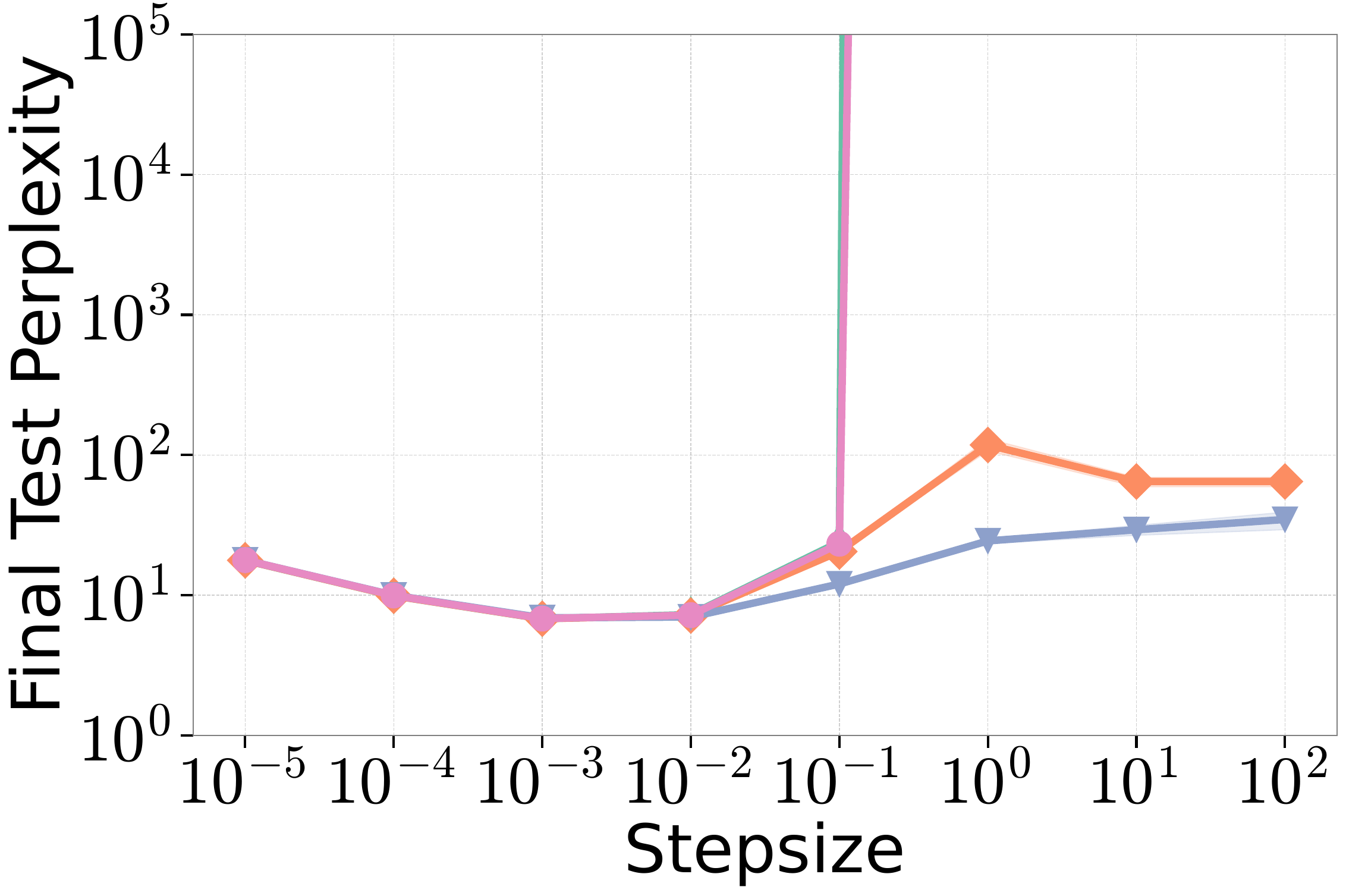}\\
       {\small MLP for MNIST} & 
       {\small LSTM for PTB} &
       \makecellnew{{\small Transformer for }\\{\small Tiny Shakespeare}}\\
       \includegraphics[width=0.3\linewidth]{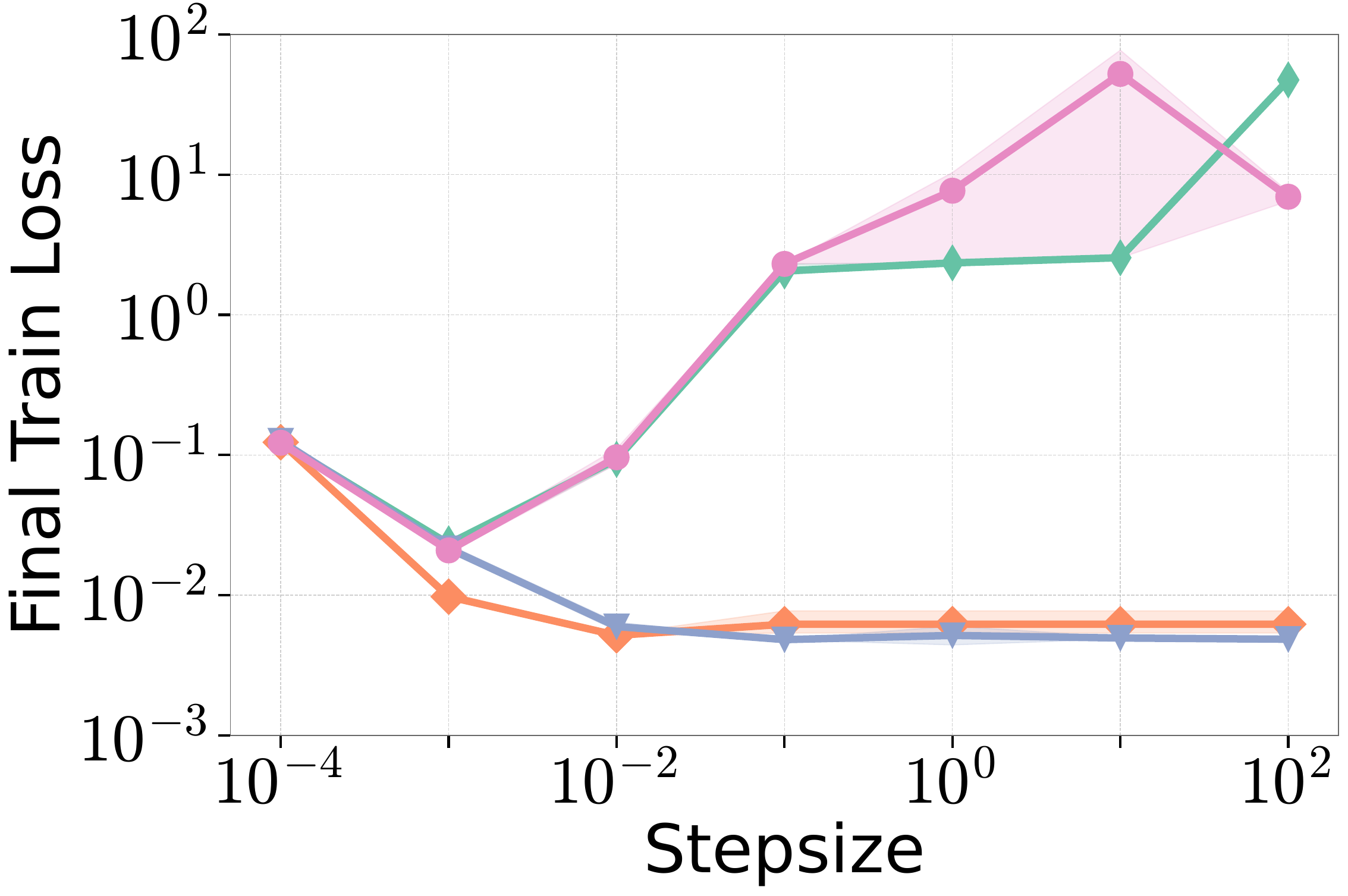} & 
       \includegraphics[width=0.3\linewidth]{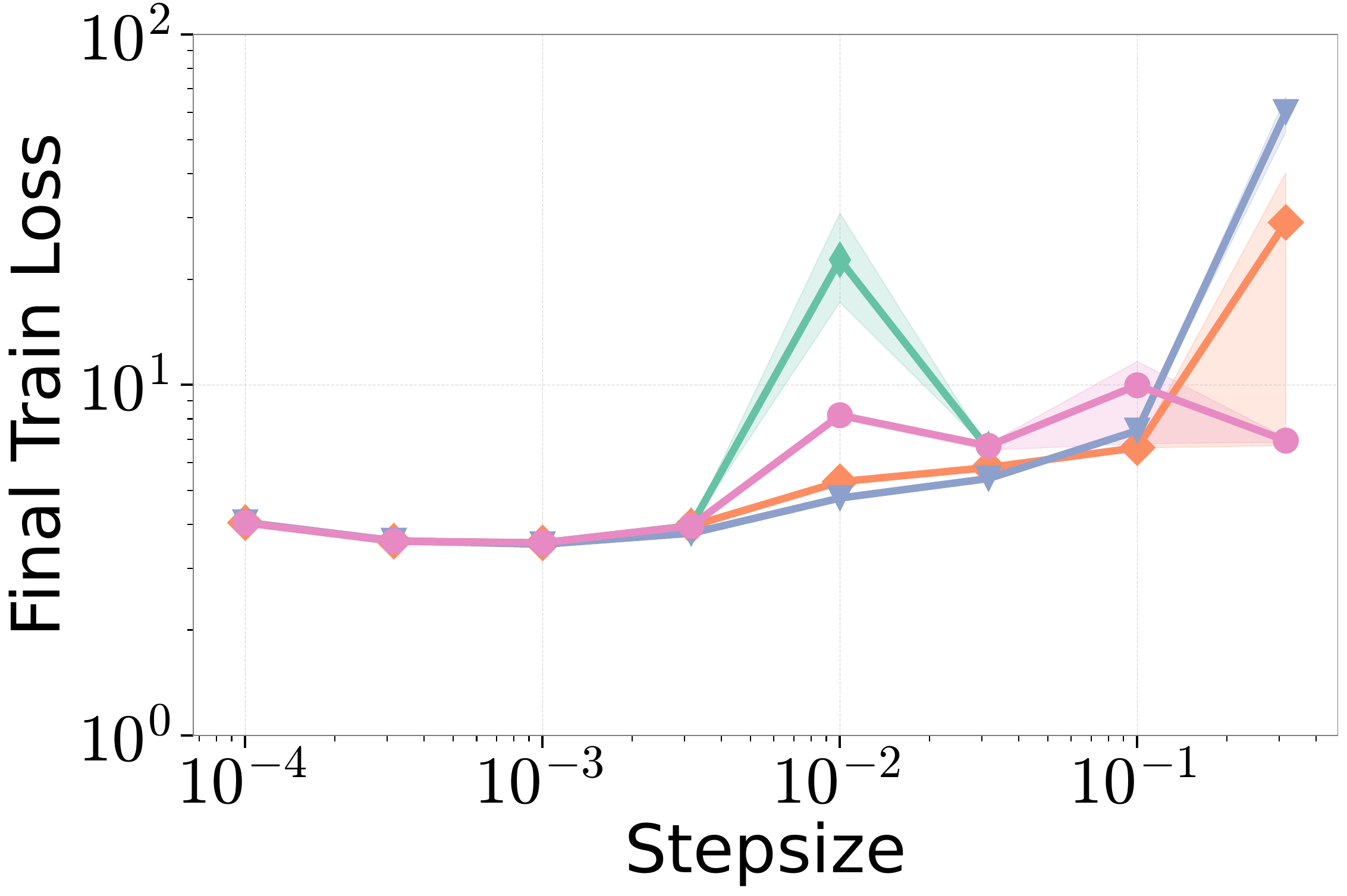} &
       \includegraphics[width=0.3\linewidth]{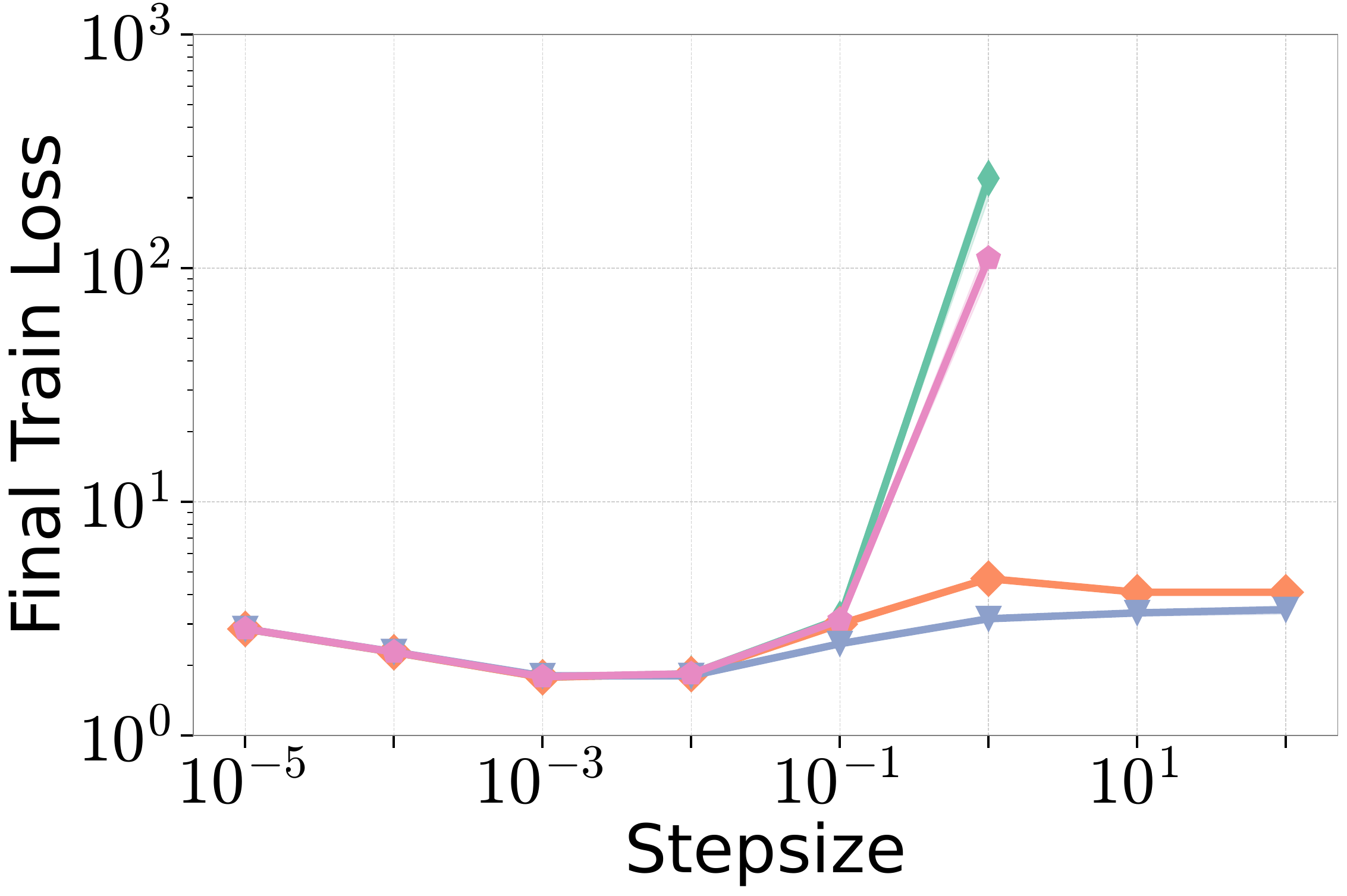}\\
       {\small MLP for MNIST} & 
       {\small LSTM for PTB} &
       \makecellnew{{\small Transformer for }\\{\small Tiny Shakespeare}}
         
    \end{tabular}
    
    \caption{Stability performance of algorithms supporting momentum and diagonal step-size varying step-size hyperparameter ($c$ for \algname{NGN-MDv1} and \algname{NGN-MDv2}, $\alpha_0$ for \algname{Momo-Adam}, and step-size for \algname{Adam}). We observe that \algname{NGN-MDv1} achieves the training loss close to the best possible for a wider range of the step-size hyperparameter. }
    \label{fig:stability_adam_type_train_loss_additional_workloads}
\end{figure*}

\subsection{Additional ImageNet Experiments}\label{sec:imagenet_appendix}

Now we turn to the experiments involving training Resnet18 on ImageNet1k and ImageNet32. In \Cref{fig:stability_imagenet_train_loss} we provide the train loss curves and results on (Resnet18, ImageNet32) workload that demonstrate that \algname{NGN-M} and \algname{NDN-MDv1} attain better resilience to the step-size hyperparameter choice than competitors not only from the train loss point of view as well. The best performance of algorithms is provided in \Cref{tab:empirical_comparison_momentum_appendix} and \ref{tab:empirical_comparison_adam_type_lion_adabound_adabelief}. According to them, both \algname{NGN-M} and \algname{NGN-M} achieve competitive performance against considered benchmarks.

\begin{figure*}[t]
    \centering
    \begin{tabular}{ccc}
       \includegraphics[width=0.3\linewidth]{Plots/momentum_all_test_acc_stability_comparison_imagenet1k_resnet18.pdf} &
       \includegraphics[width=0.3\linewidth]{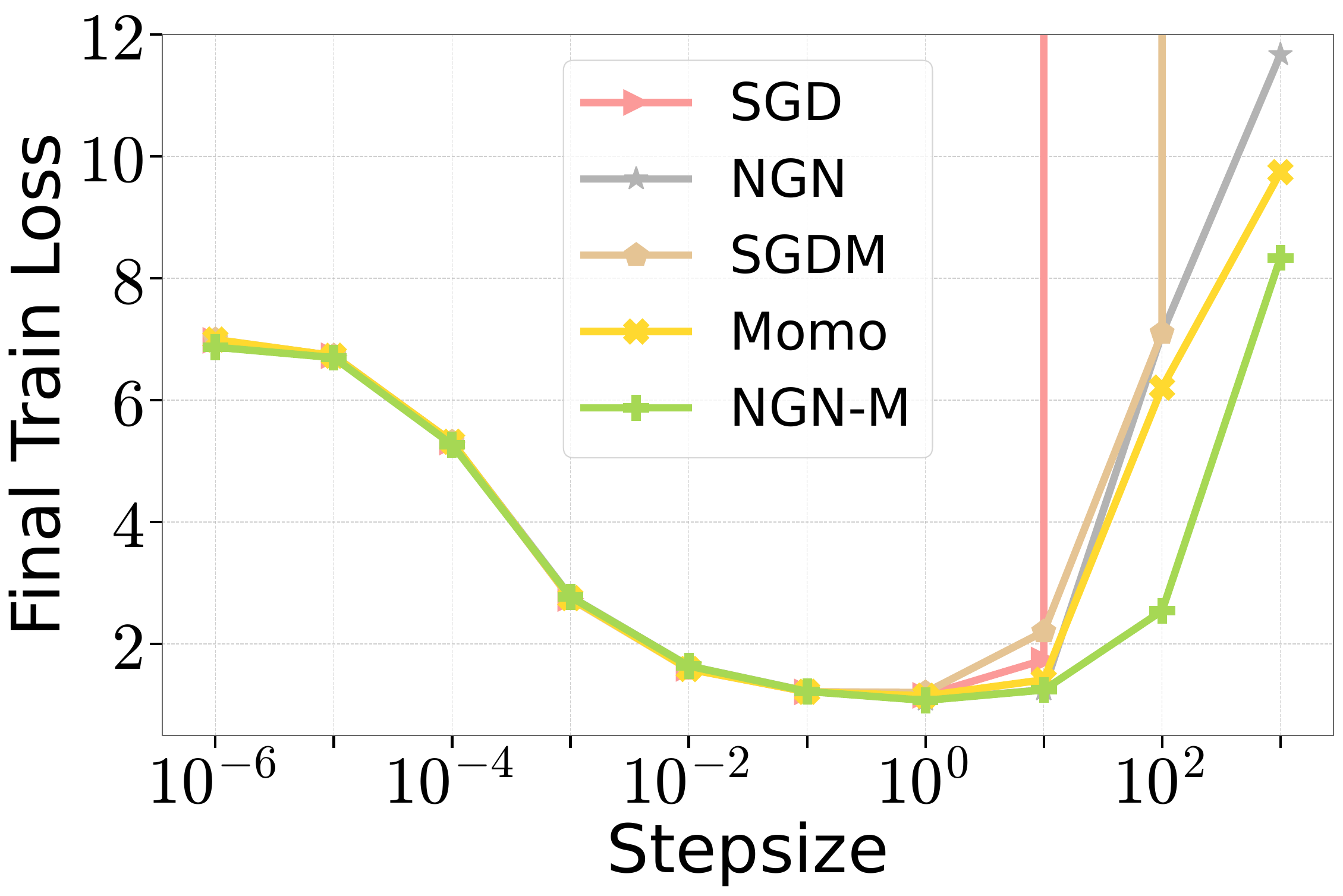} &
       \includegraphics[width=0.3\linewidth]{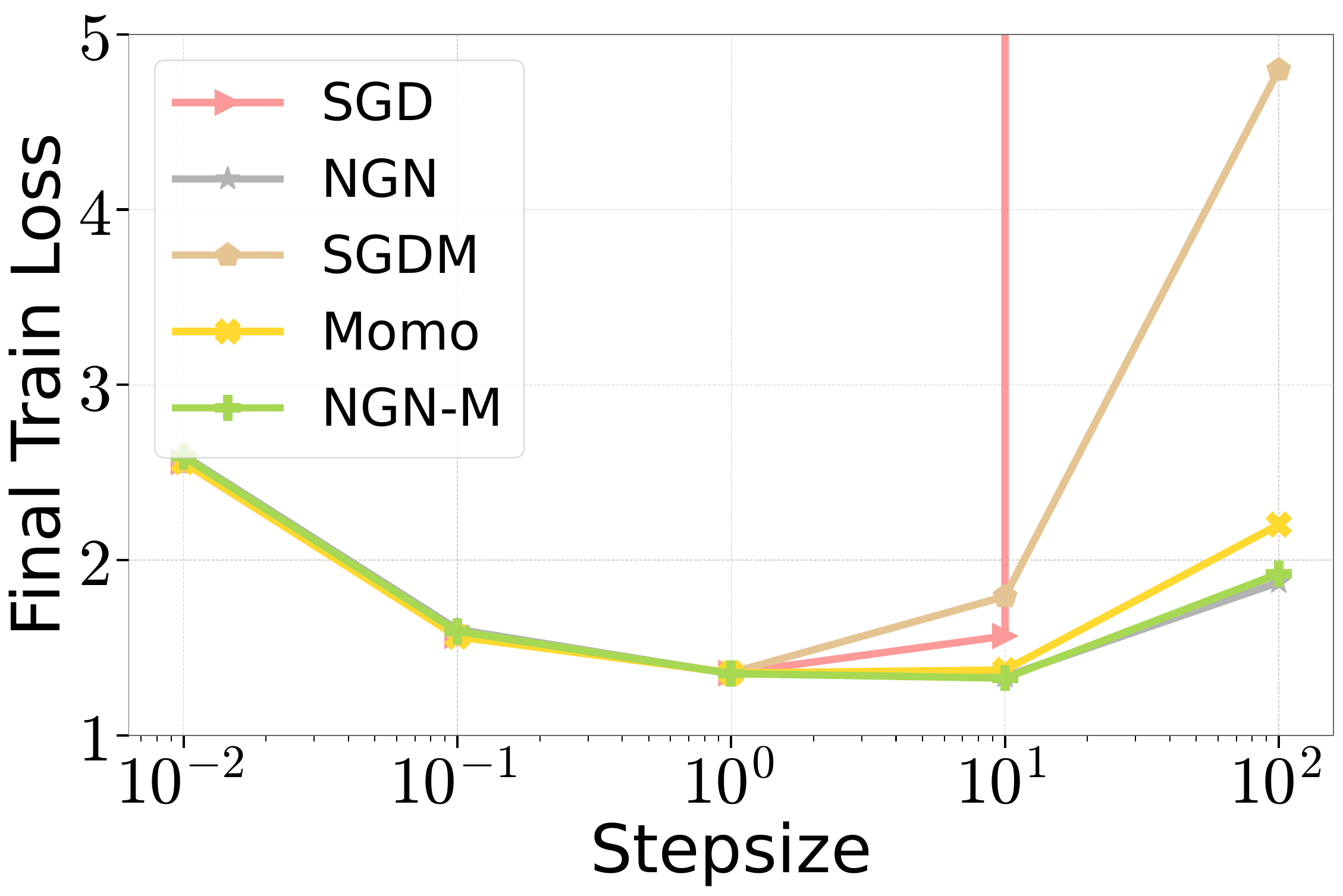} \\
       {\small Resnet18 for ImageNet32} &
       {\small Resnet18 for ImageNet32} &
       {\small Resnet18 for ImageNet1k} \\
       \includegraphics[width=0.3\linewidth]{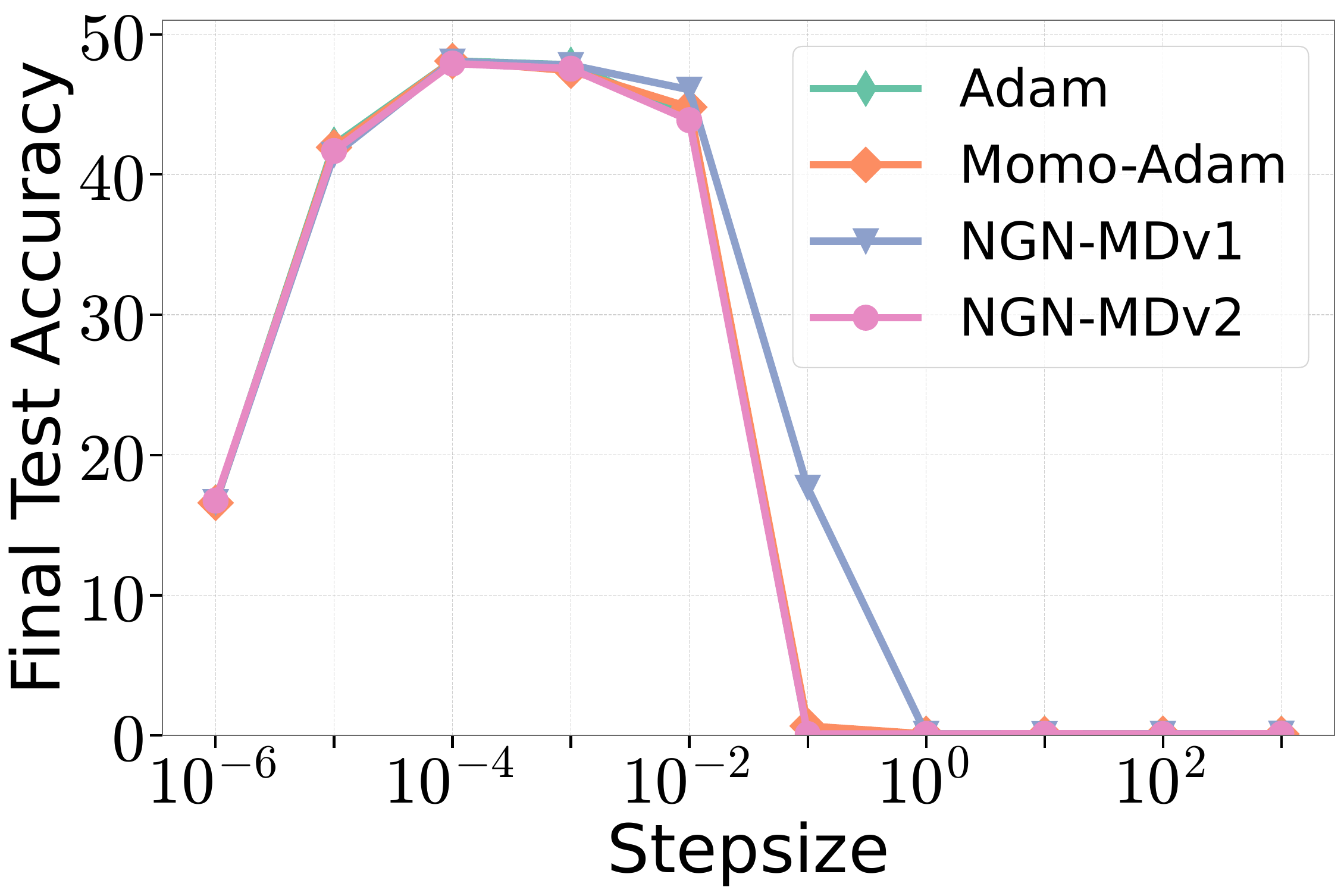} &
       \includegraphics[width=0.3\linewidth]{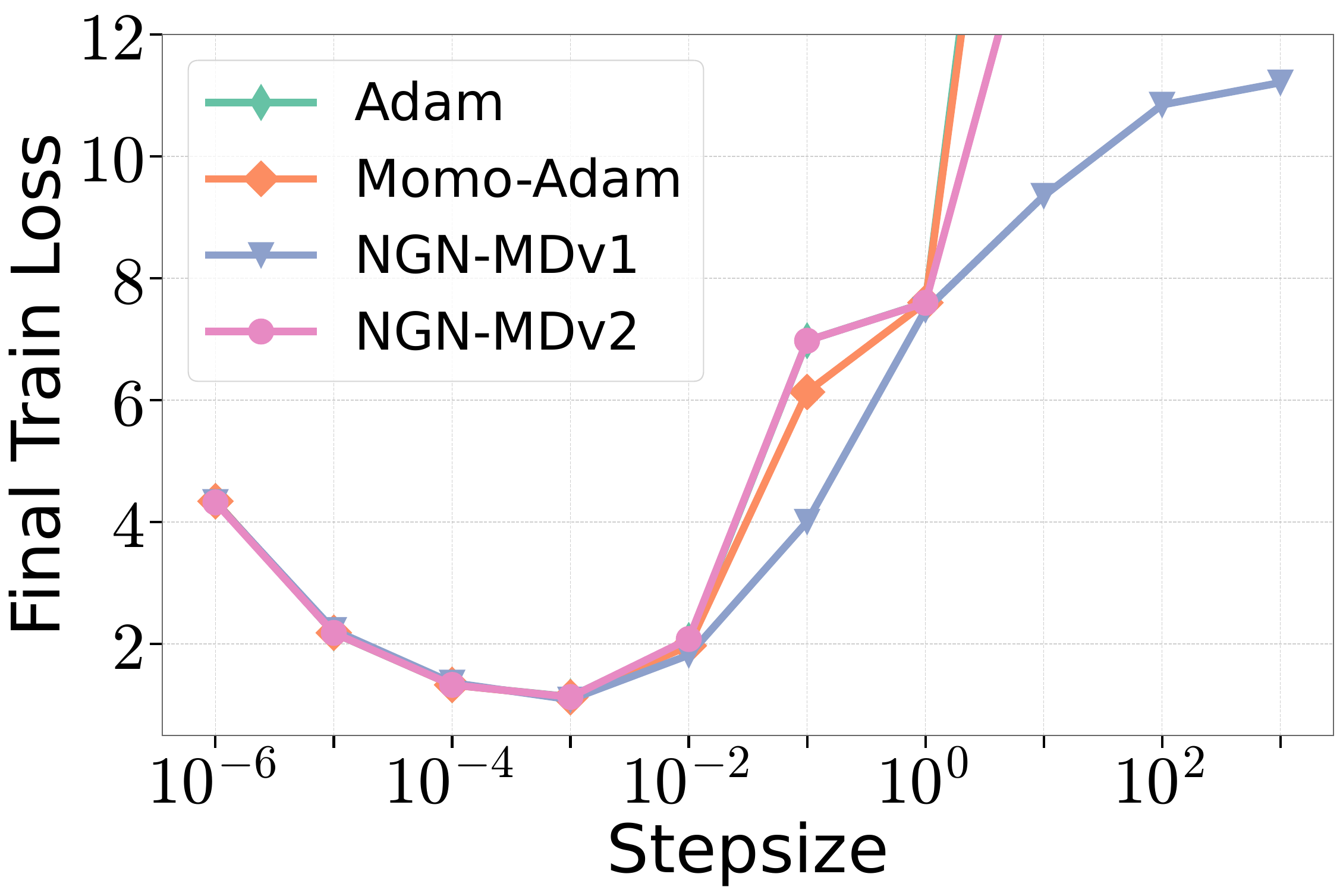} &
       \includegraphics[width=0.3\linewidth]{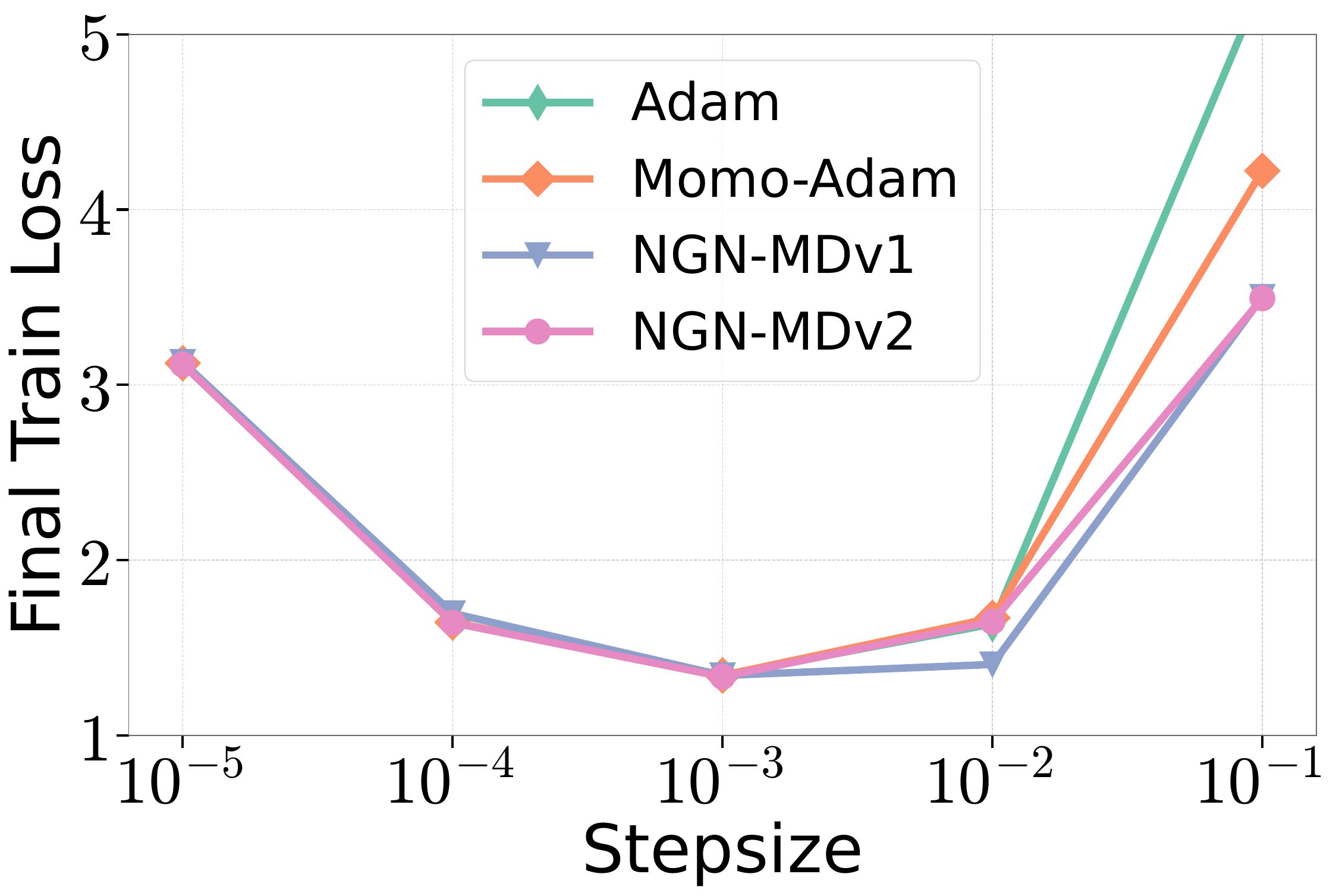} \\
       {\small Resnet18 for ImageNet32} &
       {\small Resnet18 for ImageNet32} &
       {\small Resnet18 for ImageNet1k}
    \end{tabular}
    
    \caption{Stability performance of algorithms supporting momentum ({\bf first row}), and momentum with diagonal step-size ({\bf second row}) varying step-size hyperparameter ($c$ for \algname{NGN}, \algname{NGN-M}, \algname{NGN-MDv1}, and \algname{NGN-MDv2}, $\alpha_0$ for \algname{Momo} and \algname{Momo-Adam}, and step-size for \algname{SGD}, \algname{SGDM}, and \algname{Adam}).}
    \label{fig:stability_imagenet_train_loss}
\end{figure*}



\subsection{Additional Comparison against Lion, Adabelief, Adabound}\label{sec:lion_adabelief_adabound}

This section compares algorithms from 
\Cref{sec:experiments_main}. Moreover, we include the comparison against \algname{Lion} \citep{chen2024symbolic}, \algname{Adabound} \citep{luo2019adaptive}, and \algname{Adabelief} \citep{zhuang2020adabelief}. The results are presented in \Cref{tab:empirical_comparison_adam_type_lion_adabound_adabelief}.

We observe that \algname{NGN-MDv1} and \algname{NGN-MDv2} both achieve competitive performance across various Deep Learning workloads. In \Cref{fig:also_adabound_adabelief_lion_resnet20}, we observe that \algname{Lion}, \algname{Adabound} and \algname{Adabelief} algorithms do not match always the performance of \algname{NGN-MDv1} and \algname{Adam}: \algname{Adabelief} has worse performance on (Resnet20, CIFAR10) workload; \algname{Adabound} has worse performance on (Resnet20, CIFAR10), (Resnet110, CIFAR100), and (ViT, CIFAR10) workloads; \algname{Lion} has worse performance on (Resnet110, CIFAR100) workload. Moreover, their resilience to the step-size hyperparameter choice is lower than that of \algname{NGN-MDv1}. To summarize, \algname{NGN-M} and \algname{NGN-MDv1} are the most robust algorithms to the choice of step-size hyperparameter.

\begin{figure*}[th]
    \centering
    \begin{tabular}{ccc}
        \multicolumn{3}{c}{\includegraphics[width=0.95\linewidth]{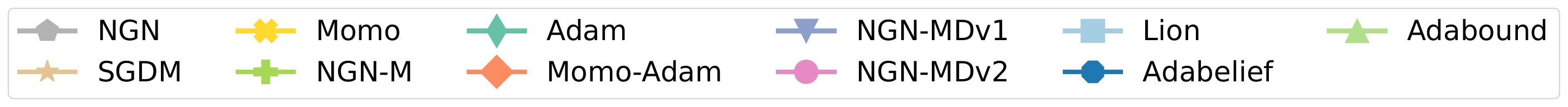} }\\
       \includegraphics[width=0.3\linewidth]{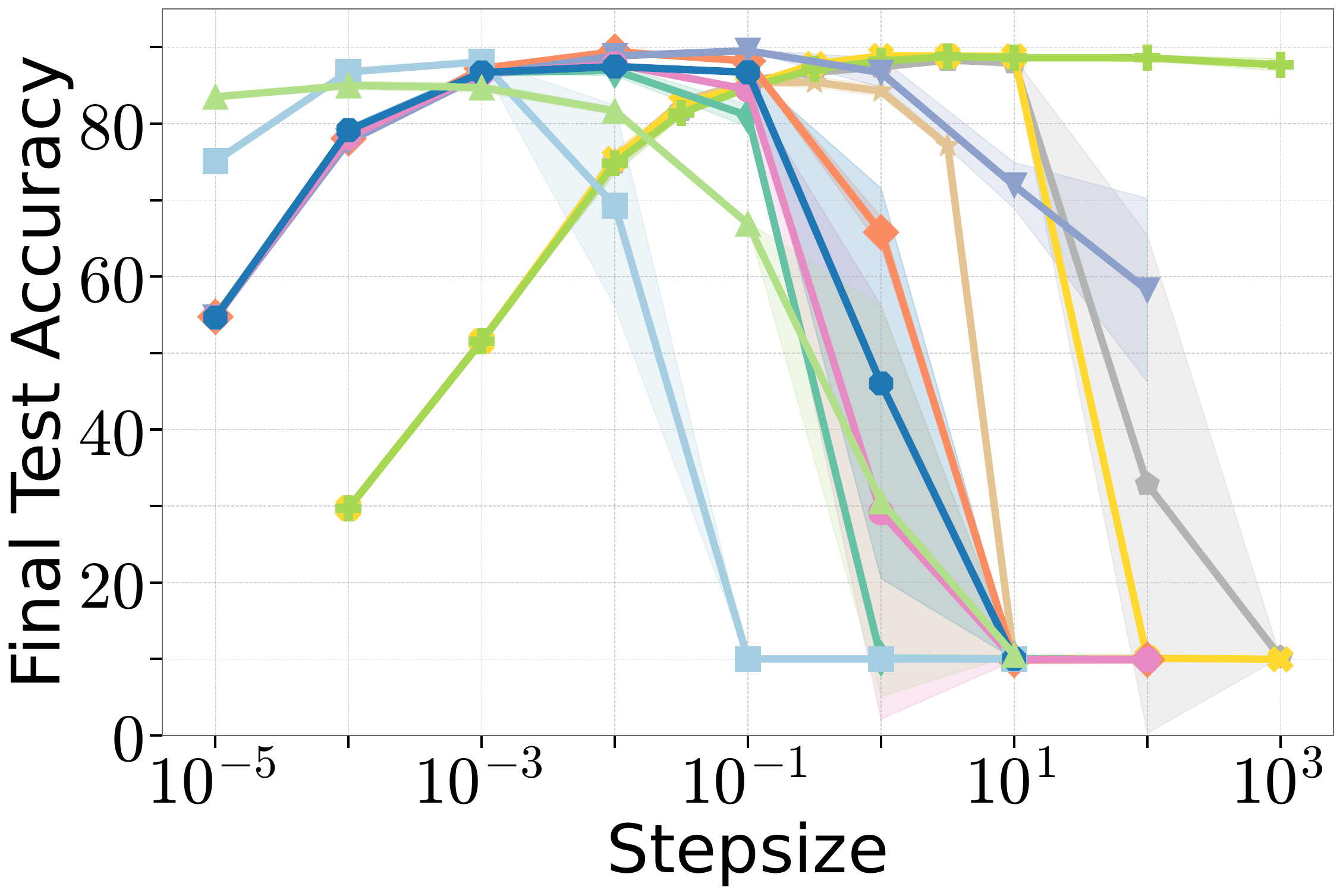} &
        \includegraphics[width=0.3\linewidth]{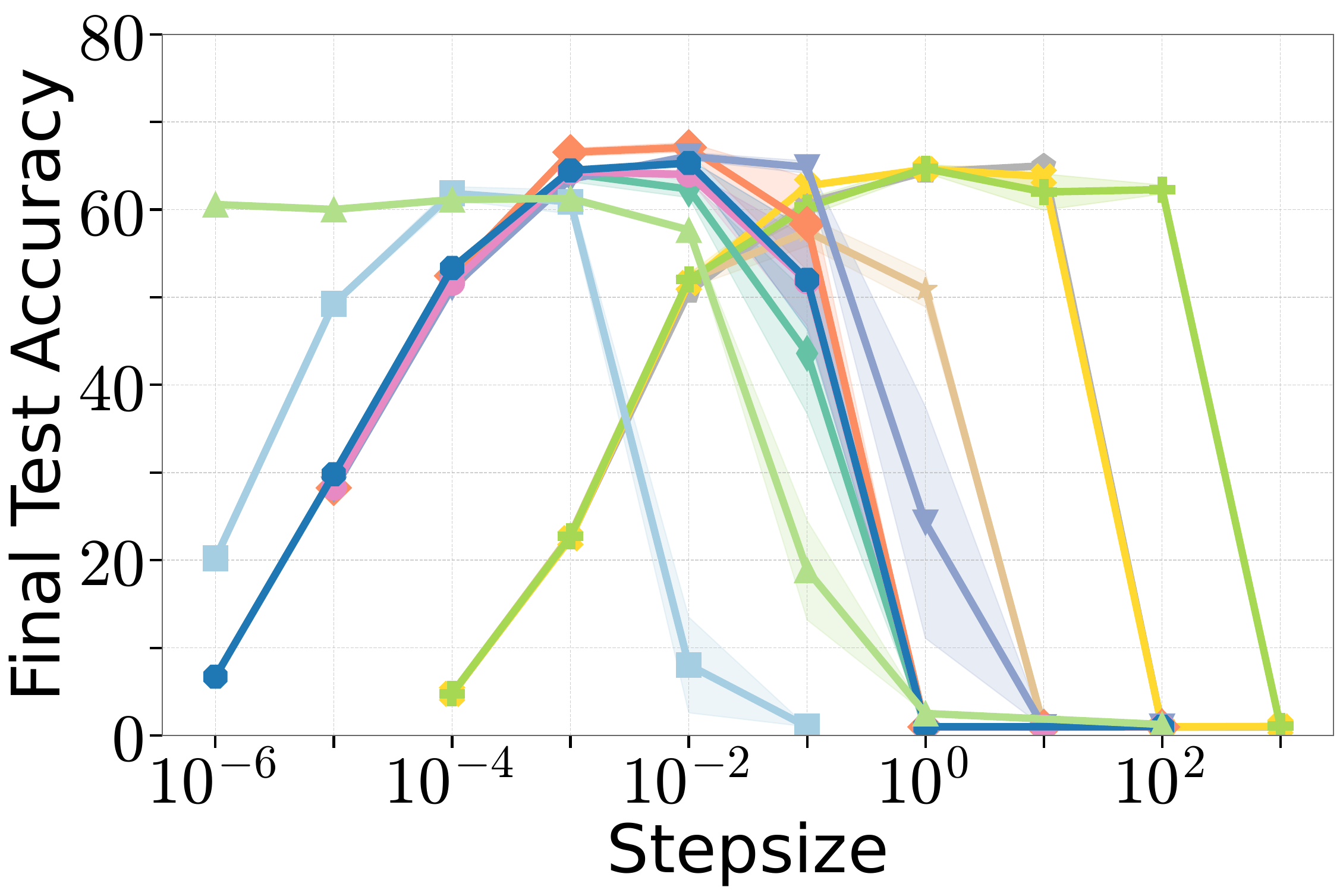} &
        \includegraphics[width=0.3\linewidth]{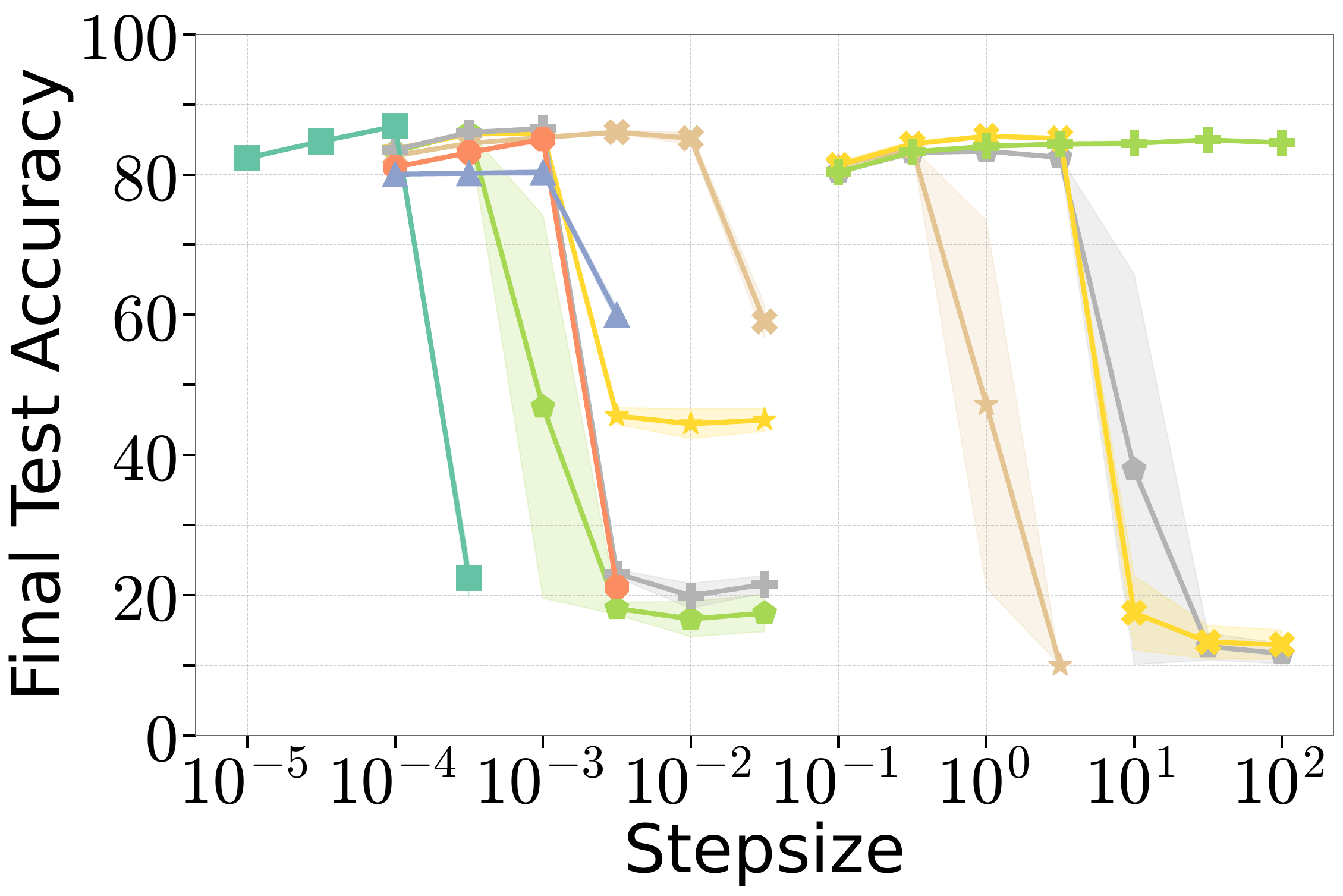} \\
       \includegraphics[width=0.3\linewidth]{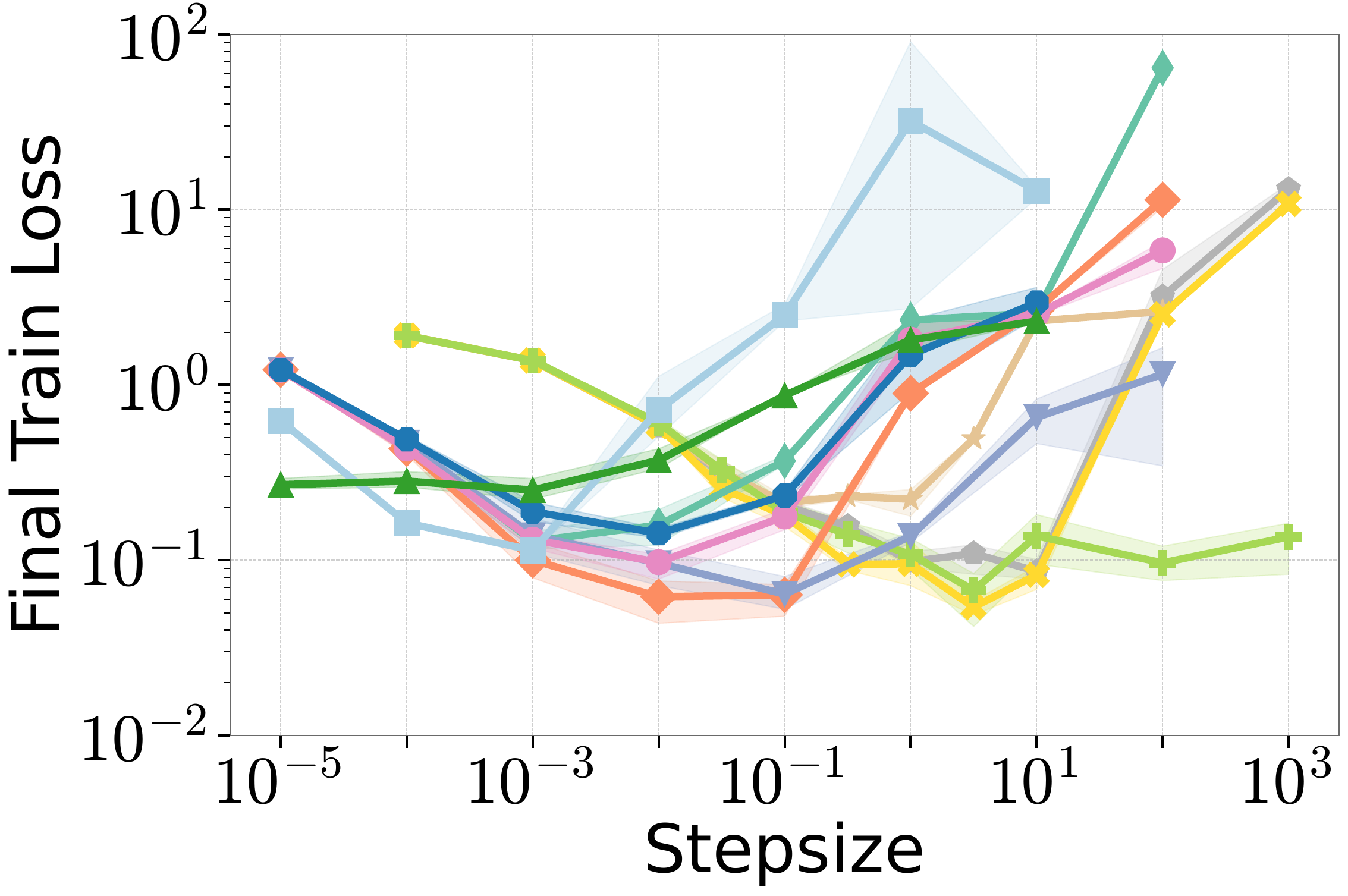} &
       \includegraphics[width=0.3\linewidth]{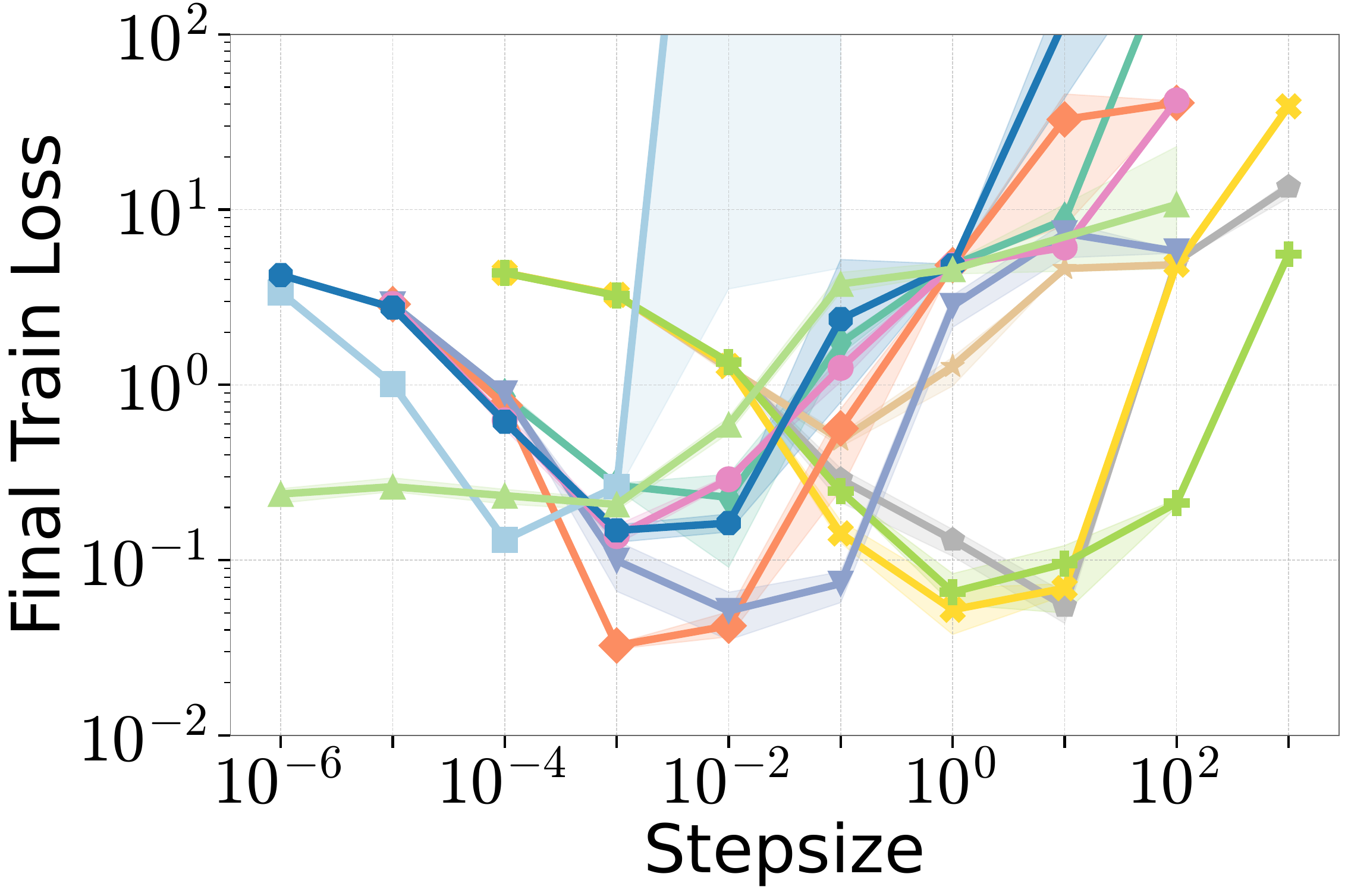} &
        \includegraphics[width=0.3\linewidth]{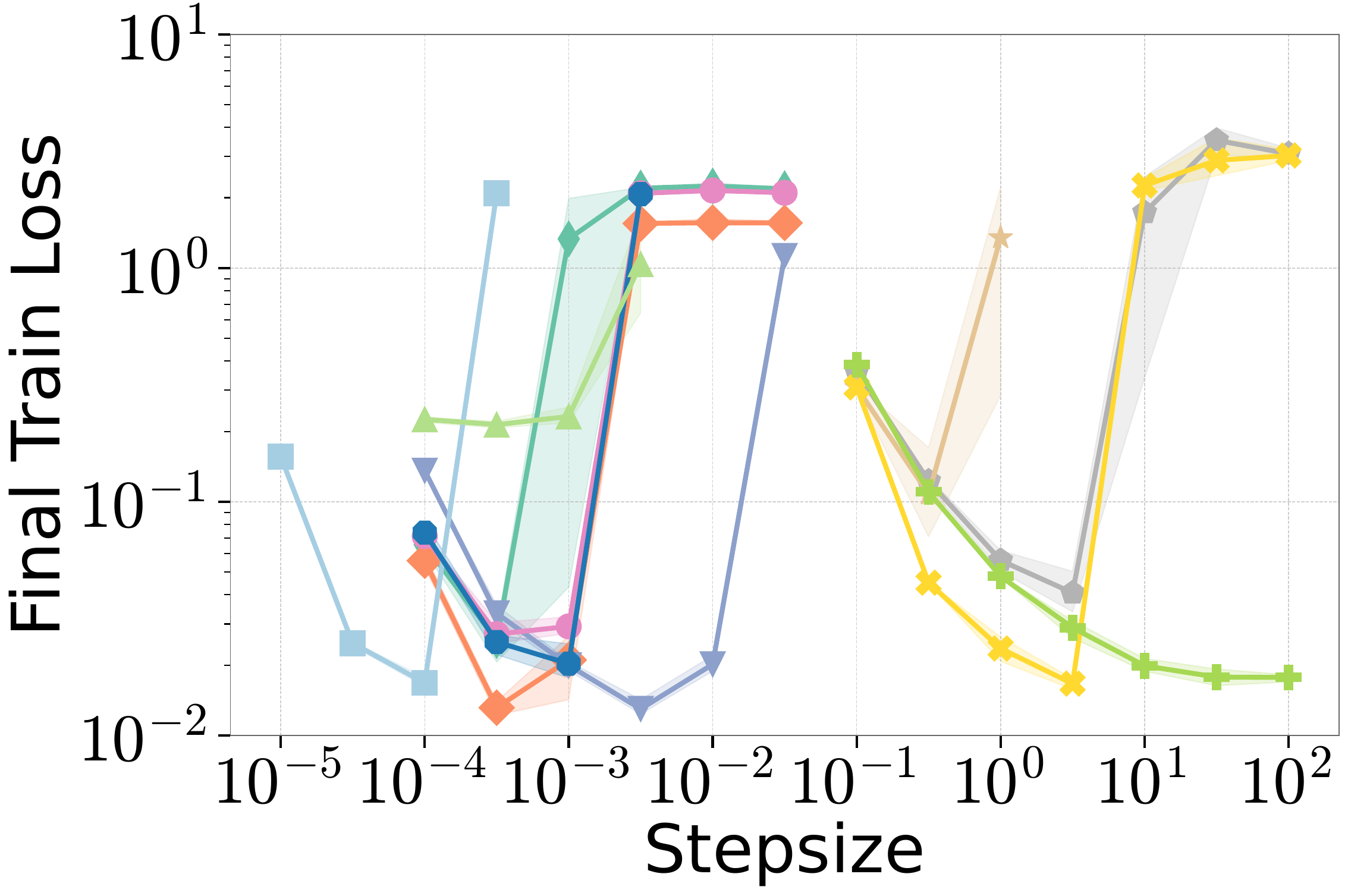} \\
        {\small Resnet20 for CIFAR10} & 
        {\small Resnet110 for CIFAR100} &
        {\small ViT for CIFAR10}
    \end{tabular}
    
    \caption{Stability performance of various optimizers for (Resnet20, CIFAR10), (Resnet110, CIFAR100), (ViT, CIFAR10) workloads.}
    \label{fig:also_adabound_adabelief_lion_resnet20}
\end{figure*}

\subsection{Comparison of Adaptive Step-sizes of \algname{Adam}, \algname{Momo-Adam}, and \algname{NGN-MDv1}}

Next, we conduct experiments to compare the adaptive step-size of \algname{Adam}, \algname{Momo-Adam}, and \algname{NGN-MDv1}. Note that ResNet20 model consists of $3$ base blocks, and each block has $3$ convolution layers. In \Cref{fig:rebuttals_resnet20_stepsize} we plot the average adaptive step-size of the layers $j \in\{\text{layer1.0.conv1, layer2.0.conv1, layer3.0.conv1}\}$ of ResNet20 that corresponds to the first convolution layer within each base block. Similarly, in \Cref{fig:rebuttals_vit_stepsize} we plot the average adaptive step-size of the layers $j \in\{ \text{layer0.0.fn.to\_qkv, layer3.0.fn.to\_qkv, layer5.0.fn.to\_qkv} \}$ that corresponds to the attention layers of the first, fourth, and sixth base blocks.  

Since the adaptivity of \algname{Adam} is only in the second-order momentum applied as a normalization, in our experiment we compare the following quantities 
\begin{equation}
    \frac{\gamma}{(\mD_k)_{(j)}} \text{ for \algname{Adam}}, \quad 
    \frac{\tau_k}{(\mD_k)_{(j)}} \text{ for \algname{Momo-Adam}}, \quad 
    \frac{\gamma_k}{(\mD_k)_{(j)}} \text{ for \algname{NGN-MDv1}},
\end{equation}
where $\gamma$ is the step-size hyperparameter of \algname{Adam}.

Let us first describe the results for ResNet20 in \Cref{fig:rebuttals_resnet20_stepsize}. We observe that \algname{NGN-MDv1} tends to set smaller effective step-size compared to two other algorithms. This is especially visible for the large step-size hyperparameter values where the adaptive step-size of \algname{NGN-MDv1} is by several orders in magnitude smaller than that of \algname{Adam} and \algname{Momo-Adam}. In contrast, the coordinate-wise adaptive step-size of \algname{Momo-Adam} is mostly follow that of \algname{Adam}. Considering that the stability performance of \algname{NGN-MDv1} is much higher for this task, this happens mainly due to the fact that the adaptation mechanism of \algname{NGN-MDv1} step-size is more conservative than that of \algname{Momo-Adam}. 

Now we switch to the results on ViT model in \Cref{fig:rebuttals_vit_stepsize}. Here both \algname{Momo-Adam} and \algname{NGN-MDv1} tend to utilize smaller effective coordinate-wise step-size, by several orders in magnitude smaller than that of \algname{Adam}. However, the adaptation mechanism of \algname{NGN-MDv1} is still more conservative than that of \algname{Momo-Adam}, especially for large step-size hyperparameters. We also highlight that in this experiment the best performance of \algname{NGN-MDv1} is achieved with $c=10^{-3}$. When we vary the step-size hyperparameter $c$, the effective coordinate-wise step-size does not change dramatically, especially for $\text{layers.0.0.fin.to\_qkv}$ layer.

\begin{figure*}[t]
    \centering
    \begin{tabular}{ccc}
     \multicolumn{3}{c}{\includegraphics[width=0.9\linewidth]{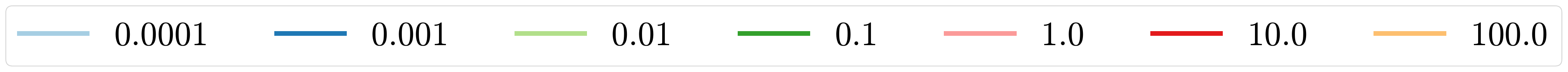}}\\
       \includegraphics[width=0.3\linewidth]{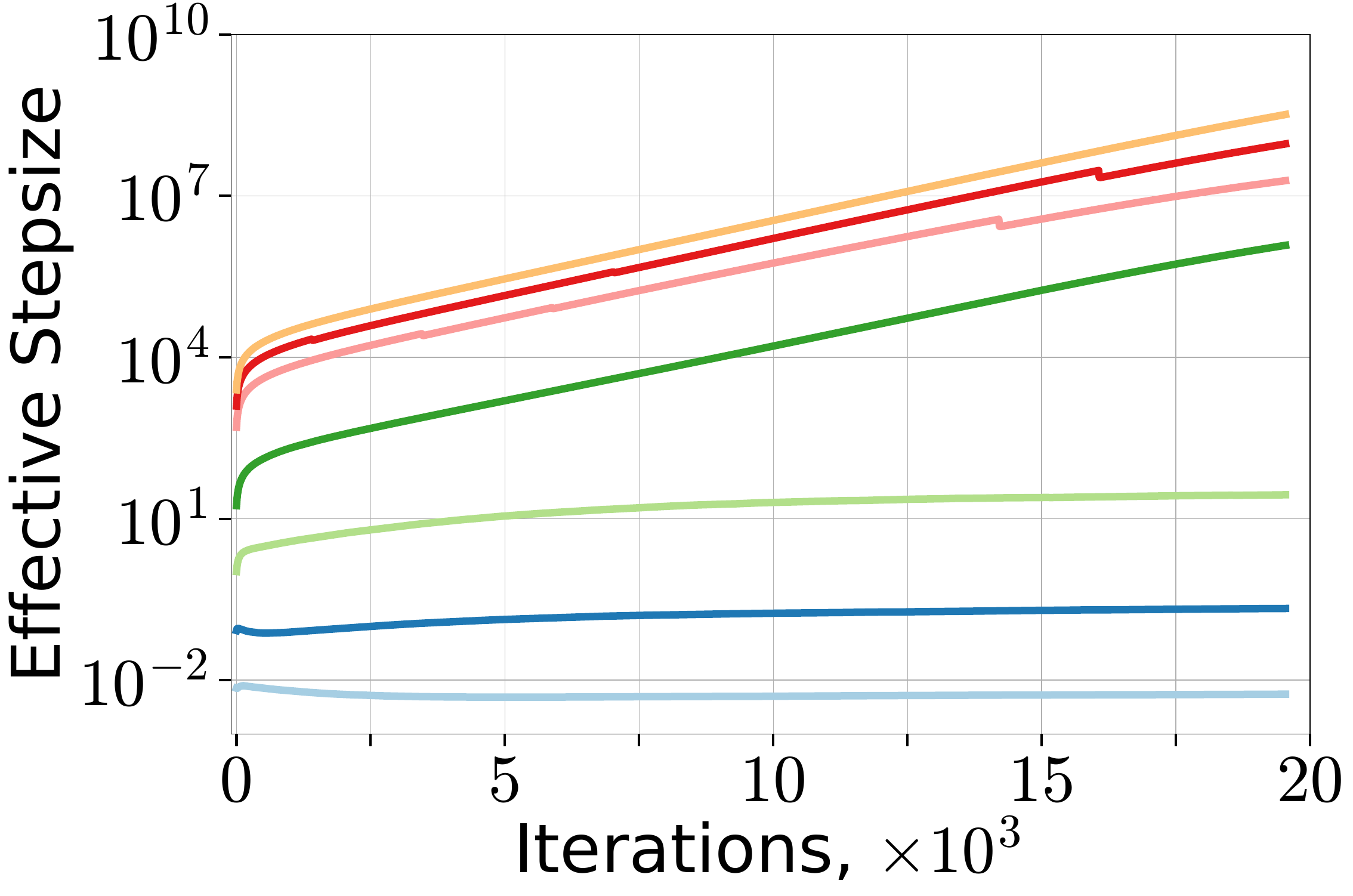} &
       \includegraphics[width=0.3\linewidth]{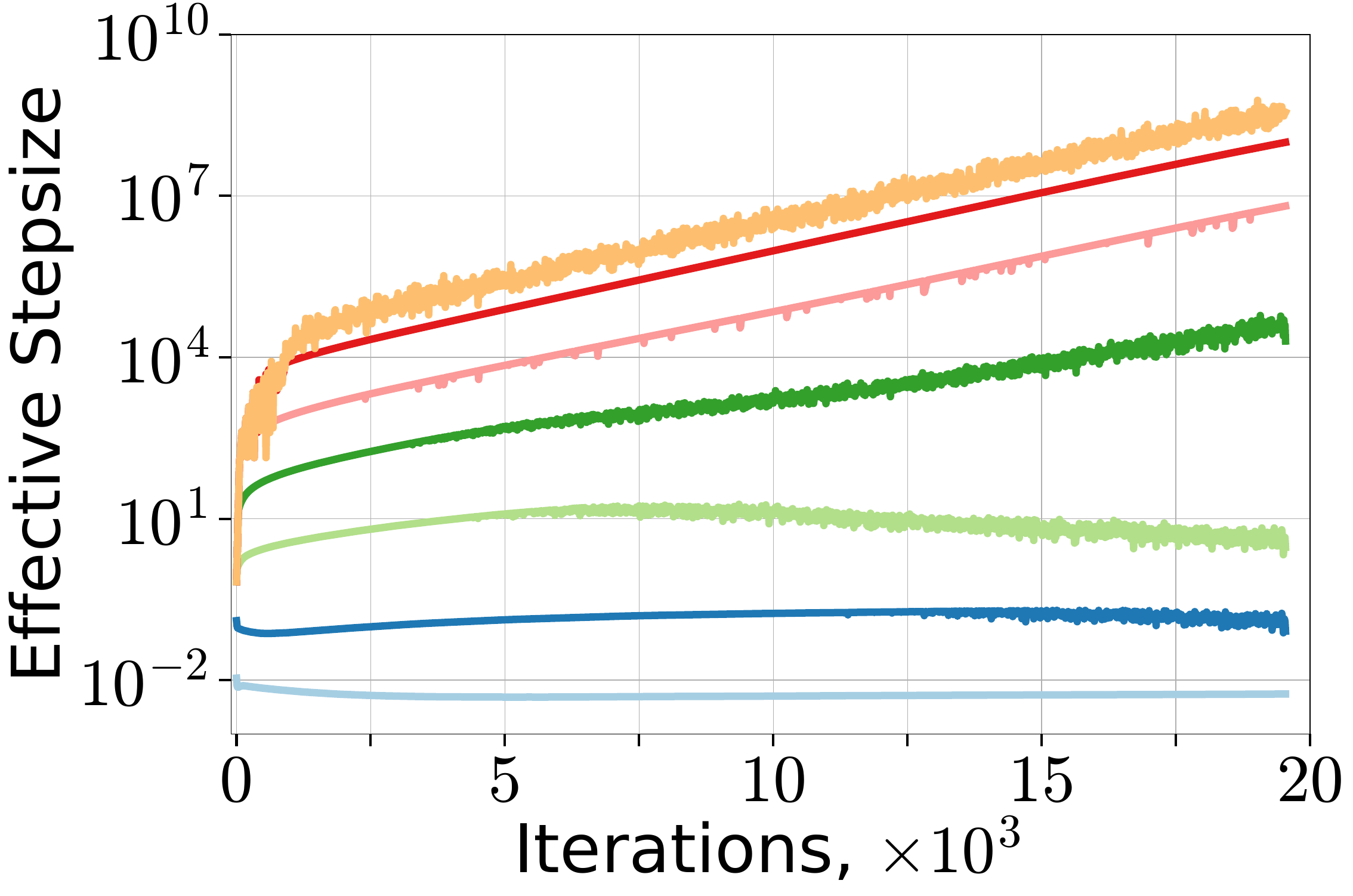} & 
       \includegraphics[width=0.3\linewidth]{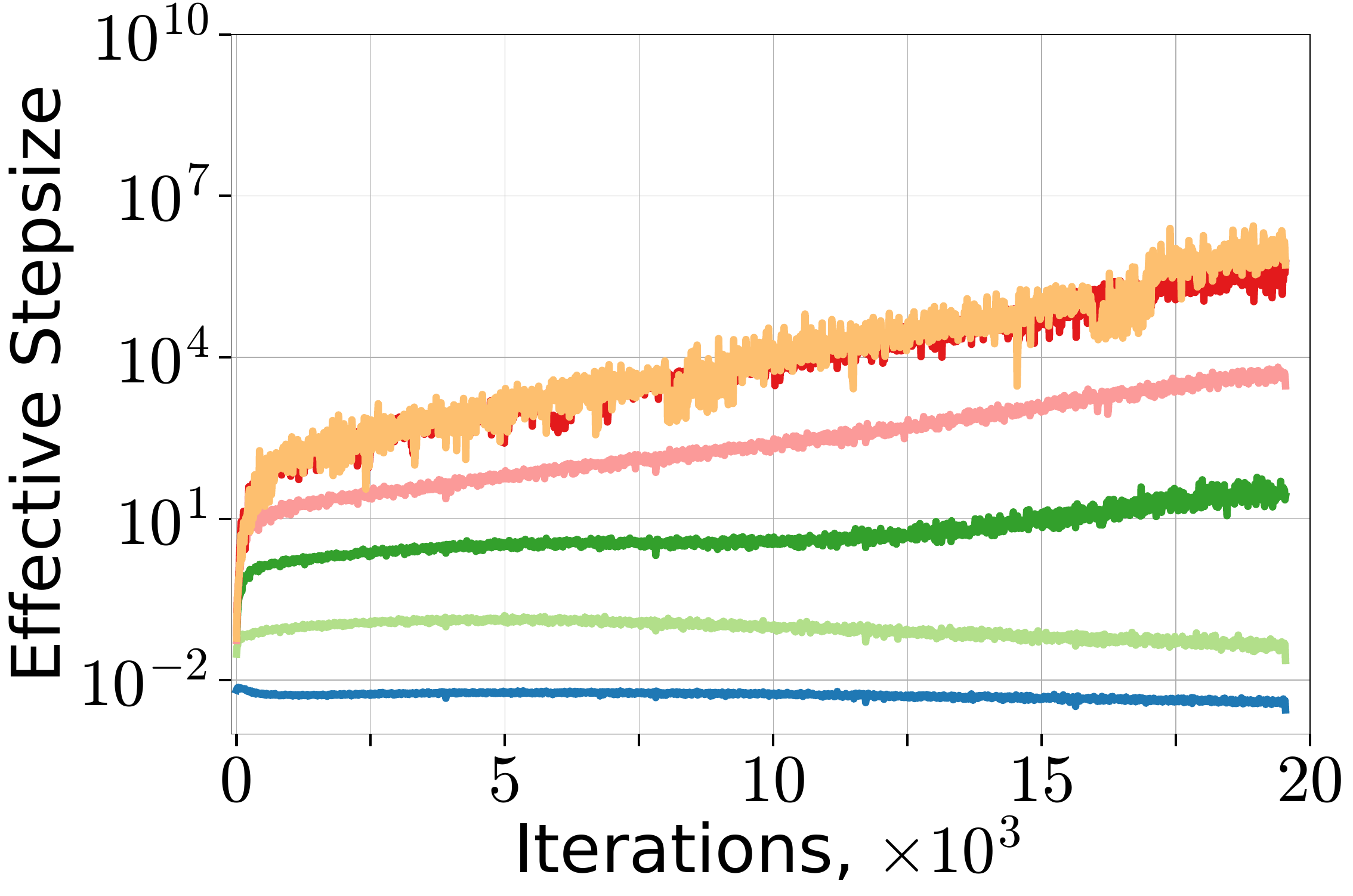} \\
       \makecellnew{{\small $\text{layer1.0.conv1}$} \\
       {\small \algname{Adam}}} &
        \makecellnew{{\small $\text{layer1.0.conv1}$} \\
       {\small \algname{Momo-Adam}}}  &
        \makecellnew{{\small $\text{layer1.0.conv1}$} \\
       {\small \algname{NGN-MDv1}}}  \\
       \includegraphics[width=0.3\linewidth]{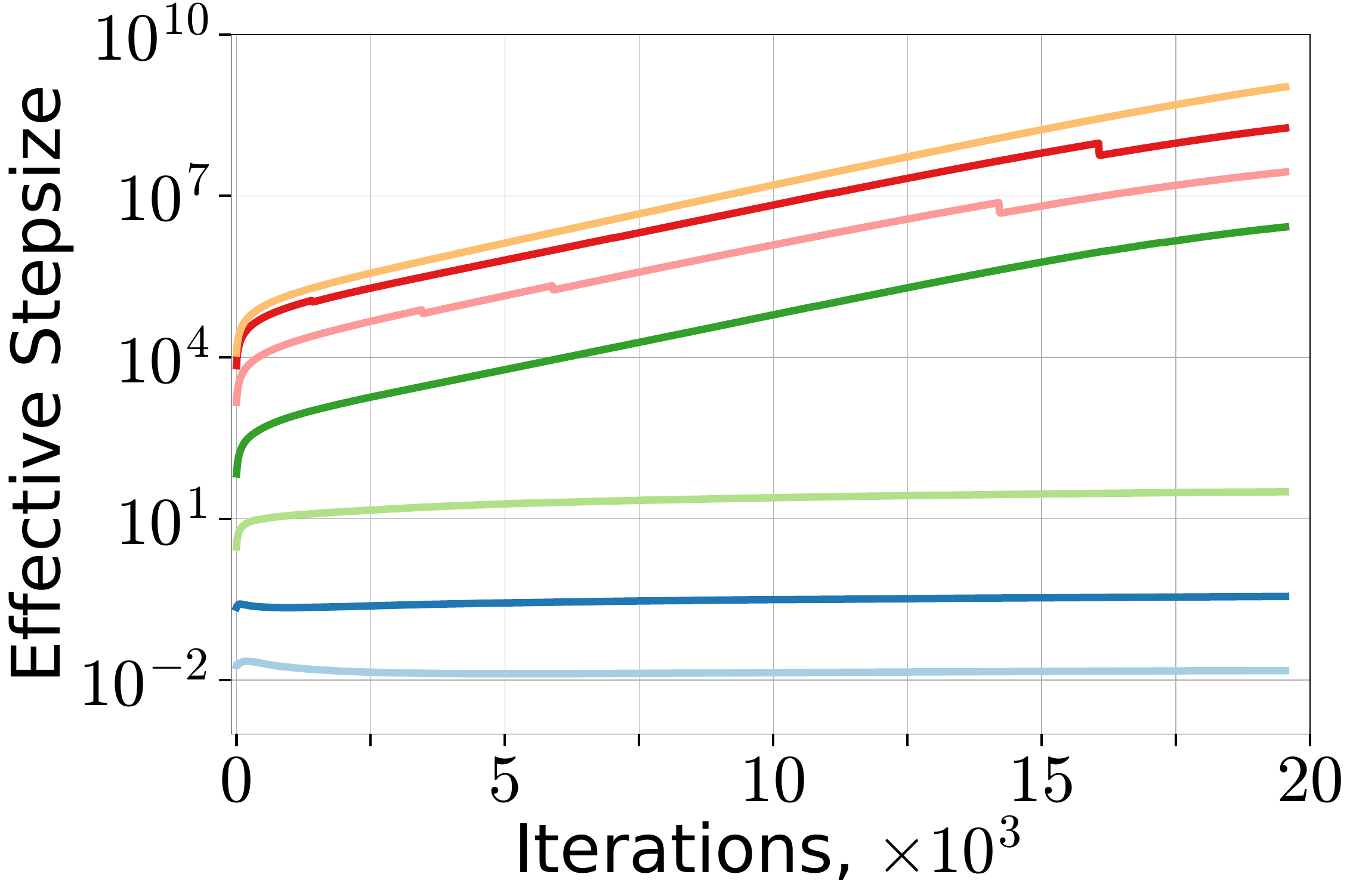} &
       \includegraphics[width=0.3\linewidth]{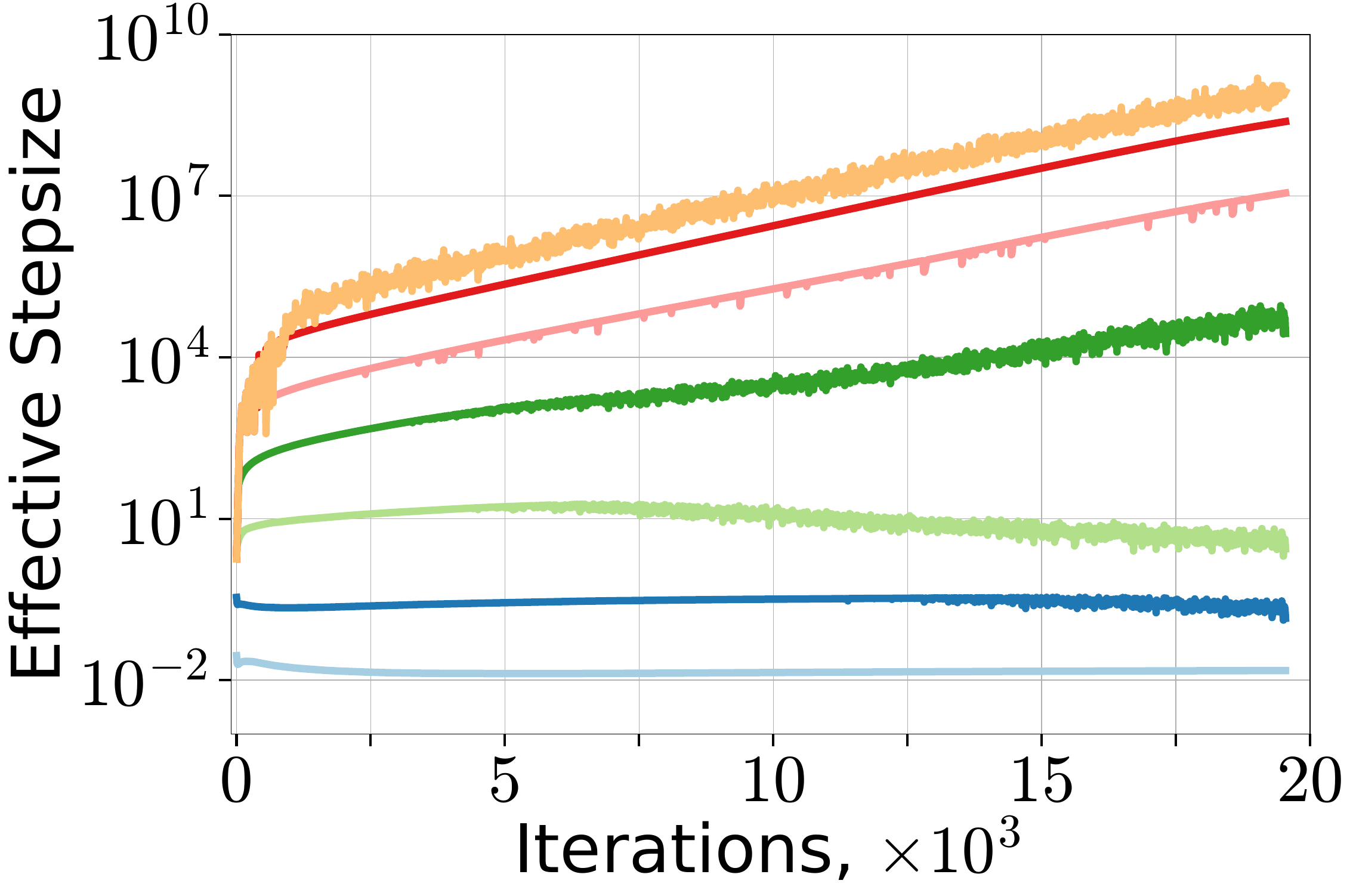} & 
       \includegraphics[width=0.3\linewidth]{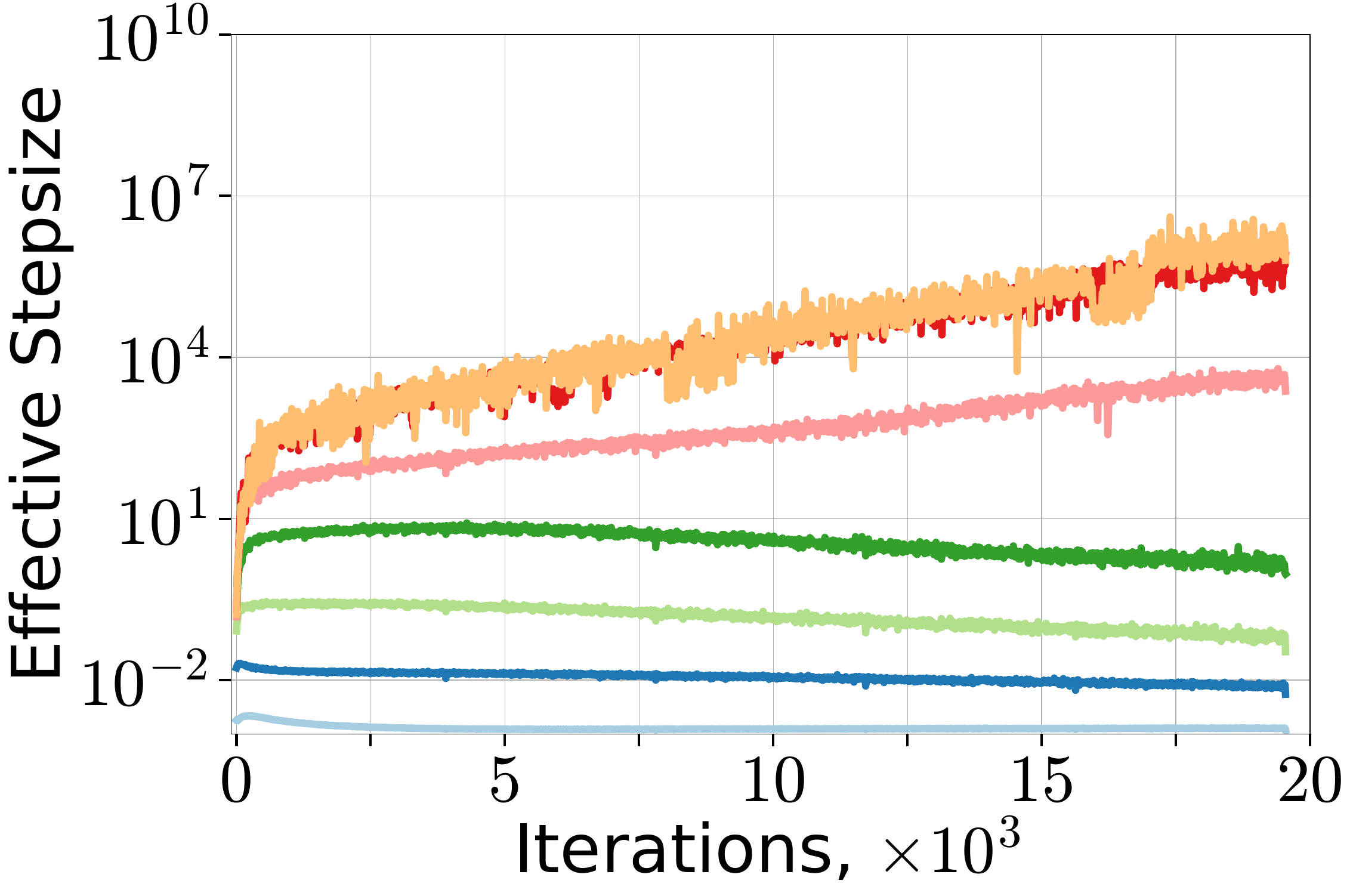} \\
       \makecellnew{{\small $\text{layer2.0.conv1}$} \\
       {\small \algname{Adam}}} &
        \makecellnew{{\small $\text{layer2.0.conv1}$} \\
       {\small \algname{Momo-Adam}}}  &
        \makecellnew{{\small $\text{layer2.0.conv1}$} \\
       {\small \algname{NGN-MDv1}}}  \\
       \includegraphics[width=0.3\linewidth]{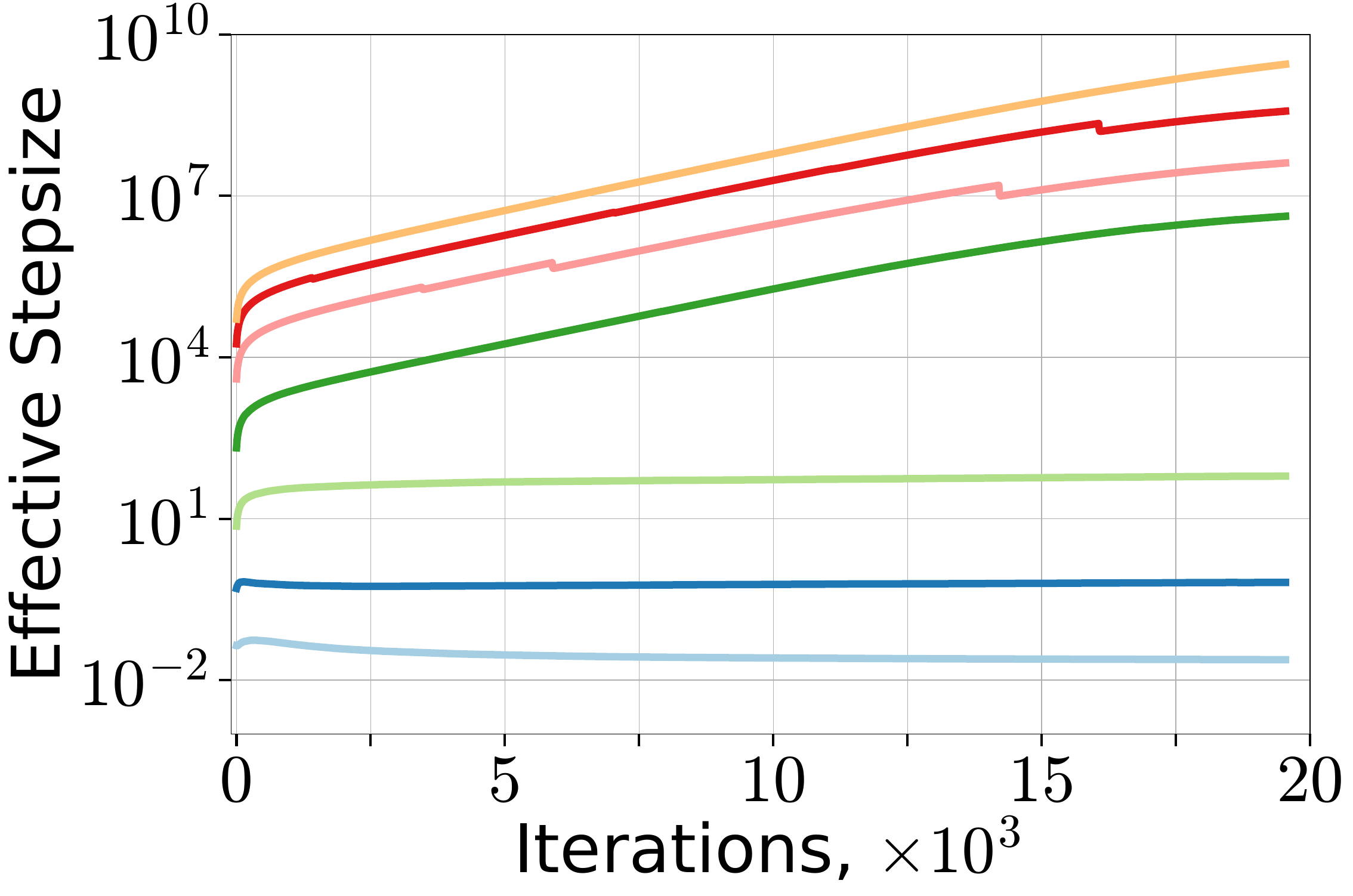} &
       \includegraphics[width=0.3\linewidth]{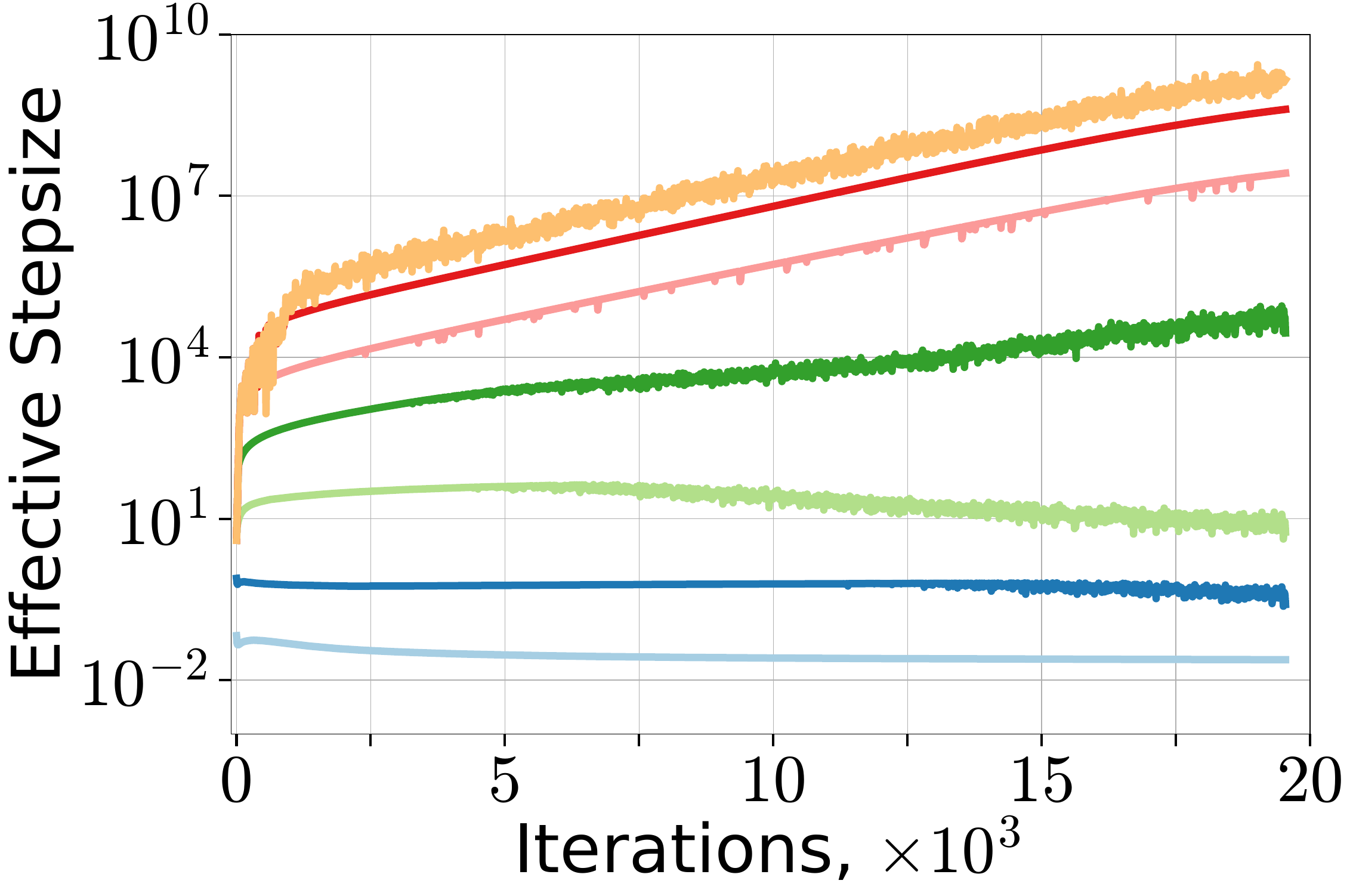} & 
       \includegraphics[width=0.3\linewidth]{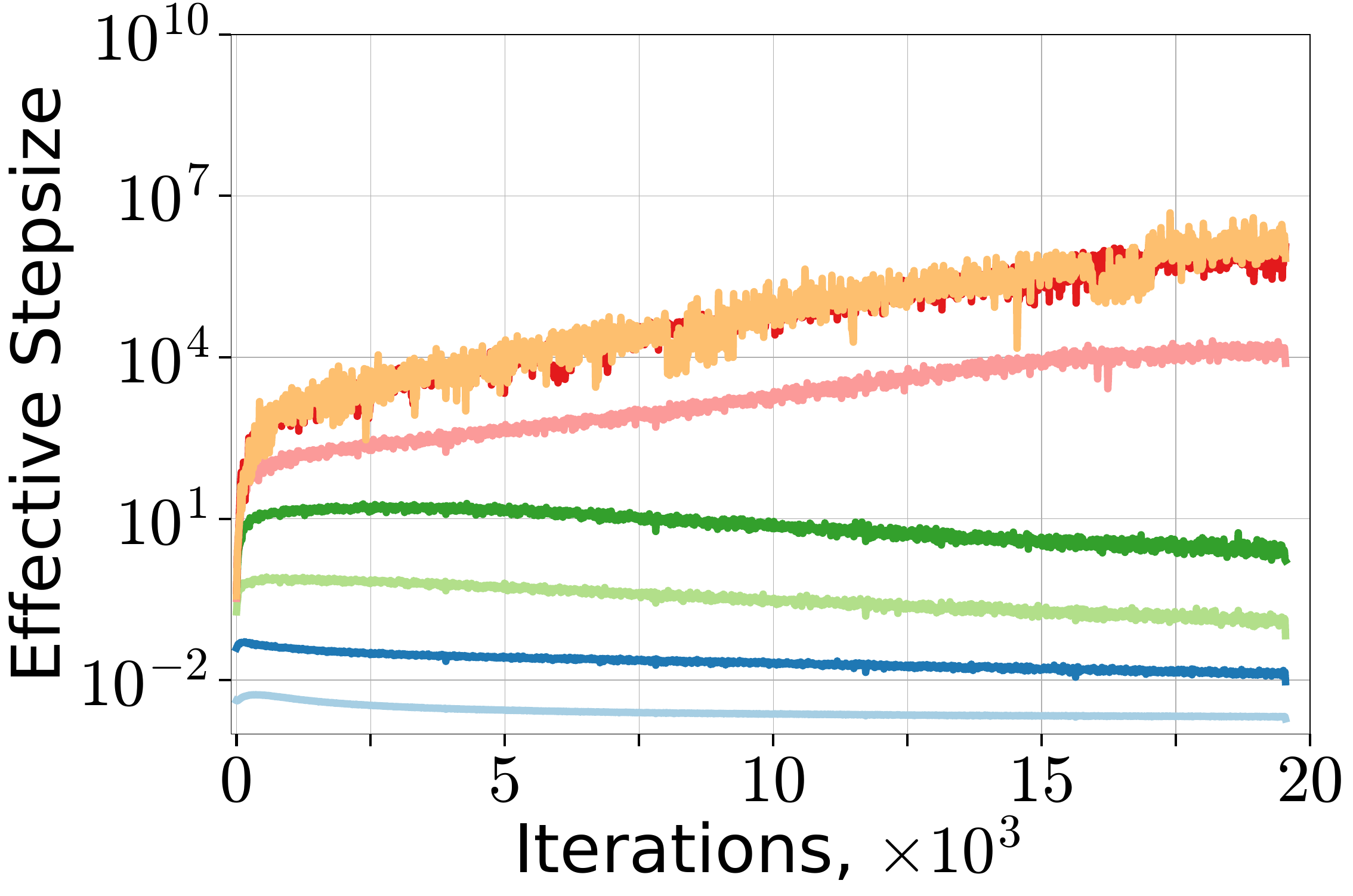} \\
       \makecellnew{{\small $\text{layer3.0.conv1}$} \\
       {\small \algname{Adam}}} &
        \makecellnew{{\small $\text{layer3.0.conv1}$} \\
       {\small \algname{Momo-Adam}}}  &
        \makecellnew{{\small $\text{layer3.0.conv1}$} \\
       {\small \algname{NGN-MDv1}}}  \\
    \end{tabular}
    \caption{The adaptive stepsize of \algname{Adam} ({\bf first column}), \algname{Momo-Adam} ({\bf second column}), and \algname{NGN-MDv1} ({\bf third column}) algorithms in training ResNet20 model on CIFAR10 dataset. We plot the average stepsize $\frac{\gamma}{(\mD_k)_{(j)}}$ (for \algname{Adam}), $\frac{\tau_k}{(\mD_k)_{(j)}}$ (for \algname{Momo-Adam}), and $\frac{\gamma_k}{(\mD_k)_{(j)}}$ (for \algname{NGN-MDv1}) for the first convolution layer within each of $3$ base blocks of ResNet20 architecture varying the step-size hyperparameter of the algorithms ($c$ for \algname{NGN-M} and \algname{NGN}, $\alpha_0$ for \algname{Momo}, and learning rate parameter for \algname{Adam}).
    }
    \label{fig:rebuttals_resnet20_stepsize}
\end{figure*}

\begin{figure*}[t]
    \centering
    \begin{tabular}{ccc}
     \multicolumn{3}{c}{\includegraphics[width=0.8\linewidth]{Plots/legend_adam-type_stepsize_train_loss_stability_comparison_cifar10_vit_512_200.pdf}}\\
       \includegraphics[width=0.3\linewidth]{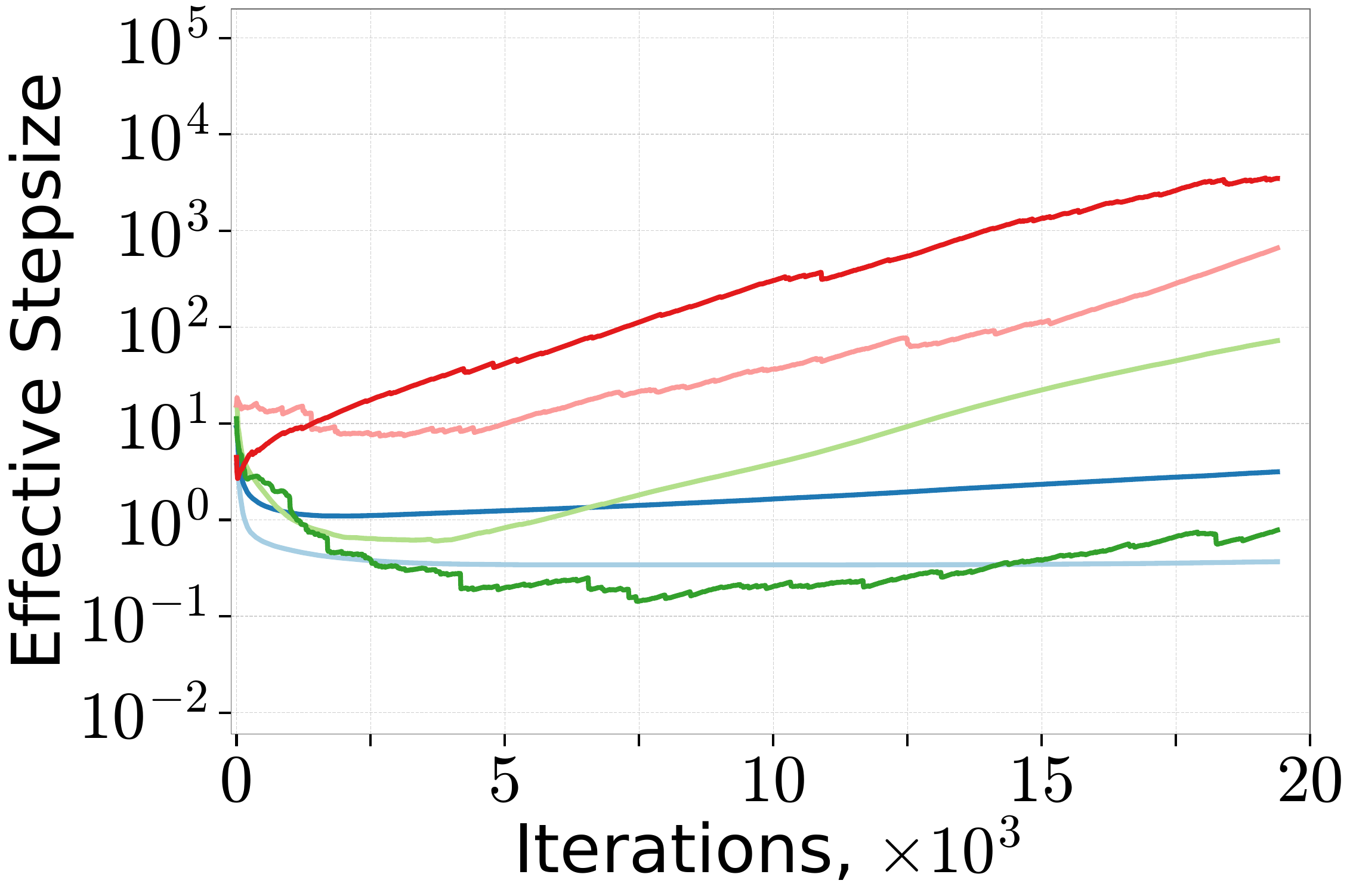} &
       \includegraphics[width=0.3\linewidth]{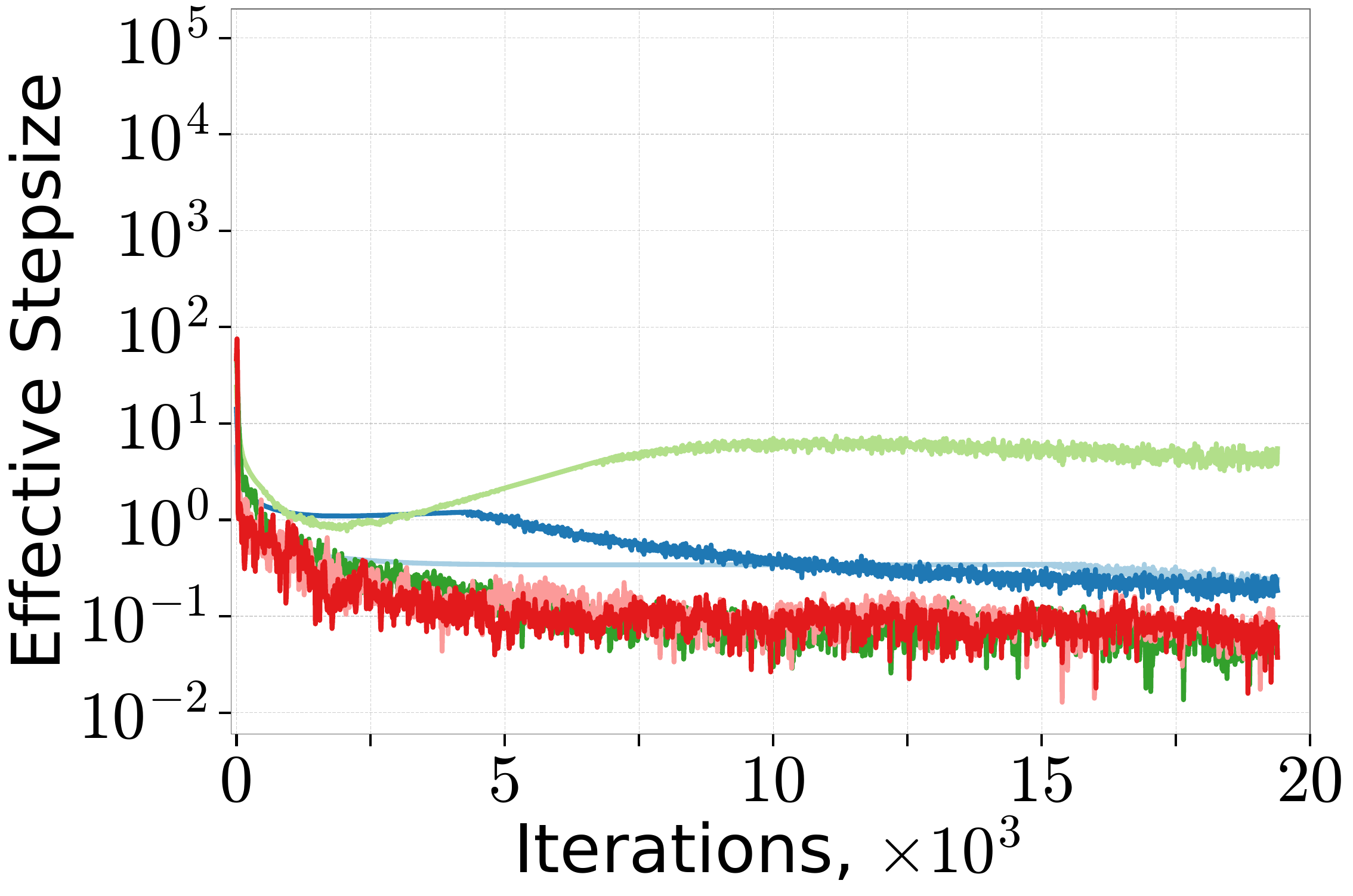} & 
       \includegraphics[width=0.3\linewidth]{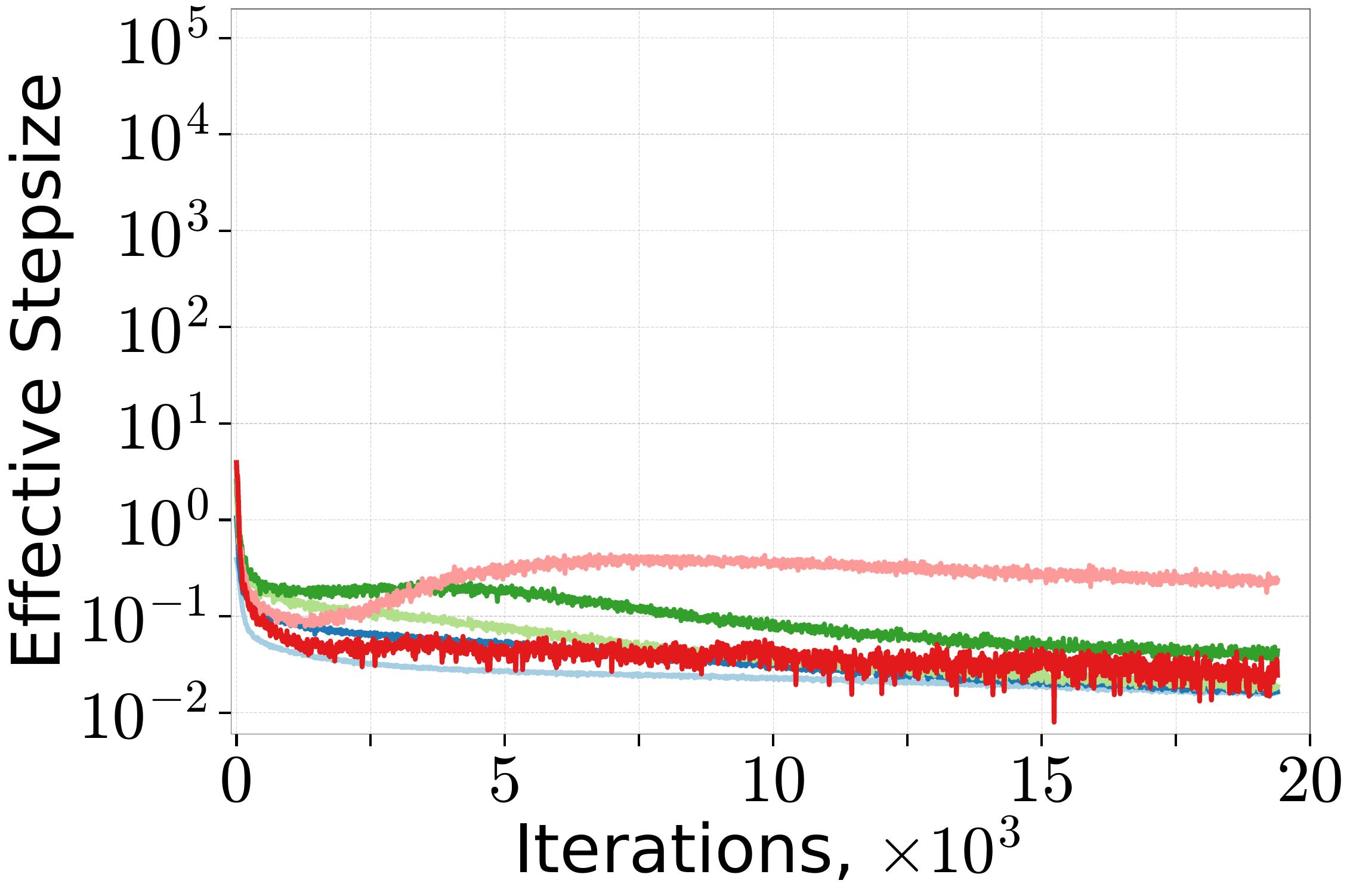} \\
       \makecellnew{{\small $\text{layers.0.0.fn.to\_qkv}$} \\
       {\small \algname{Adam}}} &
        \makecellnew{{\small $\text{layers.0.0.fn.to\_qkv}$} \\
       {\small \algname{Momo-Adam}}}  &
        \makecellnew{{\small $\text{layers.0.0.fn.to\_qkv}$} \\
       {\small \algname{NGN-MDv1}}}  \\
       \includegraphics[width=0.3\linewidth]{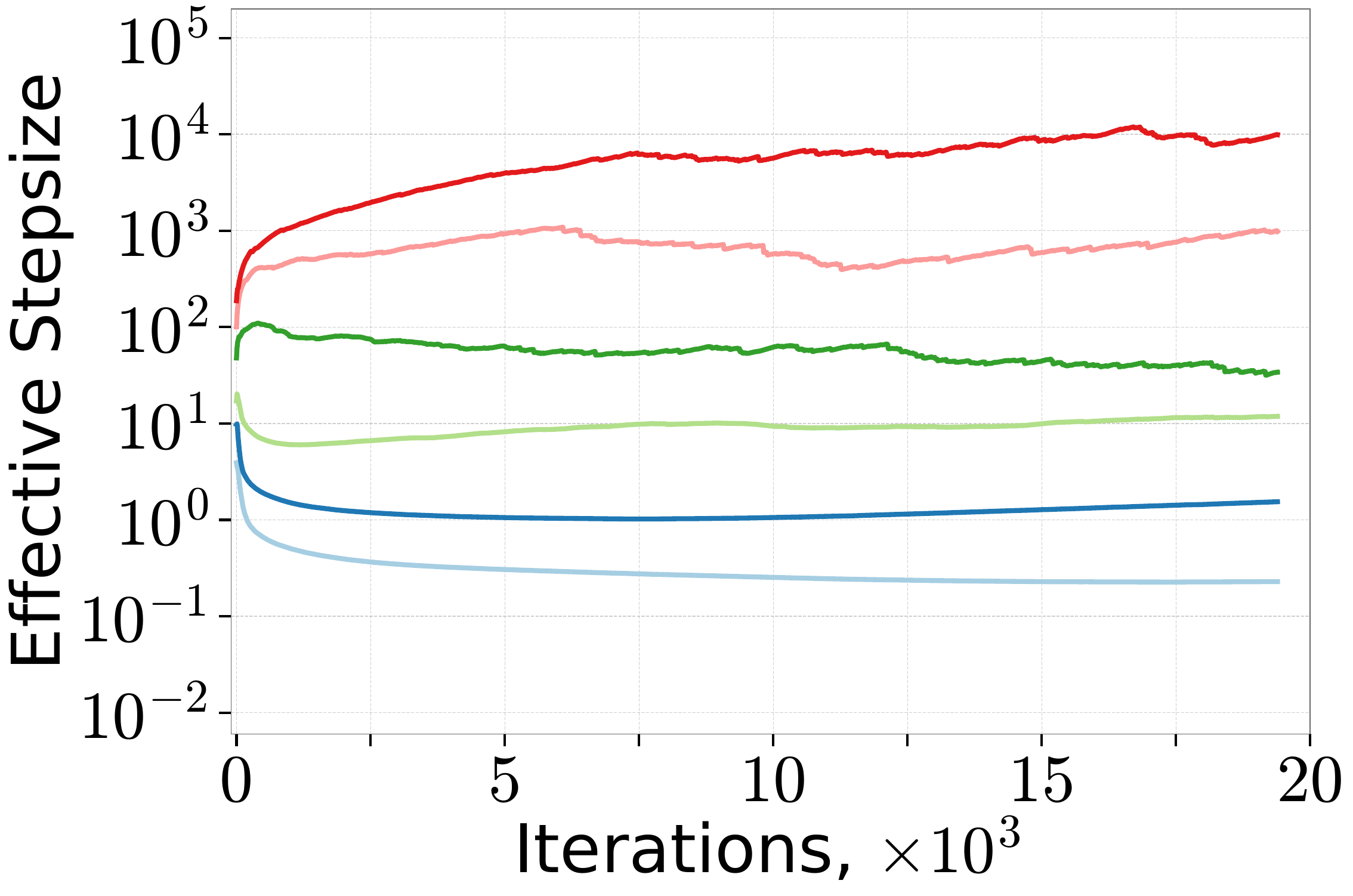} &
       \includegraphics[width=0.3\linewidth]{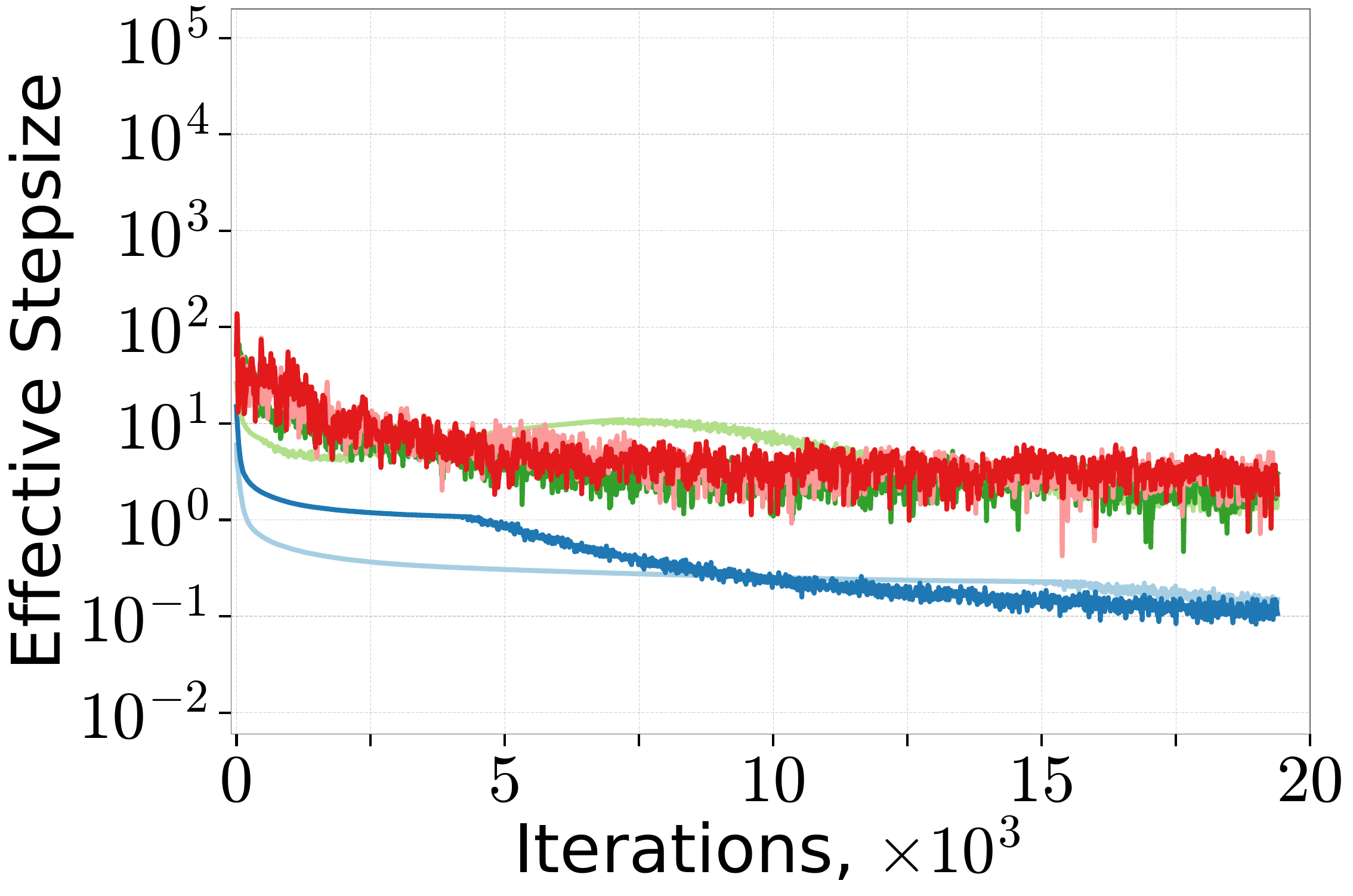} & 
       \includegraphics[width=0.3\linewidth]{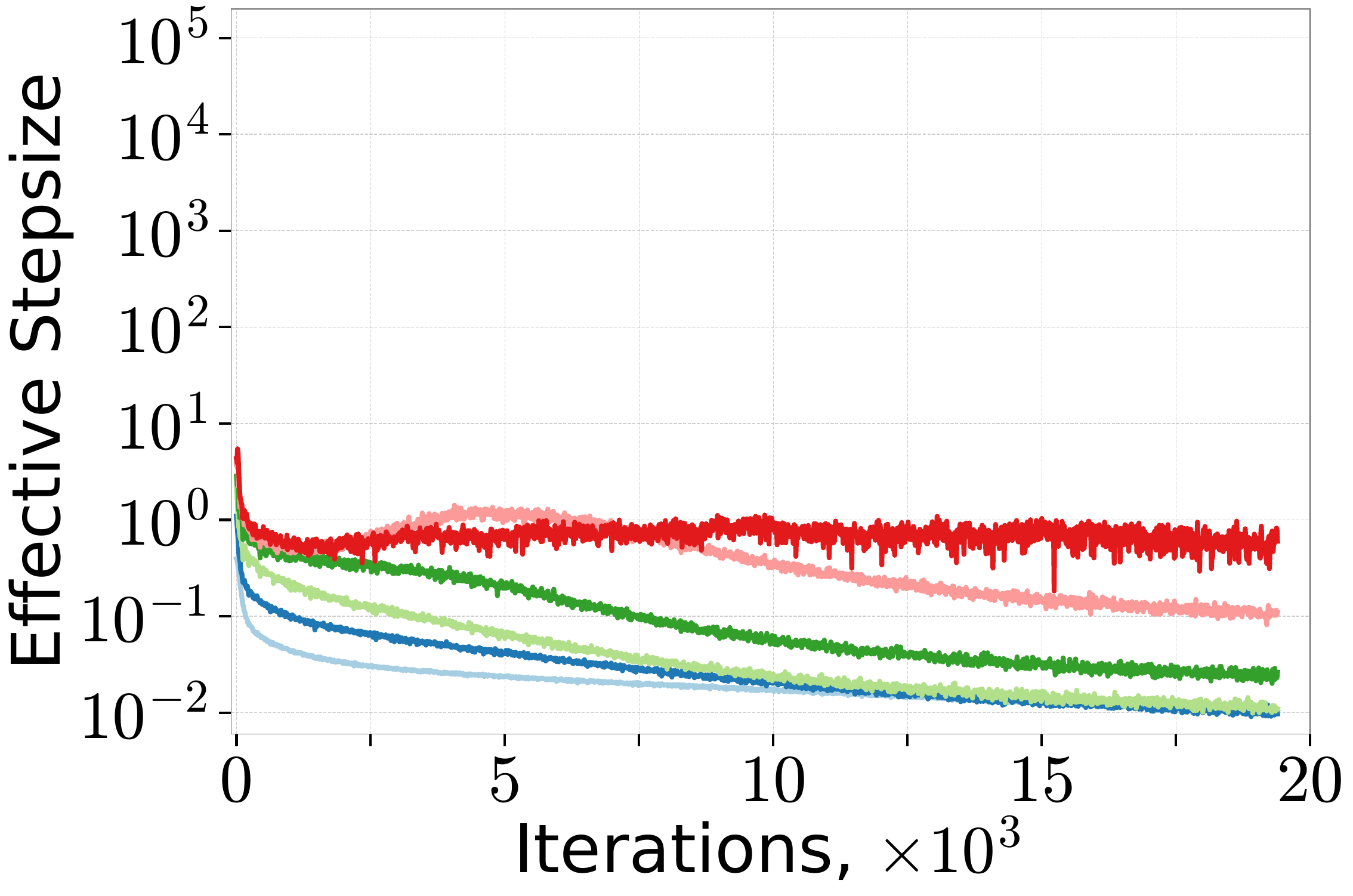} \\
       \makecellnew{{\small $\text{layers.3.0.fn.to\_qkv}$} \\
       {\small \algname{Adam}}} &
        \makecellnew{{\small $\text{layers.3.0.fn.to\_qkv}$} \\
       {\small \algname{Momo-Adam}}}  &
        \makecellnew{{\small $\text{layers.3.0.fn.to\_qkv}$} \\
       {\small \algname{NGN-MDv1}}}  \\
       \includegraphics[width=0.3\linewidth]{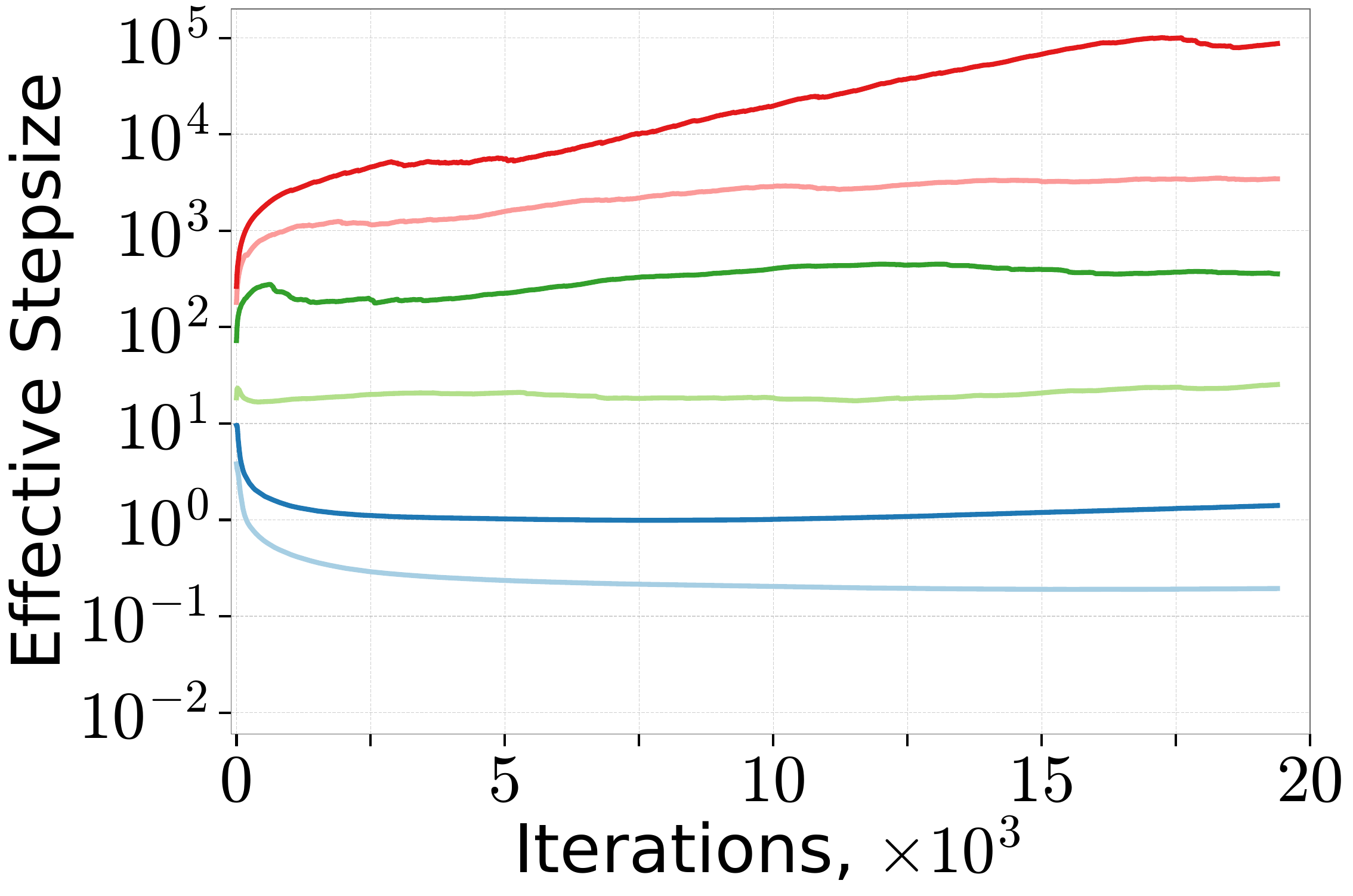} &
       \includegraphics[width=0.3\linewidth]{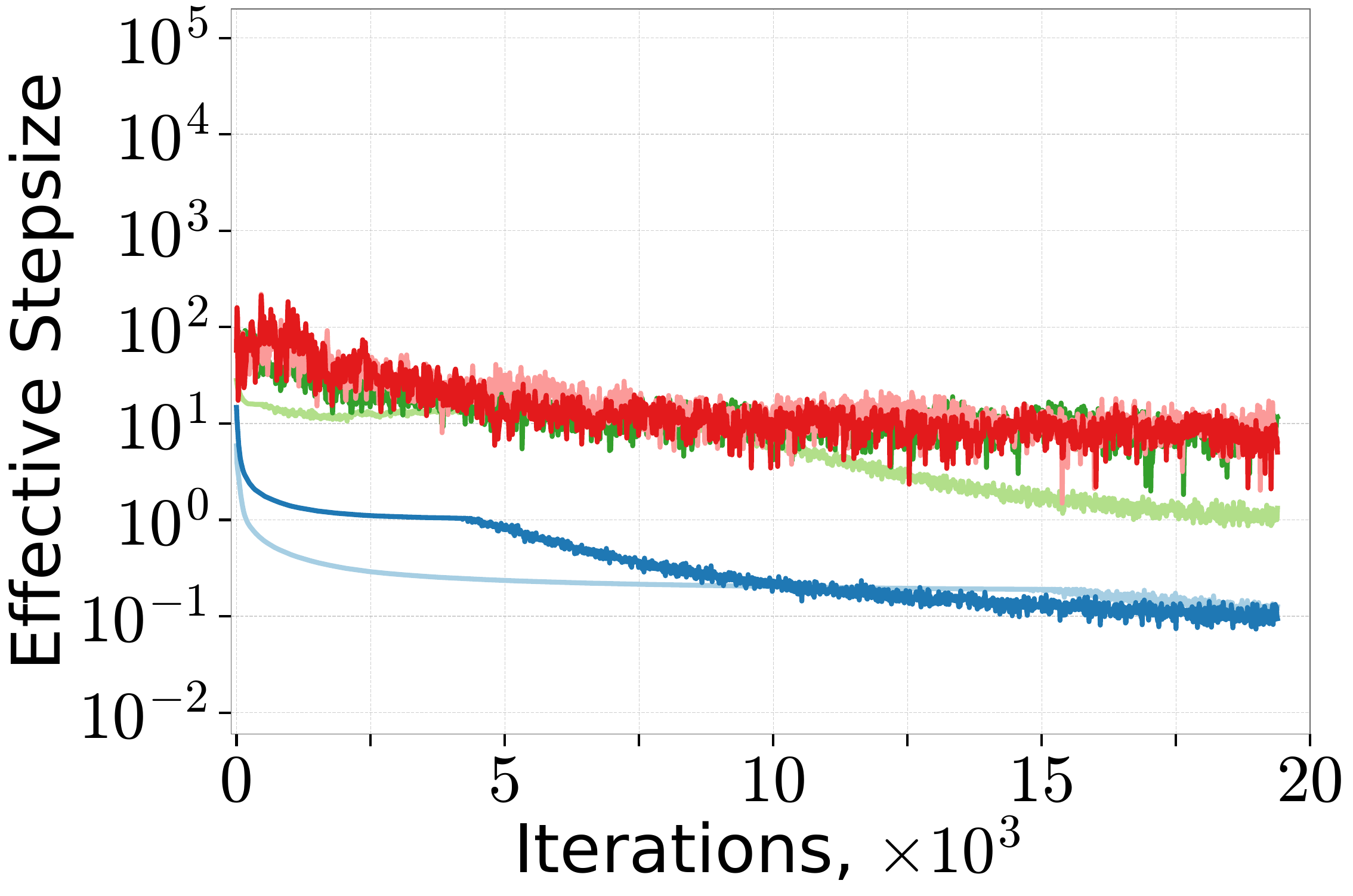} & 
       \includegraphics[width=0.3\linewidth]{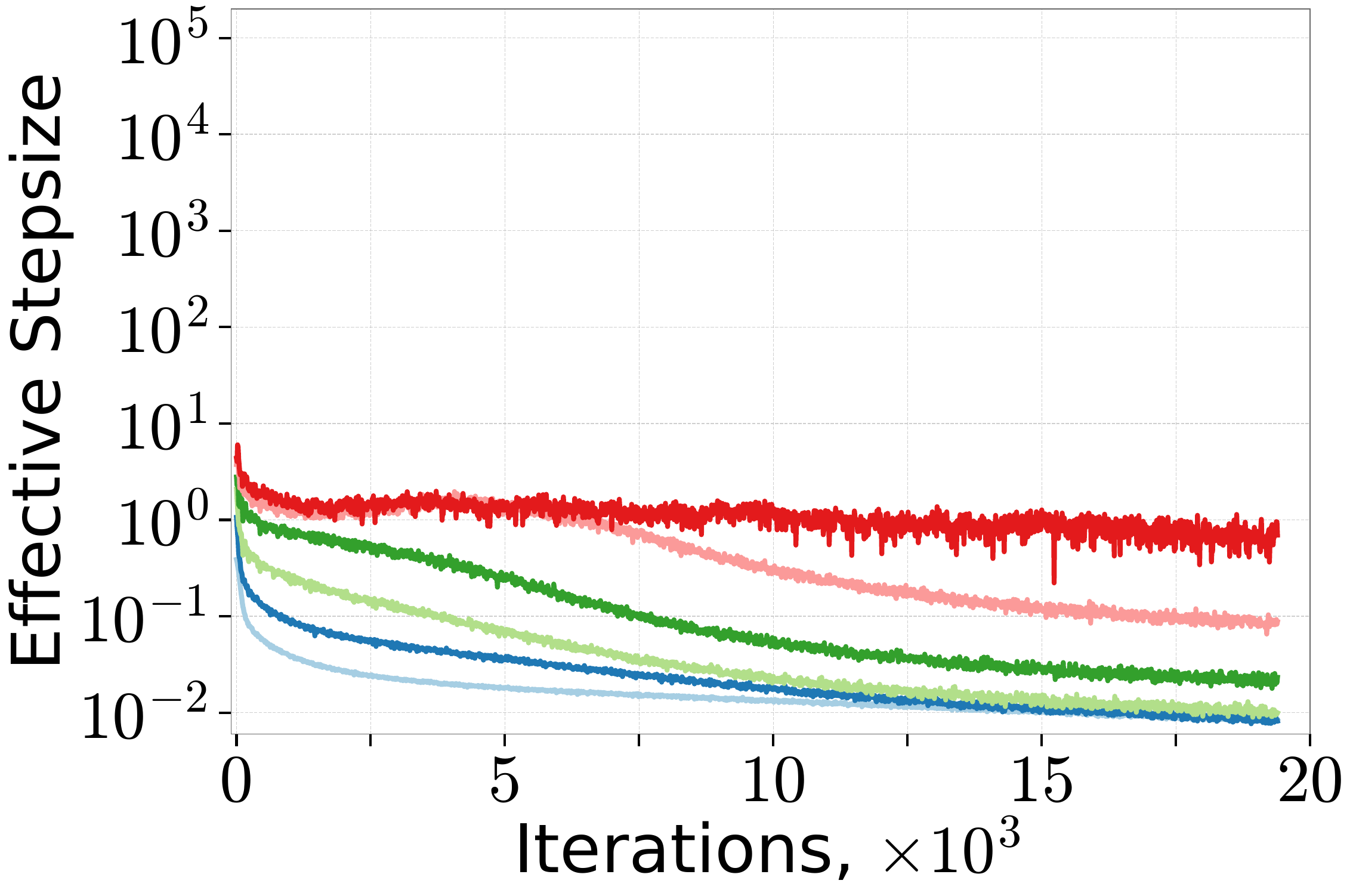} \\
       \makecellnew{{\small $\text{layers.5.0.fn.to\_qkv}$} \\
       {\small \algname{Adam}}} &
        \makecellnew{{\small $\text{layers.5.0.fn.to\_qkv}$} \\
       {\small \algname{Momo-Adam}}}  &
        \makecellnew{{\small $\text{layers.5.0.fn.to\_qkv}$} \\
       {\small \algname{NGN-MDv1}}}  \\
    \end{tabular}
    \caption{The adaptive stepsize of \algname{Adam} ({\bf first column}), \algname{Momo-Adam} ({\bf second column}), and \algname{NGN-MDv1} ({\bf third column}) algorithms in training ViT model on CIFAR10 dataset. We plot the average stepsize $\frac{\gamma}{(\mD_k)_{(j)}}$ (for \algname{Adam}), $\frac{\tau_k}{(\mD_k)_{(j)}}$ (for \algname{Momo-Adam}), and $\frac{\gamma_k}{(\mD_k)_{(j)}}$ (for \algname{NGN-MDv1}) for the attention layer within each of the first, fourth, and sixth base blocks of ViT architecture varying the step-size hyperparameter of the algorithms ($c$ for \algname{NGN-M} and \algname{NGN}, $\alpha_0$ for \algname{Momo}, and learning rate parameter for \algname{Adam}).
    }
    \label{fig:rebuttals_vit_stepsize}
\end{figure*}

\subsection{Extended Comparison of Momentum-based Algorithms on NLP Tasks}\label{sec:nlp_tasks}

We switch to comparison of \algname{NGN-M}, \algname{Momo}, \algname{NGN}, and \algname{SGDM} on NLP tasks. In particular, we consider the training of Transformer (based on NanoGPT) on the Tiny Shakespeare and Rotten Tomatoes datasets and LSTM on the Wikitext-2 dataset from \Cref{sec:adam_type_appendix}. We report the results in \Cref{fig:rebuttals_nlp_tasks} while the best performance is shown in \Cref{tab:empirical_comparison_momentum_appendix}. First, note that all algorithms do not match the best performance of those that incorporate diagonal step-size and momentum (see \Cref{tab:empirical_comparison_adam_type_lion_adabound_adabelief}). Such results are expected since the training of NLP models has significantly different coordinate-wise conditioning. Nonetheless, \algname{NGN-M} algorithm achieves better resilience to the step-size hyperparameter choice, especially in the training of Transformer models. Therefore, \algname{NGN-M} across various model architectures and task domains.

\begin{figure*}[t]
    \centering
    \begin{tabular}{ccc}
     \includegraphics[width=0.3\linewidth]{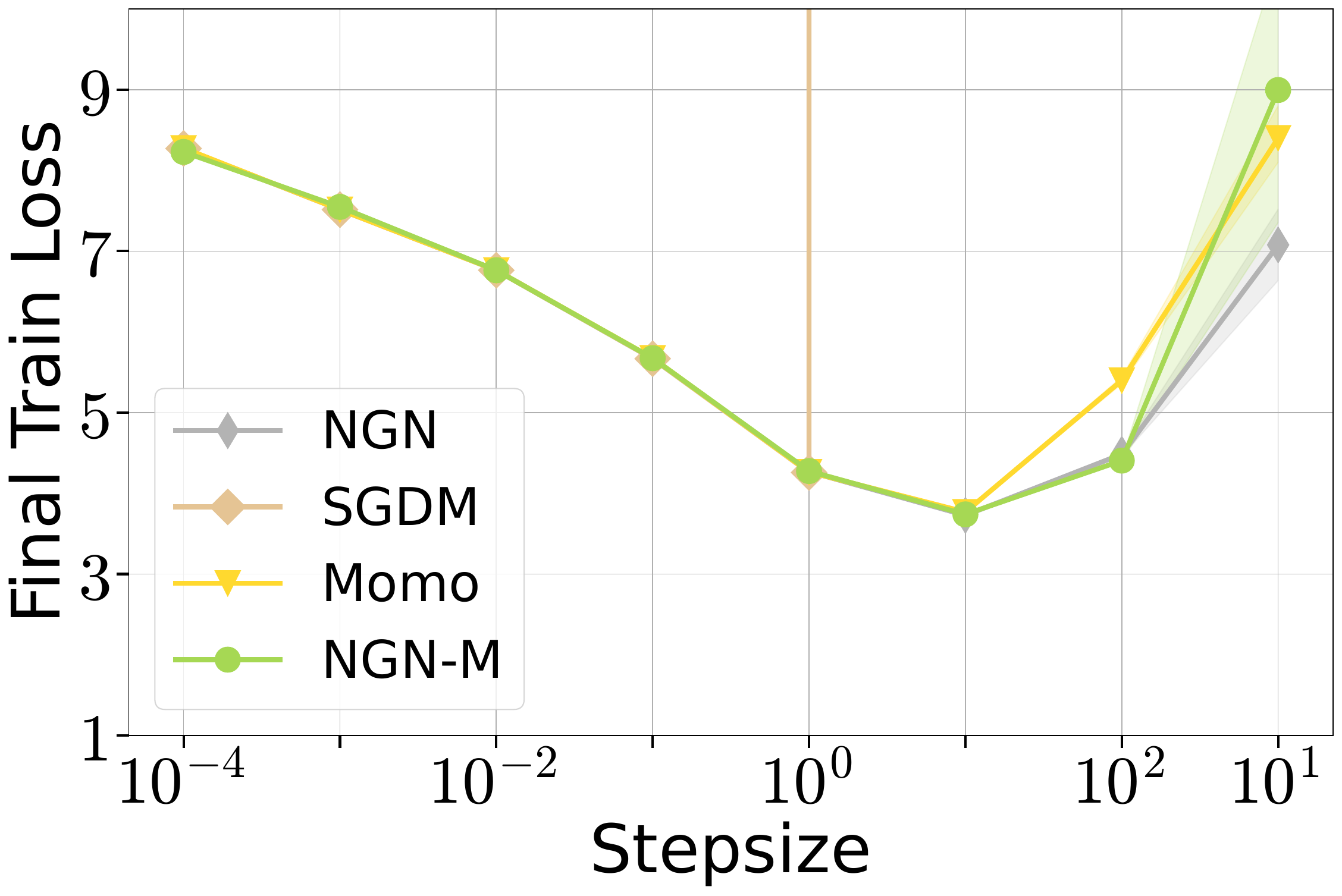} &
     \includegraphics[width=0.3\linewidth]{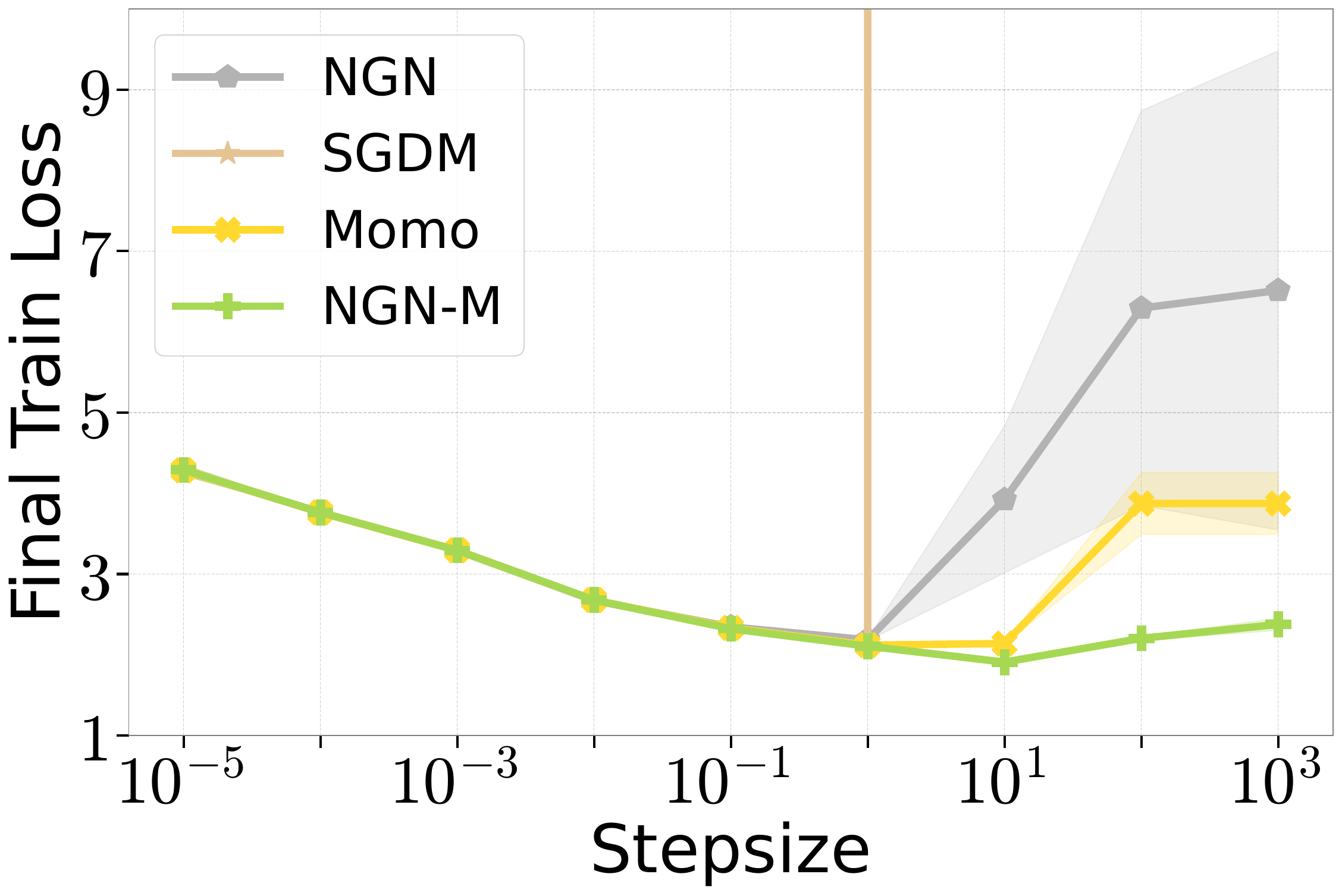} &
     \includegraphics[width=0.3\linewidth]{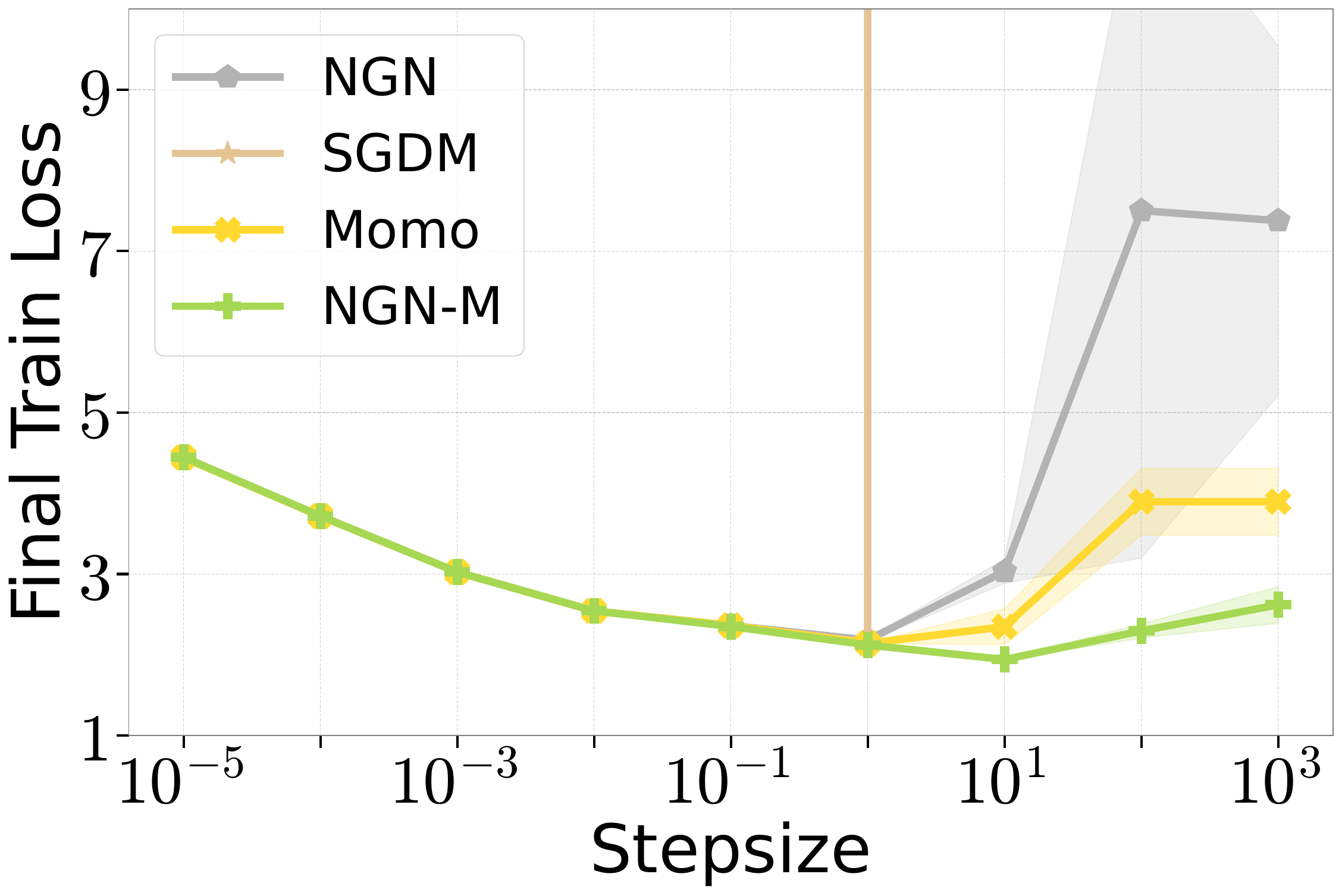} \\
     \includegraphics[width=0.3\linewidth]{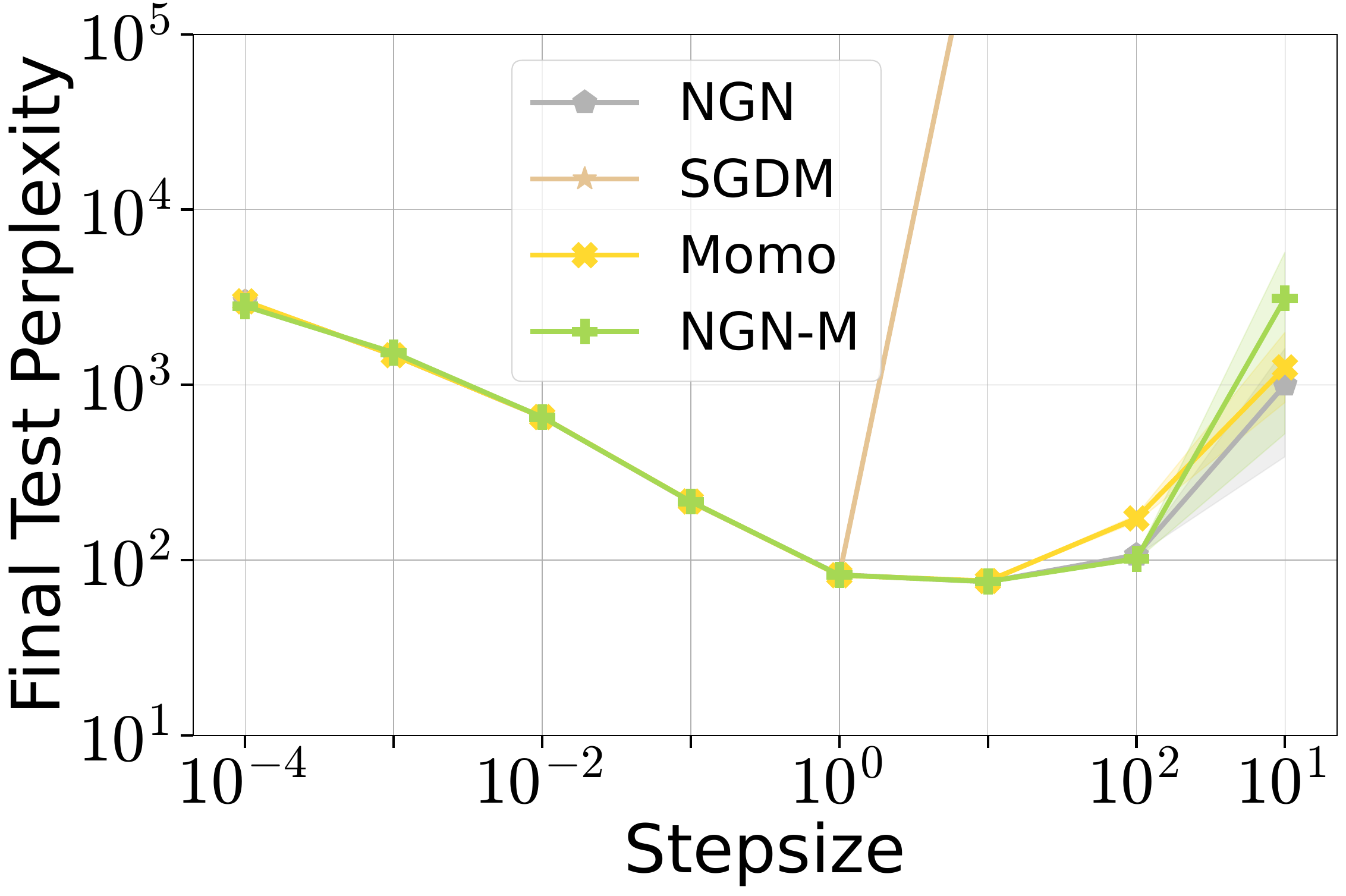} &
     \includegraphics[width=0.3\linewidth]{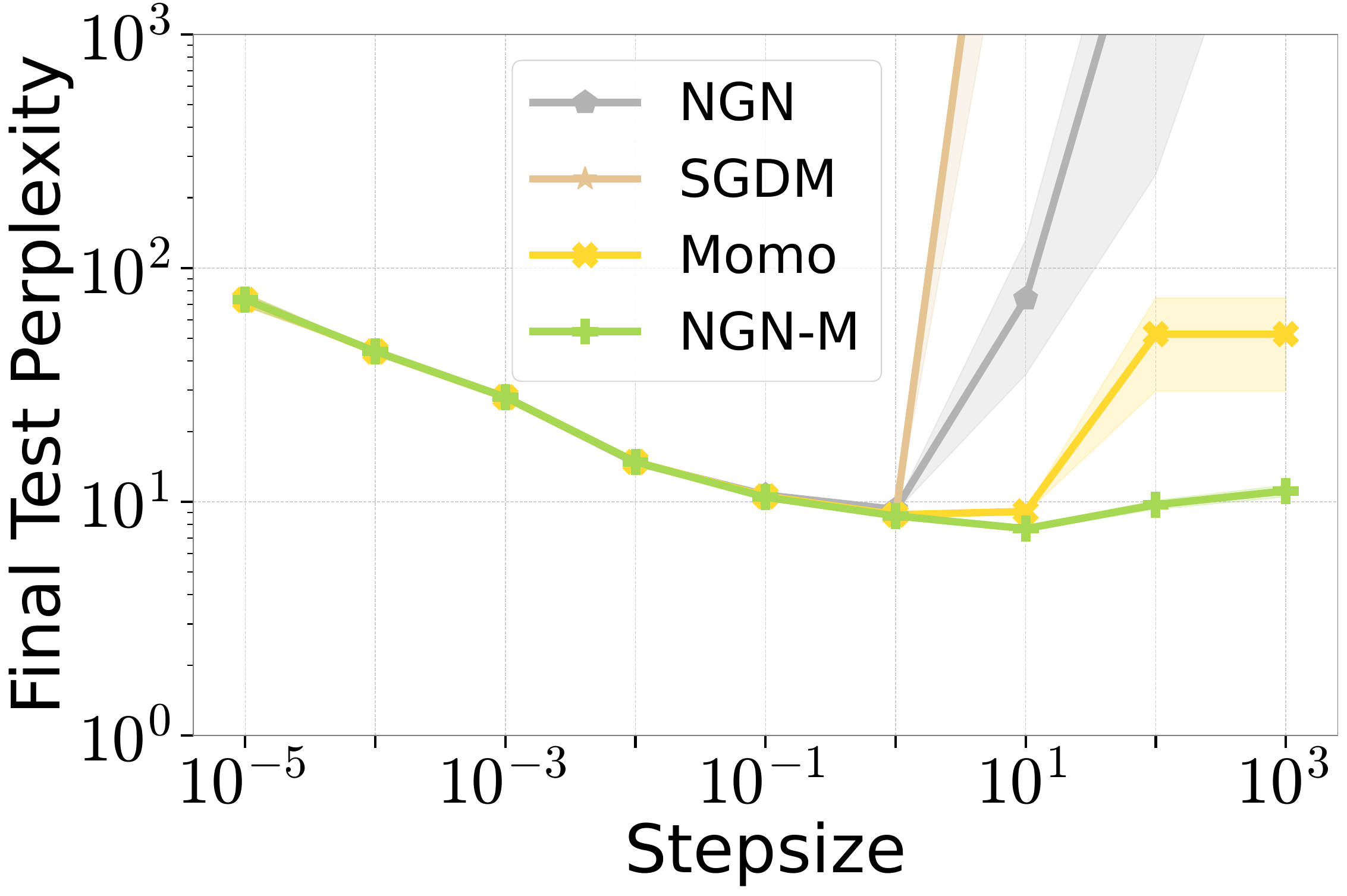} &
     \includegraphics[width=0.3\linewidth]{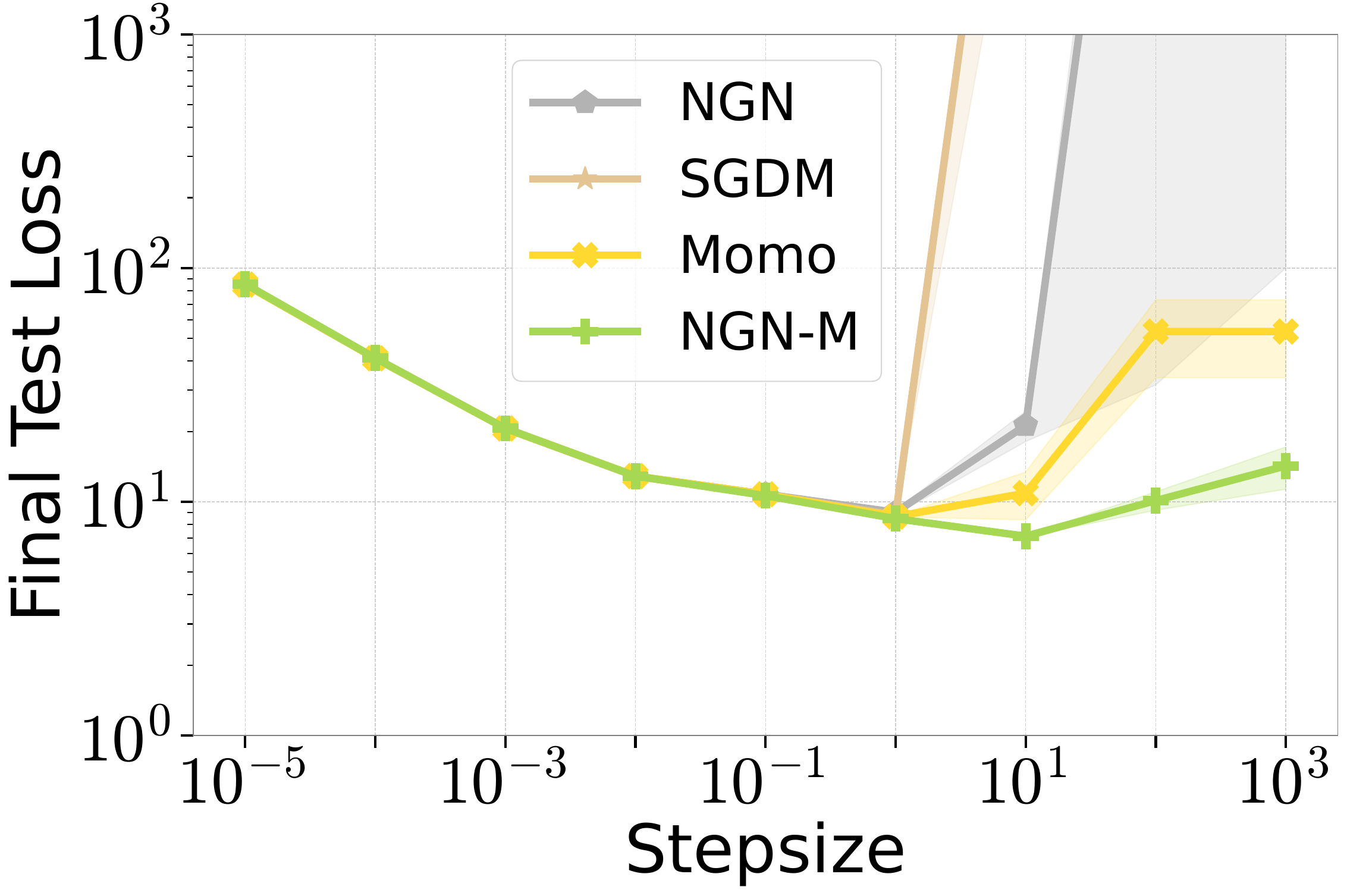} \\
     {\small \makecellnew{LSTM for \\  Wikitext-2} } &
     {\small \makecellnew{Transformer for \\  Tiny Shakespeare} } &
    {\small \makecellnew{Transformer for \\  Rotten Tomatoes} }       
    \end{tabular}
    \caption{Stability performance of algorithms supporting momentum and diagonal step-size varying step-size hyperparameter ($c$ for \algname{NGN-M} and \algname{NGN}, $\alpha_0$ for \algname{Momo}, and step-size for \algname{SGDM}). We observe that \algname{NGN-M} achieves the training loss close to the best possible for a wider range of the step-size hyperparameter. 
    }
    \label{fig:rebuttals_nlp_tasks}
\end{figure*}

\subsection{Comparison of Algorithms with Diagonal Step-size}

Now we compare algorithms with diagonal step-size such as \algname{NGN-D}, \algname{Adagrad} \cite{duchi2011adaptive}, and \algname{RMSprop} \cite{Kingma2015adam}. Since \algname{NGN-D} requires to find constants $\{c_j\}_{j=1}^d$ where $d$ is the size of the model. Finding sufficiently good constants $c_j$ might be a challenging task since $d$ is a large number. Therefore, we use \algname{RMSprop} preconditioner $\mD_k$ to set them as $c_j = c/(\mD_k)_{(j)}$. We leave the exploration of how to set constants $c_j$ properly for future research. 

For each method, we tune its learning rate hyperparameter over the powers of $10$: $\{10^{-4}, \dots, 10^{2}\}$ and present the best performance averaged across $3$ random seeds in \Cref{tab:empirical_comparison_cd_type}. We observe that \algname{NGN-D} performs similarly to \algname{RMSprop}. \algname{NGN-D} has slightly worse performance on (LSTM, PTB) dataset but significantly better on (LSTM, Wikitext-2) workload. Besides, \algname{Adagrad} always has the worst performance. Moreover, these algorithms do not have high resilience to the choice of hyperparameter. Therefore, we omit their comparison from this perspective. 

\begin{table*}[t]
    \centering
    \caption{The best validation score (with one standard deviation; accuracy for image classification; perplexity for language modeling) for the best learning rate choice for each method that supports diagonal step-sizes.}
    \label{tab:empirical_comparison_cd_type}
    \resizebox{0.7\textwidth}{!}{
        \begin{tabular}{ccccc}
            \toprule
            {\bf Model} & {\bf Dataset} & \algname{Adagrad} & \algname{RMSprop} & \algname{NGN-D} 
            \\ \toprule

            Resnet20 &
            CIFAR10 &
            $85.90_{\pm 0.30}$ &
            $86.71_{\pm 0.64}$ &
            $86.98_{\pm 0.15}$
            \\ \midrule

            Transformer &
            Rotten Tomatoes &
            $7.77_{\pm 0.02}$ &
            $6.87_{\pm 0.05}$ &
            $6.92_{\pm 0.03}$ 
            
            \\ \midrule

            Transformer &
            Tiny Sheaksper &
            $7.77_{\pm 0.05}$ & 
            $7.00_{\pm 0.13}$ &
            $6.90_{\pm 0.05}$ 
            \\ \midrule
            
            LSTM & 
            PTB &
            $99.24_{\pm 2.13}$  & 
            $69.00_{\pm 0.17}$ &
            $71.54_{\pm 0.11}$

            \\ \midrule

            LSTM & 
            Wikitext-2 &
            $113.19_{\pm 4.36}$ & 
            $79.48_{\pm 0.45}$ &
            $75.44_{\pm 0.12}$ 

            \\
        
            \bottomrule 
        
        \end{tabular}
        }

\end{table*}

\subsection{Effective Step-size of \algname{NGN-M}, \algname{Momo}, \algname{NGN-MDv1}, and \algname{Momo-Adam}}

Next, we compare the effective step-size applied throughout the training with \algname{NGN-M}, \algname{Momo}, \algname{NGN-MDv1}, and \algname{Momo-Adam} in \Cref{fig:stepsize_momentum_appendix,fig:stepsize_adam_type_appendix}. First, both \algname{NGN-M} and \algname{Momo} perform a warm-up in the beginning: the effective step-size increases at the beginning of the training. Then we observe the main difference between the two algorithms above: effective step-size of \algname{Momo} for sufficiently large step-size hyperparameter is not adaptive within some part of the training, it always hits the upper bound. Consequently, during that part of the training \algname{Momo} reduces to \algname{SGDM}. In contrast, the effective step-size of \algname{NGN-M} is always adaptive: it gradually decreases after a short warm-up. This trend is similar to the state-of-the-art learning rate schedulers used in practice. Similar observations can be made in comparison of \algname{NGN-MDv1} and \algname{Momo-Adam}.

\begin{figure*}[th]
    \centering
    \begin{tabular}{cccc}
    \multicolumn{4}{c}{ \includegraphics[width=0.9\linewidth]{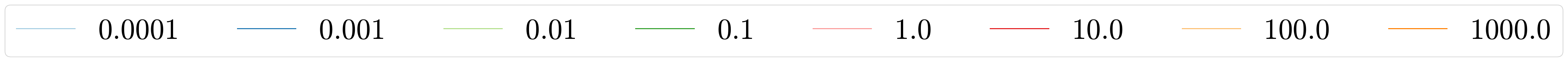}} \\
       \includegraphics[width=0.25\linewidth]{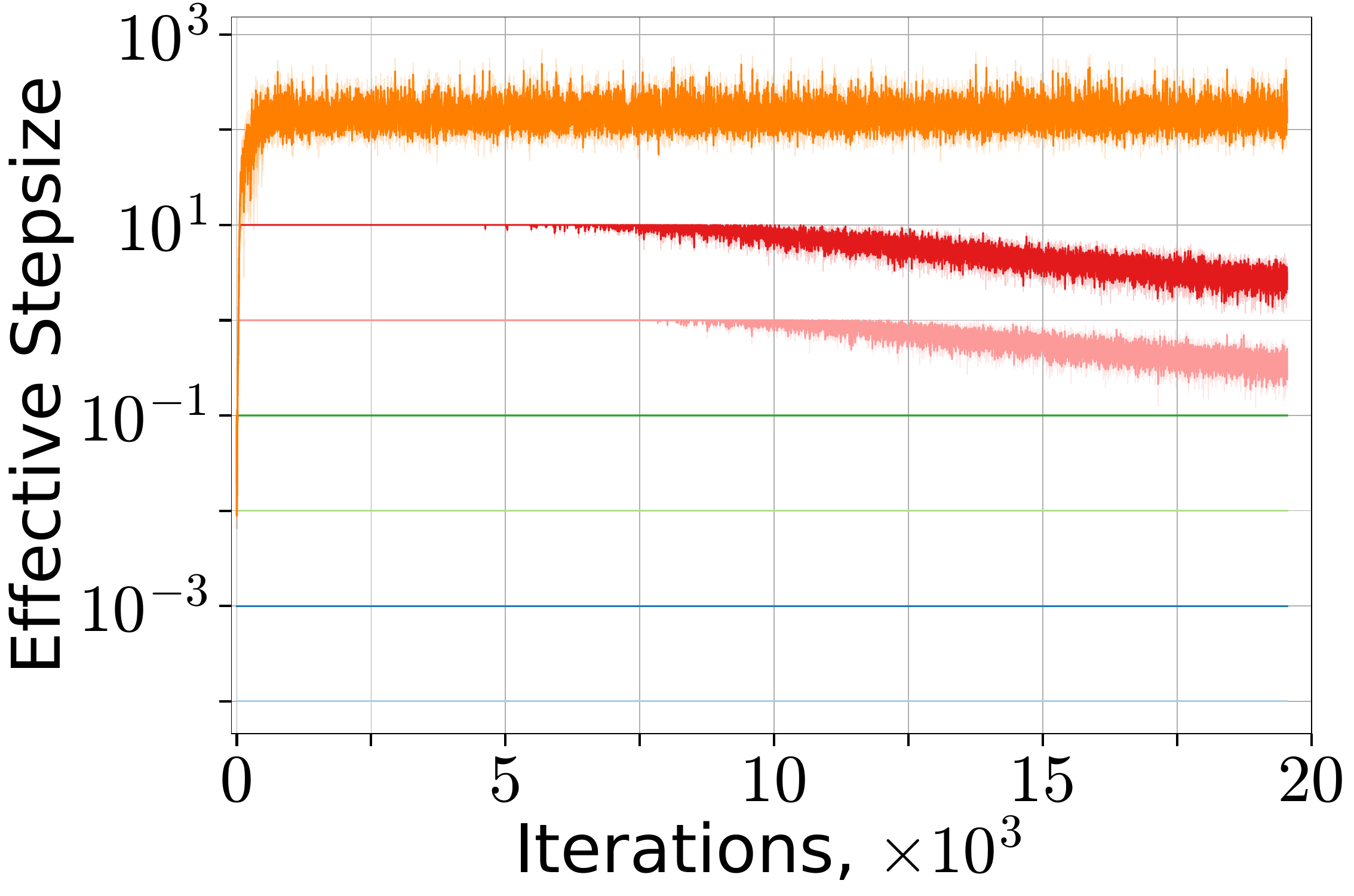}  &  
       \hspace{-5mm}\includegraphics[width=0.25\linewidth]{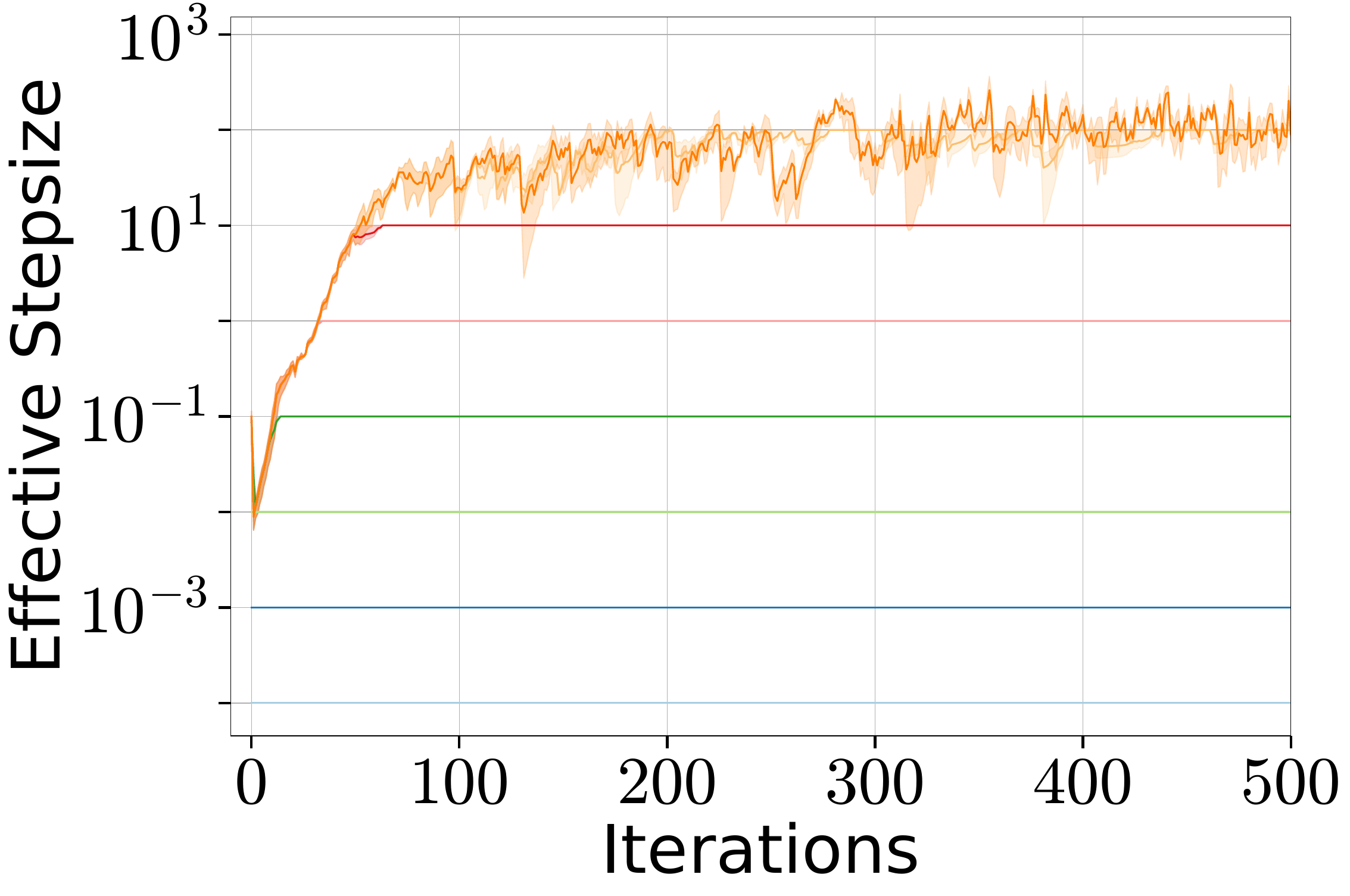} &
       \hspace{-5mm}\includegraphics[width=0.25\linewidth]{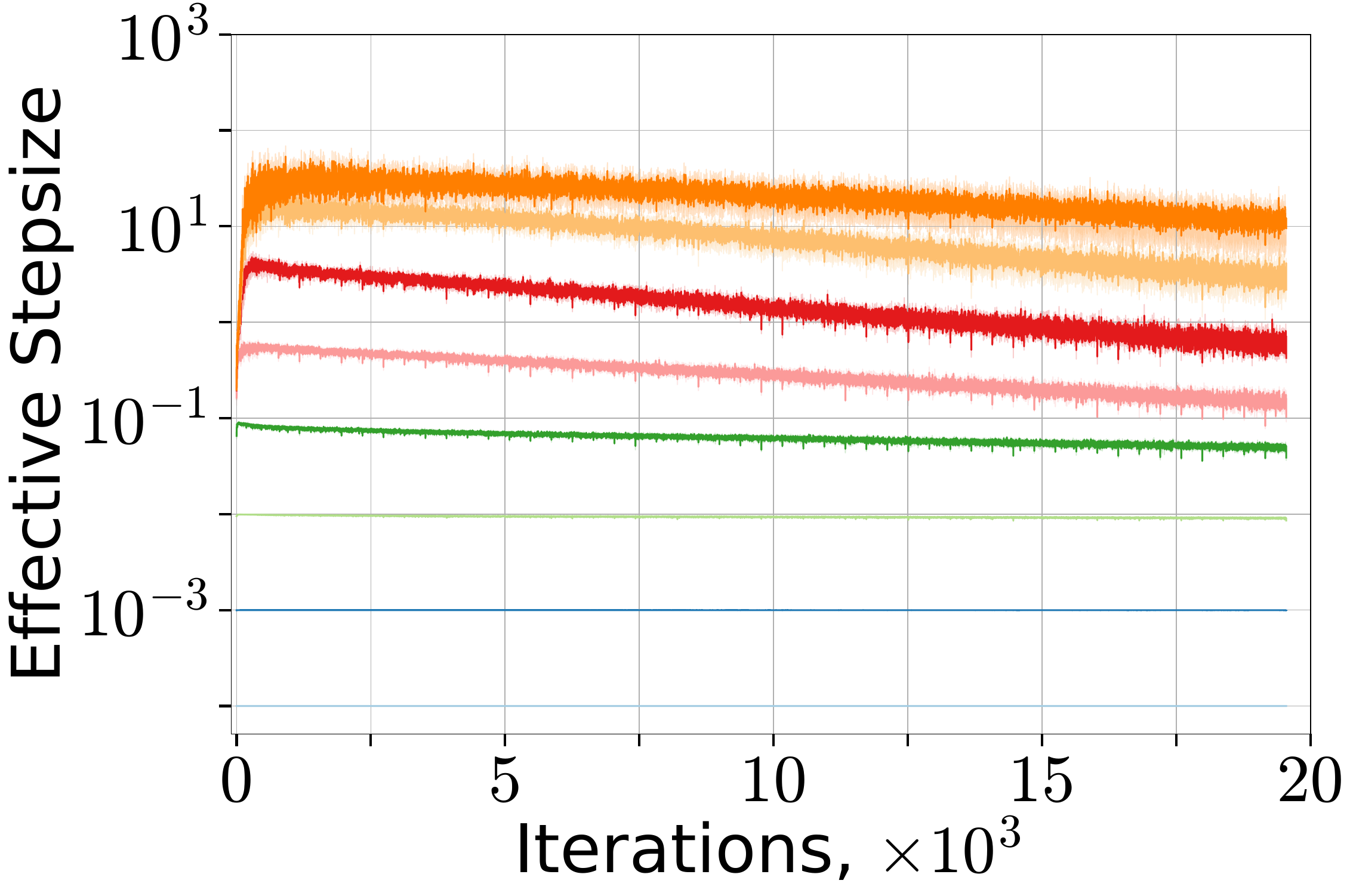} &
      \hspace{-5mm} \includegraphics[width=0.25\linewidth]{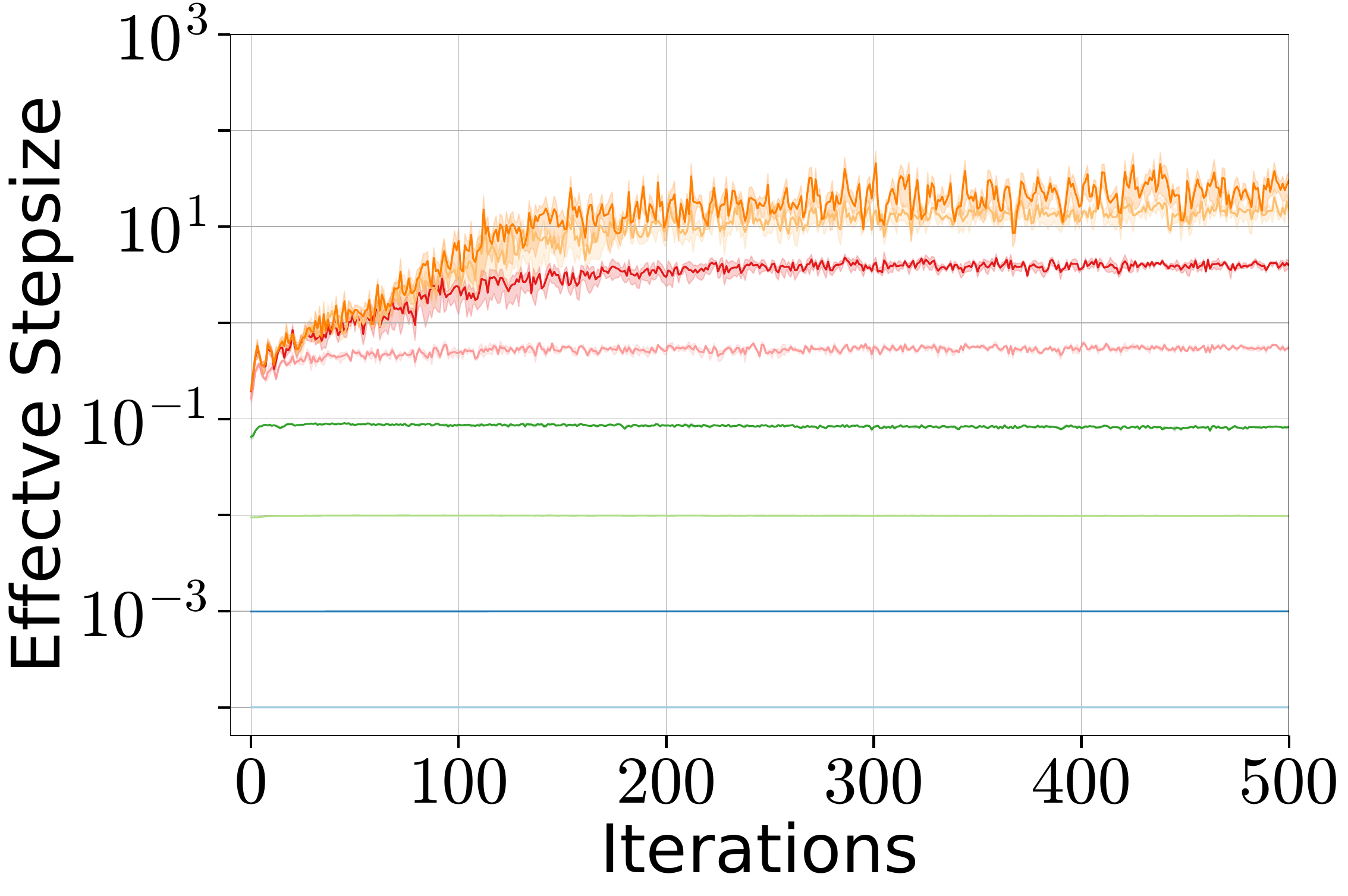} \\
       \multicolumn{2}{c}{\algname{Momo} for Resnet20 for CIFAR 10} & 
       \multicolumn{2}{c}{\algname{NGN-M} for Resnet20 for CIFAR 10} \\ 
       \multicolumn{4}{c}{ \includegraphics[width=0.7\linewidth]{Plots/legend_stepsize_train_loss_stability_comparison_cifar10_vit_512_200.pdf}} \\
       \includegraphics[width=0.25\linewidth]{Plots/momentum_stepsize_cifar10_vit_Momo_512_200.pdf}  &  
       \hspace{-5mm}\includegraphics[width=0.25\linewidth]{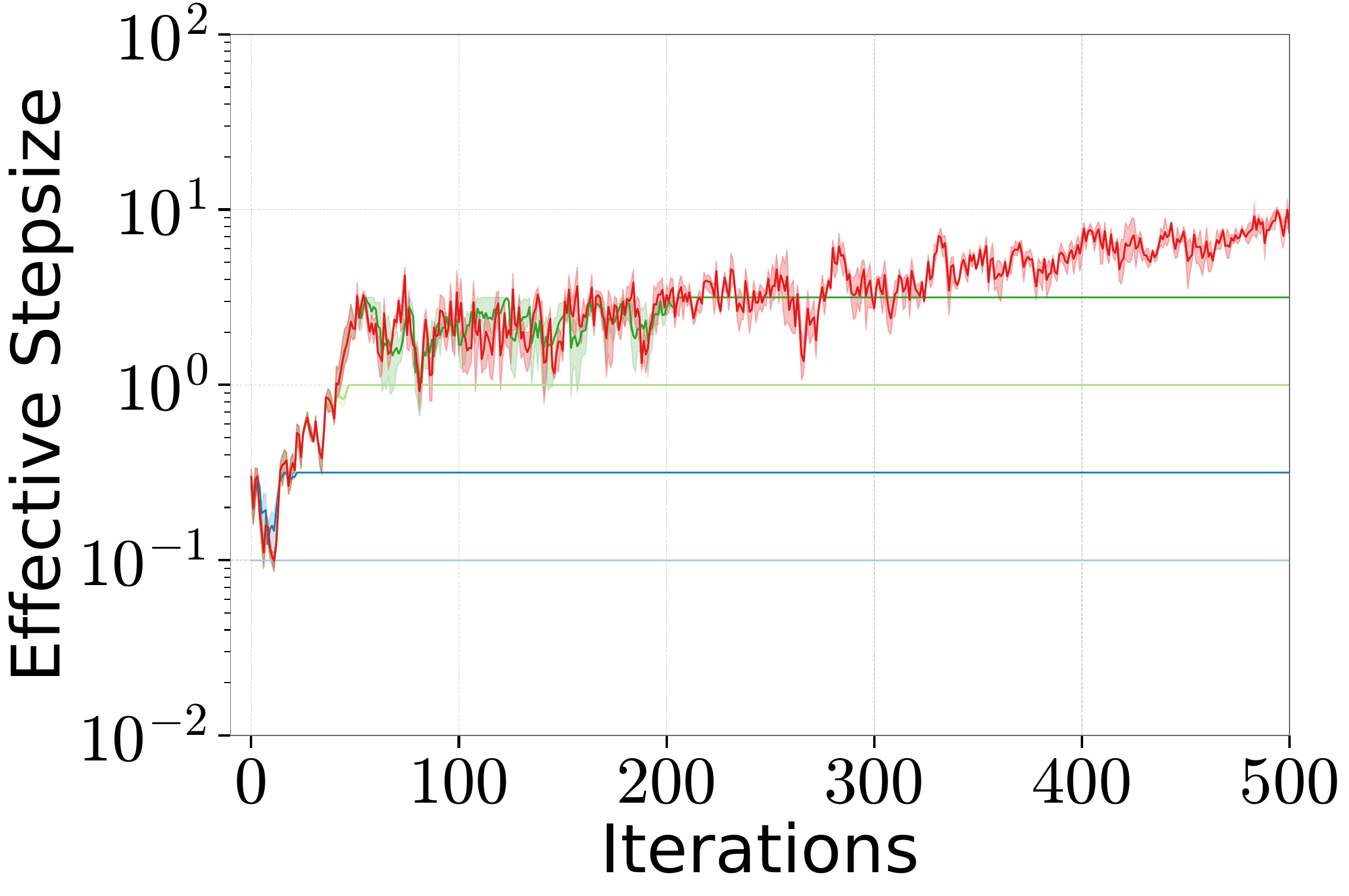} &
      \hspace{-5mm} \includegraphics[width=0.25\linewidth]{Plots/momentum_stepsize_cifar10_vit_MomNGN_512_200.pdf} &
       \hspace{-5mm}\includegraphics[width=0.25\linewidth]{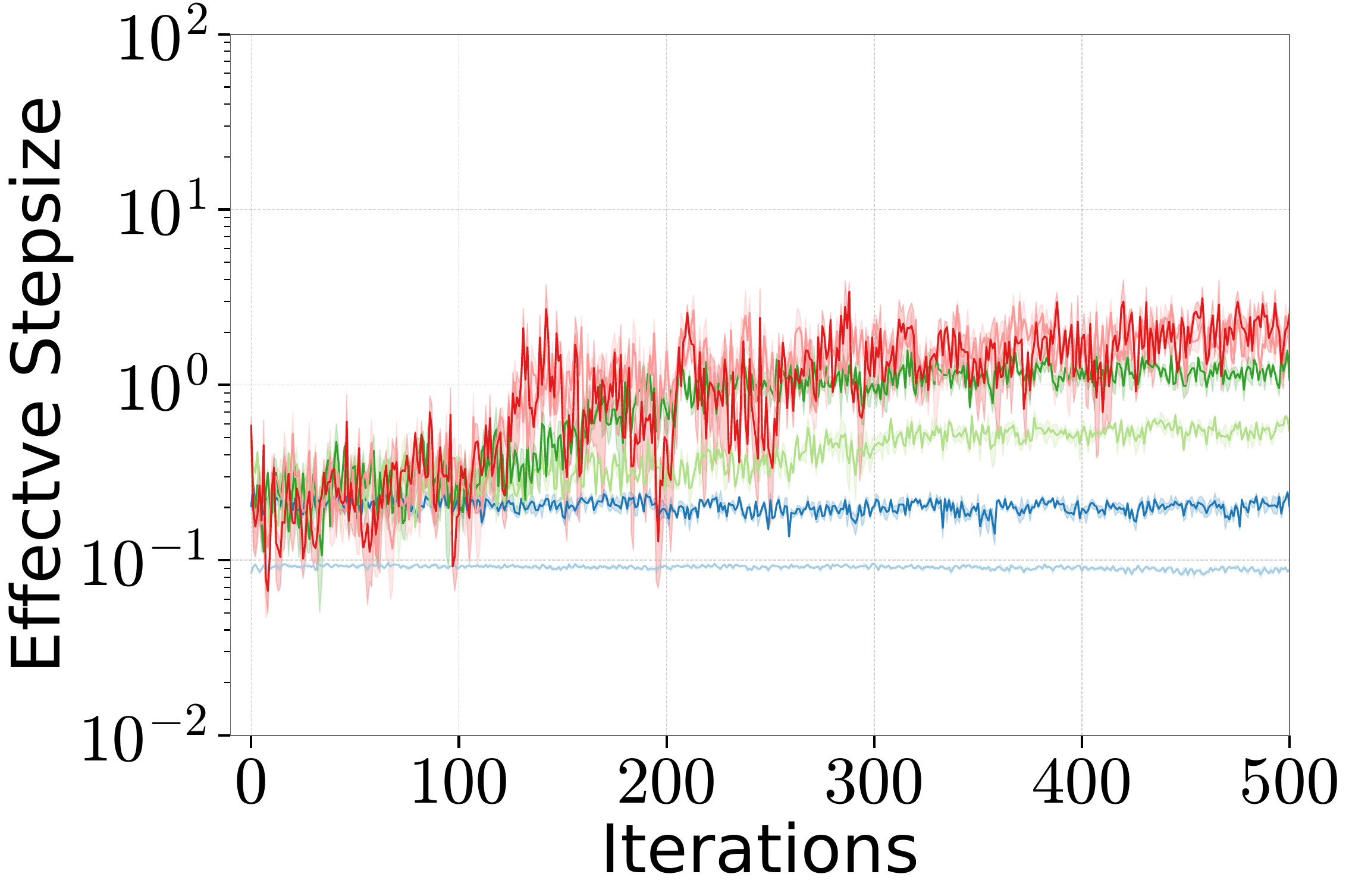} \\
       \multicolumn{2}{c}{\algname{Momo} for ViT for CIFAR 10} & 
       \multicolumn{2}{c}{\algname{NGN-M} for ViT for CIFAR 10} \\ 
       \\
    \end{tabular}
    
    \caption{The step-size of \algname{Momo} and \algname{NGN-M} during the training. We demonstrate the step-sizes $\tau_k$ for \algname{Momo} and $\gamma_k$ for \algname{NGN-M} varying step-size parameters $\alpha_0$ for \algname{Momo} and $c$ for \algname{NGN-M}.}
    \label{fig:stepsize_momentum_appendix}
\end{figure*}

\begin{figure*}[t]
    \centering
    \begin{tabular}{cccc}
    \multicolumn{4}{c}{ \includegraphics[width=0.9\linewidth]{Plots/legend_stepsize_train_loss_stability_comparison_cifar10_resnet20_stepsize_128_50.pdf}} \\
       \includegraphics[width=0.25\linewidth]{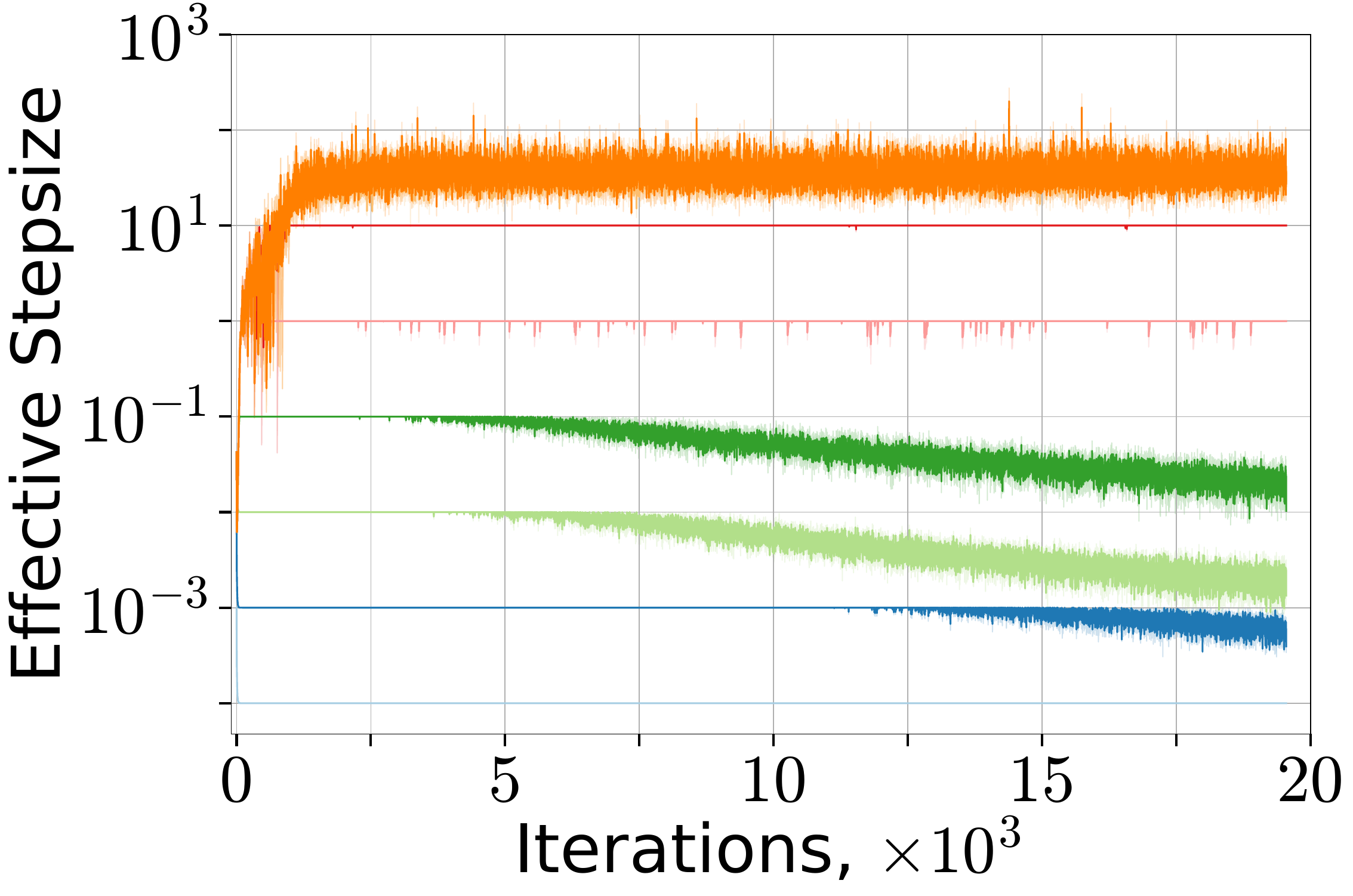}  &  
      \hspace{-5mm} \includegraphics[width=0.25\linewidth]{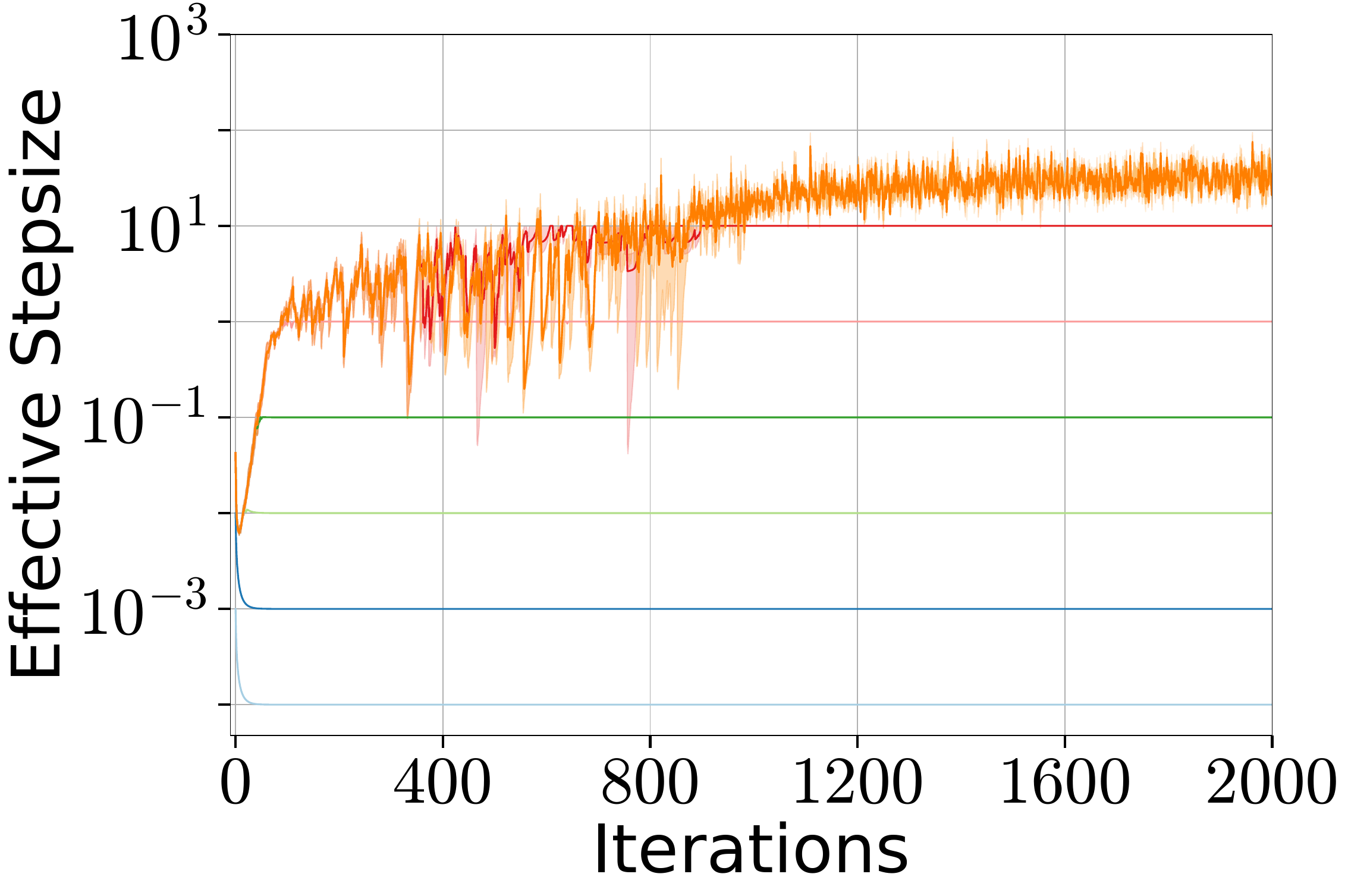} &
      \hspace{-5mm} \includegraphics[width=0.25\linewidth]{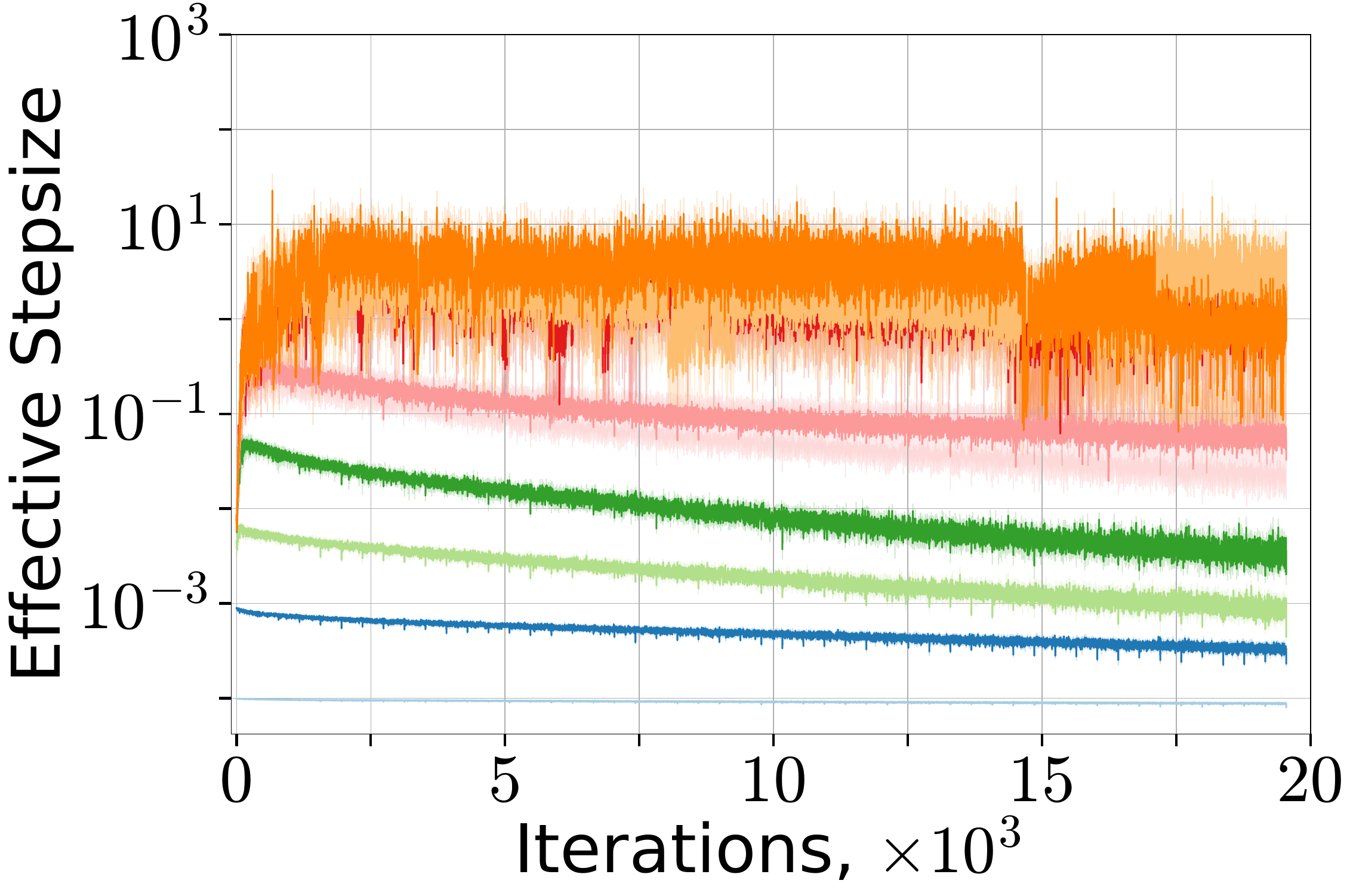} &
       \hspace{-5mm}\includegraphics[width=0.25\linewidth]{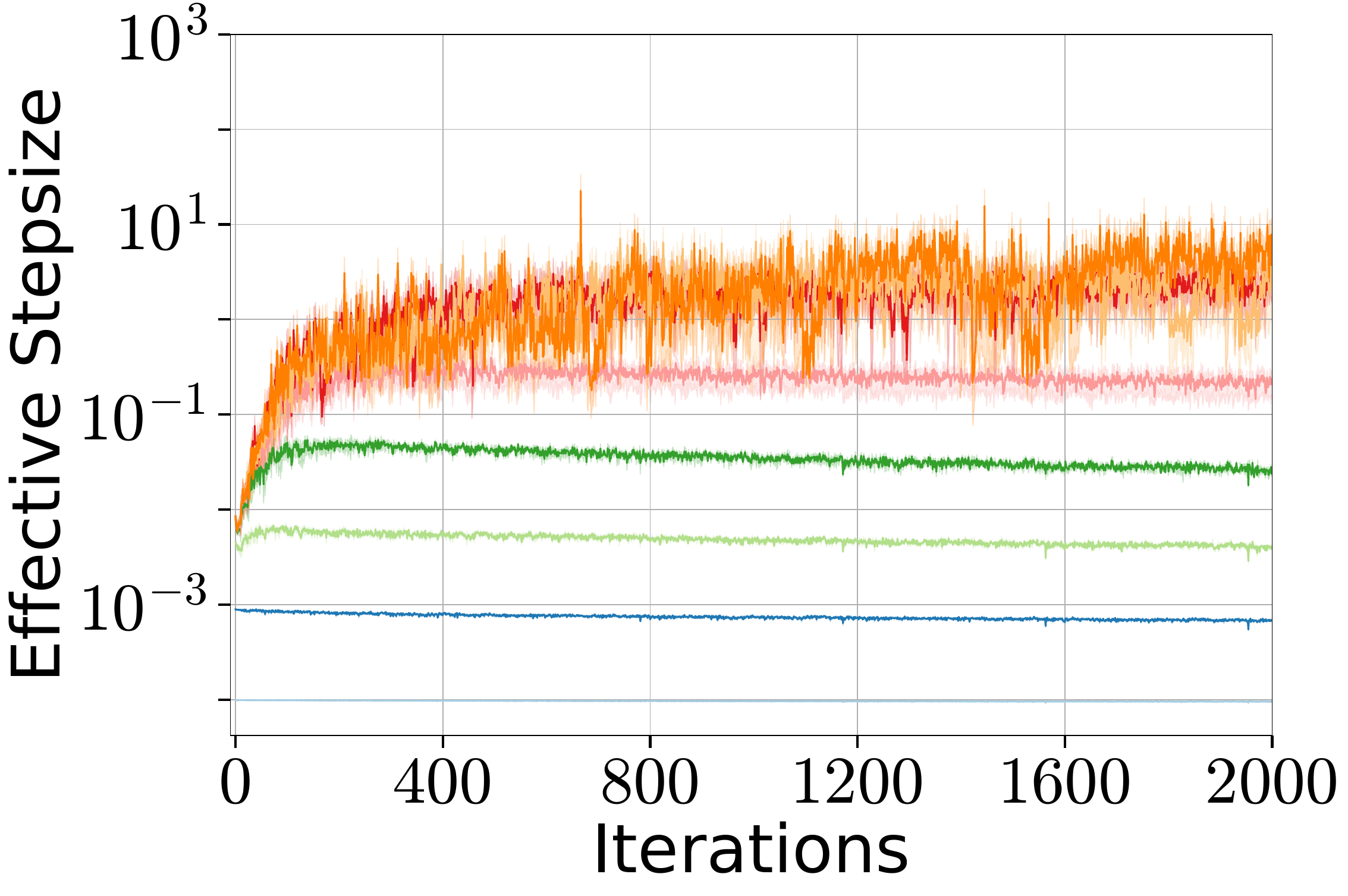} \\
       \multicolumn{2}{c}{\algname{Momo-Adam} for Resnet20 for CIFAR 10} & 
       \multicolumn{2}{c}{\algname{NGN-MDv1} for Resnet20 for CIFAR 10} \\ 
        \multicolumn{4}{c}{ \includegraphics[width=0.7\linewidth]{Plots/legend_adam-type_stepsize_train_loss_stability_comparison_cifar10_vit_512_200.pdf}} \\
       \includegraphics[width=0.25\linewidth]{Plots/adam-type_stepsize_cifar10_vit_Momo-Adam_512_200.pdf}  
       &  
      \hspace{-5mm} \includegraphics[width=0.25\linewidth]{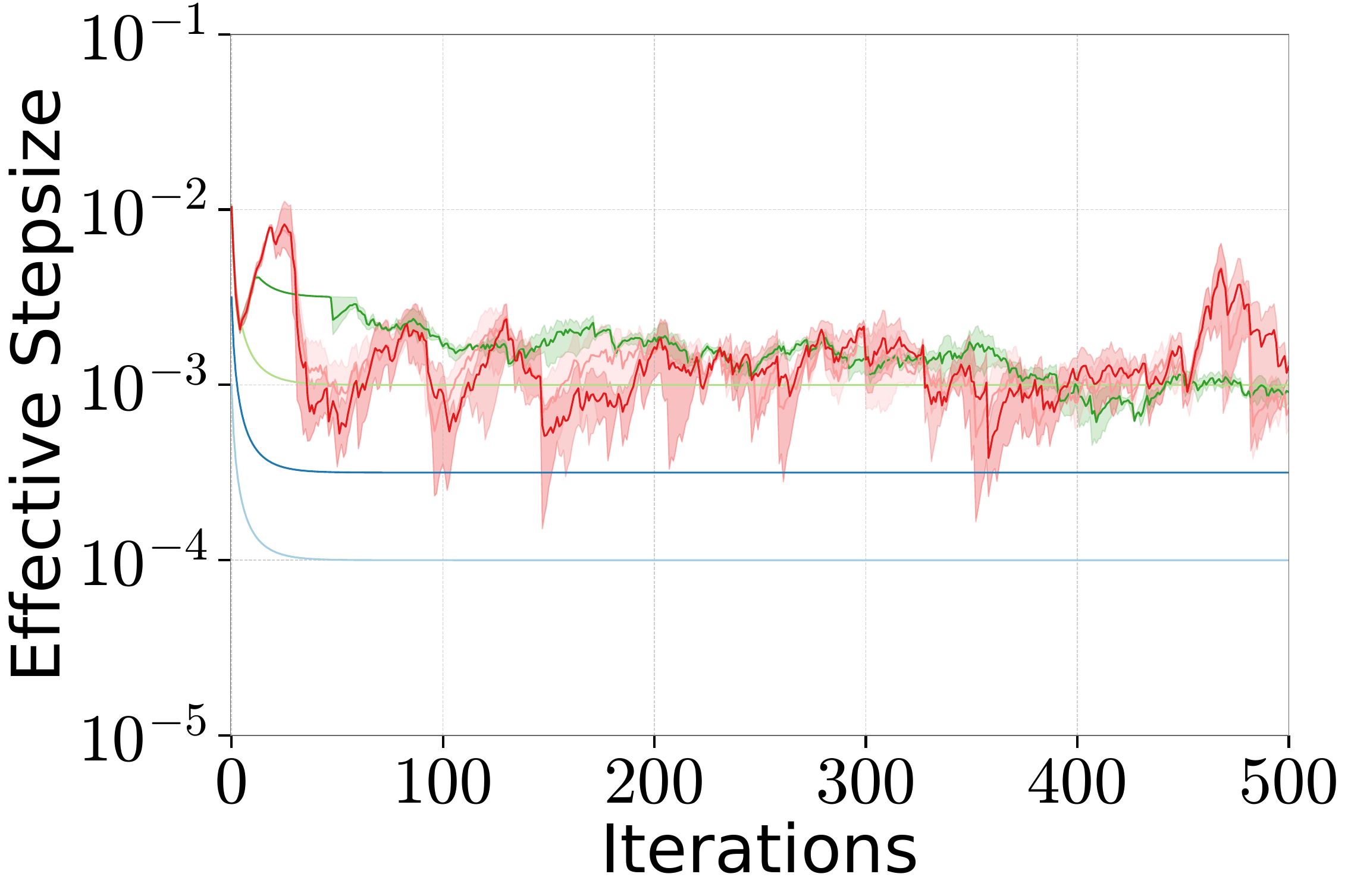} 
       &
       \hspace{-5mm}\includegraphics[width=0.25\linewidth]{Plots/adam-type_stepsize_cifar10_vit_MomNGN-RMSprop_512_200.pdf} &
       \hspace{-5mm}\includegraphics[width=0.25\linewidth]{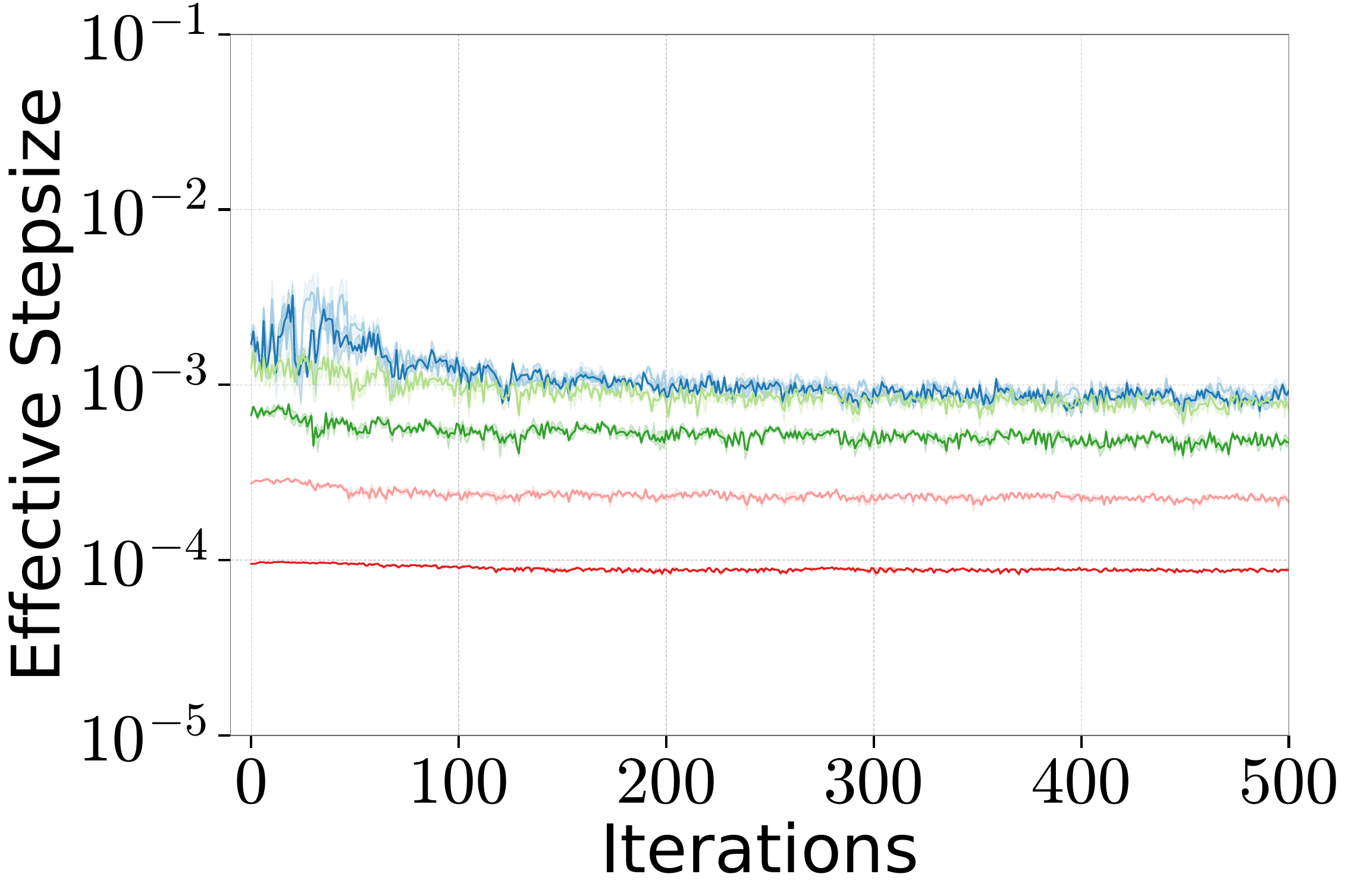} \\
       \multicolumn{2}{c}{\algname{Momo-Adam} for ViT for CIFAR 10} & 
       \multicolumn{2}{c}{\algname{NGN-MDv1} for ViT for CIFAR 10} \\ 
    \end{tabular}
    
    \caption{The step-size of \algname{Momo-Adam} and \algname{NGN-MDv1} during the training. We demonstrate the step-sizes $\tau_k$ for \algname{Momo-Adam} and $\gamma_k$ for \algname{NGN-MDv1} varying step-size parameters $\alpha_0$ for \algname{Momo} and $c$ for \algname{NGN-MDv1}.}
    \label{fig:stepsize_adam_type_appendix}
\end{figure*}

\subsection{Effective Updates in Training Language Models}

In this section, we demonstrate the magnitude of updates when training $160$M language model with \algname{Adam} and \algname{NGN-MDv1} and varying the step-size hyperparameter across different layers of the model: see the results in \Cref{fig:update_magnitude_lr0003,fig:update_magnitude_lr001,fig:update_magnitude_lr003}. We demonstrate that \algname{NGN-MDv1} is a more conservative algorithm: the effective update is smaller than that of \algname{Adam} due to the adaptive nature of the step-size. This is especially evident when training $160$M language model with a step-size hyperparameter $0.03$: The updates of \algname{Adam} become considerably larger than the update of \algname{NGN-MDv1}. This property is a key factor behind the difference in training dynamics: \algname{NGN-MDv1} can stabilize at a significantly lower training loss.

\begin{figure*}[h]
    \centering
    \includegraphics[width=1\linewidth]{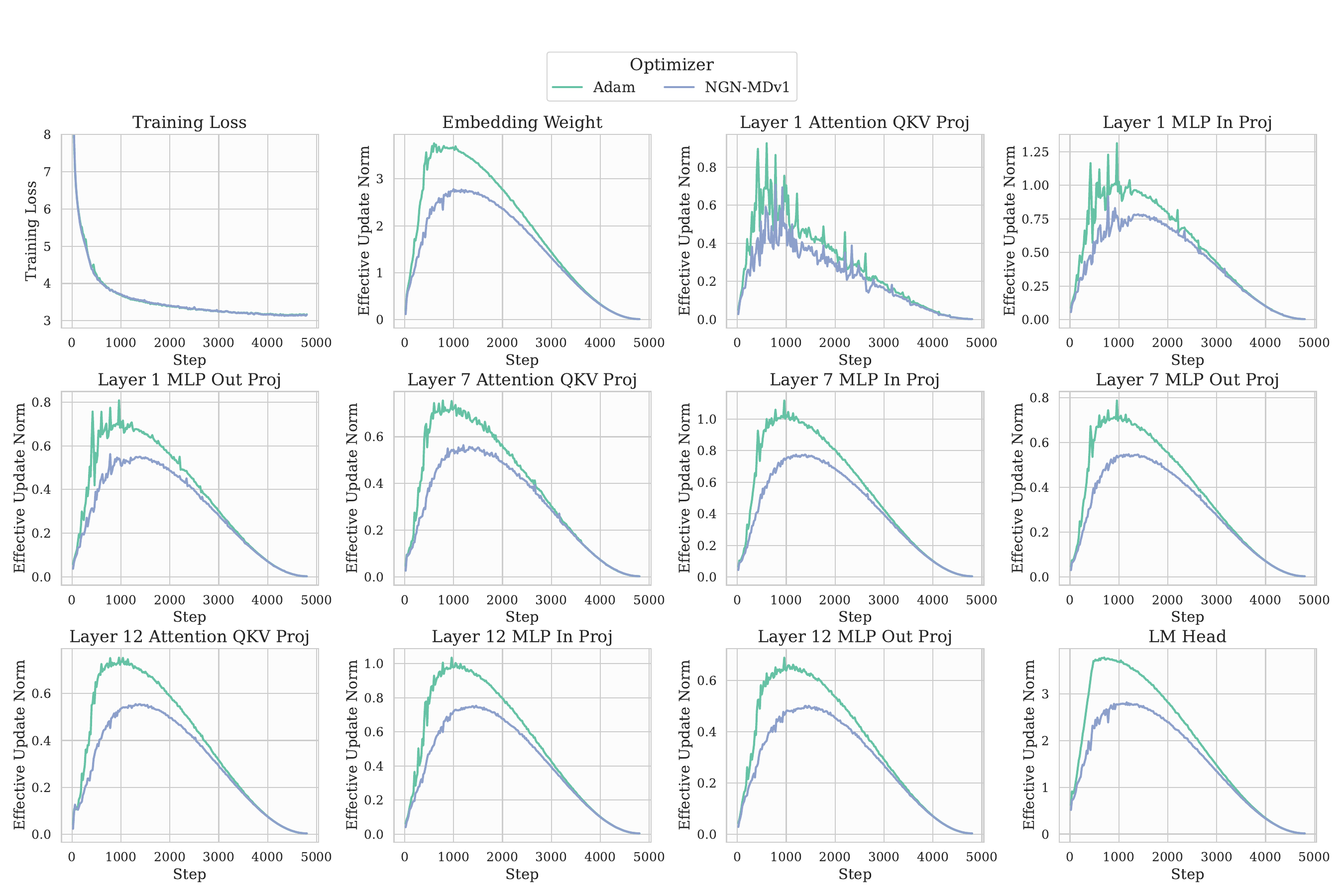}

    \caption{Magnitude of updates when training $160$M language model with \algname{Adam} and \algname{NGN-MDv1} and step-size hyperparameter $0.003$.}
    \label{fig:update_magnitude_lr0003}
\end{figure*}

\begin{figure*}[h]
    \centering
    \includegraphics[width=1\linewidth]{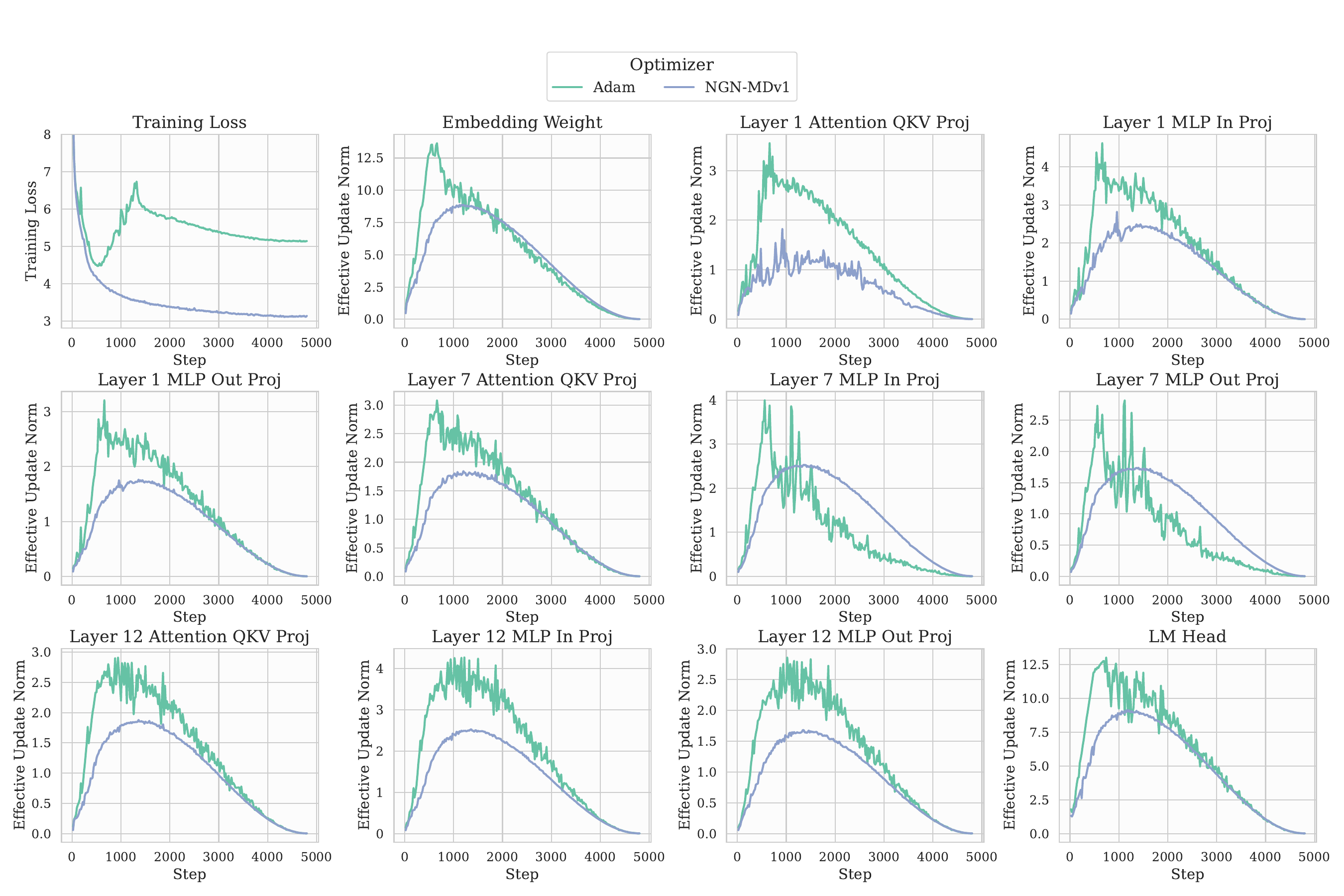}

    \caption{Magnitude of updates when training $160$M language model with \algname{Adam} and \algname{NGN-MDv1} and step-size hyperparameter $0.01$.}
    \label{fig:update_magnitude_lr001}
\end{figure*}

\begin{figure*}[h]
    \centering
    \includegraphics[width=1\linewidth]{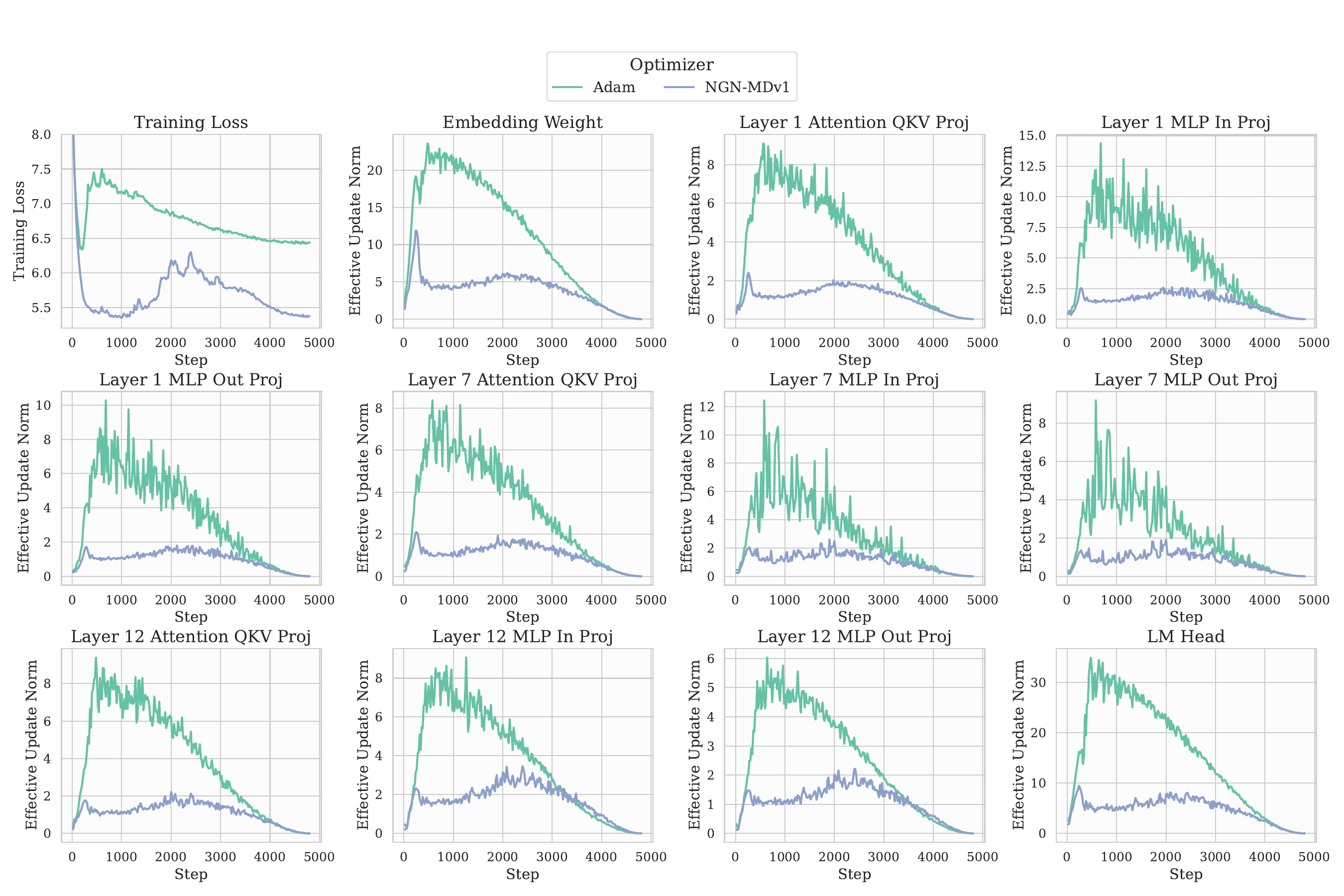}

    \caption{Magnitude of updates when training $160$M language model with \algname{Adam} and \algname{NGN-MDv1} and step-size hyperparameter $0.03$.}
    \label{fig:update_magnitude_lr003}
\end{figure*}

\subsection{Training Dynamics in Training Language Models}

Now we report the training dynamics in the training language across all tested sizes.

\begin{figure*}[h]
    \centering
    \begin{tabular}{cc}
       \includegraphics[width=0.5\linewidth]{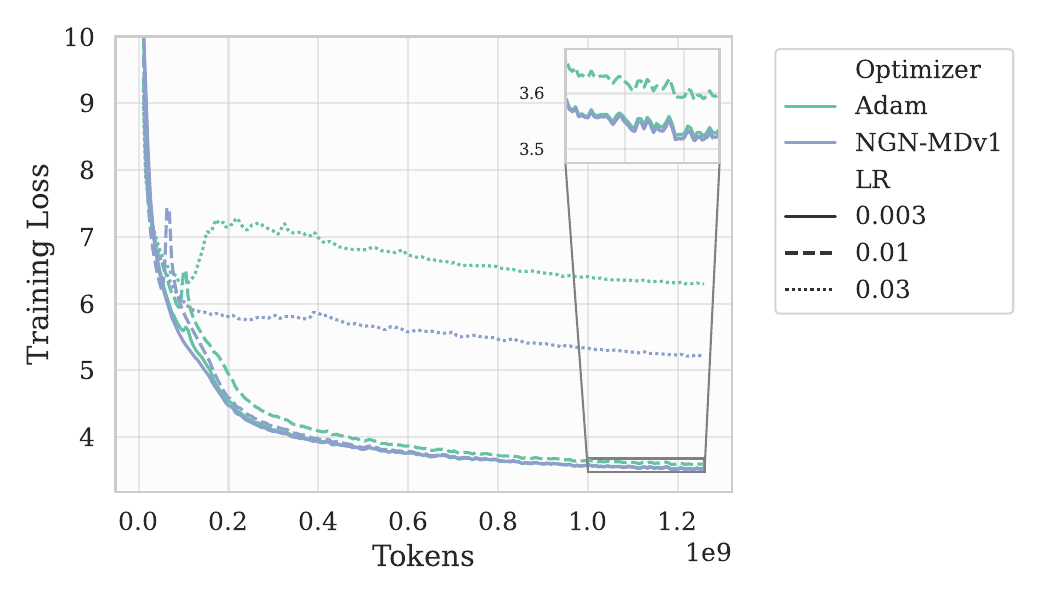}  & 
       \hspace{-3mm}\includegraphics[width=0.5\linewidth]{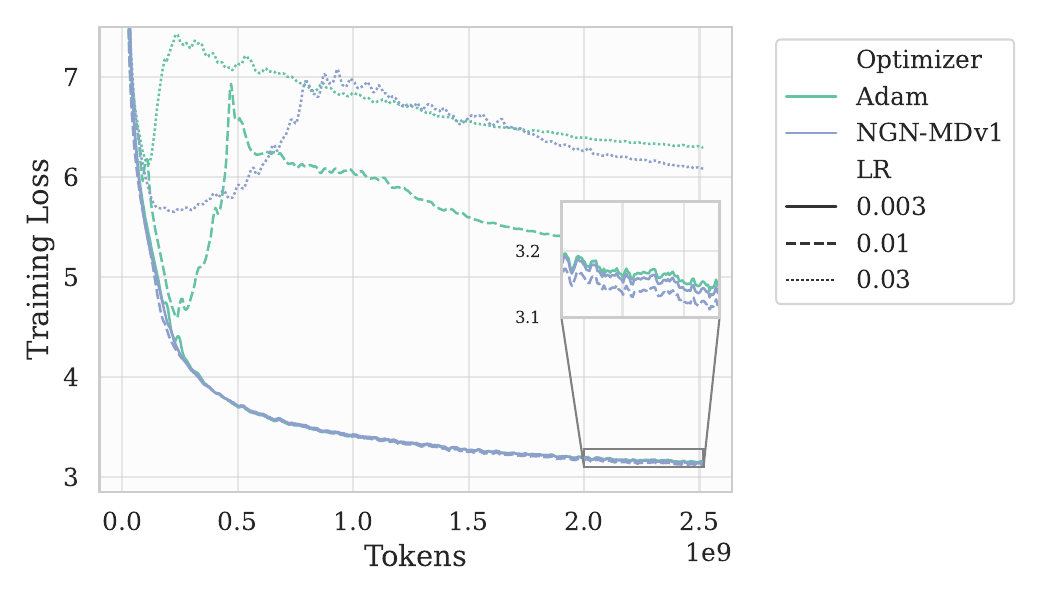} \\
       {\small 70M Transformer$++$} &
       {\small 160M Transformer$++$} \\
        \includegraphics[width=0.5\linewidth]{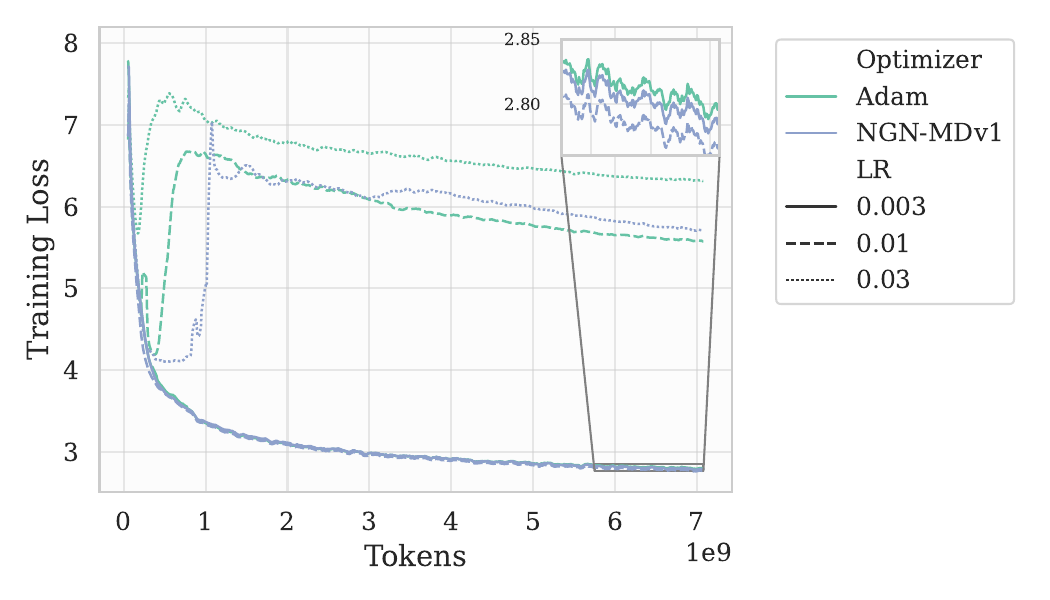} & 
        \hspace{-3mm}\includegraphics[width=0.5\linewidth]{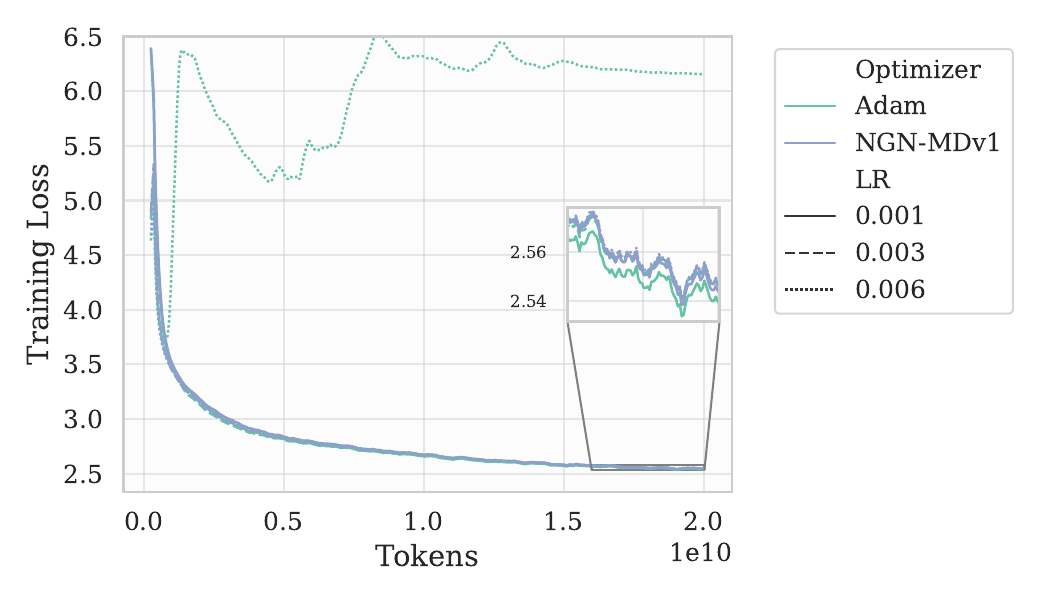} \\
        {\small 410M Transformer$++$} & 
        {\small 1B Transformer$++$}
    \end{tabular}

    \caption{Training dynamics when training language model at different sizes.}
    \label{fig:llm_training_dynamics}
\end{figure*}

\subsection{Ablation Study of Momentum Parameters}

In this section, we study the sensitivity of \algname{NGN-MDv1} and \algname{Adam} to the choice of the learning rate and momentum hyperparameters, when training 70M language model on FineWeb dataset \citep{penedo2024fineweb}. To do that, we fix $\beta_1=0.9$ (or $\beta_2=0.95$) and make a sweep over the learning rate hyperparameter and $\beta_2$ (or learning rate hyperparameter and $\beta_1$). We report the final test perplexity averaged over 3 runs for each set of hyperparameters. 

We summarize our findings from \Cref{tab:ngnmdv1_ablation_beta1}, \Cref{tab:adam_ablation_beta1}, \Cref{tab:ngnmdv1_ablation_beta2}, and \Cref{tab:adam_ablation_beta2}  as follows:

\begin{itemize}
    \item Low lr (3e-3): \algname{NGN-MDv1} and \algname{Adam} show similar sensitivity to changes in both $\beta_1$ and $\beta_2$.

    \item Moderate lr (1e-2): \algname{NGN-MDv1} is noticeably more robust than \algname{Adam} to extremes of $\beta_1$, while both optimizers perform similarly across $\beta_2$ (though \algname{Adam}’s performance degrades slightly at $\beta_2=0.999$).

    \item High lr (3e-2): Both methods suffer when $\beta_1$ is small (or $\beta_2$ is large), but \algname{NGN-MDv1} recovers lower perplexity at larger $\beta_1$ values (smaller $\beta_2$ values), whereas \algname{Adam} fails to reach comparable performance.
\end{itemize}

To conclude, \algname{NGN-MDv1} demonstrates greater robustness to changes in momentum parameters at high lr, and consistently attains lower perplexity than Adam, even when both methods’ performance deteriorates (we refer to the cases when both algorithms cannot achieve perplexity around $50$).

\begin{table*}[t]
\begin{minipage}[t]{0.48\linewidth}
    \centering
    \caption{Test perplexity of \algname{NGN-MDv1} when varying the learning rate and $\beta_1$ hyperparameters when training 70M language model on the FineWeb dataset.}
    \label{tab:ngnmdv1_ablation_beta1}
    \resizebox{\linewidth}{!}{
        \begin{tabular}{ccccc}

            \toprule
           \textbf{lr} & 
           $\beta_1=0.6$ &
            $\beta_1=0.8$ &
           $\beta_1=0.9$ & 
           $\beta_1=0.99$ \\ \toprule

            3e-4 &
            $49.9\pm0.2$ &$47.4\pm0.2$ & $47.0\pm0.2$ & $49.7\pm0.3$
            \\

            \midrule

            1e-3 &
            $41.5\pm0.2$ & 
            $39.9\pm0.2$ & 
            $38.6\pm0.1$ & 
            $40.2\pm0.3$
            \\

            \midrule

            3e-3 &
            $40\pm1$ & 
            $36.9\pm0.3$ & 
            $35.9\pm0.1$ & 
            $37.2\pm0.1$
            \\

            \midrule

            1e-2 &
            $54\pm16$ & 
            $37\pm2$ & 
            $34.7\pm0.3$ & 
            $35.9\pm0.1$
            \\

            \midrule
            3e-2 &
            $278\pm6$ & 
            $129\pm2$ & 
            $34.6\pm0.1$ & 
            $35.6\pm0.1$
            \\

            \bottomrule 
        
        \end{tabular}
        }
\end{minipage}
\hfill
\begin{minipage}[t]{0.48\linewidth}
    \centering
    \caption{Test perplexity of \algname{Adam} when varying the learning rate and $\beta_1$ hyperparameters when training 70M language model on the FineWeb dataset.}
    \label{tab:adam_ablation_beta1}
    \resizebox{\linewidth}{!}{
        \begin{tabular}{ccccc}

            \toprule
           \textbf{lr} & 
           $\beta_1=0.6$ &
            $\beta_1=0.8$ &
           $\beta_1=0.9$ & 
           $\beta_1=0.99$ \\ \toprule

            3e-4 &
            $49.4\pm0.2$ & 
            $46.5\pm0.1$ & 
            $46.2\pm0.3$ & 
            $57\pm1$
            \\

            \midrule

            1e-3 &
            $41.4\pm0.2$ & 
            $39.6\pm0.1$ & 
            $38.5\pm0.1$ & 
            $45.0\pm0.2$
            \\

            \midrule

            3e-3 &
            $40.7\pm0.1$ & 
            $37.0\pm0.1$ & 
            $36.0\pm0.1$ & 
            $220\pm70$
            \\

            \midrule

            1e-2 &
            $160\pm60$ & 
            $41\pm2$ & 
            $36\pm2$ & 
            $210\pm110$
            \\

            \midrule
            3e-2 &
            $420\pm20$ & 
            $340\pm50$ & 
            $320\pm60$ & 
            $330\pm130$
            \\

            \bottomrule 
        
        \end{tabular}
        }

\end{minipage}
\end{table*}

\begin{table*}[t]
\begin{minipage}[t]{0.48\linewidth}
    \centering
    \caption{Test perplexity of \algname{NGN-MDv1} when varying the learning rate and $\beta_2$ hyperparameters when training 70M language model on the FineWeb dataset.}
    \label{tab:ngnmdv1_ablation_beta2}
    \resizebox{\linewidth}{!}{
        \begin{tabular}{cccccc}

            \toprule
           \textbf{lr} & 
           $\beta_2=0.6$ &
            $\beta_2=0.8$ &
           $\beta_2=0.9$ & 
           $\beta_2=0.95$ &
           $\beta_2=0.999$
           \\ \toprule

            3e-4 &
            $51.8\pm0.6$ & 
            $49.2\pm0.4$ & 
            $47.8\pm0.3$ & 
            $47.0\pm0.2$ & 
            $47.0\pm0.2$
            \\

            \midrule

            1e-3 &
            $42.6\pm0.3$ & 
            $40.5\pm0.1$ & 
            $39.3\pm0.2$ & 
            $38.6\pm0.1$ & 
            $38.9\pm0.1$
            \\

            \midrule

            3e-3 &
            $39.4\pm0.2$ & 
            $37.5\pm0.2$ & 
            $36.3\pm0.1$ & 
            $35.9\pm0.1$ & 
            $36.5\pm0.4$
            \\

            \midrule

            1e-2 &
            $37.8\pm0.2$ & 
            $35.9\pm0.1$& 
            $35.1\pm0.3$ & 
            $34.7\pm0.3$ & 
            $35.0\pm0.3$
            \\

            \midrule
            3e-2 &
            $37.8\pm0.3$ & 
            $35.8\pm0.1$ & 
            $34.9\pm0.1$ & 
            $34.6\pm0.1$ & 
            $250\pm50$
            \\

            \bottomrule 
        
        \end{tabular}
        }
\end{minipage}
\hfill
\begin{minipage}[t]{0.48\linewidth}
    \centering
    \caption{Test perplexity of \algname{Adam} when varying the learning rate and $\beta_2$ hyperparameters when training 70M language model on the FineWeb dataset.}
    \label{tab:adam_ablation_beta2}
    \resizebox{\linewidth}{!}{
        \begin{tabular}{cccccc}

            \toprule
           \textbf{lr} & 
           $\beta_2=0.6$ &
            $\beta_2=0.8$ &
           $\beta_2=0.9$ & 
           $\beta_2=0.95$ &
           $\beta_2=0.999$
           \\ \toprule

            3e-4 &
            $46.1\pm0.2$ & 
            $46.6\pm0.1$ & 
            $46.5\pm0.2$ & 
            $46.2\pm0.3$ & 
            $46.5\pm0.1$
            \\

            \midrule

            1e-3 &
            $38.8\pm0.1$ & 
            $39.0\pm0.2$ & 
            $38.9\pm0.1$ & 
            $38.5\pm0.1$ & 
            $39.5\pm0.6$
            \\

            \midrule

            3e-3 &
            $38.8\pm0.3$ & 
            $36.3\pm0.1$ & 
            $36.1\pm0.2$ & 
            $36.0\pm0.1$ & 
            $36.7\pm0.8$
            \\

            \midrule

            1e-2 &
            $35.4\pm0.2$ & 
            $35.0\pm0.1$ & 
            $34.9\pm0.3$ & 
            $36\pm2$ & 
            $41\pm3$
            \\

            \midrule
            3e-2 &
            $550\pm250$ & 
            $120\pm80$ & 
            $160\pm5$ & 
            $210\pm60$ & 
            $500\pm20$
            \\

            \bottomrule 
        
        \end{tabular}
        }

\end{minipage}
\end{table*}

\end{document}

%% file: math_commands.tex
\usepackage{amsmath,amsfonts,bm}

\def\eqref#1{equation~\ref{#1}}

\def\1{\bm{1}}

\DeclareMathAlphabet{\mathsfit}{\encodingdefault}{\sfdefault}{m}{sl}
\SetMathAlphabet{\mathsfit}{bold}{\encodingdefault}{\sfdefault}{bx}{n}

\DeclareMathOperator*{\argmin}{arg\,min}